\documentclass[format=acmsmall, review=false]{acmart}

\usepackage{amsmath}

\usepackage{amssymb}
\usepackage{mathtools}
\usepackage{amsthm}

\usepackage{acm-ec-25}    
\usepackage{booktabs}       
\usepackage[table]{xcolor}
\usepackage[algo2e,linesnumbered,ruled,lined,noend]{algorithm2e}
\SetAlFnt{\small}
\SetAlCapFnt{\small}
\SetAlCapNameFnt{\small}
\SetAlCapHSkip{0pt}
\IncMargin{-\parindent}

\setcitestyle{authoryear}   

\usepackage{microtype}
\usepackage{graphicx}
\usepackage{subfigure}
\usepackage{hyperref}

\usepackage[capitalize,noabbrev]{cleveref}
\usepackage{thmtools, thm-restate}

\theoremstyle{plain}
\newtheorem{theorem}{Theorem}[section]
\newtheorem{proposition}[theorem]{Proposition}
\newtheorem{lemma}[theorem]{Lemma}
\newtheorem{corollary}[theorem]{Corollary}
\theoremstyle{definition}
\newtheorem{definition}[theorem]{Definition}
\newtheorem{assumption}[theorem]{Assumption}
\theoremstyle{remark}

\usepackage{mleftright}
\usepackage{bm}

\usepackage{enumitem}

\usepackage[svgnames,x11names,table]{xcolor}
\definecolor{myMaroon}{RGB}{128, 0, 0}

\AtEndPreamble{
\hypersetup{
    colorlinks=true,       %
    linkcolor=blue,        %
    citecolor=myMaroon,         %
    filecolor=magenta,     %
    urlcolor=cyan          %
}
}

\usepackage{pifont}
\usepackage{scalerel}

\usepackage{tikz}
\usetikzlibrary{positioning, calc, fit,shapes.misc,arrows.meta}

\def\[#1\]{\begin{align*}#1\end{align*}}

\newcommand{\kours}{\texttt{Kernelized-COMWU}\xspace}

\newcommand{\COMWU}{\texttt{COMWU}\xspace}
\newcommand{\OFTRL}{\texttt{OFTRL}\xspace}
\newcommand{\OOMD}{\texttt{OOMD}\xspace}
\newcommand{\FTRL}{\texttt{FTRL}\xspace}

\newcommand{\LRLOFTRL}{\texttt{LRL-OFTRL}\xspace}
\newcommand{\ours}{\texttt{COFTRL}\xspace}

\newcommand{\Opthedge}{\texttt{OMWU}\xspace}
\newcommand{\Hedge}{\texttt{MWU}\xspace}

\newcounter{qst}
\crefname{qst}{Question}{Questions}

\numberwithin{equation}{section}

\def\ind{k}

\DeclareMathOperator{\reg}{Reg}
\DeclareMathOperator{\tildereg}{{\tilde R}eg}

\newcommand{\regdep}{\mathfrak{R}}

\newcommand{\defeq}{\coloneqq}
\renewcommand{\^}[1]{^{(#1)}}
\newcommand{\bbR}{\mathbb{R}}

\renewcommand{\vec}[1]{\bm{#1}}
\newcommand{\vstack}[2]{\begin{pmatrix} #1 \\ #2 \end{pmatrix}}

\let\hat\widehat
\let\tilde\widetilde

\DeclareMathOperator*{\dprime}{{\prime \prime}}

\newcommand{\argmin}{\mathop{\mathrm{arg\,min}}}
\newcommand{\argmax}{\mathop{\mathrm{arg\,max}}}

\newcommand{\ent}[1]{\textup{H}(#1)}

\newcommand{\expect}{\mathbb{E}}

\newcommand{\vP}{\vec{P}}
\newcommand{\vH}{\vec{H}}
\newcommand{\vT}{\vec{T}}
\newcommand{\vertex}{\vec{v}}

\newcommand{\ut}{\vec{u}}
\newcommand{\Ut}{\vec{U}}

\newcommand{\nut}{\vec{\nu}}
\newcommand{\mut}{\vec{\mu}}
\newcommand{\at}{\vec {\mathsf{r}}}

\newcommand{\Nut}{\vec{ \mathcal{V}}}
\newcommand{\At}{\vec {\mathcal{R}}}

\newcommand{\vebar}{\overline{\vec{e}}}
\newcommand{\vb}{\vec{b}}
\newcommand{\vvv}{\vec \upsilon}
\newcommand{\vchi}{\vec{\chi}}
\newcommand{\vx}{{\vec{x}}}

\newcommand{\lambdap}{{\lambda^\prime}}
\newcommand{\vy}{\vec{y}}
\newcommand{\vz}{\vec{z}}
\newcommand{\vtheta}{\vec{\theta}}

\usepackage{amsthm}

\newenvironment{restatetheorem}[1]{%
  \renewcommand{\thetheorem}{\ref{#1}}%
  \addtocounter{theorem}{-1}%
  \begin{theorem}
    }{%
  \end{theorem}
  \addtocounter{theorem}{1}%
}

\newenvironment{restatelemma}[1]{%
  \renewcommand{\thelemma}{\ref{#1}}%
  \addtocounter{lemma}{-1}%
  \begin{lemma}
    }{%
  \end{lemma}
  \addtocounter{lemma}{1}%
}

\newenvironment{restateproposition}[1]{%
  \renewcommand{\theproposition}{\ref{#1}}%
  \addtocounter{proposition}{-1}%
  \begin{proposition}
    }{%
  \end{proposition}
  \addtocounter{proposition}{1}%
}

\newenvironment{restatecorollary}[1]{%
  \renewcommand{\thecorollary}{\ref{#1}}%
  \addtocounter{corollary}{-1}%
  \begin{corollary}
    }{%
  \end{corollary}
  \addtocounter{corollary}{1}%
}

\theoremstyle{definition}

\numberwithin{theorem}{section}

\newenvironment{restatedefinition}[1]{%
  \renewcommand{\thedefinition}{\ref{#1}}%
  \addtocounter{definition}{-1}%
  \begin{definition}
    }{%
  \end{definition}
  \addtocounter{definition}{1}%
}

\NewDocumentCommand{\numberthis}{om}{%
  \IfNoValueTF{#1}{%
    \refstepcounter{equation}\tag{\theequation}%
  }{%
    \tag{#1}%
  }%
  \label{#2}%
}

\definecolor{darkgrey}{gray}{0.3}
\definecolor{commentcolor}{gray}{0.5}
\SetKwComment{Hline}{}{\vspace{-3mm}\textcolor{gray}{\hrule}\vspace{1mm}}
\SetKwComment{Comment}{\color{commentcolor}[$\triangleright$\ }{}
\SetCommentSty{}
\SetNlSty{}{\color{darkgrey}}{}
\SetKwProg{Fn}{function}{}{}
\SetKwProg{Subr}{subroutine}{}{}
\crefalias{AlgoLine}{line}%
\crefname{algocf}{Algorithm}{Algorithms}

\newcommand{\crossmark}{\ding{55}}
\renewcommand{\checkmark}{\ding{51}}

\title{Cautious Optimism: A Meta‐Algorithm for Near‐Constant Regret in General Games}

\author{Ashkan Soleymani}
\email{ashkanso@mit.edu}
\affiliation{%
  \institution{MIT}
  \city{Cambridge}
  \country{USA}
}

\author{Georgios Piliouras}
\email{gpil@google.com}
\affiliation{%
  \institution{Google DeepMind}
  \city{London}
  \country{UK} 
}

\author{Gabriele Farina}
\email{gfarina@mit.edu}
\affiliation{%
  \institution{MIT}
  \city{Cambridge}
  \country{USA}
}

\begin{document}

\begin{abstract}

We introduce \emph{Cautious Optimism}, a framework for substantially faster regularized learning in general games. Cautious Optimism, as a variant of Optimism, adaptively controls the learning pace in a dynamic, non-monotone manner to accelerate no-regret learning dynamics. Cautious Optimism takes as input any instance of Follow-the-Regularized-Leader (\FTRL) and outputs an accelerated no-regret learning algorithm (\ours) by pacing the underlying \FTRL with minimal computational overhead. Importantly, it retains uncoupledness, that is, learners do not need to know other players’ utilities. Cautious Optimistic \FTRL (\ours) achieves near-optimal $O_{\scaleto{T}{4pt}}(\log T)$ regret in diverse self-play (mixing and matching regularizers) while preserving the optimal $O_{\scaleto{T}{4pt}}(\sqrt{T})$ regret in adversarial scenarios. In contrast to prior works (e.g., \citet{Syrgkanis15}, \citet{daskalakis2021near}), our analysis does not rely on monotonic step sizes, showcasing a novel route for fast learning in general games. Moreover, instances of \ours achieve new state-of-the-art regret minimization guarantees in general convex games, exponentially improving the dependence on the dimension of the action space $d$ over previous works~\citep{farina2022near}.%
\footnote{
Preliminary versions of this work appeared in the proceedings of the ACM Symposium on Theory of Computing (STOC 2025, \citep{Soleymani25:Faster}) and as a one-page extended abstract at the ACM Conference on Economics and Computation (EC 2025, \citep{soleymani2025cautious}). The present version provides a comprehensive and unified treatment of this framework, substantially extending the preliminary versions.
}

\end{abstract}

\maketitle

\setcounter{tocdepth}{1}
\tableofcontents

\newpage

\section{Introduction}

Understanding multi-agent systems is fundamental to a wide range of disciplines, including economics, game theory, artificial intelligence, optimization, control theory, and human sciences~(e.g., \citet{von1947theory,rawls1971atheory,shoham2008multiagent,egerstedt2001formation,lewis2008convention,gintis2014bounds}). This area has gained importance with the advent of machine learning, where multi-agent systems are integral to many fundamental architectures and applications~\citep{goodfellow2014generative,silver2017mastering,goodfellow2020generative,bighashdel2024policy}. In these systems, multiple autonomous agents interact, often with individual objectives that may be cooperative, competitive, or even a combination of both. It is difficult to overstate the importance of equilibrium concepts in the study of multi-agent systems, characterizing stable outcomes where no agent has an incentive to unilaterally deviate. In the context of games, the Nash equilibrium stands as the primary solution concept~\citep{nash1950equilibrium}.

Online learning and game theory are intimately linked through the celebrated framework of \emph{regret minimization}, wherein an algorithm iteratively updates its decisions to minimize the difference between its cumulative utility and that of the best fixed strategy in hindsight.
Originally formalized within the context of single-agent decision making~\citep{robbins1952some}, regret minimization has deep historical ties to game theory, tracing back to fictitious play \citep{brown1949some, robinson1951iterative}. This connection is not merely conceptual but fundamental, as arguably the most scalable and effective algorithms for finding equilibrium concepts are \emph{uncoupled no-regret learning} algorithms~\citep{cesa2006prediction}, leading to pioneering developments in artificial intelligence~\citep{bowling2015heads,brown2018superhuman,moravvcik2017deepstack,brown2019superhuman}. Moreover, beyond its theoretical significance and practical application in equilibrium computation, regret minimization serves as a powerful model for rational behavior in real-world applications, where agents adaptively respond to their environments based on past experiences~\citep{halpern2012iterated}.

Within the class of two-player zero-sum matrix games, no-regret algorithms are known to converge in a time-average sense to the Nash equilibrium of the game. In particular, if both players follow a no-regret learning algorithm, the time-average of their marginal strategies results in an approximate Nash equilibrium~\citep{cesa2006prediction}. However, while this result is appealing, it does not extend even to the relatively restricted setting of multi-player normal-form games. On the contrary, strong impossibility results demonstrate that no \emph{uncoupled} learning dynamics can guarantee convergence to a Nash equilibrium in such games~\citep{hart2003uncoupled,milionis2023impossibility}. There exists a fundamental barrier: not only do uncoupled learning dynamics fail to converge to Nash equilibria in general settings, but even centralized algorithms with full access to game parameters face inherent computational hardness in Nash equilibrium computation, even in two-player general-sum games~\citep{daskalakis2009complexity,chen2006settling}.  

On the positive side, regret minimization extends to general games, converging to \emph{Coarse Correlated Equilibria} (CCE)~\citep{cesa2006prediction}.  
Formally, the empirical play of no-regret learners in a game converges to an approximate CCE, where the approximation level is controlled by the maximum regret of the players. As the average regret diminishes over time, the equilibrium approximation improves, making regret minimization a natural and scalable framework for equilibrium computation 
\citep{cesa2006prediction,monnot2017limits,roughgarden2017price,papadimitriou2008computing,roughgarden2015intrinsic,babichenko2014simple,keiding2000correlated}. Beyond convergence to CCE in game settings, regret minimization is recognized as a fundamental concept with wide-ranging applications, including learning and generalization~\citep{littlestone1988learning,blum1990learning,lugosi2023online}, von Neumann’s minimax theorem and Blackwell approachability~\citep{freund1996game,abernethy2011blackwell}, boosting~\citep{freund1997decision,freund1996game}, combinatorial optimization~\citep{plotkin1995fast,arora2007combinatorial}, complexity theory~\citep{klivans1999boosting,barak2009uniform}, differential privacy~\citep{hardt2010multiplicative,hsu2013differential}, prediction markets~\citep{chen2010new,abernethy2013efficient}, evolutionary dynamics~\citep{chastain2013multiplicative}, and more.

Online learning traditionally focuses on no-regret algorithms designed to perform optimally in adversarial environments, achieving the minimax optimal regret bound of $O_{\scaleto{T}{4pt}}(\sqrt{T})$\footnote{In $O_{\scaleto{T}{4pt}}$ notation, only the dependence on the time horizon $T$ is explicitly represented.
} in such settings. However, this worst-case perspective is overly pessimistic in the context of \emph{self-play} in games, where the environment exhibits a much higher degree of predictability, despite being nonstationary. This predictability arises from the structured dynamics of the game, where players' utilities evolve smoothly as a function of their actions over time. This phenomenon naturally begs the question: what are the fastest no-regret learning algorithms in games? This remains a fundamental open problem in game theory.

The breakthrough work of \citet{Syrgkanis15} took the first major step in this line, by showing that the framework of \emph{Optimism} framework~\citep{rakhlin2013online,rakhlin2013optimization} leads to faster regrets in the self-play settings. They proved that optimistic variants of regularized learning in games enjoy $O_{\scaleto{T}{4pt}}(T^{1/4})$ regret, an $O_{\scaleto{T}{4pt}} (T^{1/4})$ improvement over the online learning with adversarial utilities. Once theoretical improvements were proven to be feasible, a series of works~\citep{Chen20:Hedging,daskalakis2021near,farina2022near,piliouras2022beyond} aimed to achieve faster regret rates in games by designing specialized, tailor-made algorithms. 
We discuss recent works in detail in \Cref{sec:prior_work}. Optimistic regularized learning algorithms achieve \emph{constant social regret}, ensuring that the total regret of all players in self-play remains bounded over time~\citep{Syrgkanis15}. Unfortunately, controlling \emph{individual regrets}—essential for convergence to Coarse Correlated Equilibria (CCE)—remains considerably more challenging, and despite nearly a decade of substantial progress, the optimal regret rate is yet to be determined. Moreover, although recent works~\citep{daskalakis2021near,farina2022near} have achieved a remarkable exponential improvement, reducing individual regret to $O_{\scaleto{T}{4pt}}(\log T)$, these advances remain confined to single-instance algorithms and lack the generality of the \emph{Optimism} framework of \citet{Syrgkanis15}. In this work, we aim to bridge these gaps.

We introduce the framework of \emph{Cautious Optimism}, a broad characterization of exponentially faster regret minimization, achieving $O_{\scaleto{T}{4pt}}(\log T)$ for regularized learning in games. The central idea is to \emph{pace the dynamics of learners} by employing a non-monotone dynamic learning rate. Intuitively, this mechanism prevents agents from accumulating runaway negative regret. Our key insight is to refine Optimism into a form of Cautious Optimism, where agents decrease their learning rates once their regret becomes too negative, that is, when they substantially outperform all fixed actions. Perhaps unexpectedly, we demonstrate that this principle
can be systematically instantiated to establish new state-of-the-art regret bounds.

Cautious Optimism takes as input an instance of the Follow-the-Regularized-Leader (\FTRL) algorithm—typically achieving $O_{\scaleto{T}{4pt}}(T^{1/4})$ regret\footnote{For general \OFTRL algorithms~\citep{Syrgkanis15}.}—and produces a substantially accelerated no-regret learning algorithm, Cautious Optimistic Follow-the-Regularized-Leader (\ours), with $O_{\scaleto{T}{4pt}}(\log T)$ regret. This acceleration is achieved by dynamically adjusting the pace of the underlying \FTRL{} method, using a structured, nonmonotonic adaptation of the learning rate that enables significantly faster convergence while maintaining minimal computational overhead. In \Cref{fig:coftrl}, we summarize the structure of \ours.

Our framework, through its design of non-monotone dynamic learning rates, establishes a new paradigm for learning dynamics in games. By substantially departing from traditional approaches in online learning and optimization—which primarily rely on constant or monotonically decreasing learning rates—our framework leverages the novel principle of \emph{dynamic pacing}, introducing a principled mechanism to accelerate regret minimization in multi-agent learning settings.

Cautious Optimism mirrors the generality of Optimism \citep{Syrgkanis15}, enables \emph{diverse self-play} by allowing the mixing and matching of regularizers, recovers \LRLOFTRL~\citep{farina2022near} as a special case, and introduces \emph{Cautious Optimistic Multiplicative Weight Updates} (\COMWU) algorithm
and 
two other new \ours instances with regret guarantees achieving a new state-of-the-art regret bound in finite games---see also \Cref{table:regularizers,table:results}. Cautious Optimism applies naturally to the broad class of convex games, with one instance achieving a new state-of-the-art regret bound in convex games.  We discuss our contributions in more depth in \Cref{sec:contributions}. 

\subsection{Prior Work and Context for Cautious Optimism} \label{sec:prior_work}

In this section, we discuss major prior work on accelerating no-regret learning in games. Let $T$ denote the time horizon, $d$ denote the number of actions per player, and $n$ denote the number of players. Please refer to \Cref{table:results} for a summarized comparison of \ours with existing methods.

The quest for optimal regret bounds in general games was initiated by the seminal work of \citet{Syrgkanis15}, which established the first $o_{\scaleto{T}{4pt}}(T^{1/2})$ regret rate for self-play (better than the optimal rate for online learning with adversarial utilities). A key contribution of their work was the introduction of the \emph{RVU property}, an adversarial regret bound that applies broadly to a class of optimistic no-regret learning algorithms, including Follow-the-Regularized-Leader and Mirror Descent. By leveraging this property, they demonstrated that while social regret remains constant, individual player regret grows at a rate of $O_{\scaleto{T}{4pt}}(T^{1/4})$, leaving the challenging problem of determining the optimal individual regret rate largely unresolved. 

More recently, \citet{daskalakis2021near} achieved a major breakthrough by exponentially improving the regret guarantees for Optimistic Multiplicative Weight Updates (\Opthedge). Their work marked the first derivation of a polylogarithmic regret bound in $T$, specifically $O(n \log d \log^4 T)$. This result was obtained through a careful analysis of the smoothness properties of higher-order discrete differentials of the softmax operator in the Multiplicative Weight Updates algorithm and their implications in the frequency domain.  

However, due to their specific choice of learning rate, their analysis is valid only in the regime where $T > C n \log^4 T$ where $C$ is a large constant—see Lemmas 4.2 and C.4 in \cite{daskalakis2021near} for details. Additionally, the hidden constants in their order notation have a very large magnitude.  Most importantly, it remains unknown how these ideas can be extended to the regret analysis of other regularized learning algorithms.

\citet{piliouras2022beyond} introduced Clairvoyant MWU (CMWU), where players leverage the game dynamics as a broadcast channel to implement Nemirovski’s Conceptual Proximal Method~\citep{nemirovski2004prox,farina2022clairvoyant} in a decentralized, uncoupled manner. However, this attitude towards the dynamics of the game play comes at a cost: not all iterations of Clairvoyant MWU satisfy the no-regret property. Instead, only a subsequence of $\Theta(T / \log T)$ iterations exhibits bounded regret of $O(n \log d)$, while the remaining iterations are dedicated to broadcasting.  As a result, the empirical play of this subsequence converges to a CCE of the game at a rate of $O(\frac{n \log d \log T}{T})$. Additionally, CMWU does not guarantee the no-regret property in the presence of adversarial utilities.

\citet{farina2022near} established a logarithmic dependence on $T$ throughout the entire history of play using the Log-Regularized Lifted Optimistic FTRL (\LRLOFTRL) algorithm. Their approach builds upon an instantiation of optimistic Follow-the-Regularized-Leader (FTRL) over a suitably lifted space with a log regularizer. However, this improved dependence on $T$ comes at the cost of an exponentially worse dependence on the number of actions $d$, resulting in an overall regret bound of $O(n d \log T)$.  As we see later in \Cref{app:find_gamma}, the dynamics of \LRLOFTRL are equivalent to those of \ours with a log regularizer, which is a special instance of Cautious Optimism.

Despite all these efforts~\citep{Chen20:Hedging,daskalakis2021near,piliouras2022beyond,farina2022near} following the celebrated result of \citet{Syrgkanis15}, no \emph{unified} framework or analysis broadly characterized faster convergence than \citet{Syrgkanis15} for regularized learning in games is known.  %
We address this gap by introducing \emph{Cautious Optimism}, which establishes a unified framework that exponentially accelerates regularized learning in games in an uncoupled manner, improving the $O_{\scaleto{T}{4pt}}(T^{1/4})$ regret bound of \citet{Syrgkanis15} to $O_{\scaleto{T}{4pt}}(\log T)$. Ultimately, it becomes evident that an instance of Cautious Optimistic \FTRL (\ours) with a log regularizer recovers the dynamics of \LRLOFTRL~\citep{farina2022near}. However, this recovery arises as a simple and immediate byproduct of our general framework and analysis, in contrast to \citep{farina2022near}, where the spectral properties and structural details of the chosen regularizers were heavily exploited.

Moreover, we introduce \COMWU, a refined variant of the Optimistic Multiplicative Weights Update (\Opthedge) algorithm, as an instance of \ours with negative entropy regularizer. \COMWU and two other instances of \ours achieve a new state-of-the-art regret bound of $O(n \log^2 d \log T)$. Interestingly, one of these instances of \ours extends to the class of convex games, exponentially improving dependence on the dimension of the action space $d$ upon the previous state-of-the-art, \LRLOFTRL~\citep{farina2022near}, in this broad setting. Our approach is grounded in dynamic learning rate control, ensuring that learners progress in a coordinated manner throughout the game while preventing agents from developing excessively negative regret. 

\COMWU enjoys an exponential improvement in $d$ compared to \LRLOFTRL~\citep{farina2022near} and significantly reduces the dependence on the time horizon, improving from $\log^4 T$ in \Opthedge~\citep{daskalakis2021near} to $\log T$. Moreover, unlike \citet{daskalakis2021near}, which requires $T > C n \log^4 T$ for some constant $C$ (see Lemmas 4.2 and C.4 in~\citep{daskalakis2021near}), our analysis remains valid for all values of $n$, $d$, and $T$, and holds simultaneously across all time steps $t$, thereby ensuring uniform convergence. As a further remark, the regret analysis of \Opthedge~\citep{daskalakis2021near} obscures extremely large constants in the asymptotic notation, particularly in comparison to ours.

The idea of agents adjusting their behavior depending on whether they feel content or discontent is both simple and intuitive, and has inspired a variety of game dynamics that provably concentrate around pure Nash equilibria~\citep{young2009learning}, as well as adjusted replicator dynamics in evolutionary game theory~\citep{weibull1997evolutionary}. In the context of regret minimization, this idea can be traced back to \citet{bowling2002multiagent}, who introduced the Win or Learn Fast (WoLF) principle. Under WoLF, agents increase their learning rate when they are losing, enabling faster adaptation to the environment and to the strategies of other agents. \citet{bowling2002multiagent} established convergence of gradient ascent–descent with WoLF to a Nash equilibrium in the restricted setting of two players with two actions. Subsequently, \citet{bowling2004convergence} extended WoLF to multiplayer games, showing that it achieves no-regret dynamics with $O(d \sqrt{T})$ regret. Despite the success of WoLF and related heuristics in small-scale games~\citep{abdallah2008multiagent,ratcliffe2019win,xi2015novel,awheda2016exponential,georgila2014single,busoniu2006multi,zhu2021analysis}, their theoretical benefits in general games remain largely unresolved.

On the negative side, it has recently been shown that the multiplicative weights update algorithm, even when equipped with a continuous learning rate (stronger than the original fast and slow rates of WoLF~\citep{bowling2002multiagent}), exhibits chaotic behavior in nonatomic congestion games~\citep{vlatakis2023chaos}. While our motivation for adaptive learning rates and our methodology differ from the WoLF principle, to the best of our knowledge, \ours is the {first algorithmic framework to demonstrate theoretical benefits for such ideas in decision-making, and in particular for general regularized learning in games}.

\begin{table}[t]
    \centering
    \newcommand{\ldarrow}{\raisebox{-.7mm}{\tikz \draw[->] (0,0) -- (.25,0) -- +(0, -.2);}}%
    
    \resizebox{\textwidth}{!}{ 
    \begin{tabular}{>{\arraybackslash}m{5.0cm} >{\arraybackslash}m{4.5cm} lc}
        \bf Method                                            & \bf Regret in Games     & \bf Adversarial Regret & \bf General Learners \\
        \toprule
        OFTRL / OOMD\newline\citep{Syrgkanis15}            & $O(\sqrt{n}\hspace{0.5 mm} \regdep(d) T^{1/4})$    &  $\Tilde{O}(\sqrt{T \log d})$ & \checkmark \\
        \midrule

        COFTRL \newline\textbf{[This paper]}             &  $O(n \hspace{0.5 mm} \Gamma(d) \log T)$                     &   $\Tilde{O}(\sqrt{T \log d})$ & \checkmark\\
        \midrule
        \midrule
        OMWU\newline\citep{Chen20:Hedging}\!                             & $O(n \log^{5/6} d \: T^{1/6})\: \dagger $                      &    $\Tilde{O}(\sqrt{T \log d})$ & \crossmark \\ 
        \midrule
        OMWU\newline\citep{daskalakis2021near}\!                             & $O(n \log d \log^4 T)$                       &   $\Tilde{O}(\sqrt{T \log d})$   & \crossmark   \\
        \midrule
        Clairvoyant MWU\newline\citep{piliouras2022beyond}              & $O(n \log d)$\newline for a subsequence only~$\ddagger$      & No guarantees    & \crossmark                                                           \\
        \midrule
          LRL-OFTRL\newline\citep{farina2022near} \newline\textbf{[$\equiv$ \ours w/ log regularizer]}              &  $O(n \hspace{0.5 mm} d  \log T)$       &   $\Tilde{O}(\sqrt{T \log d})$    & \crossmark                                                                        \\
        \midrule
          \COMWU \newline\textbf{[This paper]}              &  $O(n \log^2 d \log T)$                  &   $\Tilde{O}(\sqrt{T \log d})$     & \crossmark                                                        \\
        \midrule
          \ours with $\ell_{p^*}$ \newline\textbf{[This paper]}             &  $O(n \log^2 d \log T)$                     &   $\Tilde{O}(\sqrt{T \log d})$       & \crossmark                                                      \\
          \midrule
          \ours with  $q^*$-Tsallis entropy \newline\textbf{[This paper]}             &  $O(n \log^2 d \log T)$                     &   $\Tilde{O}(\sqrt{T \log d})$      & \crossmark                                                       \\
        \bottomrule
    \end{tabular}
    }    
    \caption{ 
    Comparison of existing no-regret learning algorithms in general finite games. We define $n$ as the number of players, $T$ as the number of game repetitions, and $d$ as the number of available actions. For simplicity, dependencies on smoothness and utility range are omitted.  $\dagger$ Applicable only to two-player games ($n = 2$).  $\ddagger$ Unlike other algorithms, Clairvoyant MWU (CMWU) does not guarantee sublinear regret for its full sequence of iterates. Instead, after $T$ iterations, only a subsequence of length $\Theta(T/\log T)$ achieves the regret bound stated in the table.
    }
    \label{table:results}
\end{table}

\subsection{Contributions and Techniques} 

In this paper, we introduce Cautious Optimism, a uncoupled framework (meta-algorithm) that takes an instance of the Follow-the-Regularized-Leader (\FTRL) algorithm as input and produces an accelerated no-regret learning algorithm—Cautious Optimistic \FTRL (\ours)—by adaptively regulating the pace of the underlying \FTRL method through dynamic control of the learning rate.

Our work provides the first comprehensive characterization of exponentially faster regret minimization for regularized learning in general games. \ours achieves a near-optimal regret bound of $O_{\scaleto{T}{4pt}}(\log T)$ while simultaneously maintaining the optimal $O_{\scaleto{T}{4pt}}(\sqrt{T})$ regret bound in adversarial settings.
Furthermore, \ours is the first framework that achieves near-optimal $O_{\scaleto{T}{4pt}}(\log T)$ regret through \emph{diverse self-play}, where players mix and match their choice of regularizers for \FTRL, without requiring all players to use the exact same instance of \ours.

We show that social regret of \ours has the same optimal $O_{\scaleto{T}{4pt}}(1)$ as its optimistic counterparts, e.g., \OFTRL. Thus, \ours can enjoy best of both worlds type of guarantees, i.e., $O_{\scaleto{T}{4pt}}(\log T)$ individual regret and $O_{\scaleto{T}{4pt}}(1)$ social regret simultaneously, while for the optimistic counterparts the current choices of learning rates for $O_{\scaleto{T}{4pt}}(T^{1/4})$ individual regret and $O_{\scaleto{T}{4pt}}(1)$ social regret are different~\citep[Corollary 8 and Corollary 12]{Syrgkanis15}.

In addition to the exponentially faster convergence, \ours retains the generality of \citet{Syrgkanis15} and reproduces the dynamics of \LRLOFTRL\citep{farina2022near} as a special case (Please refer to \Cref{table:results} for details). Specifically, when instantiated with \FTRL using a logarithmic regularizer, \ours recovers the dynamics of \LRLOFTRL~\citep{farina2022near}. There, \emph{Cautious Optimism} is applied as the overarching adaptive learning rate control mechanism. This connection sheds new light on the underlying reasons for the rapid convergence of \LRLOFTRL~\citep{farina2022near}. While the analyses of~\citep{farina2022near} are heavily tailored to their specific choices of regularizers and dynamics—such as the strong multiplicative stability of actions induced by the high curvature of the log-regularizer in \LRLOFTRL, we recover their convergence guarantees as an immediate consequence of our general analysis.

We introduce the first uncoupled no-regret learning algorithms as three instances of \ours: Cautious Optimistic Multiplicative Weight Update (\COMWU), \ours with $\ell_{p^*}$, and \ours with $q^*$-Tsallis entropy, achieving regret guarantees of order $O(n \log^2 d \log T)$, a new state-of-the-art in learning in games (see \Cref{table:results} for details). These instances of \ours achieve an exponential in $d$ improvement relative to \LRLOFTRL~\citep{farina2022near} and substantially reduce the horizon dependence—from $\log^4 T$ under \Opthedge~\citep{daskalakis2021near} down to $\log T$. In contrast to \citet{daskalakis2021near}, our guarantees hold uniformly for all regimes of $n$, $d$, and $T$.

Furthermore, inspired by Kernelized \Opthedge~\citep{farina2022kernelized}, we introduce a kernelized version of \COMWU (denoted as \kours), applicable to convex $0$/$1$-polyhedral games such as extensive-form games and flows on directed graphs. In this way, we show that \COMWU inherits the fundamental and intriguing properties of \Opthedge.

Finally, we show that Cautious Optimism extends naturally to the general class of convex games, and when instantiated with the $\ell_{p^*}$ norm, attains $O(n \log^2 d \log T)$ regret, replacing the linear dependence on $d$ in the previous state-of-the-art, \LRLOFTRL~\citep{farina2022near}, with a $\log^2 d$ factor in this broad setting.

At the technical level, we conceptualize the idea of learning-rate control for no-regret learning  and formalize it to design optimization algorithms for learning-rate control, resulting in \ours, a computationally efficient algorithm. The concept of dynamic learning rates has the potential to be beneficial in other areas involving regret minimization, particularly in multi-agent settings.

Moreover, we introduce a relaxed notion of Lipschitz continuity, termed \emph{intrinsic Lipschitzness}, which quantifies changes in the regularizer value through its \emph{Bregman divergence}. \ours applies to any \FTRL algorithm with an \emph{intrinsically Lipschitz} regularizer. As we will demonstrate, intrinsic Lipschitzness is a mild condition, satisfied by a broad class of regularizers, including all Lipschitz continuous regularizers. 

We show that the dynamics of \ours are equivalent to an instance of \OFTRL on the lifted space $(0, 1]\Delta^d$ with a specific composite regularizer obtained as the sum of a strongly convex part and a nonconvex transformation of the original regularizer used by the underlying \FTRL. We prove strong convexity (and self-concordance) of the composite regularizer in a specific regime of hyperparameters and provide technical steps to convert the regret analysis of the resulting \OFTRL on the lifted space $(0, 1]\Delta^d$ to that of \ours. These techniques hold broadly for any intrinsic Lipschitz regularizer and are not constrained by the structural properties of the learning rate control problem.

By studying the geometry induced by the dynamic learning-rate control problem, we prove the multiplicative stability of consecutive learning rates, design efficient 0th-, 1st-, and 2nd-order algorithms to solve the learning-rate control problem up to multiplicative accuracy, and analyze the guarantees of \ours under approximate iterates. This way, we characterize the iteration complexity of \ours and show that the computational overhead compared to \OFTRL is minimal.

And lastly, with \ours, we provide the first theoretical demonstration that dynamic learning-rate adjustments yield benefits in general regularized learning in games—resolving a longstanding open question motivated by heuristics and practice~\citep{bowling2002multiagent,bowling2004convergence,abdallah2008multiagent,busoniu2006multi,donancio2024dynamic}. To our knowledge, \ours is the first framework to translate WoLF-style intuitions into rigorous, general guarantees for regularized learning in games, thereby closing this theoretical gap, as discussed in \Cref{sec:prior_work}.

\label{sec:contributions}

\subsection{Organization of the Paper}

The paper is organized as follows. In \Cref{sec:finite_games_def}, we introduce notation and basic concepts needed in the sequel. In \Cref{sec:cautious_optimism} we introduce and analyze the learning rate control procedure that enables accelerating general FTRL learners and forms the core technical idea behind \ours. \Cref{sec:design} discusses alternative perspectives on the learning rate control subroutine, including a reformulation of the underlying optimization problem which will be especially useful in analyzing the regret cumulated by \ours. \Cref{sec:analysis} presents the analysis of \ours in normal-form games, establishing accelerated $O_T(\log T)$ rates for individual regret.  \Cref{sec:social_regret} analyzes the social regret incurred by \ours, establishing $O_T(1)$ rates consistent with the Optimism framework. \Cref{sec:approx} bounds the effect of approximation errors in the solution of the learning rate control problem, showing limited impact on individual regret. \Cref{sec:iteration} analyzes optimization algorithms for efficiently solving the learning rate control problem. \Cref{sec:special} discusses particular instantiations of our framework, and their regret guarantees. One such notable instantiation is Cautious Optimistic Multiplicative Weights Update (\COMWU). We show in the same section that \COMWU can be efficiently extended to 0/1-polyhedral games (including extensive-form games). Finally, \Cref{sec:convex_games} discusses extensions of \ours to general convex games.

\section{Finite Games and Learning Dynamics} \label{sec:finite_games_def}

Consider a finite $n$-player game, each player $i \in [n]$ having a deterministic strategy space $\mathcal{A}_i$, resulting mixed strategy space $\mathcal{X}_i \defeq \Delta(\mathcal{A}_i) = \Delta^{|\mathcal{A}_i|} $ and a set of utility functions $\mathcal{U}_i : \bigtimes_{j=1}^n \mathcal{A}_j \rightarrow  \mathbb{R}$, for a joint strategy profile $\vx = (\vx_1, \vx_2, \ldots, \vx_n) \in \bigtimes_{j=1}^n \mathcal{X}_j$, we indicate the expected utility of player $i$ by $\nut_i(\vx) \defeq \mathbb{E}_{\vec{s} \sim \vx}[\mathcal{U}_i(\vec{s})] = \langle \vx_i, \nabla_{\vx_i} \nut_i(\vx) \rangle$ where $\nabla_{\vx_i} \nut_i(\vx)[\xi] = \mathbb{E}_{\vec{s}_{-i} \sim \vx_{-i}} [\mathcal{U}_i(\xi, \vec{s}_{-i})]$ for all $\xi \in \mathcal{A}_i$. We assume that the utility functions are bounded by one. Additionally, for simplicity, we define $d \coloneqq \max_{i \in [n]} |\mathcal{A}_i|$. Given that the rates depend on $\max_{i \in [n]} |\mathcal{A}_i|$, we assume, for mathematical simplicity, that $d = |\mathcal{A}_i|$ for all players $i \in [n]$. For convenience, except in \Cref{sec:convex_games}, we use $\mathcal{X}$ and $\Delta^d$ interchangeably and define $\Omega \defeq (0, 1]\mathcal{X}$.

We study the regret in the self-play setting, where the game is repeated over $T$ time steps. At each time step $t$, each player $i$ simultaneously selects an action $\vx_i\^t \in \mathcal{X}_i$. Then each player $i$, observes the reward vector $\nut\^t_i = \nabla_{\vx_i} \nut_i(\vx)[\xi]$ and gains the expected utility $\langle \nut_i\^t, \vx_i\^t \rangle$. We drop the subscript $i$ and write down $\nut\^t$ and $\vx\^t$, whenever it is clear from the context.

Our goal is to design no-regret learning algorithms for the players such that after $T$ rounds, their regret,
\[
  \reg\^T \defeq \max_{\vx^* \in \mathcal{X}}  \sum_{t=1}^T \langle \nut\^t, \vx^* \rangle  - \sum_{t=1}^T \langle \nut\^t, \vx\^t  \rangle
\]
is sublinear $o(T)$ and small as a function of parameters of the game $n, d$ and time horizon $T$.

For details on notation, problem formulation and background, please refer to \Cref{sec:background}.

\begin{figure}[t]
    \centering
    \scalebox{.87}{\begin{tikzpicture}[
        node distance=1.8cm and 2.0cm, %
        thick, 
        every node/.style={align=center},
        box1/.style={rectangle, draw=black, fill=red!10, rounded corners, minimum height=1.2cm, minimum width=3.4cm, text width=3.4cm, font=\small}, %
        box2/.style={rectangle, draw=black, fill=green!10, rounded corners, minimum height=1.2cm, minimum width=3.4cm, text width=3.4cm, font=\small} %
    ]

    \node (X) at (0,0) {\large $\mathbf{r}^{(t)} \in \mathbb{R}^d$};
    \node[box1, right=1.0cm of X] (pi) {{Dynamic learning} \\ {rate control}};
    \node[box2, right=2.2cm of pi] (F) {{Optimistic FTRL with} \\ {learning rate} $\lambda^{(t)}$}; %
    \node (output) [right=1.0cm of F] {\large $\vx^{(t)} \in \mathcal{X}$};
    
    \draw[->] (X) -- (pi);
    \draw[->] (pi) -- (F) node[midway, above, yshift=3pt] { $\lambda^{(t)} \in (0, \eta]$};
    \draw[->] (F) -- (output);
    
    \draw[->] (X.south) |- ([yshift=-0.8cm] F.south) -| (F.south); %

    \node[draw=black, dashed, inner sep=10pt, fit=(pi) (F)] {}; %

    \node at ([yshift=1.2cm, xshift=0.2\linewidth] pi) {\textbf{Cautious Optimistic FTRL algorithm}}; %
    \end{tikzpicture}}
    
    \caption{Dynamics of Cautious Optimistic Follow-the-Regularized-Leader (\ours) Algorithms. \ours takes as input an instance of an \OFTRL algorithm and equips it with a dynamic learning rate control mechanism that nonmonotonically adjusts the learning rate of the underlying \OFTRL instance. We prove that this simple and lightweight overhead on top of \OFTRL leads to exponentially faster convergence guarantees for no-regret learning in games~\citep{Syrgkanis15}, for a broad class of regularizers.}
    \label{fig:coftrl}
\end{figure}

\section{Cautious Optimism} \label{sec:cautious_optimism}

Cautious Optimism is a variant of the Optimistic Follow-the-Regularized-Leader (\OFTRL) algorithms, but with non-monotone, uncoupled and adaptive adjustment of the learning rate based on the regret accumulated up to the current iteration $t$. As depicted in \Cref{fig:coftrl}, our framework can be seen as a module that takes an instance of \OFTRL as input and accelerates its convergence guarantees substantially by adding a learning rate control problem on top.

In the standard version of \OFTRL with regularizer $\psi$, the actions of the play are picked according to 
\[
\vx\^t \leftarrow \argmax_{\vx \in \mathcal{X}} \left\{ \lambda\^t \langle \at\^t, \vx \rangle - \psi(\vx) \right\}, \numberthis{eq:ftrl_step_fixed_lambda}
\]
where learning rate at time $t$ is represented by $\lambda\^t > 0$, and $\at\^t$ denotes the vector containing the accumulated optimistically-corrected regrets for each action up to time $t$. This is given by,
\[
    \at\^t[k] & \defeq (\nut\^{t-1}[k] - \langle \nut\^{t-1}, \vx\^{t-1}\rangle) + \sum_{\tau=1}^{t-1} \left[\nut\^\tau[k] - \langle \nut\^\tau, \vx\^\tau\rangle\right],
\]
for all $k \in \mathcal{A}$. To simplify further, let us define the corrected reward signal as $\ut\^t \defeq \nut\^t - \langle \nut\^t, \vx\^t\rangle \vec{1}_d$, while the accumulated signal is expressed as $\Ut\^t \defeq \sum_{\tau=1}^{t-1} \left[\nut\^\tau - \langle \nut\^\tau, \vx\^\tau\rangle \vec{1}_d\right]$.

The celebrated work of \citet{Syrgkanis15} showed that when all players in a game employ \OFTRL with a fixed learning rate $\lambda\^t = \eta$, the maximum regret accumulated by the players grows at most as $O_{\scaleto{T}{4pt}} (T^{1/4})$.\footnote{Only the dependence on the time horizon $T$ is shown. For details, please refer to \Cref{table:results}.} In this paper, we improve this result exponentially to an $O_{\scaleto{T}{4pt}} (\log T)$ dependence for a general class of regularizers, building on a novel technique that we term \emph{dynamic learning rate control}.

Unlike conventional methods in optimization and learning that enforce a monotonically decreasing learning-rate schedule, our approach allows for more flexible, non-monotone adjustments based on the learner's performance. And opposed to conventional methods, conceptually, it is \emph{not} designed as a means of circumventing uncertainty about the problem’s conditioning (e.g., Lipschitz constants). Instead, we introduce a \emph{universal dynamic learning control} mechanism that \emph{deliberately slows down the learning} process of the underlying \OFTRL whenever \emph{regret becomes excessively negative}.

Due to the counterintuitive nature of this idea, which deliberately slows down the learner when it performs too well, we term our framework \emph{Cautious Optimism}—an extension of Optimism~\citep{rakhlin2013online}, where the learner is additionally cautious about its regret so far. The \emph{universality} of our framework stems from the fact that our approach provides a general learning rate control mechanism applicable to a broad class of \OFTRL algorithms. This leads to a new general class of algorithms with exponentially faster rates, which we coin \emph{Cautious Optimistic Follow-the-Regularized-Leader} (\ours).

Our dynamic learning rate control aims to pace the learner—\emph{slowing it down when it is performing too well}—that is, \emph{when its maximum regret becomes too negative}. Such a goal might appear backwards---after all, if a learner is doing so well, why pace them down? The apparent contradiction is resolved when considering the learning system as a whole. 

In self-play settings, the predictability of the players' actions due to the smooth evolution of the learning dynamics guarantees faster convergence compared to adversarial settings. This observation motivates us to consider the convergence of the learning dynamics of multiple players as a whole, rather than at an individual level. Hence, a learner performing exceedingly well can create an imbalance that hinders others from keeping pace. Consequently, aiming for harmonic learning among players during self-play seems a natural approach to better exploit the predictability of the dynamics. While hindering learning when a player is performing exceptionally well may seem counterintuitive and even unfavorable at first glance, it helps maintain a balance among the players, thereby improving the performance of the hindered player in the long run.

There is another way to justify at the conceptual level why such ideas lead to faster convergence. As discussed in \Cref{sec:background}, no-regret dynamics converge to Coarse Correlated Equilibria at a rate dictated by the \emph{worst-performing} player—the one with the highest regret. Thus, ensuring a balanced performance among players naturally accelerates convergence.  Moreover, when all players’ regrets remain nonnegative, one can establish desirable overall properties of the learning process. These include not only small swap regret~\citep{Anagnostides22:Uncoupled}, but also iterate-level convergence to equilibrium~\citep{Anagnostides22:Last-Iterate}, and the discovery of strongly incentive-compatible equilibria~\citep{Anagnostides22:Optimistic}.

Cautious Optimism works for a broad class of convex regularizers, which we term \emph{intrinsically Lipschitz}, and achieves a regret of $O(n \Gamma_\psi(d) \log T)$ in self-play and the optimal $O (\sqrt {T \log d})$ rate in adversarial settings, where $\Gamma_\psi(d)$ depends on the properties of the chosen regularizer $\psi$. We present the structure of Cautious Optimistic \FTRL (\ours) in \Cref{fig:coftrl} and its pseudocode in \Cref{algo:cftrl}.

\begin{algorithm2e}[th]
    \SetNoFillComment
    \caption{Cautious Optimistic \FTRL (\ours)
    }\label{algo:cftrl}
    \DontPrintSemicolon
    \KwData{Learning rate $\eta$, parameters $\alpha$}\vspace{2mm}
    Set $ {\Ut}\^1, \ut^{(0)} \gets \vec{0} \in \bbR^{d}$\;
    \For{$t=1,2,\dots, T$}{
    \tcc{\color{commentcolor}\texttt{Optimism}}
    Set $\at\^t \gets {\Ut}\^t + {\vec\ut}\^{t-1}$\; \medskip
    \tcc{\color{commentcolor}\texttt{Dynamic Learning Rate Control}}
    Set  $\displaystyle \lambda\^t \gets \argmax_{\lambda \in (0, \eta]} \left\{ \alpha \log \lambda + \psi^*_{\mathcal{X}} (\lambda \at\^t)   \right\}$\label{line:oftrl0} \;  \medskip 
    \tcc{\color{commentcolor} \texttt{\OFTRL with Dynamic Learning Rate}} 
    Set $\displaystyle\vx\^t \leftarrow \argmax_{\vx \in \mathcal{X}} \left\{ \lambda\^t \langle \at\^t, \vx \rangle - \psi(\vx) \right\} $\label{line:norm0} \;  \medskip
    Play strategy $\displaystyle\vx\^t$ \;
    Observe $\vec \nu \^t \in \bbR^d$\;
    \medskip
    \label{line:lift0}
    \tcc{\color{commentcolor}\texttt{Empirical Cumulated Regrets}} 
    Set $\displaystyle {\vec\ut}\^t \gets \vec \nu \^t -\langle \nut\^t, \vx\^t\rangle \vec1_d $ \; 
    Set $ {\Ut}\^{t+1} \gets  {\Ut}\^t +  {\vec\ut}\^t$
    }
\end{algorithm2e}

\subsection{Intrinsic Lipschitzness} \label{sec:intr_lips}

A core building block of Cautious Optimism is our newly defined notion of intrinsic Lipschitzness for convex regularizers, defined as follows. This concept characterizes the Lipschitz properties of regularizers within the geometry they induce as generators.

\begin{definition} \label{def:intrin_lips}
    Let $\psi: \mathcal{X} \rightarrow \mathbb{R}$ be an arbitrary convex regularizer for the simplex, we call $\psi$, $\gamma$-intrinsically Lipschitz ($\gamma$-IL) if for all $\vx^\prime, \vx \in \mathcal{X}$,
    \[
    \mleft| \psi(\vx^\prime) - \psi(\vx) \mright|^2 \leq \gamma D_{\psi} (\vx^\prime \; \| \; \vx).
    \]
\end{definition}

This condition is general and holds for many choices of regularizers. In fact, it is easy to verify that any Lipschitz and strongly convex function $\psi$ is intrinsically Lipschitz. Thus, intrinsic Lipschitzness is a weaker condition than standard Lipschitzness in the context of regularizers. As we will see later (e.g., in \Cref{table:regularizers}), there exist many regularizers that are intrinsically Lipschitz but not Lipschitz in the standard sense, such as negative entropy and negative Tsallis entropy.

\begin{proposition}
\label{prop:lip_intr_lip}
Any regularizer $\psi$ that is $\mu$-strongly convex w.r.t. norm $\|.\|$, and $L$-Lipschitz w.r.t. the same norm, is trivially ($2 L^2/\mu$)-IL.
\end{proposition}

A more refined analysis of the intrinsically Lipschitz parameter $\gamma$ for regularizers may take into account its dependence on the input dimension $d$, rather than relying solely on the direct analysis in \Cref{prop:lip_intr_lip}. Furthermore, we introduce the following localized version of the $\gamma$-IL condition, which is more general.

\begin{definition} [Formal version in \Cref{app:proofs_intr_lips}]\label{def:local_intrin_lips}
    Let $\psi: \mathcal{X} \rightarrow \mathbb{R}$ be an arbitrary convex regularizer for the simplex, we call $\psi$, $\gamma$-locally intrinsically Lipschitz ($\gamma$-LIL) if,
    \[
    \mleft| \psi(\vx') - \psi(\vx) \mright|^2 \leq \gamma (\epsilon) D_{\psi} (\vx'\; \| \; \vx),
    \]
    whenever $\vx, \vx'$ are $\epsilon$-close to each other in a specific way and $\gamma$ is a function of $\epsilon$.
\end{definition}

We note that \Cref{def:local_intrin_lips} is more relaxed than \Cref{def:intrin_lips}—any $\gamma$-IL function $\psi$ is also $\gamma$-LIL. For $\gamma$-LIL functions, the parameter $\gamma$ may additionally depend on the locality parameter $\epsilon$, as well as the dimension $d$. Indeed, the smaller the value of the locality parameter $\epsilon$, the lower the local Lipschitzness parameter $\gamma(\epsilon)$ is expected to be. Examples of $\gamma$-LIL and strongly convex regularizers are provided in \Cref{table:regularizers}. For a mathematical derivation of the $\gamma$ parameter for these regularizers, please refer to \Cref{app:find_gamma}.

\begin{table}[!t]
\renewcommand\arraystretch{1.6}
\rowcolors{2}{gray!25}{white}
\resizebox{\textwidth}{!}{\begin{tabular}{m{2.7cm}m{3cm}m{3cm}m{3cm}m{3.3cm}m{3mm}}
\textbf{Regularizer} & \bf Formulation $\psi$ & \bf Strong convexity $\mu$ & \bf (L)IL parameter $\gamma$ & \bf Regret & \bf \hfill\makebox[0mm][r]{Globally IL} \\
\toprule
Negative entropy    & $\sum_{\ind = 1}^d \vx[\ind] \log \vx[\ind]$ & $\Omega(1)$ & $O(\log^2 d)$ & $O(n 
\log^2 d \log T)$ & \checkmark \\ %

Log & $- \sum_{\ind = 1}^d \log \vx[\ind]$ & $\Omega(1)$ & $O(d)$ & $O(n d \log T)$ & \crossmark \\ %

Squared $\ell_2$ norm & $\frac{1}{2} \| \vx \|_2^2$ & $\Omega(d^{-1})$ & $O(1)$ & $O(n 
d \log T)$ & \checkmark \\ %

Squared $\ell_p$ norm \newline {{($p \in (1, 2]$)}} & {$\frac{1}{2} \| \vx \|_p^2$} & {$\Omega\mleft((p - 1)d^{2/p - 2}\mright)$} & {$O\mleft(\frac{1}{p - 1}\mright)$} & {$O\mleft( n \frac{d^{2 - 2/p}}{( p - 1)^2}  \log T \mright)$ } & {\checkmark} \\ %

Squared $\ell_{p^*}$ norm \newline {($p^* = 1 + 1/\log d$)}  & {$\frac{1}{2} \| \vx \|_{p^*}^2$} & {$\Omega\mleft((p^* - 1)d^{2/p^* - 2}\mright)$} & {$O\mleft(\frac{1}{p^* - 1}\mright)$} & {$O( n \log^2 d \log T )$ } & {\checkmark} \\ %

$1/2$-Tsallis entropy     & $2 \mleft( 1 - \sum_{\ind = 1}^d \sqrt{\vx[\ind]} \mright)$ & $\Omega(1)$  & $O(\sqrt{d}) $ & $O( n \sqrt{d} \log T )$ & \checkmark \\ %

$q$-Tsallis entropy  \newline {($q \in (0,1) $)}   & $\frac{1}{1 - q} \mleft( 1 - \sum_{\ind = 1}^d \sqrt{\vx[\ind]} \mright)$ & $\Omega(q)$  & $O\mleft(\frac{d^{1-q}}{(1 - q)^2}\mright) $ & $O\mleft( n \frac{d^{1-q}}{q (1 - q)^2} \log T \mright)$ & \checkmark \\ %

$q^*$-Tsallis entropy  \newline {($q^* = 1 - 1/\log d$)}   & $\frac{1}{1 - q^*} \mleft( 1 - \sum_{\ind = 1}^d \sqrt{\vx[\ind]} \mright)$ & $\Omega(q^*)$  & $O\mleft(\frac{d^{1-q^*}}{(1 - q^*)^2}\mright) $ & $O\mleft( n \log^2 d \log T  \mright)$ & \checkmark \\ %

$L$-Lipschitz and \newline $\mu$-strongly convex & general & $\mu$ & $O\mleft(\frac{L^2}{\mu}\mright)$ & $O\mleft( n \frac{L^2}{\mu^2} \log T\mright)$ & \checkmark \\
\bottomrule
\end{tabular}}
\vskip 0.1in
\caption{Various examples of appropriate regularizers $\psi$ for Cautious Optimism, which are (locally) intrinsically Lipschitz and strongly convex, along with the corresponding regret rates of \ours with each choice of $\psi$. \ours instantiated with negative entropy, the squared $\ell_{p^*}$ norm (for an appropriate choice of $p^* = 1 + 1/\log d$), or the $q^*$-Tsallis entropy (for an appropriate choice of $q^* = 1 - 1/\log d$) leads to new state-of-the-art no-regret algorithms in games. The parameter $\mu$ represents the strong convexity parameter with respect to the $\ell_1$ norm. The log regularizer is $O(d)$-LIL, whereas the other examples in this table are intrinsically Lipschitz.
}\label{table:regularizers}
\end{table}

\subsection{Properties of Intrinsically Lipschitz Regularizers} \label{sec:circuit}

For the purpose of Cautious Optimism, any $\gamma$-IL regularizer $\psi$ that is strongly convex with respect to the $\ell_1$ norm leads to fast-vanishing regret in learning dynamics. In this section, we demonstrate how this class of regularizers constructs a circuit that, through their combination, allows us to define new regularizers.

\begin{proposition}
\label{prop:ciruit}
Let $\psi_1$ and $\psi_2$ be two functions that are $\gamma_1$- and $\gamma_2$-intrinsically Lipschitz, then,
\begin{enumerate}
    \item $\psi(\vx) = a \psi_1 (\vx) + b$ is $a \gamma_1$-intrinsically Lipschitz for any choices of $a \in \mathbb{R}_+, b \in \mathbb{R}$.
    \item $\psi(\vx) = \psi_1 (\vx) + \psi_2(\vx) $ is $(\gamma_1 + \gamma_2)$-intrinsically Lipschitz.
\end{enumerate}
\end{proposition}

Given \Cref{prop:lip_intr_lip} and the properties of strong convexity with respect to the $\ell_1$ norm, we can immediately deduce that if $\psi_1$ and $\psi_2$ are $\mu_1$- and $\mu_2$-strongly convex functions, respectively, with intrinsic Lipschitz constants $\gamma_1$ and $\gamma_2$, then for any choice of $a_1, a_2 \in \mathbb{R}_+$ and $b \in \mathbb{R}$, the resulting regularizer  
\[
\psi(\vx) = a_1 \psi_1(\vx) + a_2 \psi_2(\vx) + b
\]
is a valid candidate for Cautious Optimism. This regularizer has an intrinsic Lipschitz parameter of $a_1 \gamma_1 + a_2 \gamma_2$ and a strong convexity parameter of $a_1 \mu_1 + a_2 \mu_2$ with respect to the $\ell_1$ norm.

From now on, unless stated otherwise, any reference to strong convexity is with respect to the $\ell_1$ norm.

\subsection{Learning Rate Control Problem}

Given a suitable regularizer $\psi$ as discussed in \Cref{sec:intr_lips,sec:circuit}, i.e., $\gamma$-IL (or $\gamma$-LL) and $\mu$-strongly convex w.r.t. $\ell_1$ norm, we generalize the notion of "Cautious Optimism", by referring to the general framework that picks action iterates $\vx\^t$ according to the \OFTRL update in \eqref{eq:ftrl_step_fixed_lambda} with a \emph{dynamic learning rate} $\lambda\^t$ carefully chosen as the solution to the following optimization problem,
\[
\lambda\^t \gets \argmax_{\lambda \in (0, \eta]} \left\{ \alpha \log \lambda + \psi^*_{\mathcal{X}} (\lambda \at\^t)   \right\}, \numberthis{eq:dynamic_learning_rate}
\]
where constant $\eta > 0$ is a constant that caps the maximum learning rate, $\alpha$ is a key parameter of the algorithm, and the convex conjugate
$\psi_{\mathcal{X}}^*(\lambda \at\^t)$ defined as
\[
\psi_{\mathcal{X}}^*(
\lambda \at\^t) \defeq \max_{\vx \in \mathcal{X}} \langle \lambda \at\^t, \vx \rangle - \psi(\vx).
\]
We refer to this problem~\eqref{eq:dynamic_learning_rate} as the \emph{learning rate control problem}. In our analysis, we will demonstrate that if $\psi$ is a $\gamma$-IL regularizer, it suffices to choose $\alpha \ge 4 \gamma + \mu$ in the learning rate control problem~\eqref{eq:dynamic_learning_rate}. We recall that, this is step corresponds to Line~\ref{line:oftrl0} of the \ours Algorithm is outlined in \Cref{algo:cftrl}.

Under normal operating conditions, when the maximum regret $\max_k \{\at\^t[k]\}$ accumulated over actions is not excessively negative, the optimal learning rate remains fixed at $\lambda\^t = \eta$, corresponding to a regime with a constant step size for \OFTRL. However, as the maximum regret decreases further into negative values, the optimal $\lambda\^t$ gradually shrinks toward $0$, effectively degrading the player's performance by reducing their reliance on past observations, history of the game and thereby learning. 

In the extreme case where $\lambda\^t \to 0$, the player’s decision-making is entirely governed by the regularizer, leading to actions drawn from the distribution $\argmin_{\vx \in \mathcal{X}} \psi(\vx)$. If the regularizer $\psi$ is symmetric with respect to permutations of the actions, for example, when $\psi$ is negative entropy, then the player acts uniformly over its actions. In \Cref{fig:learning_rate_regime}, we summarize the dynamics of the learning rate as a function of the maximum regret under symmetric regularizers $\psi$.

\begin{figure}[!t]
    \centering
    \scalebox{.87}{
    
    \begin{tikzpicture}
        \pgfdeclarehorizontalshading{gradientlambda}{100bp}{color(0bp)=(blue); color(100bp)=(red)}

        \shade[shading=gradientlambda] (0,0) rectangle (10,0.3);

        \draw[thick,->] (0,0) -- (10.5,0) node[anchor=west] {$\lambda$}; %
        \node[below] at (0,0) {$0$}; %
        \node[above] at (1.7,0.5) {$\max_{r \in [d]} \{ \at[r]\} = - \infty$};
        \node[above] at (9.5,0.5) {$\max_{r \in [d]} \{ \at[r]\} \geq 0$};
        \node[below] at (10,0) {$\eta$}; %
        \node[below] at (1.7,-0.1) {\textbf{Acting uniformly}};
        \node[below] at (8.5,-0.1) {\textbf{Learning fast}};
    \end{tikzpicture}
    }
    \caption{Learning-rate regime under symmetric regularizers. When $\max_k {\at\^t[k]}$ is not excessively negative, the optimal $\lambda\^t=\eta$ (constant step size in \OFTRL). As $\max_k {\at\^t[k]}$ becomes more negative, the optimal $\lambda\^t$ shrinks toward $0$, damping learning by down-weighting history. In the limit $\lambda\^t \to 0$, actions follow $\argmin_{\vx \in \mathcal{X}} \psi(\vx)$; for symmetric $\psi$ (e.g., negative entropy) this yields a uniform policy.}
    \label{fig:learning_rate_regime}
\end{figure}

\Cref{fig:learning_rate_map} illustrates how the learning rate $\lambda\^t$ varies as a function of optimistic cumulative regrets in a two-action scenario, considering different choices of the regularizer $\psi$. Notably, across all cases, $\lambda\^t$ remains a monotonically non-increasing function of $\{\at\^t[1]\}$ and $\{\at\^t[2]\}$, highlighting the structured dependence of the learning rate on regret dynamics.

\begin{figure}[!t]
    \centering
    \includegraphics[width=\textwidth]{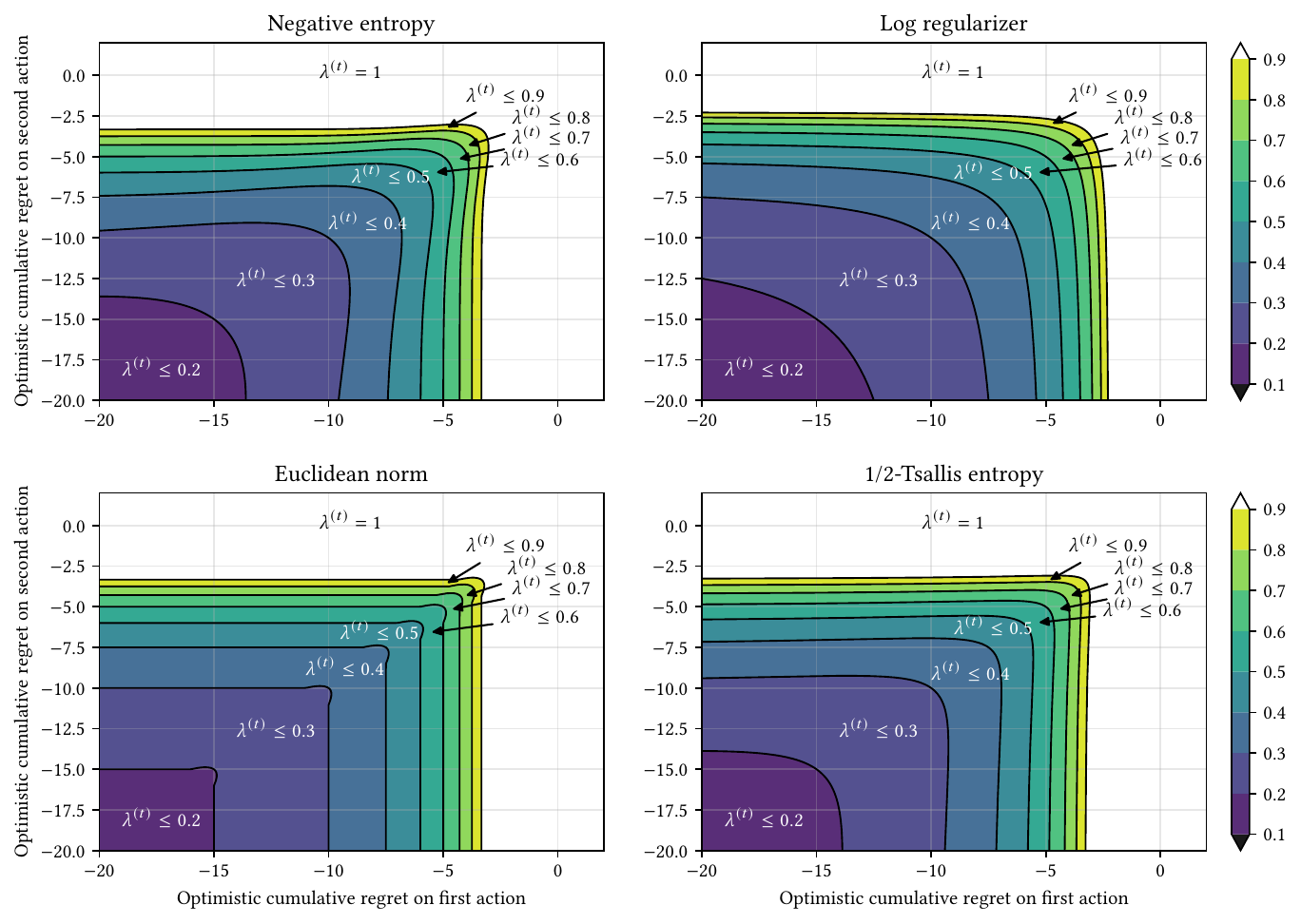}
    \caption{
    Learning rate control landscape for various choices of regularizers $\psi$: negative entropy, log regularizer, Euclidean norm, and Tsallis entropy. Visualization of how $\lambda\^t$, as defined in \eqref{eq:dynamic_learning_rate}, evolves in response to optimistic regrets in the 2-action simplex.  The plot is generated using parameter values $\eta = 1$ and $\alpha = 4$.
    }
    \label{fig:learning_rate_map} %
\end{figure}

\subsection{Properties of the Learning Rate Control Objective} \label{sec:prop_learning_rate_objective}

At first glance, it may not be immediately apparent that the maximization objective of the learning-rate control problem \eqref{eq:dynamic_learning_rate}  
\[
f(\lambda) \defeq \alpha \log \lambda + \psi_\mathcal{X}^*(\lambda \at),
\]
is tractable or even concave in $\lambda$. Interestingly, the convex conjugate term $\psi_\mathcal{X}^*(\lambda \at)$ is strongly convex—quite the opposite of what we hope for. However, as we will characterize later, by choosing a reasonably large hyperparameter $\alpha$, the $\alpha \log \lambda$ term ensures the concavity of the entire objective.

In formal terms, as we show in \Cref{thm:strong_convex_learning_rate_control}, this objective is $(\alpha - \gamma)$-strongly concave for any $\gamma$-IL regularizer $\psi$ and admits high curvature locally. When the regularizer $\psi$ is additionally Legendre, the learning-rate control objective exhibits high curvature globally. Specifically, the curvature of the objective $f$ is sandwiched between constant factors of the curvature of the $\log \lambda$ term; consequently, the $\alpha \log \lambda$ component essentially governs the curvature of $f$. For the later result, the local version of the intrinsic Lipschitzness condition suffices.

\begin{theorem}[Strong concavity of learning rate control problem] \label{thm:strong_convex_learning_rate_control}
The dynamic learning-rate control optimization problem \eqref{eq:dynamic_learning_rate},  
\[
f(\lambda) \defeq \alpha \log \lambda + \psi_\mathcal{X}^*(\lambda \at),
\]
\begin{itemize}
    \item is $(\alpha - \gamma)$-strongly concave, differentiable and admits high curvature locally,  
    \[
    f(\lambda') \leq f(\lambda) + f'(\lambda)(\lambda' - \lambda) 
    - \frac{\alpha - \gamma}{2 \max \{\lambda',\lambda\} ^2} (\lambda' - \lambda)^2,
    \]  
    whenever the regularizer $\psi$ is $\gamma$-intrinsic Lipschitz.

    \item is $(\alpha - \gamma)$-strongly concave, twice differentiable, and admits high curvature globally,  
    \[
    - \frac{\alpha}{\lambda^2} \leq f''(\lambda) \leq -\frac{\alpha - \gamma}{\lambda^2},
    \]
    whenever the regularizer $\psi$ is $\gamma$-(locally) intrinsic Lipschitz and additionally Legendre.
\end{itemize}
\end{theorem}

The detailed proof of this result is presented in \Cref{app:proof_convexity}, and we provide a proof outline here. The general idea is based on \Cref{prop:convex_phi}, which we will encounter later, and involves translating the strong concavity of the composite regularizer $\phi$ in \eqref{eq:def_phi} over the lifted space along the $\lambda$ ray into the strong concavity of the learning-rate control problem $f$. When the regularizer $\psi$ is additionally Legendre, we can establish a stronger form of concavity.

\Cref{thm:strong_convex_learning_rate_control} demonstrates that for a Legendre regularizer $\psi$, the $\log \lambda$ term controls the second-order derivative of the learning-rate control objective $f$. Given this observation, one may hope that the objective $f$ also inherits appealing higher-order properties of $\log \lambda$, in particular self-concordance. In \Cref{thm:log_dom_self_con}, we show that when the regularizer $\psi$ additionally satisfies a mild condition, which we term log-dominance, the learning-rate control objective $f$ is $M$-generalized self-concordant. We defer the formal definition of log-dominance and other mathematical details to \Cref{def:log_dominated,lemma:g_derivatives,lemma:v_formula} in \Cref{proof_selfconcordance}.

\begin{theorem}\label{thm:log_dom_self_con}
    When the $\gamma$-(L)IL regularizer $\psi$ is $\kappa$-log-dominated, then the dynamic learning rate control objective (\ref{eq:dynamic_learning_rate}),
    \[
    f(\lambda) = \alpha \log \lambda + \psi_\mathcal{X}^*( \lambda \at).
    \]
    is $M$-generalized self-concordant, i.e.,
    \[
    \mleft| f'''(\lambda) \mright|^2 \leq M \mleft| f''(\lambda) \mright|^3,
    \]
    where 
    \[
    M = \frac{\mleft(2 \alpha + \kappa\mright)^2}{\mleft( \alpha - \gamma \mright)^3}.
    \]
\end{theorem}

The Log-dominance condition is satisfied by well-known Legendre regularizers $\psi$: negative entropy, the log regularizer, and $q$-Tsallis entropy. Hence, the learning-rate control problem for these choices of regularizer is self-concordant. We defer the derivations to \Cref{app:find_kappa} and summarize the results in \Cref{table:regularizers_self_concordant}.

\begin{table}[!th]
\renewcommand\arraystretch{1.6}
\rowcolors{2}{gray!25}{white}
\resizebox{\textwidth}{!}{\begin{tabular}{m{2.7cm}m{3cm}m{3cm}m{3cm}m{3.3cm}}
\textbf{Regularizer} & \bf Formulation $\psi$ & \bf (L)IL parameter $\gamma$ & \bf Log-dominance $\kappa$ & \bf self-concordance $M$ \\
\toprule
Negative entropy    & $\sum_{\ind = 1}^d \vx[\ind] \log \vx[\ind]$ & $\Theta(\log^2 d)$ & $\Theta(\log^3 d)$ & $\Theta(1)$ \\ %

Log & $- \sum_{\ind = 1}^d \log \vx[\ind]$ & $\Theta(d) $ & $\Theta(d)$ & $\Theta(\frac{1}{d})$ \\ %

$q$-Tsallis entropy  \newline {($q \in (0,1) $)}   & $\frac{1}{1 - q} \mleft( 1 - \sum_{\ind = 1}^d \sqrt{\vx[\ind]} \mright)$ & $\Theta\mleft(\frac{d^{1-q}}{(1 - q)^2}\mright)$ & $\Theta \mleft(\frac{q (2 - q)}{(1 - q)^3} \mleft(d^{\frac{(3 - q)(1 - q)}{(2 - q)}} \mright) \mright)$ & $\Theta\mleft( q^2 (2 - q)^2 d^{\frac{q (1 - q)}{2 - q}}\mright)$ \\ %

$\gamma$-(L)IL and \newline $k$-log-dominated & general Legendre & $\gamma$ & $\kappa$ & $\dfrac{\mleft(2 \alpha + \kappa\mright)^2}{\mleft( \alpha - \gamma \mright)^3}$ \\
\bottomrule
\end{tabular}}
\vskip 0.1in
\caption{We provide different examples of appropriate Legendre regularizers $\psi$ for Cautious Optimism that are log-dominated and yield a self-concordant dynamic learning-rate control optimization problem~\eqref{eq:dynamic_learning_rate}. For details and derivations, see \Cref{app:find_kappa}. As a result, the per-step computational overhead of \ours with these choices of regularizers is $O(\log \log T)$ (see \Cref{app:approx}). We consider the choice $\alpha = \Theta(\gamma + \mu)$ in the formula  $
M = \frac{(2\alpha + \kappa)^2}{(\alpha - \gamma)^3}$,
as required in \Cref{theorem:nonnegative_rvu}.
}\label{table:regularizers_self_concordant}
\end{table}

\section{Design of \ours and Alternative Perspectives} \label{sec:design}

In this section, we discuss how to mathematically formalize the concept of dynamic learning-rate control leading to \eqref{eq:dynamic_learning_rate}, the design of \ours, and alternative perspectives on \ours that highlight its algorithmic aspects and analysis.

Recall that, as discussed in \Cref{sec:cautious_optimism}, we aim to design a no-regret learning algorithm that \emph{paces down} the learner when it is \emph{performing too well}. We achieve this goal and the formulation of \ours through the following step-by-step construction.

We begin with the standard dynamics of the Optimistic Follow-the-Regularized-Leader (\OFTRL) algorithm with a potentially dynamic learning rate $\lambda\^t$,  
\[
\vx\^t \leftarrow \argmax_{\vx \in \mathcal{X}} \left\{ \lambda\^t \langle \at\^t, \vx \rangle - \psi(\vx) \right\}, \numberthis{eq:FTRL_naive}
\]
where $\lambda\^t \in (0, \eta]$ is chosen according to a separate dynamic that we will design later, and $\psi$ is a regularizer over the space $\mathcal{X}$. When the negative entropy is chosen as the regularizer, $\psi(\vx) = \sum_{\ind = 1}^d \vx[\ind] \log \vx[\ind]$, Formulation~\ref{eq:FTRL_naive} recovers the celebrated Optimistic Multiplicative Weights Update (\Opthedge) algorithm with learning rate $\lambda\^t$.

Not only do we intend to have diminishing regret dynamics by incentivizing actions with highest returns, but we also aim to hinder the learner when it is performing too well. We incorporate this phenomenon in the update step by integrating the dynamic learning rate $\lambda\^t$ mechanism into Formulation~\ref{eq:FTRL_naive}, leading to  
\[
\vstack{\lambda\^t}{\vx\^t} \gets \argmax_{\lambda \in (0, \eta], \vx \in \mathcal{X}} \left\{ \lambda \langle \at\^t, \vx \rangle - \psi(\vx) \right\}. \numberthis{eq:FTRL_naive1}
\]
Unfortunately, this naive integration leads to a nonsmooth dynamic for the learning rate $\lambda\^t$ over time $t$, since $\lambda\^t$ exhibits behavior akin to that of a step function. When $\langle \at\^t, \vx\^t \rangle > 0$, this formulation naively reduces to $\lambda\^t = \eta$, corresponding to the original \OFTRL with a constant learning rate; otherwise, it trivially sets $\lambda\^t = 0$ and $\vx\^t = \argmin_{\vx \in \mathcal{X}} \psi(\vx)$.

To increase the predictability of the dynamics and smooth the evolution of $\lambda\^t$ over time, we add the necessary inertia by incorporating the log regularizer over the learning rate $\lambda\^t$ into the dynamics. This modification leads to the following update rule:  
\[
\vstack{\lambda\^t}{\vx\^t} \gets \argmax_{\lambda \in (0, \eta], \vx \in \mathcal{X}} \left\{ \lambda \langle \at\^t, \vx \rangle + \alpha \log \lambda - \psi(\vx) \right\}, \numberthis{eq:OFTRL}
\]
where the hyperparameter $\alpha$, that depends on the properties of the regularizer $\psi$, controls the smooth thresholding procedure of the learning rate and, as a result, determines what it means to “\emph{perform too well}” relative to the regret vector $\at\^t$.

The formulation \eqref{eq:OFTRL} is not jointly concave in $(\lambda, \vx)$, making its \emph{computational} aspects and \emph{regret analysis} unclear in this naive formulation.  From a computational perspective, however, it is immediate to observe that this formulation is equivalent to \ours in \Cref{algo:cftrl}, for which we demonstrate very low computational overhead beyond \OFTRL update rule.

Moreover, a key technical observation enables us to analyze the regret of \ours. By performing the change of variable $\vy = \lambda \vx$, which is invertible on the simplex as $\vx = \vy / (\vec 1^\top \vy)$, we derive the following equivalent update
\[
\vy\^t  \leftarrow  \argmax_{\vy \in (0, 1] \mathcal{X}} \mleft\{ \eta \langle \at\^t, \vy \rangle + \alpha \log(\vec 1^\top \vy) - \psi\mleft(\frac{\vy}{{\vec 1}^\top \vy} \mright) \mright\}. \numberthis{eq:lifted_FTRL}
\]
This formulation corresponds to \OFTRL on the space $(0,1] \mathcal{X}$ and is a key observation central to the analysis of \ours. We note that, at first glance, it is not even clear whether this \OFTRL update step~\eqref{eq:lifted_FTRL} is concave. As we will show later in \Cref{sec:stepII}, not only is this objective concave, but it also exhibits a special form of strong concavity along the $\lambda$-ray  that leads to multiplicative stability in the subsequent dynamic learning rates.

A similar idea to \OFTRL on the lifted space $(0, 1] \mathcal{X}$ has previously been studied~\citep{farina2022near}, albeit only for a very limited choice of regularizer. \citet{farina2022near} utilize the log regularizer to obtain self-concordance of the \OFTRL dynamics and enforce elementwise multiplicative stability in actions by leveraging its high curvature and the structure of the induced intrinsic norms. In contrast, in this work, we introduce a \emph{unified theory of Cautious Optimism} for a broad class of regularizers, identifying their algorithm (\LRLOFTRL) as a \emph{special case}.

\section{Analysis of Cautious Optimism} \label{sec:analysis}

We primarily focus on Formulation~\ref{eq:lifted_FTRL} for the analysis of Cautious Optimism. As illustrated in the previous section, the dynamics are equivalent to the following \OFTRL update:
\[
\vy\^t  \leftarrow \argmax_{\vy \in (0, 1] \mathcal{X}} \mleft\{ \eta \langle \at\^t, \vy \rangle - \phi(\vy) \mright\}, \numberthis{eq:lifted_FTRL_y}
\]
where the regularizer $\phi: (0, 1] \mathcal{X} \rightarrow \mathbb{R}$ is defined as
\[
\phi(\vy) \defeq - \alpha \log (\vec{1}^\top \vy) + \psi \mleft(\frac{\vy}{{\vec 1}^\top \vy} \mright), \numberthis{eq:def_phi}
\]
with $\psi$ being $\gamma$-(locally) intrinsically Lipschitz.  

At first glance, it is not immediately clear whether the function $\phi$ is convex. However, as we will demonstrate later in this section, $\phi$ is indeed convex in $\vy$ provided that $\alpha \geq \gamma$. Consequently, $D_\phi (\cdot \; \| \; \cdot)$ is well-defined.

We split the analysis of the regret cumulated by \ours into four technical steps, discussed in the next subsections.

\subsection{Step I: A Key Identity for $D_\phi (\cdot \; \| \; \cdot)$} \label{sec:stepI}

A key step in analyzing the dynamics of \OFTRL in Formulation~\ref{eq:lifted_FTRL_y} is the following identity, which relates the Bregman divergence of the regularizer $\phi(\vy)$ over $\vy \in (0, 1] \mathcal{X}$ to the Bregman divergences of $-\log \lambda$ over the learning rate $\lambda \in (0, 1]$ and $\psi(\vx)$ over $\vx \in \mathcal{X}$.

\begin{lemma}
\label{lemma:breg_identity}
For arbitrary values of $\vy, \vy^\prime \in (0, 1]\mathcal{X}$, let $\lambda \defeq \vec 1^\top \vy, \lambda^\prime \defeq \vec 1^\top \vy^\prime$. Then,
\[
  D_\phi (\vy^\prime \; \| \; \vy) & = \alpha D_{- \log} (\lambda^\prime \; \| \; \lambda) + \frac{\lambda^\prime}{\lambda} D_\psi (\vx^\prime, \vx) + \mleft(1 - \frac{\lambda^\prime}{\lambda} \mright) [\psi(\vx^\prime) - \psi(\vx)].
\]
\end{lemma}

This identity is important as it illustrates how the curvature of the regularizer $\phi$ in $\vy$ and, consequently, the stability of $\vy$ in the corresponding \OFTRL Step~\ref{eq:lifted_FTRL}, are affected by the parameter $\alpha$ and the choice of regularizer $\psi$. We will use this identity repeatedly in the analysis.

\subsection{Step II: Sensitivity of Learning Rate to Cumulated Regret} \label{sec:stepII}

In this section, we demonstrate that the learning rate $\lambda$ in Formulation~\ref{eq:lifted_FTRL_y} remains multiplicatively stable under small additive perturbations to the regret vector $\at$.  

Furthermore, for a $\gamma$-IL choice of regularizer $\psi$, we establish that the function $\phi(\vy)$ is convex with respect to $\vy$. We begin with the following proposition, which provides a lower bound on $D_\phi(\vy' \; \| \; \vy)$ in terms of $\lambda'/\lambda$.

\begin{proposition} \label{prop:convex_phi}
    For arbitrary choices of $\vy',\vy \in (0, 1]\mathcal{X}$, let $\lambda \coloneqq \vec 1^\top \vy, \lambda' \coloneqq \vec 1^\top \vy'$. Then, we infer that,
    \[
    D_\phi(\vy' \; \| \; \vy) + D_\phi(\vy \; \| \; \vy') \geq (\alpha - \gamma) \mleft[ \frac{\lambda'}{\lambda} + \frac{\lambda}{\lambda'} - 2 \mright],
    \]
    where $\psi$ is $\gamma$-intrinsically Lipschitz.
\end{proposition}

An immediate byproduct of this proposition is that the function $\phi$ is convex for a proper choice of $\alpha$, provided that the function $\psi$ is $\gamma$-IL.

\begin{corollary} \label{corr:convex_phi}
    If $\psi$ is $\gamma$-intrinsically Lipschitz and $\alpha \geq \gamma$,
    the function $\phi(\vy)$ is convex in $\vy \in (0,1]\mathcal{X}$.
\end{corollary}

For the weaker notion of $\gamma$-local intrinsic Lipschitzness, we can prove a similar result, but only for $\vy',\vy \in (0, 1]\mathcal{X}$ satisfying $D_{\phi} (\vy^\prime \; \| \; \vy) + D_{\phi} (\vy \; \| \; \vy^\prime) \leq \epsilon$. This result is formalized in \Cref{prop:local_convex_phi}. In contrast to the global assumption, $\gamma$-local intrinsic Lipschitzness does not directly imply that the function $\psi$ is convex; rather, it suggests that $\nabla \psi$ is a locally monotone operator.

With all the necessary ingredients in place, we are now ready to show that the regularizer, and consequently Cautious Optimism, induces proximal steps for learning rates that remain multiplicatively stable.

\begin{theorem}[Stability of learning rates]\label{theorem:mult_stable}
    Given the regularizer $\phi$ with $\gamma$-IL component $\psi$ and choice of $\alpha \geq 4 \gamma$, consider regret vectors $\at',\at$ such that $\| \at' - \at\|_\infty \leq 4$ the corresponding proximal steps of \OFTRL,
    \[
    \vy  \leftarrow \argmax_{\vy \in (0, 1] \mathcal{X}} \mleft\{ \eta \langle \at, \vy \rangle - \phi(\vy) \mright\},
    \] 
    and 
    \[
    \vy'  \leftarrow \argmax_{\vy \in (0, 1] \mathcal{X}} \mleft\{ \eta \langle \at', \vy \rangle - \phi(\vy) \mright\}.
    \]
    Additionally, let $\lambda' = \vy'/(\vec 1^\top \vy'), \lambda = \vy/(\vec 1^\top \vy)$. Then, for a small enough learning rate $\eta \leq 3 \gamma/80$, dynamic learning rates are multiplicatively stable,
    \[
        \frac{\lambda'}{\lambda} \in \mleft[\frac{1}{2}, \frac{3}{2}\mright].
    \]
\end{theorem}

Similar results to \Cref{theorem:mult_stable} also hold for $\gamma$-locally Lipschitz functions, with the additional caveat that the learning rate must satisfy $\eta \leq \min\{3 \gamma(1)/80, 1/8\}$. This result is formalized in \Cref{theorem:mult_stable_local_lipsc} and deferred to \Cref{app:proof_stepII} in the interest of space. The key idea is that the proximal nature of optimization problems ensures the locality required for \eqref{eq:local_int} in \Cref{def:local_intrin_lips}.

In this section, we showed that the learning rates, and consequently the actions chosen by \ours, evolve smoothly throughout the course of the game. Proofs for this section are provided in \Cref{app:proof_stepII}.

\subsection{Step III: High Curvature of Proximal Steps under Stability of Learning Rates $\lambda\^t$}  \label{sec:stepIII}

In this section, we demonstrate that under the stability of learning rates—which, as shown in \Cref{sec:stepII}, holds for successive steps—$D_\phi$ induces curvature on the space $(0, 1] \mathcal{X}$ that is at least as strong as the curvature induced by $D_{-\log}$ over the learning rates and $D_{\psi}$ over the simplex $\mathcal{X}$. This step is crucial for analyzing the dynamics of \OFTRL in Formulation~\ref{eq:lifted_FTRL_y}.

\begin{proposition}\label{prop:curvature}
    For any $\vy',\vy \in (0,1]\mathcal{X}$, denote $\vx' = \frac{\vy'}{\vec 1^\top \vy'}, \vx = \frac{\vy}{\vec 1^\top \vy}$, $\lambda' = \vec 1^\top \vy'$, and $\lambda = \vec 1^\top \vy$.
    If $\alpha \geq 4 \gamma + a$ for some positive number $a \in \mathbb{R}_+$,
    for all $\vy', \vy \in (0, 1]\mathcal{X}$ such that
    \[
    \frac{\lambda'}{\lambda} = \frac{\vec 1^\top \vy'}{\vec 1^\top \vy} \in \mleft[ \frac{1}{2}, \frac{3}{2} \mright].
    \]
    
    we have that,
    \[
    D_\phi (\vy' \;\|\; \vy) \geq (\gamma + a) D_{- \log} (\lambda' \; \| \; \lambda) + \frac{1}{4} D_\phi(\vx' \; \| \; \vx)
    \]
\end{proposition}

The preceding proposition, combined with \Cref{theorem:mult_stable}, guarantees that the high curvature of $D_\phi$ holds when considering subsequent actions in the \OFTRL formulation of \ref{eq:lifted_FTRL_y}. %

\subsection{Step IV: Regret Analysis} \label{sec:regret}

Having developed the key steps in \Cref{sec:stepI,sec:stepII,sec:stepIII}, we are now ready to state our main result and discuss the details of the regret analysis for \ours.
We begin our regret analysis by establishing the \emph{nonnegative RVU} property for \ours\footnote{A detailed discussion on the history and development of RVU is given in \Cref{sec:background}.}. This property is inspired by \citet{farina2022near}, who first introduced the concept. Compared to the standard RVU bounds~\citep{Syrgkanis15} discussed in \Cref{sec:background}, nonnegative RVU bounds provide a stronger framework for analyzing regret, as they not only directly imply the standard version but also ensure the nonnegativity of the right-hand side of the RVU bound~\eqref{eq:positive_rvu_rhs}.

\begin{theorem}[Nonnegative RVU bound of \ours] \label{theorem:nonnegative_rvu}
    Under the bounded utilities assumption~\ref{assumption:bounded}, by choosing a sufficiently small learning rate $\eta \leq \min\{ 3\gamma/80, \mu/(32\sqrt{2}) \}$ and $\alpha \geq 4 \gamma + \mu$, the cumulative regret $\reg\^T$ incurred by \ours up to time $T$ is bounded as follows:
    \[
        \max\{\reg\^T, 0\} & \leq 3 + \frac{1}{\eta} \mleft( \alpha \log T + \mathcal{R} \mright)  + \eta \frac{48}{\mu} \sum_{t = 1}^{T-1} \|\nut\^{t+1} - \nut\^{t}\|_{\infty}^2 - \frac{1}{\eta} \frac{\mu}{64} \sum_{t = 1}^{T-1} \| {\vx}\^{t+1} -  {\vx}\^{t} \|_1^2 \numberthis{eq:positive_rvu_rhs},
    \]
    where $\mathcal{R} \defeq \max_{\vx \in \mathcal{X}} \psi(\vx)$.
\end{theorem}

A similar result can be established for $\gamma$-LIL regularizers, which we defer the formalism version of \Cref{thm:main_ILI} in \Cref{app:rvu}.

As mentioned at the beginning of \Cref{sec:analysis}, we analyze the regret of \ours by studying  
\[
\tildereg\^T \defeq \max_{\vy^* \in [0,1]\mathcal{X}} \sum_{t = 1}^{T} \langle  {\ut}\^{t}, \vy^* - \vy\^{t} \rangle
\]
for the \OFTRL formulation on $(0, 1]\mathcal{X}$ in \eqref{eq:lifted_FTRL_y}. In \Cref{prop:reg+} of \Cref{app:nonneg}, we show that $\tildereg\^T = \max\{0, \reg\^T\}$.  However, as seen in \eqref{eq:positive_rvu_rhs}, the RVU property is formulated as if it were proven for $\reg\^T$ rather than $\tildereg\^T$. We bridge this gap through the steps (I), (II), and (III) outlined in \Cref{sec:stepI,sec:stepII,sec:stepIII}. The key idea is that, in subsequent iterations of the algorithm, the learning rates remain multiplicatively close, as shown in \Cref{theorem:mult_stable}. Given this closeness, \Cref{prop:curvature} allows us to directly translate the Bregman divergence $D_\phi$ defined on $\vy, \vy' \in (0,1] \mathcal{X}$ into $D_\psi$ on $\vx, \vx' \in \mathcal{X}$ and $D_{-\log}$ over the learning rates.  We postpone the formal reasoning to \Cref{app:detailed}.

In our analysis, the use of positive regret, $\max\{0, \reg\^T\}$, plays a key role. Due to its nonnegativity, any upper bound on the total regret, $ \sum_{i=1}^{n} \max\{0, \reg\^T_i\}$, also serves as an upper bound on the maximum regret, $\max_{i \in [n]} \reg\^T_i$, among the $n$ players.  This result guarantees that the empirical distribution of joint strategies converges rapidly to an approximate CCE of the game, with the convergence rate determined by $\max_{i \in [n]} \reg\^T_i$. We formalize this intuition in the following proposition, which leverages the nonnegative RVU bound from \Cref{theorem:nonnegative_rvu} to establish an upper bound on the total path length of the play, $
\sum_{i = 1}^{n} \sum_{t = 1}^{T - 1} \|\vx\^{t+1}_i - \vx\^{t}_i \|_1^2$,
by summing $\max\{0, \reg\^T_i\}$ terms.

\begin{proposition}
[Bound on total path length]\label{prop:path_length}
    Under \Cref{assumption:bounded}, if all players adhere to the \ours algorithm with  a learning rate $\eta \leq \min\{ 3\gamma/80, \mu/(32\sqrt{2}), \mu/ (L n 32 \sqrt{6}) \}$, the total path length is bounded as follows,
    \[
    \sum_{i = 1}^{n} \sum_{t = 1}^{T - 1} \|\vx\^{t+1}_i - \vx\^{t}_i \|_1^2 \leq \frac{128 \eta}{\mu} \mleft( 3 n + n \frac{\alpha \log T + \mathcal{R}}{\eta} \mright).
    \]
\end{proposition}

With the bound on total path length from \Cref{prop:path_length} at hand, we follow the standard framework of RVU bounds~\citep{Syrgkanis15} to establish an upper bound on the individual regret for each player.  The key idea is that, under \Cref{assumption:bounded}, the variations in utilities over time, $\sum_{t=1}^{T-1} \| \nut\^{t} - \nut\^{t-1} \|_{\infty}^2$, can be controlled using the bound on the total path length, ultimately leading to the desired result. This analysis culminates in the following theorem, which serves as the main result of this work.

\begin{theorem}[Regret bounds for \ours] \label{theorem:regret_final_bound}
Under the bounded utilities assumption~\ref{assumption:bounded}, if all players $i \in [n]$ follow \ours with a $\gamma$-IL and $\mu$-strongly convex regularizer $\psi$, and choosing a small enough learning rate $\eta \leq \min\{ 3\gamma/80, \mu/(32\sqrt{2}), \mu/ (L n 32 \sqrt{6}) \}$, then the regret for each player $i \in [n]$ is bounded as follows:
\[
\reg_i\^T = O ( n \Gamma_\psi(d) \log T).
\]
where $\Gamma_\psi(d) = \gamma/\mu$ and the algorithm for each player $i \in [n]$ is adaptive to adversarial utilities, i.e., the regret that each player incurs is $\reg_i\^T = O(\sqrt{T \log d})$. 
\end{theorem}

We know that the $\gamma$-LIL condition is a relaxation of $\gamma$-IL, and this flexibility comes at a cost. The regret guarantees of \ours with a $\gamma$-LIL regularizer are of the order $O(n \gamma(1) \log T)$ instead of $O(n \gamma/\mu \log T)$. We defer the formal details of this case to \Cref{thm:main_ILI} in \Cref{app:rvu}.

The regret guarantees in \Cref{theorem:regret_final_bound} hold for all $T \in \mathbb{N}^+$ \emph{simultaneously}, as the \ours algorithm and the learning rate $\eta$ are independent of the time horizon $T$. This ensures that \ours maintains low regret even when the horizon is unknown or players fail to coordinate, without requiring techniques such as the doubling trick.

We conclude this section with the following corollary, which demonstrates the convergence of \ours to the CCE of the game. This result follows from the well-known theorem that relates the rate of convergence to the CCE with the maximum regret among the players, $\max_{i \in [n]} \reg\^T_i$.

\begin{corollary} \label{corollary:convergence_to_CCE}
    If $n$ players follow the uncoupled learning dynamics of \ours for $T$ rounds in a general-sum multiplayer game with a finite set of $d$ deterministic strategies per player, the resulting empirical distribution of play constitutes an $O\left( \frac{n \Gamma(d) \log T}{T}\right)$-approximate coarse correlated equilibrium (CCE) of the game.
\end{corollary}

As established in \Cref{theorem:regret_final_bound}, \ours, when equipped with a $\gamma$-IL and $\mu$-strongly convex regularizer $\psi$, achieves a regret bound of $O(n \Gamma(d) \log T)$\footnote{For $\gamma$-LIL regularizers, this bound simplifies to $O(n \gamma(1) \log T)$, as stated in \Cref{thm:main_ILI}.}. We observe that the term $\Gamma(d)$, and consequently the regret guarantees, remain invariant under scaling of the regularizers. The intrinsically Lipschitz assumption is not restrictive and is satisfied by a wide range of regularizers. In \Cref{table:regularizers}, we present the resulting regret rates of \ours with different regularizers.

\section{Fast Convergence of Social Regret} \label{sec:social_regret}

In \Cref{sec:analysis}, we showed that \ours has accelerated individual regret guarantees, i.e., $\reg_i\^T = O_{\scaleto{T}{4pt}}(\log T)$ for each player $i \in [n]$. In this short section, we demonstrate that \ours additionally enjoys the constant social welfare rate $\sum_{i = 1}^n \reg_i\^T = O_{\scaleto{T}{4pt}}(1)$, matching its optimistic counterparts (\OFTRL, and \OOMD) studied by \citet{Syrgkanis15}. For this result, it suffices to show that players employing \ours incur regrets satisfying the RVU property. We present this statement below in \Cref{theorem:rvu_simple} and provide a concise proof in \Cref{app:social_regret}. Compared to the nonnegative RVU bound in \Cref{theorem:nonnegative_rvu}, the following result exhibits two principal distinctions.

First, it provides an upper bound on the regret rather than on the nonnegative regret; consequently, the left-hand side of \eqref{eq:rvu_social_main_text} may be negative. Second, the right-hand side of \eqref{eq:rvu_social_main_text} does not depend on the time horizon $T$. The key insight in the analysis is that, unlike in \Cref{theorem:nonnegative_rvu}, the absence of a nonnegativity requirement permits us to restrict the comparator set to the simplex $\Delta^d$ instead of $(0,1]\Delta^d$.

\begin{theorem} 
(RVU bound of \ours) \label{theorem:rvu_simple}
    Under \Cref{assumption:bounded}, by choosing a sufficiently small learning rate $\eta \leq \min\{ 3\gamma/80, \mu/(32\sqrt{2}) \}$ and $\alpha \geq 4 \gamma + \mu$, the cumulative regret ($\reg\^T$) incurred by \ours up to time $T$ is bounded as follows:
    \[
    \reg\^T \leq \frac{1}{\eta} \mathcal{R} + \eta \frac{48}{\mu} \sum_{t = 1}^{T-1} \|\nut\^{t+1} - \nut\^{t}\|_{\infty}^2  - \frac{1}{\eta} \frac{\mu}{64} \sum_{t = 1}^{T-1} \| {\vx}\^{t+1} -  {\vx}\^{t} \|_1^2, \numberthis{eq:rvu_social_main_text}
    \]
    where $\mathcal{R} = \argmax_{\vx \in \mathcal{X}} \psi(\vx)$.
\end{theorem}

Now, by \Cref{theorem:rvu_simple}, Theorem 4 of \citet{Syrgkanis15}, which connects RVU bounds to social regret, and a sufficiently small learning rate cap $\eta$, we infer $\sum_{i = 1}^n \reg_i\^T = O_{\scaleto{T}{4pt}}(1)$. This result is formalized below, and its proof is deferred to \Cref{app:social_regret}.

\begin{theorem} \label{theorem:constant_social}
    Under the bounded utilities assumption~\ref{assumption:bounded}, if all players $i \in [n]$ follow \ours with a $\gamma$-IL and $\mu$-strongly convex regularizer $\psi$, and choosing a small enough learning rate $\eta \leq \min\{ 3\gamma/80, \mu/(32\sqrt{2}), \mu/ (L n 32 \sqrt{6}) \}$, then the social regret is constant in time horizon $T$, i.e.,
    \[
    \sum_{i = 1}^n \reg_i\^T \leq O\mleft( L \frac{n^2 \mathcal{R}}{\mu} + \frac{n \mathcal{R}}{\gamma} \mright) = O_{\scaleto{T}{4pt}}(1),
    \]
    where $\mathcal{R} = \argmax_{\vx \in \mathcal{X}} \psi(\vx)$.
\end{theorem}

In closing this section, we note that the choice of the learning rate cap $\eta$ for the individual regrets in \Cref{theorem:regret_final_bound} and for the social regret in \Cref{theorem:constant_social} is the same; hence, \OFTRL enjoys both results simultaneously. This contrasts with the optimistic counterparts, where the studied learning rate regimes for $O_{\scaleto{T}{4pt}}(T^{1/4})$ individual regret and $O_{\scaleto{T}{4pt}}(1)$ social regret differ, see e.g., \citet[Corollary 8 and Corollary 12]{Syrgkanis15}.

\section{Extension of the Analysis for \ours under Approximate Iterates}\label{sec:approx}

While exact computation of the dynamic learning rate $\lambda\^t$,  
\[
\lambda\^t = \argmax_{\lambda \in (0, \eta]} \mleft\{ f_t(\lambda) \defeq \alpha \log \lambda + \psi_{\mathcal{X}}^* (\lambda \at\^t)\mright\}, \numberthis{eq:opt_lamba_dyna}
\]
for each iteration $t$ is ideal, in practice we can often only compute an approximate solution to \eqref{eq:opt_lamba_dyna} within some tolerance $\epsilon\^t$. For this reason, in this section we study the regret analysis of \ours when, at each iteration $t$, we compute a multiplicative $\epsilon$-approximation $\widehat{\lambda}\^t$ satisfying  
\[
\widehat{\lambda}\^t \in \mleft[ (1 - \epsilon) \lambda\^t, (1 + \epsilon) \lambda\^t \mright], \qquad \text{with } \epsilon \in \mleft(0, \frac{1}{2}\mright],
\]
using algorithms that we will discuss in \Cref{sec:iteration}. Consequently, at each iteration $t$, instead of the ideal canonical action $\vx\^t$, we compute and play $\widehat{\vx}\^t$, the action induced by the approximate learning rate $\widehat{\lambda}\^t$, i.e.,  
\[
\widehat{\vx}\^t = \argmax_{\vx \in \mathcal{X}} \mleft\{ \langle \widehat{\lambda}\^t \at\^t, \vx \rangle - \psi(\vx) \mright\},
\]
where the reward signal is defined as $\ut\^t = \nut\^t - \langle \nut\^t, \widehat{\vx}\^t \rangle \vec1_d$, and the regret vector as $\at\^t = \sum_{\tau=1}^{t-1} \ut\^t + \ut\^{t-1}$. Here, the utility vector $\nut\^t$ is received by playing the approximate action $\widehat{\vx}\^t$.

We show that computing the approximate learning rate $\widehat{\lambda}\^t$ up to a multiplicative tolerance of $\epsilon = 1/T$ is sufficient for \ours to enjoy the same guarantees as the ideal version with exact computation of learning rates. A key part of the analysis for the approximate iterates is the definition of the per-round suboptimality gap in the lifted \OFTRL objective,  
\[
\delta_t  \defeq \frac{1}{\eta} \mleft( f_t(\lambda\^t) - f_t(\widehat{\lambda}\^t) \mright),
\]
and the fact that this suboptimality gap is controlled by a quadratic function of the multiplicative tolerance $\epsilon$, as we formalize below.

\begin{lemma}\label{lemma:subopt_value}
    For every time $t \in [T]$, and for every approximate dynamic learning rate $\widehat{\lambda}\^t > 0$ that is multiplicatively $\epsilon$-close ($\epsilon \in (0, 1/2]$) to the exact dynamic learning rate $\lambda\^t \in (0, \eta]$, i.e.,  
    \[
    \frac{\widehat{\lambda}\^t}{\lambda\^t} \in [1 - \epsilon, 1 + \epsilon],
    \]
    we have that per-round suboptimality is bounded as
    \[
    \delta_t \leq \frac{\alpha}{\eta} \epsilon^2.
    \]

\end{lemma}

We now state the main result of this section and defer full proof details to \Cref{app:approx}. The main idea is that \ours with approximate iterates enjoys nonnegative RVU bounds, analogous to \Cref{theorem:nonnegative_rvu}, but with additional terms bounded by a constant factor of the sum of suboptimalities $\sum_{t = 1}^T \delta_t$.

\begin{theorem}[Extension of \Cref{theorem:regret_final_bound} under approximate dynamic learning rates]\label{thm:main_approx}
    If all players $i \in [n]$ follow \ours with a multiplicatively $\epsilon$-approximated dynamic learning rate $\lambda\^t$ at each iteration $t$,  
    \[
    \widehat{\lambda}\^t \in \mleft[ (1 - \epsilon) \lambda\^t, (1 + \epsilon) \lambda\^t \mright], \qquad \text{with } \epsilon = \frac{1}{T},
    \]
    and a small enough learning-rate cap $\eta = O_{\scaleto{T}{4pt}}(1)$, then the regret for each player $i \in [n]$ is bounded by  
    \[
    \reg_i\^T = O \big( n \Gamma_\psi(d) \log T \big), 
    \]
    where $\Gamma_\psi(d) = \gamma/\mu$. Moreover, the algorithm for each player $i \in [n]$ is adaptive to adversarial utilities; that is, the regret that each player incurs satisfies $\reg_i\^T = O(\sqrt{T \log d})$.
\end{theorem}

In the next section, we design algorithms to efficiently compute $\widehat{\lambda}\^t$ up to a multiplicative tolerance $\epsilon$, thereby showing that the computational overhead of \ours compared to \OFTRL is minimal.

\section{Iteration Complexity of \ours}\label{sec:iteration}

As we showed in \Cref{thm:strong_convex_learning_rate_control}, the dynamic learning-rate control problem is strongly concave, and first-order convex optimization methods, e.g., gradient descent, can be employed to solve the one-dimensional optimization problem \eqref{eq:dynamic_learning_rate}. Nonetheless, these algorithms provide fast convergence only with additive errors, whereas, given the analysis of \ours in \Cref{sec:approx}, we need to solve \eqref{eq:dynamic_learning_rate} to find an approximate learning rate $\widehat{\lambda}\^t$ that is multiplicatively $\epsilon$-close to the exact solution $\lambda\^t$. In particular, we seek $\widehat{\lambda}$ such that  
\[
\left| \frac{\widehat{\lambda}}{\lambda^*} - 1 \right| \leq \epsilon,
\]
where $\lambda^* = \argmax_{\lambda \in (0, \eta]} \mleft\{ f(\lambda) \defeq \alpha \log \lambda + \psi_{\mathcal{X}}^* (\lambda \at) \mright\}$ is the unique maximizer of the strongly concave learning-rate control problem.

Since this one-dimensional optimization problem is unimodal by concavity, in \Cref{app:iteration_comp}, for a \emph{general intrinsic Lipschitz} regularizer $\psi$, we design geometric variants of the golden-section (\Cref{algo:geometric_golden}) and bisection (\Cref{algo:geometric_bisection}) search algorithms, which rely only on zeroth-order and first-order oracles, respectively, to find an approximate dynamic learning rate $\widehat{\lambda}$ with multiplicative error $\epsilon$ in $O(\log (1/\epsilon))$ time. 

\begin{table}[!t]
\rowcolors{2}{gray!25}{white}
\resizebox{\textwidth}{!}{\begin{tabular}{m{3.5cm}m{6cm} m{3cm}m{3.5cm}m{3.5cm}}
\textbf{Regularizer $\psi$} & \bf Algorithm & \bf Oracle ($f, f', f''$)  & \bf Iteration Complexity\\
\toprule
$\gamma$-Intrinsic Lipschitz (IL) & Geometric Golden-section Search \newline (\Cref{algo:geometric_golden}) & 0\text{th}-order ($f$) & $O(\log T)$ \\ %
$\gamma$-Intrinsic Lipschitz (IL) & Geometric Bisection Search \newline (\Cref{algo:geometric_bisection}) & 1\text{st}-order ($f'$) &  $O(\log T)$ \\
$\gamma$-(Local) IL, \newline  Legendre \& \newline  Log-dominated  & Newton's Method & 1st- and 2nd-order
 ($f', f''$) &  $O(\log \log T)$ \\
\bottomrule
\end{tabular}}
\vskip 0.1in
\caption{This table summarizes the iteration complexity of solving the learning-rate control problem up to multiplicative error $\epsilon = 1/T$, which we show is sufficient for our guarantees in \Cref{thm:main_approx}. If the regularizer $\psi$ is $\gamma$-(locally) intrinsic Lipschitz and log-dominated, then by \Cref{thm:log_dom_self_con}, the learning-rate control problem is self-concordant, and hence the Newton algorithm converges in $O(\log \log (1/\epsilon)) = O(\log \log T)$ iterations. This setting covers negative entropy, the log regularizer, and $q$-Tsallis entropy, as shown in \Cref{table:regularizers_self_concordant,app:find_kappa}. For a general intrinsic Lipschitz regularizer $\psi$, we show that one can use geometric variants of the golden-section (\Cref{algo:geometric_golden}) and bisection search (\Cref{algo:geometric_bisection}) algorithms, which rely only on zeroth-order and first-order oracles, respectively. These algorithms output an approximate dynamic learning rate $\widehat{\lambda}$ with multiplicative error $\epsilon = 1/T$ in $O(\log (1/\epsilon)) = O(\log T)$ iterations.
}
\label{table:iteration_complexity}
\end{table}

If the regularizer $\psi$ is \emph{additionally Legendre and log-dominated} (see \Cref{def:log_dominated}), then the learning-rate control problem is self-concordant by \Cref{thm:log_dom_self_con}. In that case, we can apply Newton's method to solve the learning-rate control problem and obtain an approximation $\widehat{\lambda}$ that is $\epsilon$-close to the true solution $\lambda^*$ in the local norm induced by the second-order derivative of $\overline{f}$, in $O(\log \log (1/\epsilon))$ time. Moreover, by \Cref{thm:strong_convex_learning_rate_control_legendre}, $\widehat{\lambda}$ is also multiplicatively $\epsilon$-close to $\lambda^*$. Thus, given \Cref{table:regularizers_self_concordant}, we can solve the learning-rate control problem for negative entropy, the log regularizer, and $q$-Tsallis entropy in $O(\log \log (1/\epsilon))$ time.

With \Cref{thm:main_approx} at hand, we know that we should solve the learning-rate control problem up to a multiplicative tolerance of $\epsilon = 1/T$. Thus, the iteration complexity of \ours for general choices of regularizers is determined. We summarize these results in \Cref{thm:iter_comp,table:iteration_complexity}, and postpone the design of algorithms, proofs, and mathematical details to \Cref{app:iteration_comp}.

\begin{theorem}\label{thm:iter_comp}
    The iteration complexity of the additional overhead of \ours compared to \OFTRL, i.e., solving the learning-rate control problem, is:  
    \begin{itemize}
        \item $O(\log T)$ with access to \textbf{either a zeroth- or a first-order oracle}, whenever the regularizer is intrinsic Lipschitz,
        \item $O(\log \log T)$ with access to \textbf{first- and second-order oracles}, whenever the regularizer is (locally) intrinsic Lipschitz, Legendre, and log-dominated.
    \end{itemize}
\end{theorem}

An essential requirement at each iteration $t$ for the Geometric Golden-section search (\Cref{algo:geometric_golden}) and Geometric Bisection search (\Cref{algo:geometric_bisection}) algorithms is a positive lower bound $l_0 > 0$ such that we are assured the exact solution satisfies $\lambda\^t \geq l_0$. Similarly, for Newton's method, we require an initialization $\lambda_0$ at each iteration $t$ such that $\lambda_0$ lies in the quadratic convergence regime. In the proof of \Cref{thm:iter_comp}, we show that, by the multiplicative stability of consecutive iterations in \Cref{theorem:mult_stable}, we can warm start $l_0$ and $\lambda_0$ using the approximate dynamic learning rate from the previous iteration, $\widehat{\lambda}\^{t-1}$. Further details are provided in \Cref{app:iteration_comp}.

\section{Special Instances of \ours}\label{sec:special}

In this section, we discuss the regret guarantees of \ours for various instantiations of $\gamma$-(L)IL regularizers as shown in \Cref{table:regularizers}. A comparison to existing methods is presented in \Cref{table:results}. The three instantiations of \ours achieve an exponential improvement in their dependence on $d$ compared to Log-Regularized Lifted \OFTRL~\citep{farina2022near}, while also refining the dependence on $T$ from $\log^4 T$ to $\log T$ in comparison to \Opthedge~\citep{daskalakis2021near}. Additionally, the \emph{anytime convergence} guarantees established in \Cref{theorem:regret_final_bound} is stronger than those of \citet{daskalakis2021near}, where the learning rate $\eta$ is predefined based on the time horizon $T$. This dependency necessitates the use of auxiliary techniques such as the doubling trick when $T$ is unknown.  Furthermore, in \citet{daskalakis2021near}, the learning rate is constrained within the range $1/T \leq \eta \leq 1/(C n \log^4 T)$ for some constant $C$ (see Lemmas 4.2 and C.4 in \citep{daskalakis2021near}). Consequently, their theoretical guarantees do not extend to settings where $T \leq C n \log^4 T$, a regime frequently encountered in multi-agent environments with a large number of players or short time horizons.

Three instances of \ours yield new state-of-the-art algorithms for learning in games, achieving $O(n \log^2 d \log T)$ regret. The first is \ours with the $\ell_{p^*}$ regularizer, where $p^* = 1 + \frac{1}{\log d}$. The second is \ours with the $q^*$-Tsallis entropy, where $q^* = 1 - \frac{1}{\log d}$. The last is \ours with negative entropy, which, as an extension of \Opthedge, we term Cautious Optimistic Multiplicative Weights Update (\COMWU). Due to the importance of \COMWU, we study its behavior in detail in \Cref{sec:comwu} and develop its kernelized version for $0/1$-polyhedral games in \Cref{sec:kernel}. The motivation for the first and second instances is to optimize the resulting regret in the hyperparameters $p$ and $q$, respectively, which leads to the choices $p^*$ and $q^*$.

As shown in the first two rows of \Cref{table:results}, Cautious Optimism, as a general framework for regularized learning in games, achieves an exponential improvement in regret convergence, attaining the near-optimal rate of $O_{\scaleto{T}{4pt}}(\log T)$. This marks a significant improvement over the widely recognized Optimism framework~\citep{rakhlin2013online,rakhlin2013optimization,Syrgkanis15}.

\subsection{Cautious Optimistic Multiplicative Weights Update (\COMWU)} \label{sec:comwu}

We detail the Cautious Optimistic Multiplicative Weights Update (\COMWU), an important instantiation of \ours obtained by choosing the negative-entropy regularizer:
\[
\psi (\vx) = \sum_{\ind = 1}^d \vx[\ind] \log \vx[\ind].
\]

As in \Opthedge, the induced update assigns to each action a probability proportional to the exponential of its optimistic regret:
\[
    \vx\^{t}[k] \defeq\frac{\exp\{\lambda\^t \at\^t[k]\}}{\sum_{k' \in \mathcal{A}} \exp\{\lambda\^t \at\^t[k']\}} \qquad \forall k \in \mathcal{A}
    \numberthis{eq:softmax_main}
\]
with learning rate $\lambda\^t$ is dynamically adjusted by the learning rate control problem \eqref{eq:dynamic_learning_rate}. Interestingly, for \COMWU, we can show that the learning-rate control problem admits a closed form~\eqref{eq:comwu_closed_form_dyna_learning}. We formalize this observation in the following proposition.

\begin{proposition} \label{proposition:closed_form_comwu}
    \ours with negative entropy admits the following dynamic learning rate control problem,
    \[
    \lambda\^t \defeq \argmax_{\lambda \in (0, \eta]} \left\{f(\lambda; \at\^t) \defeq \alpha \log \lambda + \log \left( \sum_{\ind=1}^d e^{\lambda \at\^t[\ind]} \right) \right\}. 
    \numberthis{eq:comwu_closed_form_dyna_learning}
    \]
\end{proposition}

Based on our unified analysis of \ours in \Cref{theorem:regret_final_bound,table:regularizers}, we set the hyperparameter $\alpha = \Theta(\log^2 d)$. Owing to the closed form \eqref{eq:comwu_closed_form_dyna_learning}, we can better illustrate the learning-rate landscape of Cautious Optimism for this specific example. Notably, we can formally characterize the point at which the "regrets become too negative," causing the learner to begin being "paced down" in \COMWU. Technically speaking, in \Cref{lemma:lambda_is_one}, we show that when the maximum regret of the player (among its actions) exceeds $-\alpha + \log d$, the learner updates aggressively with full capacity, i.e., $\lambda\^t = \eta$. Thus, in this regime, there is no need to solve the one-dimensional learning-rate control problem~\eqref{eq:comwu_closed_form_dyna_learning}. We summarize \COMWU algorithm in \Cref{algo:comwu}. Beyond this regime, as the maximum regret becomes increasingly negative, the learning rate gradually decreases from $\eta$, causing the player to discount the history of the game more heavily.

\begin{lemma}\label{lemma:lambda_is_one}
    Let $\at\^t \in \bbR^d$ be arbitrary, and $\lambda\^t$ be the solution to
    \[
        \lambda\^t \defeq \argmax_{\lambda \in (0, \eta]} \left\{f(\lambda; \at\^t) \defeq \alpha \log \lambda + \log \Big( {\sum_{\ind=1}^d e^{\lambda \at\^t[\ind]}} \Big)  \right\}.
    \]
    If $\max_{r \in [d]} \{ \at\^t[r]\} \geq - \alpha + \log d$, then $\lambda\^t = \eta$.
\end{lemma}

\begin{algorithm2e}[!th]
    \SetNoFillComment
    \caption{Cautious Optimistic Multiplicative Weights Update  (\COMWU)
    }\label{algo:comwu}
    \DontPrintSemicolon
    \KwData{Learning rate $\eta$, parameters $\alpha$}\vspace{2mm}
    Set $ {\Ut}\^1, \ut^{(0)} \gets \vec{0} \in \bbR^{d}$\;
    
    \For{$t=1,2,\dots, T$}{
    \tcc{\color{commentcolor}\texttt{Optimism}}
    Set $\at\^t \gets {\Ut}\^t + {\vec\ut}\^{t-1}$\; \medskip
    \uIf{$\max_{\ind \in [d]} \{ \at[\ind]\} \geq - \alpha + \log d = \Theta(\log^2 d)$}{
    Set  $\displaystyle \lambda\^t \gets \eta$ \;
    }
    \Else{
    \tcc{\color{commentcolor}\texttt{Dynamic Learning Rate Control}}
     Set  $\displaystyle \lambda\^t \gets \argmax_{\lambda \in (0, \eta]} \left\{ (\alpha - 1) \log \lambda  + \log \Big( {\sum_{\ind=1}^d e^{\lambda \at\^t[\ind]}} \Big)  \right\}$ \medskip \label{line:oftrl0_COMWU} 
    }
    Set $\displaystyle\vx\^t \leftarrow \argmax_{\vx \in \Delta^d} \left\{ \lambda\^t \langle \at\^t, \vx \rangle - \psi(\vx) \right\} $\label{line:norm0_COMWU} \;  \medskip
    Play strategy $\displaystyle\vx\^t$ \;
    Observe $\vec \nu \^t \in \bbR^d$\;
    \medskip
    \label{line:lift0_COMWU}
    \tcc{\color{commentcolor}\texttt{Empirical Cumulated Regrets}} 
    Set $\displaystyle {\vec\ut}\^t \gets \vec \nu \^t -\langle \nut\^t, \vx\^t\rangle \vec1_d $ \; 
    Set $ {\Ut}\^{t+1} \gets  {\Ut}\^t +  {\vec\ut}\^t$
    }
\end{algorithm2e}

We can now take a closer look at the learning-rate control objective~\eqref{eq:comwu_closed_form_dyna_learning}. It consists of two components: $(\alpha - 1) \log \lambda$ and $ \log \Big( {\sum_{\ind=1}^d e^{\lambda \at\^t[\ind]}} \Big)$, the latter arising from the convex conjugate projected onto the simplex $\psi^*_{\mathcal{X}}$ in \eqref{eq:dynamic_learning_rate}. The term $ \log \Big( {\sum_{\ind=1}^d e^{\lambda \at\^t[\ind]}} \Big)$ is convex in $\lambda$; however, as we established earlier in \Cref{thm:strong_convex_learning_rate_control,thm:log_dom_self_con} in \Cref{sec:prop_learning_rate_objective}, the overall objective remains concave and self-concordant. Thanks to the closed form \eqref{eq:comwu_closed_form_dyna_learning}, and for illustrative purposes, we rederive this result using the closed form in \Cref{lemma:self_concordance} in \Cref{app:comwu}. Following our study in \Cref{sec:iteration}, we can employ Newton's method to achieve an iteration complexity of $\log \log T$. We conclude this section by showing that, for the special case of \COMWU, we can alternatively, at each iteration $t$, start Newton's method not from the warm start $\lambda_0 = \widehat{\lambda}\^{t-1}$, but from the initialization $\lambda_0 = \alpha / \big(- \max_{r \in [d]} \{\at[r]\}\big)$, which we demonstrate to be multiplicatively close to the solution $\lambda\^t$. We summarize this observation in the following Corollary.

\begin{corollary} \label{coro:newton_comwu}
    Given any $\at \in \bbR^d$ and a desired multiplicative accuracy $\epsilon > 0$, $O(\log \log (1/\epsilon))$ iterations of Newton's method, starting from the initialization point $\lambda_0 = \alpha / \big(- \max_{r \in [d]} \{\at[r]\}\big)$, are sufficient to compute a point $\widehat{\lambda}$ that approximates $\lambda^* \defeq \argmax_{\lambda \in (0, \eta]} f(\lambda; \at)$ with relative error at most $\epsilon$, i.e., $(1 - \epsilon)\lambda^* < \widehat{\lambda} < (1 + \epsilon)\lambda^*$.
\end{corollary}

\subsection{Extension of \COMWU to 0/1-Polyhedral Games via Kernels (\kours)} \label{sec:kernel}

In this section, We demonstrate that that—similarly to the \Opthedge algorithm~\citep{farina2022kernelized}—\COMWU can be applied efficiently to certain classes of polyhedral convex games with $0$/$1$-integral vertices.

Consider a convex $0/1$-polyhedral set $\Lambda \subseteq \mathbb{R}^d$, that is, a polytope whose set of vertices $\mathcal{V}_\Lambda$ is a subset of $\{0,1\}^d$. The \Opthedge algorithm can be directly applied to $\Lambda$ by maintaining a distribution $\vchi\^t$ over the vertices and updating it multiplicatively based on the utility $\Nut[\vertex] = \langle \nut, \vertex \rangle$ obtained by each vertex $\vertex \in \mathcal{V}_\Lambda$. Although this process requires time proportional to the number of vertices $|\mathcal{V}_\Lambda|$ if implemented naively, \citet{farina2022kernelized} show that it can, in some cases, be simulated in polynomial time per iteration even when $|\mathcal{V}_\Lambda|$ is large. Specifically, they prove that updating $\vchi\^t$ and computing the expectation $\sum_{\vertex \in \mathcal{V}_\Lambda} \vchi\^t[\vertex]\vertex$ can be performed using only $d + 1$ evaluations of a \emph{$0/1$-polyhedral kernel}. This kernel can be efficiently evaluated in extensive-form games and various other convex $0/1$-polyhedral settings, including $m$-sets, unit cubes, and flows on directed acyclic graphs, building on a prior idea of \citet{takimoto2003path} on path kernels.

Although the dynamics of \COMWU are quite similar to those of \Opthedge, extending \COMWU to its kernelized version (\kours) is not an immediate consequence of \citet{farina2022kernelized}. This extension requires additional considerations, particularly due to the need to solve the dynamic learning-rate control problem \eqref{eq:comwu_closed_form_dyna_learning} at each time step $t$. We will show that this optimization problem can also be solved using the kernel trick, with novel modifications. We begin by introducing the definitions of the $0/1$-polyhedral feature mapping, the associated kernel, and the key results from \citet{farina2022kernelized}.

\begin{definition}[0/1-polyhedral feature map and kernel \citep{farina2022kernelized}]\label{def:kernel_0_1}
Associated with a convex $0/1$-polyhedral set $\Lambda \subseteq \mathbb{R}^d$, define the \emph{0/1-polyhedral feature map} $\phi_\Lambda : \mathbb{R}^d \rightarrow \mathbb{R}^{\mathcal{V}_\Lambda}$,
\[
\phi_\Lambda(\vx)[\vertex] \defeq \prod_{\ind: \vertex[\ind] = 1} 
\vx[\ind] \quad \forall \, \vx \in \mathbb{R}^d, \, \vertex \in \mathcal{V}_\Lambda,
\]
and the corresponding \emph{0/1-polyhedral kernel} $K_\Lambda : \mathbb{R}^d \times \mathbb{R}^d \rightarrow \mathbb{R}$,
\[
K_\Lambda(\vx_1, \vx_2) \defeq \langle \phi_\Lambda(\vx_1), \phi_\Lambda(\vx_2) \rangle = \sum_{\vertex \in \mathcal{V}_\Lambda}\prod_{\ind: \vertex[\ind] = 1} \vx_1[\ind] \vx_2[\ind] \quad \forall \, \vx_1, \vx_2 \in \mathbb{R}^d.
\]
\end{definition}

We also define auxiliary indicator vectors $\vebar_i, \vebar_{ij} \in \mathbb{R}^d$ for all $i, j \in [d]$, such that
\[
\vebar_i [k] \defeq \mathbb{I}_{k \neq i} = 
\begin{cases}
    0 & \text{if } k = i, \\
    1 & \text{if } k \neq i,
\end{cases} \quad \text{and} \quad \vebar_{ij} [k] \defeq \mathbb{I}_{k \neq i \land k \neq j} = 
\begin{cases}
    0 & \text{if } k = i \text{ or } k = j, \\
    1 & \text{if } k \neq i \text{ and } k \neq j.
\end{cases}.
\]
The feature embedding of these vectors is useful in kernel computations.

We denote by $\At\^t$ the regret vector of the $0/1$-polyhedral game at time $t$. The following proposition ensures that step \eqref{eq:softmax_main},
\[
    \vchi\^t[\vertex] \defeq \frac{\exp\{\lambda\^t \At\^t[\vertex]\}}{\sum_{\vertex^\prime \in \mathcal{V}_{\Lambda}} \exp\{\lambda\^t \At\^t[\vertex']\}} \qquad \forall \vertex \in \mathcal{V}_{\Lambda},
    \numberthis{eq:softmax_vertices}
\]
can be simulated using only $d + 1$ kernel evaluations, assuming that the dynamic learning rate $\lambda\^t$ is available. The key idea is that by embedding a carefully constructed an embedding vector $\vb\^t$ into the feature mapping $\phi_\Lambda$, the algorithm’s updates—represented as the distribution over the vertices $\vchi\^t \in \Delta(\mathcal{V}_{\Lambda})$—and thus the actions $\vx\^t \in \Lambda$, become computable via the kernel $K_\Lambda$.

\begin{proposition}[Theorems 4.1 and 4.2 of \citep{farina2022kernelized}] \label{prop:kernel_paper}
For all time steps $t \in T$, 
let $\mut\^t \defeq \nut\^t + \sum_{\tau=1}^t  \nut\^\tau$ be the optimistic sum of the utility vectors $\nut\^t$, and define the embedding vector $\vb\^t \in \mathbb{R}^d$ as
\[
    \vb\^t[k] \defeq \exp\{\lambda\^t \mut\^t[k]\} \quad \forall \, k \in [d]. \numberthis{eq:b_t}
\]
Then, the distributions $\vchi\^t$ are proportional to $\phi_\Lambda(\vb\^t)$,
\[
    \vchi\^t = \frac{\phi_\Lambda(\vb\^t)}{K_\Lambda(\vb\^t, \vec{1}_d)},
\]
and the iterates $\vx\^t$ produced by \COMWU are computed as
\[
    \vx\^t = \sum_{\vertex \in \mathcal{V}_{\Lambda}} \vchi\^t [\vertex] \vertex = \mleft[1 - \frac{K_\Lambda(\vb\^t, \vebar_1)}{K_\Lambda(\vb\^t, \vec{1}_d)}, 1 - \frac{K_\Lambda(\vb\^t, \vebar_2)}{K_\Lambda(\vb\^t, \vec{1}_d)}, \ldots, 1 - \frac{K_\Lambda(\vb\^t, \vebar_d)}{K_\Lambda(\vb\^t, \vec{1}_d)} \mright]. \numberthis{eq:komw_action}
\]
\end{proposition}

\begin{algorithm2e}[thp]
    \SetNoFillComment
    \caption{Kernelized-Cautious Optimistic Multiplicative Weights Update (\kours)
    }\label{algo:kernel_comwu}
    \DontPrintSemicolon
    \KwData{Learning rate $\eta$, parameters $\alpha$ and $\beta$}\vspace{2mm}
    Set $ {\mut}\^1, \nut\^{0}, \vx\^{0} \gets \vec{0} \in \bbR^{d}, {\sigma}\^1, \lambda\^0 \gets 0 \in \bbR$\;
    \For{$t=1,2,\dots, T$}{
    Set $\mut\^t \gets {\mut}\^t + {\nut}\^{t-1}$ \Comment*{\color{commentcolor}Optimism for utility]\!\!\!\!\!}
    Set $\sigma\^t \gets {\sigma}\^t - \langle \nut\^{t-1}, \vx\^{t-1} \rangle$ \Comment*{\color{commentcolor}Optimism for correction]\!\!\!\!\!}
    
    \tcc{\color{commentcolor}\texttt{Dynamic Learning Rate Control via Kernelized Newton}}

    Set  $\displaystyle \lambda \leftarrow \lambda\^{t - 1} $ \Comment*{\color{commentcolor} Warm-start initialization for Newton]\!\!\!\!\!}
    
    \Repeat{\textup{Convergence of $\lambda$}}{
    \For{$r=1,2,\dots, d$}{
    Set $\displaystyle\vb[r] \leftarrow \exp\{\lambda \mut\^t[r]\}$ \Comment*{\color{commentcolor} See (\ref{eq:b_lambda})]\!\!\!\!\!} 
    }

    \For{$i=1,2,\dots, d$}{
    Set $\displaystyle \expect\mleft[\vertex \mright]_i \leftarrow 1 - \frac{K_\Lambda(\vb, \vebar_i)}{K_\Lambda(\vb, \vec{1}_d)} $ \Comment*{\color{commentcolor} See (\ref{eq:moment1})]\!\!\!\!\!} 
    }

    \For{$i, j=1,2,\dots, d$}{
    \uIf{$ i \neq j$}{
    Set $\displaystyle \expect\mleft[\vertex \vertex^\top \mright]_{ij}  \leftarrow 1 + \frac{K_\Lambda(\vb, \vebar_i)}{K_\Lambda(\vb, \vec{1}_d)} + \frac{K_\Lambda(\vb, \vebar_j)}{K_\Lambda(\vb, \vec{1}_d)} - \frac{K_\Lambda(\vb, \vebar_{ij})}{K_\Lambda(\vb, \vec{1}_d)} $ \Comment*{\color{commentcolor} See (\ref{eq:moment2})]\!\!\!\!\!}
    }
    
    \Else{
    Set $\displaystyle \expect\mleft[\vertex \vertex^\top \mright]_{ii}  \leftarrow 1 + \frac{K_\Lambda(\vb, \vebar_i)}{K_\Lambda(\vb, \vec{1}_d)}  $ \Comment*{\color{commentcolor} See (\ref{eq:moment2.2})]\!\!\!\!\!}
    }
     
    }

    Set $\displaystyle f^\prime(\lambda; \At) \leftarrow (\mut\^t)^\top \expect\mleft[\vertex \mright] + \sigma^t + \frac{\alpha}{\lambda}$

    Set $\displaystyle f^{\prime\prime}(\lambda; \At) \leftarrow (\mut\^t)^\top \expect\mleft[\vertex \vertex^\top \mright] \mut\^t - ((\mut\^t)^\top \expect\mleft[\vertex\mright])^2 - \frac{\alpha}{\lambda^2}$

    \tcc{\color{commentcolor}\texttt{Newton Update}}

    Set $\displaystyle \lambda \leftarrow \lambda + \frac{f^\prime(\lambda; \At)}{f^{\prime\prime}(\lambda; \At)}$
    
    }

    Set  $\displaystyle \lambda\^{t} \leftarrow \lambda $ \medskip  
    
    \tcc{\color{commentcolor}\texttt{Kernelized \Opthedge}}
    \For{$r=1,2,\dots, d$}{
    $\displaystyle\vb\^t[r] \leftarrow \exp\{\lambda\^t \mut\^t[r]\}$ \Comment*{\color{commentcolor} See (\ref{eq:b_t})]\!\!\!\!\!} 
    }
    \For{$r=1,2,\dots, d$}{
    $\displaystyle\vx\^t[r] \leftarrow 1 - \frac{K_\Lambda(\vb\^t, \vebar_r)}{K_\Lambda(\vb\^t, \vec{1}_d)} $ \Comment*{\color{commentcolor} See (\ref{eq:komw_action})]\!\!\!\!\!} 
    }
    Play strategy $\displaystyle\vx\^t$ and observe $\vec \nu \^t \in \bbR^d$\;

    Set $\displaystyle \mut\^t \gets {\mut}\^t + {\nut}\^{t} - {\nut}\^{t-1}$ \Comment*{\color{commentcolor}Empirical cumulated corrections]\!\!\!\!\!}
    Set $\displaystyle \sigma\^t \gets {\sigma}\^t - \langle \nut\^t, \vx\^t \rangle + \langle \nut\^{t-1}, \vx\^{t-1} \rangle $ \Comment*{\color{commentcolor}Empirical cumulated utilities]\!\!\!\!\!}
    }
\end{algorithm2e}

The proof of \Cref{prop:kernel_paper} is identical to the proof of \citet[Theorems 4.1 and 4.2]{farina2022kernelized}, except for a minor difference: in \eqref{eq:softmax_vertices}, we must also account for the correction terms $- \langle \Nut\^t, \vchi\^t \rangle$ in the regret vector,  
\[
\At\^t[\vertex] = (\Nut\^t[\vertex] - \langle \Nut\^t, \vchi\^t \rangle) + \sum_{\tau=1}^t \left[\Nut\^\tau[\nu] - \langle \Nut\^\tau, \vchi\^\tau \rangle\right].
\]
We provide the detailed proof for completeness in \Cref{app:kernelized_comwu}.

It remains to show that, for each time step $t$, the dynamic learning rate $\lambda\^t$ can be determined via kernelization. By \Cref{coro:newton_comwu}, we need to verify that the Newton steps of the optimization problem in \eqref{eq:comwu_closed_form_dyna_learning} can be simulated by the kernel $K_\Lambda$. The Newton algorithm requires the calculation of $f^\prime(\lambda; \At)$ and $f^{\prime\prime}(\lambda; \At)$,\footnote{To simplify notation, we omit the superscript $t$ from this point onward in this section whenever it is clear from the context.}
\[
 f^\prime(\lambda; \At) & = \frac{\sum_{\vertex \in \mathcal{V}_{\Lambda}} \At[\vertex] e^{\lambda \At[\vertex]}}{\sum_{\vertex \in \mathcal{V}_{\Lambda}} e^{\lambda \At[\vertex]}} + \frac{\alpha}{\lambda} \\
 f^{\prime\prime}(\lambda; \At) & = \frac{\sum_{\vertex \in \mathcal{V}_{\Lambda}} \At[\vertex]^2 e^{\lambda \At[\vertex]} }{\sum_{\vertex \in \mathcal{V}_{\Lambda}} e^{\lambda \At[\vertex]}} - \mleft( \frac{\sum_{\vertex \in \mathcal{V}_{\Lambda}} \At[\vertex] e^{\lambda \At[\vertex]}}{\sum_{\vertex \in \mathcal{V}_{\Lambda}} e^{\lambda \At[\vertex]}}\mright)^2 - \frac{\alpha}{\lambda^2},
\]
where the vector $\At$ is potentially of exponential size. We recall that by \eqref{eq:softmax_vertices}, $\vchi$ can be seen as a discrete random variable. We revisit the representation of $f^\prime(\lambda; \At)$ and $f^{\prime\prime}(\lambda; \At)$. Interestingly, they can be rewritten as
\[
f^\prime(\lambda; \At) & = \expect\mleft[\At[\vertex] \mright] + \frac{\alpha}{\lambda} \\
f^{\prime\prime}(\lambda; \At) & = \expect\mleft[ \At[\vertex]^2\mright] - \expect\mleft[\At[\vertex]  \mright]^2 - \frac{\alpha}{\lambda^2},
\]
where the expectations are taken with respect to the distribution $\vertex \sim \vchi$. Consequently, it suffices to verify that the first and second moments of $\At$ with respect to the distribution $\vchi$ can be computed via the kernelization approach. Given that
$\At[\vertex] = \langle \mut , \vertex \rangle + \sigma$, where
\[
\mut \defeq \nut\^t + \sum_{\tau=1}^t  \nut\^\tau \qquad\text{and}\qquad \sigma \defeq - (\langle \nut\^t, \vx\^t \rangle + \sum_{\tau=1}^t \langle \nut\^\tau, \vx\^\tau \rangle),
\]
we infer that
\[
    \expect\mleft[\At[\vertex] \mright] = \expect\mleft[\langle \mut, \vertex \rangle \mright] + \sigma = \langle  \mut, \expect\mleft[\vertex \mright] \rangle + \sigma. 
\]
And,
\[
\quad \expect\mleft[\At[\vertex]^2 \mright] = \expect\mleft[(\langle \mut, \vertex \rangle + \sigma) ^2 \mright] = \expect\mleft[\langle \mut, \vertex \rangle^2 \mright] + 2 \sigma \expect\mleft[\langle \mut, \vertex \rangle \mright] + \sigma^2 = \mut^\top \expect\mleft[\vertex \vertex^\top \mright] \mut + 2 \sigma \langle \mut, \expect\mleft[\vertex \mright] \rangle + \sigma^2,
\]
by multilinearity. We can simplify the $f^{\prime\prime}(\lambda; \At)$ term further,
\[
f^{\prime\prime}(\lambda; \At) & = \mut^\top \expect\mleft[\vertex \vertex^\top \mright] \mut + 2 \sigma \mut^\top \expect\mleft[\vertex \mright] + \sigma^2 - (\mut^\top \expect\mleft[\vertex \mright] + \sigma)^2 - \frac{\alpha}{\lambda^2} \\
& = \mut^\top \expect\mleft[\vertex \vertex^\top \mright] \mut - (\mut^\top \expect\mleft[\vertex\mright])^2 - \frac{\alpha}{\lambda^2}
\]
Hence, it is adequate to calculate $\expect\mleft[\vertex \mright]$ and $\expect\mleft[\vertex \vertex^\top \mright]$, which we formalize in the following proposition.

\begin{proposition} \label{prop:kernel_order_two}
    Define the vector $\vb \in \mathbb{R}^d$ as
    \[
        \vb[k] \defeq \exp\{\lambda \mut[k]\} \quad \forall \, k \in [d]. \numberthis{eq:b_lambda}
    \]
    Then,
    \[
    \expect\mleft[\vertex \mright] = \mleft[1 - \frac{K_\Lambda(\vb, \vebar_1)}{K_\Lambda(\vb, \vec{1}_d)}, 1 - \frac{K_\Lambda(\vb, \vebar_2)}{K_\Lambda(\vb, \vec{1}_d)}, \ldots, 1 - \frac{K_\Lambda(\vb, \vebar_d)}{K_\Lambda(\vb, \vec{1}_d)} \mright], \numberthis{eq:moment1}
    \]
    And,
    \[
    \expect\mleft[\vertex \vertex^\top \mright]_{ij} = 1 + \frac{K_\Lambda(\vb, \vebar_i)}{K_\Lambda(\vb, \vec{1}_d)} + \frac{K_\Lambda(\vb, \vebar_j)}{K_\Lambda(\vb, \vec{1}_d)} - \frac{K_\Lambda(\vb, \vebar_{ij})}{K_\Lambda(\vb, \vec{1}_d)} \numberthis{eq:moment2},
    \]
    for all $i, j \in [d]$, where $i \neq j$, and
    \[
        \expect\mleft[\vertex \vertex^\top \mright]_{ii} = 1 + \frac{K_\Lambda(\vb, \vebar_i)}{K_\Lambda(\vb, \vec{1}_d)} \numberthis{eq:moment2.2},
    \]
    for all $i \in [d]$.
\end{proposition}

An immediate byproduct of \Cref{prop:kernel_order_two} is that $d^2 + 1$ evaluations of the kernel $K_\Lambda$ are sufficient for each iteration of the Newton optimization algorithm that optimizes the learning-rate control problem objective \eqref{eq:comwu_closed_form_dyna_learning}. The full version of the \kours algorithm is presented in \Cref{algo:kernel_comwu}. The results of this section are summarized in the following corollary.

\begin{corollary} 
    \kours (\Cref{algo:kernel_comwu}) requires $(d + 1) + (d^2 + 1) \, O(\log \log T)$ kernel evaluations at each time step $t$, where the first $d + 1$ evaluations are used in step \eqref{eq:softmax_vertices}, and the remaining $(d^2 + 1) \, O(\log \log T)$ evaluations are for computing the dynamic learning rate $\lambda\^t$ via the Newton algorithm. \kours achieves $O(n \log^2 |\mathcal{V}_\Lambda| \log T)$ regret in the self-play setting and $\widetilde{O}(\sqrt{T \log |\mathcal{V}_\Lambda|})$ regret in adversarial settings.
\end{corollary}

We conclude this section with the final piece of the puzzle for \Cref{algo:kernel_comwu} (\kours). For the initialization $\lambda_0$ of Newton's method at each iteration $t$, we can either use the warm start $\lambda\^{t-1}$, as discussed in \Cref{sec:iteration,app:iteration_comp}, or, in line with \Cref{coro:newton_comwu}, solve a linear program to find $\lambda_0 = \alpha / ( - \max_{\vertex \in  \mathcal{V}_\Lambda}\At[\vertex])$.

\section{Cautious Optimism for Convex Games} \label{sec:convex_games}

Up to this point, our focus has centered on games with finite action spaces for illustrative purposes. However, we show that Cautious Optimism extends beyond normal-form games to general convex games, a topic that has recently garnered substantial attention~\citep{roughgarden2015local,even2009convergence,stoltz2007learning,stein2011correlated,mertikopoulos2019learning,harks2011demand}. More specifically, in this section, we demonstrate that Cautious Optimism, as a unified framework, applies naturally to the general class of \emph{convex games}.

\subsection{Convex Games}

We use the formalism of convex games, following the same notation as \citet{farina2022near}. In these games, each player $i \in [n]$ has a nonempty, convex, and compact set of strategies $\mathcal{X}_i \subset \mathbb{R}^{d_i}$, where $d = \argmax_{i \in [n]} d_i$.  Similar to \Cref{sec:finite_games_def}, we assume, for simplicity, that $d_i = d$ for all players $i \in [n]$. For each joint strategy profile $\vx = (\vx_1, \vx_2, \ldots, \vx_n) \in \bigtimes_{j=1}^n \mathcal{X}_j$, the utility of each player $i \in [n]$ is characterized by a continuously differentiable utility function $\nu_i: \bigtimes_{j=1}^n \mathcal{X}_j \rightarrow \mathbb{R}$, subject to the following standard assumption.

\begin{assumption} \label{assumption:convex}
    For every player $i \in [n]$ of the convex game, the utility function $\nu_i(\vx_1, \vx_2, \ldots, \vx_n)$ satisfies the following properties:
    \begin{enumerate}
        \item \textbf{(Concavity)}: The utility function $\nu_i(\vx_i, \vx_{-i})$ is concave in $\vx_i$ for all choices of actions of other players $\vx_{-i} = (\vx_1, \ldots, \vx_{i-1}, \vx_{i+1}, \ldots, \vx_n) \in \bigtimes_{j \neq i} \mathcal{X}_j$.
        \item \textbf{(Bounded gradients)}: The utility function has $1$-bounded gradients, i.e., 
        \[
        \| \nabla_{\vx_i} \nu_i(\vx_1, \vx_2, \ldots, \vx_n) \|_{\infty} \leq 1,
        \]
        for any strategy profile $\vx \in \bigtimes_{j=1}^n \mathcal{X}_j$.
        \item \textbf{(Smoothness)}: The utility functions are $L$-smooth, i.e., for every two joint strategy profiles $\vx, \vx^\prime \in \bigtimes_{j=1}^{n} \mathcal{X}_j$,
        \[
        \| \nabla_{\vx_j} \nu_j(\vx) - \nabla_{\vx_j} \nu_j(\vx^\prime) \|_{\infty} \leq L \sum_{i \in [n]} \| \vx_i - \vx_i^\prime \|_1.
        \]
    \end{enumerate}
\end{assumption}

Moreover, we assume that $\| \mathcal{X}_i \|_1 = \sup_{\vx \in \mathcal{X}_i} \| \vx \|_1 \leq 1$ for all players $i \in [n]$. This assumption is without loss of generality, since we can always rescale the action space to fit within the unit $\ell_1$ ball and use the scaling parameter to adjust the utility function $\nut$ and the constant $L$, thereby yielding a strategically equivalent game.

Beyond games with finite action spaces such as \emph{normal-form and extensive-form games}~\citep{romanovskii1962reduction,koller1996efficient} (studied in the previous sections), convex games subsume many applications, including \emph{splittable routing games}~\citep{roughgarden2015local}, \emph{Cournot competition games}~\citep{even2009convergence}, \emph{provision of public goods}~\citep{bergstrom1986private}, \emph{linear–quadratic network games}~\citep{ballester2006s}, \emph{proportional allocation problems}~\citep{kelly1998rate}, and \emph{wireless precoding games}~\citep{scutari2008optimal}.

\begin{algorithm2e}[th]
    \SetNoFillComment
    \caption{Cautious Optimistic \FTRL (\ours) for Convex Games
    }\label{algo:cftrl_convex}
    \DontPrintSemicolon
    \KwData{Learning rate $\eta$, parameters $\alpha$}\vspace{2mm}
    Set $ {\Ut}\^1, \ut^{(0)} \gets \vec{0} \in \bbR^{d+1}$\;
    \For{$t=1,2,\dots, T$}{
    \tcc{\color{commentcolor}\texttt{Optimism}}
    Set $\at\^t \gets {\Ut}\^t + {\vec\ut}\^{t-1}$\; \medskip
    \tcc{\color{commentcolor}\texttt{Dynamic Learning Rate Control}}
    Set  $\displaystyle \lambda\^t \gets \argmax_{\lambda \in (0, \eta]} \left\{ \alpha \log \lambda + \psi^*_{\mathcal{X}} (\lambda \at\^t)   \right\}$\label{line:oftrl0_convex} \;  \medskip 
    \tcc{\color{commentcolor} \texttt{\OFTRL with Dynamic Learning Rate}} 
    Set $\displaystyle \begin{pmatrix}
    1 \\
    \vx\^t
    \end{pmatrix} \leftarrow \argmax_{\vx \in \widetilde{\mathcal{X}}} \left\{ \lambda\^t \langle \at\^t, \vx \rangle - \psi(\vx) \right\} $\label{line:norm0_convex} \;  \medskip
    Play strategy $\displaystyle\vx\^t$ \;
    Observe $\vec \nu \^t \in \bbR^d$\;
    \medskip
    \label{line:lift0_convex}
    \tcc{\color{commentcolor}\texttt{Empirical Cumulated Regrets}} 
    Set $\displaystyle {\vec\ut}\^t \gets \begin{pmatrix}
    -\langle \nut\^t, \vx\^t \rangle \\
    \nut\^t
    \end{pmatrix} $ \; 
    Set $ {\Ut}\^{t+1} \gets  {\Ut}\^t +  {\vec\ut}\^t$
    }
\end{algorithm2e}

\subsection{Nonnegative Regret for Convex Games} \label{sec:nonreg_convex}

By the concavity assumption in \Cref{assumption:convex}, for each player $i$, we can derive a linearized regret that serves as an upper bound on the regret in convex games:  
\[
\max_{\vx_i^* \in \mathcal{X}_i} \sum_{t=1}^T \mleft( \nu_i( \vx^*_i, \vx_{-i}\^t ) - \nu_i(\vx\^t ) \mright) \leq \max_{\vx_i^* \in \mathcal{X}_i}  \sum_{t=1}^T \langle \vx_i^* - \vx_i\^t, \nabla_{\vx_i} \nu_i(\vx\^t) \rangle.
\]
Thus, in the spirit of \Cref{sec:finite_games_def}, it is sufficient to study the linearized regret for each player $i$\footnote{We omit the subscript $i$ for simplicity henceforth.}  
\[
  \reg\^T \defeq \max_{\vx^* \in \mathcal{X}}  \sum_{t=1}^T \langle \nut\^t, \vx^* \rangle  - \sum_{t=1}^T \langle \nut\^t, \vx\^t  \rangle,
\]
where the utility vector is defined as $\nut\^t \defeq \nabla_{\vx_i} \nu_i(\vx\^t)$. By \Cref{assumption:convex}, we also have $\| \nut\^t \|_\infty \leq 1$. To define nonnegative regret, we augment the action set $\mathcal{X}$ with an additional dimension that consistently takes the value one,  
\[
\widetilde{\mathcal{X}} \defeq \mleft\{ (1, \vx) \; \big| \; \vx \in \mathcal{X} \mright\} \subset \mathbb{R}^{d + 1}.
\]
Our nonnegative regret takes the form  
\[
\tildereg\^T \defeq \max_{\vy^* \in [0,1]\widetilde{\mathcal{X}}} \sum_{t = 1}^{T} \langle  {\ut}\^{t}, \vy^* - \vy\^{t} \rangle,
\]
where we define the corrected rewards as  
\[
\ut\^t \defeq 
\begin{pmatrix}
    -\langle \nut\^t, \vx\^t \rangle \\
    \nut\^t
\end{pmatrix},
\]
and the actions $\vy\^t \in (0, 1] \widetilde{\mathcal{X}}$ on the lifted space satisfy  
\[
\widetilde{\vx}\^t \defeq \begin{pmatrix}
1 \\
\vx\^t
\end{pmatrix}
= \frac{\vy\^t}{\langle \vy\^t, {\vec 1}_{d+1} \rangle}.
\]

Similar to \Cref{prop:reg+} for finite games, in the following proposition, we show that $\tildereg\^T = \max\{0, \reg\^T\}$. The proof is postponed to \Cref{app:convex}.

\begin{proposition}[Nonnegative Regret]\label{prop:reg+_convex}
    For any time horizon $T \in \mathbb{N}$, it holds that $\tildereg\^T = \max\{0, \reg\^T\}$. Consequently, $\tildereg\^T \geq 0$ and $\tildereg\^T \geq \reg\^T$.
\end{proposition}

\begin{table}[t]
    \centering
    \newcommand{\ldarrow}{\raisebox{-.7mm}{\tikz \draw[->] (0,0) -- (.25,0) -- +(0, -.2);}}%
    
    \resizebox{\textwidth}{!}{ 
    \begin{tabular}{>{\arraybackslash}m{5.0cm} >{\arraybackslash}m{4.5cm} lc}
        \bf Method                                            & \bf Applies to      & \bf Regret in Games & \bf General Learners \\
        \toprule
        OFTRL / OOMD\newline\citep{Syrgkanis15}     & general convex set $\mathcal{X} \subset \mathbb{R}^d $     & $O(\sqrt{n}\hspace{0.5 mm} \regdep(d) T^{1/4})$    & \checkmark \\
        \midrule

        COFTRL \newline\textbf{[This paper]}    & general convex set $\mathcal{X} \subset \mathbb{R}^d $         &  $O(n \hspace{0.5 mm} \Gamma(d) \log T)$                      & \checkmark\\
        \midrule
        \midrule
          LRL-OFTRL\newline\citep{farina2022near} \newline\textbf{[$\equiv$ \ours w/ log regularizer]}               & general convex set $\mathcal{X} \subset \mathbb{R}^d $ & $O(n \hspace{0.5 mm} d  \log T)$    & \crossmark 
          \\
        \midrule
          \ours with $\ell_{2}$ \newline\textbf{[This paper]}  & general convex set $\mathcal{X} \subset \mathbb{R}^d $           &  $O(n \hspace{0.5 mm} d \log T)$                           & \crossmark                                            \\
        \midrule
          \ours with $\ell_{p^*}$ \newline\textbf{[This paper]}             & general convex set $\mathcal{X} \subset \mathbb{R}^d $ & $O(n \log^2 d \log T)$                          & \crossmark                          \\
        \bottomrule
    \end{tabular}
    }    
    \caption{ 
    Comparison of existing no-regret learning algorithms in the broad class of convex games. We define $n$ as the number of players, $T$ as the number of game repetitions, and $d$ as dimension of the action set $\mathcal{X} \subset \mathbb{R}^d$. For simplicity, dependencies on smoothness, $\| \mathcal{X} \|_1$, and utility range are omitted. We choose $p^* = 1 + \frac{1}{\log d}$ for \ours instantiated with the $\ell_{p^*}$ norm to optimize the regret as a function of $p$.
    }
    \label{table:results_convex}
\end{table}

\subsection{\ours for Convex Games and Regret Guarantees} \label{sec:coftrl_convex_games}

\ours for convex games adopts a formulation similar to that of the finite-game case. The actions of the player are chosen according to
\[
\begin{pmatrix}
1 \\
\vx\^t
\end{pmatrix} \leftarrow \argmax_{\vx \in \widetilde{\mathcal{X}}} \left\{ \lambda\^t \langle \at\^t, \vx \rangle - \psi(\vx) \right\}, \numberthis{eq:x_update_convex}
\]
where the dynamic learning rate $\lambda\^t$ is determined by the learning-rate control problem,
\[
\lambda\^t \gets \argmax_{\lambda \in (0, \eta]} \left\{ \alpha \log \lambda + \psi^*_{\mathcal{X}} (\lambda \at\^t) \right\}, \numberthis{eq:lambda_update_convex}
\]
with the optimistically corrected regret vector,
\[
\at\^t \defeq \ut\^{t-1} + \sum_{\tau=1}^{t-1} \ut\^{\tau}, \quad \text{and} \quad 
\ut\^\tau \defeq 
\begin{pmatrix}
    -\langle \nut\^\tau, \vx\^\tau \rangle \\
    \nut\^\tau
\end{pmatrix}.
\]
The complete algorithm is presented in \Cref{algo:cftrl_convex}, and we establish the following regret guarantees, analogous to those for finite games.

\begin{theorem}[Regret bounds of \ours for convex games]\label{theorem:regret_final_bound_convex}
    If all players $i \in [n]$ in a convex game satisfying \Cref{assumption:convex} follow \ours with a $\gamma$-(L)IL and a $\mu$-strongly convex regularizer $\psi$ on the set $\mathcal{X}_i$, and use a sufficiently small learning-rate cap $\eta = O_{\scaleto{T}{4pt}}(1)$, then the following holds for both the individual and social regrets:
    \[
    \reg_i\^T = O\big(n \Gamma_\psi(d) \log T\big), \quad \text{and} \quad \sum_{i \in [n]} \reg_i\^T = O_{\scaleto{T}{4pt}}(1),
    \]
    where $\Gamma_\psi(d) = \gamma / \mu$. Moreover, the algorithm for each player $i \in [n]$ is adaptive to adversarial utilities; that is, each player’s regret satisfies $\reg_i\^T = O(\sqrt{T \log d})$.
\end{theorem}

For the analysis, we again consider the alternative perspective on \ours in the lifted space $(0, 1]\widetilde{\mathcal{X}}$, as in \Cref{sec:design}, obtained through the invertible change of variables $\vy = \lambda \widetilde{\vx} = \lambda \begin{pmatrix}
1 \\
\vx
\end{pmatrix}$.
The update is then given by
\[
\vy\^t \leftarrow \argmax_{\vy \in (0, 1]\widetilde{\mathcal{X}}} \mleft\{ \eta \langle \at\^t, \vy \rangle + \alpha \log(\vec 1^\top \vy) - \psi\mleft(\frac{\vy}{{\vec 1}^\top \vy}\mright) \mright\}. \numberthis{eq:lifted_FTRL_convex}
\]
The proof follows arguments similar to those in \Cref{sec:analysis,sec:social_regret} for finite games, and we defer the detailed discussion, particularly a clarification on the differences, to \Cref{app:convex}.

According to \Cref{theorem:regret_final_bound_convex}, \ours serves as a unified framework for regularized learning in convex games, exponentially improving the $O_{\scaleto{T}{4pt}}(T^{1/4})$ regret bound of \citet{Syrgkanis15} to $O_{\scaleto{T}{4pt}}(\log T)$. As we show later, an instance of \ours also achieves a new state-of-the-art in regularized learning in games, exponentially improving the dependence on the dimension $d$.

\subsection{Instances of \ours for Convex Games} \label{sec:instances_convex}

In this section, we discuss two instances of \ours for the general class of convex games. For a given action set $\mathcal{X}$\footnote{Without loss of generality, we assume $\|\mathcal{X}\|_1 \leq 1$; otherwise, the norm of the set $\|\mathcal{X}\|_1$ appears in the (local) intrinsic Lipschitzness parameter.}, depending on its structure, one may construct other strongly convex and intrinsic Lipschitz regularizers, and thus obtain additional instances of \ours, analogous to those in \Cref{table:regularizers} for finite games. It is easy to observe that, for any convex and compact action set $\mathcal{X}$, the log regularizer and the squared $\ell_p$ norm with $p \in (1, 2]$ are strongly convex and (locally) intrinsic Lipschitz, using the same derivations as in \Cref{app:find_gamma}. \ours with the log regularizer $\psi(\vx) = - \sum_{\ind = 1}^d \log \vx[\ind]$ recovers the dynamics of the \LRLOFTRL~\citep{farina2022near} algorithm, attaining $O(n\,d \log T)$ regret, while \ours with the squared $\ell_{p^*}$ norm $\psi(\vx) = \frac{1}{2} \|\vx\|\^2_{p^*}$, where $p^* = 1 + \frac{1}{\log d}$, attains an exponentially faster dependence on the dimension $d$ with regret $O(n \log^2 d \log T)$, setting a new state-of-the-art for regularized learning in convex games. We summarize these results in \Cref{table:results_convex}.

\section{Conclusion}

We designed, introduced, and analyzed \emph{Cautious Optimism}, a framework for the broad characterization of accelerated regularized learning in games. As a meta-algorithm for achieving near-constant regret in general games, Cautious Optimism takes as input an instance of \FTRL and, on top of that, paces the learners using \FTRL with non-monotone adaptive learning rate control.  

This approach represents a fundamental shift in the paradigms of online learning, optimization, and game theory, where constant or decreasing step sizes have traditionally been employed. The move from time-variant step-size to adaptive but state-dependent step-sizes is not merely a syntactic difference but has major implications for future work. The vast majority of standard tools in analyzing the behavior of dynamical systems (e.g., the celebrated Poincar\'{e} recurrence theorem~\cite{barreira}, Center-stable-manifold theorem~\cite{perko2013differential}, period three implies chaos theorem~\cite{LY}, a.o.) are only applicable to autonomous (i.e., time-independent) smooth maps or flows. Thus, although such ideas have been applied successfully in multi-agent learning in games it has been either for continuous-time systems~(e.g.~\cite{Kleinberg09multiplicativeupdates,piliouras2014optimization,mertikopoulos2017cycles}), which are arguably very idealized approximations or fixed-step dynamics, in which case we typically have to sacrifice black-box strong regret guarantees significantly hindering the analysis (e.g.,~\cite{CFMP2019,piliouras2023multi,BaileyEC18,bailey2020finite,wibisono2022alternating,katona2024symplectic}). Cautious Optimism paves a way forward where we do not have to make such concessions.   
This shift opens a wide range of new research questions, including how the dynamics of such algorithms evolve, their implications for the chaotic behavior of learning algorithms, the properties of their continuous-time counterparts, their connections to settings beyond external regret, such as swap regret, and their relationship to social welfare and strict equilibria—\emph{and the list continues}.

\bibliographystyle{ACM-Reference-Format}
\bibliography{refs}

\appendix
\newpage

\phantomsection
\addcontentsline{toc}{section}{Appendix}

\vspace{1.5cm}
\begin{center}
  {\LARGE\bfseries Appendix}
\end{center}
\vspace{0.5cm}

\section{Background} \label{sec:background}

\paragraph{Notation. }
We use lowercase letters (e.g., $x \in \mathbb{R}$) to represent scalar values and uppercase boldface letters (e.g., $\vx \in \mathbb{R}^d$) to denote vectors. The set $\{1, 2, \dots, n\}$ is abbreviated as $[n]$ for any natural number $n \in \mathbb{N}$. For a vector $\vx \in \mathbb{R}^d$, the $\ind$-th coordinate, where $\ind \in [d]$, is expressed as $\vx[\ind]$. We define $\vec{0}_d \in \mathbb{R}^d$ and $\vec{1}_d \in \mathbb{R}^d$ to represent the $d$-dimensional vectors of all zeros and all ones, respectively. Calligraphic uppercase letters (e.g., $\mathcal{A}$) are used to denote sets. For a finite set $\mathcal{A}$ of size $d = |\mathcal{A}|$, we denote the set of probability distributions over $\mathcal{A}$ by $\Delta(\mathcal{A}) \defeq \Delta^d$. We use subscript $\circ$ to represent the interior of the set, e.g., $\Delta^d_\circ$ is the interior of the simplex. We denote the vector $\vx$ associated with the $i$-th player by $\vx_i$, and use superscripts, such as $\vx_i^t$, to indicate dependence on timesteps. We denote the natural logarithm by $\log$. We use $\mathcal{X}$ and $\Delta^d$ interchangeably. Additionally, we define $\Omega \defeq (0, 1] \mathcal{X}$.

Given a convex function $\psi: \mathcal{X} \rightarrow \mathbb{R}$, we denote its convex conjugate by $\psi^*$, and further we define,
\[
\psi_{\mathcal{X}}^*(\vec g) \defeq \sup_{x \in \mathcal{X}} \langle \vec {g}, \vx \rangle  - \psi(\vx). \numberthis{eq:def_psi_x}
\]
$\psi_{\mathcal{X}}^*$ is the conjugate function of $\psi$ restricted to the convex set $\mathcal{X}$. Additionally, we consider the Bregman divergence $D_\psi(. \;\|\; .)$ generated by $\psi$,
\[
D_\psi(\vx^\prime \;\|\; \vx) = \psi(\vx^\prime) - \psi(\vx) - \langle \nabla \psi(\vx), \vx^\prime - \vx \rangle.
\]
\paragraph{Finite Games.} We study general-sum games involving $n$ players, indexed by the set $[n]$. Each player $i \in [n]$ is associated with a finite, nonempty set of pure strategies, denoted by $\mathcal{A}_i$. The mixed strategy space for player $i$ is defined as $\mathcal{X}_i = \Delta(\mathcal{A}_i)$, the set of probability distributions over $\mathcal{A}_i$. Collectively, the joint action space of all players is represented as $\bigtimes_{j=1}^n \mathcal{A}_j$. The game's utility structure is specified by functions $\mathcal{U}_i: \bigtimes_{j=1}^n \mathcal{A}_j \to \mathbb{R}$ for each player $i \in [n]$, where $\mathcal{U}_i$ assigns a real-valued payoff to every combination of pure strategies. When players adopt mixed strategies $\vx = (\vx_1, \vx_2, \dots, \vx_n) \in \bigtimes_{j=1}^n \mathcal{X}_j$, the expected utility for player $i$ is given by
\[
\nut_i(\vx) \coloneqq \mathbb{E}_{\vec{s} \sim \vx}[\mathcal{U}_i(\vec{s})] = \sum_{s \in \bigtimes_{j=1}^n \mathcal{A}_j} \vx(\vec{s}) \mathcal{U}_i(\vec{s}) = \sum_{\xi_i \in \mathcal{A}_i} \vx_i(\xi_i)  \mathbb{E}_{\vec{s}_{-i} \sim \vx_{-i}} [\mathcal{U}_i(\xi_i, \vec{s}_{-i}) ] ,
\]
which can also be expressed as $\nut_i(\vx) = \langle \vx_i, \nabla_{\vx_i} \nut_i(\vx) \rangle$, where
\[
\nabla_{\vx_i} \nut_i(\vx) = ( \mathbb{E}_{\vec{s}_{-i} \sim \vx_{-i}} [\mathcal{U}_i(\xi_1, \vec{s}_{-i}) ] \ldots , \mathbb{E}_{\vec{s}_{-i} \sim \vx_{-i}} [\mathcal{U}_i(\xi_i, \vec{s}_{-i}) ], \ldots, \mathbb{E}_{\vec{s}_{-i} \sim \vx_{-i}} [\mathcal{U}_i(\xi_{|\mathcal{A}_i|}, \vec{s}_{-i}) ] ),
\]
represents the gradient of $\nut_i$ with respect to $\vx_i$.

To streamline analysis, we let $d \coloneqq \max_{i \in [n]} |\mathcal{A}_i|$ denote the largest number of pure strategies available to any player. We further assume that $|\mathcal{A}_i| \geq 2$ for all $i \in [n]$, as players with a single strategy can be trivially excluded from the game without loss of generality.

Following the literature \citep{Syrgkanis15,daskalakis2021near,piliouras2022beyond,farina2022near,anagnostides2024interplay,legaccino,Soleymani25:Faster}, we make the following standard assumption on the boundedness of the utilities of the game.

\begin{assumption} \label{assumption:bounded}
    For every player $i$ of the game, the utilities are bounded,
    \[
    \max_{s \in {\bigtimes_{i=1}^n \mathcal{A}_i}} |\mathcal{U}_j(s)| \leq 1.
    \]
    and they are $L$-smooth, i.e.,  for every two joint strategy profiles $\vx, \vx^\prime \in \bigtimes_{j = 1}^{n} \Delta (\mathcal{A}_j)$,
    \[
    \| \nabla_{\vx_j} \nut_j (\vx) - \nabla_{\vx_j} \nut_j (\vec x^\prime) \|_{\infty} \leq L \sum\limits_{i \in [n]} \|\vx_i - {\vec x^\prime_i} \|_1.
    \]
\end{assumption}

The assumption of bounded utilities is quite general and does not impose significant restrictions, as any utility function with bounded values can be rescaled to satisfy this condition without altering the problem's essence. Under this boundedness condition, it can be shown that the game exhibits $L$-smoothness with $L = 1$. However, the actual smoothness parameter $L$ often depends on the underlying structure of the game and can be much smaller. For example, in games characterized by vanishing sensitivity $\epsilon_n$, the smoothness parameter reduces to $L = \epsilon_n$, where $\epsilon_n \ll 1$ \citep{anagnostides2024interplay}. By distinguishing $L$ from the boundedness condition, we account for structural nuances in the game that may enable improved convergence rates.

\paragraph{No-regret Learning.} In the \emph{online learning framework}, an agent interacts with the environment by choosing a strategy $\vx^t \in \mathcal{X} = \Delta^d$ at each timestep $t \in \mathbb{N}$. We consider the \emph{full information} setting, where the agent receives feedback in the form of a \emph{linear utility function}, given by $\vx \mapsto \langle \vx, \nu^t \rangle$, where $\nu^t \in \mathbb{R}^d$ denotes the utility vector at time $t$. The agent's performance over a horizon of $T$ rounds is quantified using the notion of \emph{external regret}, defined as:
\begin{equation}
    \label{eq:linreg}
    \reg^T \coloneqq \max_{\vx^* \in \mathcal{X}} \left\{ \sum_{t=1}^T \langle \nut^t, \vx^* \rangle \right\} - \sum_{t=1}^T \langle \nut^t, \vx^t \rangle.
\end{equation}
This metric captures the difference between the cumulative utility obtained by the agent and the best \emph{fixed strategy} in hindsight. It is important to note that the regret can be negative. The goal is to minimize regret as a function of the time horizon $T$. In adversarial settings, established algorithms such as Online Mirror Descent (OMD) and Follow-the-Regularized-Leader (FTRL) achieve the optimal regret rate of $\reg^T = \Theta(\sqrt{T})$~\citep{hazan2016introduction,orabona2019modern}. 

However, in \emph{self-play} dynamics within games, which is the focus of this work, regret minimization can exhibit faster rates due to the structured nature of the interactions. Unlike adversarial scenarios, where the utility vectors are independent of the agent's actions, in games, each player's utility vector evolves as a function of the joint strategy profile. Specifically, at timestep $t$, the utility vector for player $i \in [n]$ is given by $\nu^t = \nabla_{\vx_i} \nu_i(\vx^t)$, where $\vx^t = (\vx_1^t, \vx_2^t, \dots, \vx_n^t)$ denotes the collective strategy of all players.

\emph{Optimistic} predictive no-regret learning algorithms, such as Optimistic Online Mirror Descent and Optimistic Follow-the-Regularized-Leader, form a key class of no-regret algorithms with fast convergence properties in game settings. Originally introduced by \citet{rakhlin2013online}, these algorithms gained significant attention following the groundbreaking work of \citet{Syrgkanis15}, who demonstrated their faster convergence by proving that their regret satisfies the \emph{Regret bounded by Variation in Utilities (RVU)} property. For completeness, we restate this property below.

\begin{definition}[RVU property~\citep{Syrgkanis15}] \label{def:rvu}
    A no-regret learning algorithm is said to satisfy the Regret bounded by Variation in Utilities (RVU) property if its regret over any sequence of utilities $\{\nut\^t\}_{t=1}^T$ is bounded as:
    \[
        \reg\^T \leq a + b \sum_{t=1}^{T} \| \nut\^{t} - \nut\^{t-1} \|_{*}^2 - c \sum_{t=1}^{T} \| \vx\^{t+1} - \vx\^{t}  \|^2,
    \]
    where $a \geq 0$, $0 < b \leq c$, and $a$, $b$, and $c$ are algorithm-dependent parameters. The bounds are defined in terms of a pair of dual norms $(\| .\|, \| . \|_{*})$, which quantify variations in utilities and actions, respectively.
\end{definition}

In line with \Cref{assumption:bounded}, usually the pair of dual norms $(\ell_1, \ell_\infty)$ is considered for the RVU property in \Cref{def:rvu}.

\paragraph{Coarse Correlated Equilibrium.} A probability distribution $\sigma \in \Delta\left(\bigtimes_{j=1}^n \mathcal{A}_j\right)$ over the joint action space $\bigtimes_{j=1}^n \mathcal{A}_j$ is called an $\epsilon$-\emph{approximate coarse correlated equilibrium (CCE)} if, for every player $i \in [n]$ and every deterministic strategy $s_i' \in \mathcal{A}_i$, unilateral deviations do not increase the player's expected utility in the game. In terms,
\[
\mathbb{E}_{\mathbf{s} \sim \sigma}\left[\mathcal{U}_i(\mathbf{s})\right] \geq \mathbb{E}_{\mathbf{s} \sim \sigma}\left[\mathcal{U}_i(s_i', \mathbf{s}_{-i})\right] - \epsilon,
\]
where $\mathbf{s} = (s_1, s_2, \ldots, s_n)$ represents the joint action drawn from $\sigma$, and $\mathbf{s}_{-i}$ denotes the actions of all players except player $i$.

A well-known result in game theory connects no-regret learning algorithms to coarse correlated equilibria (CCE) of a game. Specifically, when the players in a game follow no-regret learning algorithms with regrets $\reg_i\^T$ for all $i \in [n]$, the empirical play of the game, $\sigma = \frac{1}{T} \sum_{t=1}^{T} \vx^t$, forms a $\left(\frac{1}{T} \max_{i \in [n]} \reg_i\^T\right)$-approximate CCE of the game. Consequently, achieving faster rates of no-regret learning in self-play directly translates to faster convergence to the CCE of the game.

\section{Details and Proofs for \Cref{sec:intr_lips}} \label{app:proofs_intr_lips}

\begin{restatedefinition}{def:local_intrin_lips}
    Let $\psi: \mathcal{X} \rightarrow \mathbb{R}$ be an arbitrary convex regularizer for the simplex, we call $\psi$, $\gamma$-locally intrinsically Lipschitz ($\gamma$-LIL) if,
    \[
    \mleft| \psi(\vx') - \psi(\vx) \mright|^2 \leq \gamma (\epsilon) D_{\psi} (\vx'\; \| \; \vx), \numberthis{eq:local_int} 
    \]
    for all $\vy^\prime, \vy \in (0, 1]\mathcal{X}$ that 
    \[
    D_{\phi} (\vy^\prime \; \| \; \vy) + D_{\phi} (\vy \; \| \; \vy^\prime) \leq \epsilon,
    \]
    where $ \vx' = \vy' / \vec 1^\top \vy', \vx = \vy / \vec 1^\top \vy$, and $\gamma(\epsilon)$ is a function of $\epsilon$. 
\end{restatedefinition}

We note that function $\phi$ is the same function as \eqref{eq:def_phi}
\[
\phi(\vy) \defeq - \alpha \log (\vec{1}^\top \vy) + \psi \mleft(\frac{\vy}{{\vec 1}^\top \vy} \mright).
\]

\begin{restateproposition}{prop:lip_intr_lip}
Any regularizer $\psi$ that is $\mu$-strongly convex, and $L$-Lipschitz w.r.t. the same norm $\|.\|$, is trivially ($2 L^2/\mu$)-IL.
\end{restateproposition}
\begin{proof}
    By Lipschitzness of the function $\psi$,
    \[
    \mleft| \psi(\vx^\prime) - \psi(\vx) \mright|^2 \leq L^2 \| \vx^\prime - \vx \|^2 \leq \frac{2 L^2}{\mu} D_\psi(\vx^\prime \; \| \; \vx),
    \]
    where in the last equation, we leveraged strong convexity of $\psi$.
\end{proof}

\begin{restateproposition}{prop:ciruit}
Let $\psi_1$ and $\psi_2$ be two functions that are $\gamma_1$- and $\gamma_2$-intrinsically Lipschitz, then,
\begin{enumerate}
    \item $\psi(\vx) = a \psi_1 (\vx) + b$ is $a \gamma_1$-intrinsically Lipschitz for any choices of $a \in \mathbb{R}_+, b \in \mathbb{R}$.
    \item $\psi(\vx) = \psi_1 (\vx) + \psi_2(\vx) $ is $(\gamma_1 + \gamma_2)$-intrinsically Lipschitz.
\end{enumerate}
\end{restateproposition}

\begin{proof}
    Proof of the first part follows immediately,
    \[
    \mleft| \psi(\vx^\prime) - \psi(\vx) \mright|^2  = a^2 \mleft| \psi_1(\vx^\prime) - \psi_1(\vx)  \mright|^2 \leq a^2 D_{\psi_1}(\vx^\prime \: \| \: \vx) = a D_\psi(\vx^\prime \: \| \: \vx).
    \]
    For the second part,
    \[
     \mleft| \psi(\vx^\prime) - \psi(\vx) \mright|^2 & = \mleft| (\psi_1(\vx^\prime) - \psi_1(\vx)) + (\psi_2(\vx^\prime) - \psi_2(\vx)) \mright|^2 \\
     & \leq (1 + c) \mleft| \psi_1(\vx^\prime) - \psi_1(\vx)  \mright|^2 + (1 + \frac{1}{c}) \mleft| \psi_2(\vx^\prime) - \psi_2(\vx)  \mright|^2 \numberthis{eq:circuit_proof1}\\
    & \leq (1 + c) \gamma_1 D_{\psi_1} (\vx^\prime \; \| \; \vx) + (1 + \frac{1}{c}) \gamma_2  D_{\psi_2} (\vx^\prime \; \| \; \vx) \\
    & \leq \max \{(1 + c) \gamma_1, (1 + \frac{1}{c}) \gamma_2\} D_{\psi} (\vx^\prime \; \| \; \vx),
    \]
    where \eqref{eq:circuit_proof1} follows by Young's inequality for any $c \in \mathbb{R}_+$. Choosing $c = \gamma_2/\gamma_1$,
    \[
     \mleft| \psi(\vx^\prime) - \psi(\vx) \mright|^2 & \leq (\gamma_1 + \gamma_2) D_{\psi} (\vx^\prime \; \| \; \vx).
    \]
\end{proof}

\subsection{Characterization of (Local) Intrinsic Lipschitzness and Strong Convexity Parameters of Different Regularizers} \label{app:find_gamma}

In this section, we state and prove the (local) intrinsic Lipschitzness and strong convexity for different regularizers, noted in \Cref{table:regularizers}.

\begin{enumerate}
    \setlength{\itemsep}{10pt}  %
    \setlength{\parskip}{0pt}
    \item \textbf{Negative Entropy}  ($\psi(\vx) = \sum_{\ind = 1}^d \vx[\ind] \log \vx[\ind]$): \hfill $\gamma = 3 \log^2 d, \mu = 1$.
    
    By Theorem 2 of \citet{reeb2015tight}, we know that entropy difference of distributions in $\Delta^d$ is upper bounded by the KL divergence as,
    \[
    \| \psi(\vx') - \psi(\vx) \|^2 \leq 3 \log^2 d \textup{KL}(\vx' \; \| \; \vx),
    \]
    for all $\vx', \vx \in \Delta^d$.
    Hence $\psi$ is $3 \log^2 d$-intrinsically Lipschitz.

    Additionally, it is widely known that negative entropy is $1$-strongly convex w.r.t. $\ell_1$ norm in the simplex $\Delta^d$, since for any vector $\vec v \in \mathbb{R}^d$,
    \[
    \vec v^\top \nabla^2 \psi(\vx) \vec v & = \vec v^\top \textup{diag}([\frac{1}{\vx[1]}, \ldots, \frac{1}{\vx[d]}]) \vec v \\
    & = \sum_{\ind = 1}^d \frac{\vec v[\ind]^2}{\vx[\ind]} \\
    & \geq \frac{1}{\sum_{\ind = 1}^d \vx[\ind] } \mleft( \sum_{\ind = 1}^d \vec | v[\ind] | \mright)^2  \\
    & = \| \vec v \|_1^2,
    \]
    as $\sum_{\ind = 1}^d \vx[\ind] = 1$ on the simplex $\Delta^d$.

    \ours with the negative entropy regularizer is equivalent to \COMWU~\citep{Soleymani25:Faster}, albeit the proof outline of \ours is substantially different from that of \COMWU, where the authors directly analyze the regret by leveraging self-concordance and stability properties of the closed-form expression of the learning rate control.

    \item \textbf{Log} ($\psi(\vx) = - \sum_{\ind = 1}^d \log \vx[\ind]$): \hfill $\gamma(1) = 18 d$.
    
    We show that Log regularizer is $d$-locally intrinsically Lipschitz. Consider the regularizer $\phi: (0, 1] \mathcal{X} \rightarrow \mathbb{R}$,
    \[
    \phi(\vy) & =  - (\alpha + 1) \log (\vec{1}^\top \vy) + \psi (\frac{\vy}{{\vec 1}^\top \vy}) \\
    & = - (d + 1) \log (\vec{1}^\top \vy) - \sum_{\ind = 1}^d \log (\frac{\vy}{{\vec 1}^\top \vy}[\ind]) \\
    & = - (d + 1) \log (\vec{1}^\top \vy) - \sum_{\ind = 1}^d \log \vy[\ind] + d \log (\vec{1}^\top \vy) \\
    & = - \log (\vec{1}^\top \vy) - \sum_{\ind = 1}^d \log \vy[\ind]
    \]
    Hence,
    \[
    D_\phi(\vy' \; \| \; \vy) & = \phi(\vy') - \phi(\vy) - \nabla \phi(\vy)^\top (\vy' - \vy) \\
    & = - \log (\vec{1}^\top \vy) - \sum_{\ind = 1}^d \log \vy[\ind] + \log (\vec{1}^\top \vy') + \sum_{\ind = 1}^d \log \vy'[\ind] + \sum_{\ind = 1}^d \mleft(\frac{1}{\vec 1^\top \vy} + \frac{1}{\vy[\ind]}\mright) (\vy'[\ind] - \vy[\ind]). 
    \]
    Thus,
    \[
    D_\phi(\vy' \; \| \; \vy) + D_\phi(\vy \; \| \; \vy') & = \sum_{\ind = 1}^d \mleft( \frac{\vy'[\ind]}{\vy[\ind]} + \frac{\vy[\ind]}{\vy'[\ind]} - 2\mright) + \mleft( \vec 1^\top \vy'  - \vec 1^\top \vy \mright) \mleft( \frac{1}{\vec 1^\top \vy} - \frac{1}{\vec 1^\top \vy'} \mright) .
    \]
    Hence, locality condition $D_\phi(\vy' \; \| \; \vy) + D_\phi(\vy \; \| \; \vy') \leq 1$ implies that,
    \[
    \sum_{\ind = 1}^d \mleft( \frac{\vy'[\ind]}{\vy[\ind]} + \frac{\vy[\ind]}{\vy'[\ind]} - 2\mright) + \mleft( \frac{\vec 1^\top \vy'}{\vec 1^\top \vy} + \frac{\vec 1^\top \vy}{\vec 1^\top \vy'} - 2 \mright) \leq 1.
    \]
    Thus, by the nonnegativity of $w + \frac{1}{w} - 2$ for any positive $w \in \mathbb{R}_+$, we infer that for all $\ind \in [d]$,
    \[
    & \frac{\vy'[\ind]}{\vy[\ind]} + \frac{\vy[\ind]}{\vy'[\ind]} - 2 \leq 1 \\
    \Rightarrow & \frac{1}{3} \leq \frac{\vy'[\ind]}{\vy[\ind]} \leq 3 \numberthis{eq:stab_log_y}.,
    \]
    and additionally,
    \[
    \frac{1}{3} \leq \frac{\vec 1^\top \vy'}{\vec 1^\top \vy} = \frac{\lambda'}{\lambda} \leq 3 \numberthis{eq:stab_log_lambda}.
    \]
    Combining \eqref{eq:stab_log_y} and \eqref{eq:stab_log_lambda}, we get,
    \[
    \frac{1}{9} \leq \frac{\vy'[\ind]}{\vy[\ind]} \frac{\vec 1^\top \vy}{\vec 1^\top \vy'} = \frac{\vx'[\ind]}{\vx[\ind]}\leq 9. \numberthis{eq;mult_stab_x_log} 
    \]
    In sequel with \eqref{eq;mult_stab_x_log} at hand, observe that,
    \[
    \nabla \psi (\vx) = \mleft[-\frac{1}{\vx_1}, \dots, -\frac{1}{\vx_d}\mright]^\top
    \]
    and 
    \[
    \nabla^2 \psi (\vx) = \operatorname{diag}\mleft( \mleft[\frac{1}{\vx_1^2}, \dots, \frac{1}{\vx_d^2} \mright] \mright).
    \]
    Hence,
    \[
    D_{\psi} (\vx^\prime \,\big\|\,   \vx) & = \psi(\vx^\prime) - \psi(\vx) - \nabla \psi(\vx)^\top (\vx^\prime - \vx) \\
    & = - \sum_{\ind = 1}^d \log \vx^\prime[\ind] + \sum_{\ind = 1}^d \log \vx[\ind] + \sum_{\ind = 1}^d \frac{1}{\vx[\ind]} (\vx[\ind]^\prime - \vx[\ind]) \\
    & = \sum_{\ind = 1}^d \mleft( - \log \frac{\vx^\prime[\ind]}{\vx[\ind]} + \frac{\vx^\prime[\ind]}{\vx[\ind]} - 1 \mright).
    \]
    At the same time, by Taylor's Reminder theorem, for a vector $\omega \in \mathbb{R}^d$ in the line segment of $\vx^\prime$ and $\vx$,
    \[
    D_{\psi} (\vx^\prime \,\big\|\,   \vx) & \geq \frac{1}{2} (\vx^\prime - \vx)^\top  \nabla^2 \psi (\omega) (\vx^\prime - \vx) \\
    & = \frac{1}{2} \sum_{\ind = 1}^d \frac{(\vx^\prime [\ind] - \vx[\ind])^2}{\omega[\ind]^2} \\
    & \geq \frac{1}{2}  \sum_{\ind = 1}^d \frac{(\vx^\prime [\ind] - \vx[\ind])^2}{\max\{\vx^\prime[\ind], \vx[\ind]\}^2} \\
     & \geq \frac{1}{2d} \mleft( \sum_{\ind = 1}^d \frac{|\vx^\prime [\ind] - \vx[\ind]|}{\max\{\vx^\prime[\ind], \vx[\ind]\}} \mright)^2 \numberthis{eq:log_intrinc_1} \\
    & = \frac{1}{2d} \mleft( \sum_{\ind = 1}^d 1 - \frac{\min\{\vx^\prime[\ind], \vx[\ind]\}}{\max\{\vx^\prime[\ind], \vx[\ind]\}}\mright)^2 \numberthis{eq:log_intrinc_2}
    \]
    where \eqref{eq:log_intrinc_1} follows by Cauchy-Schwarz inequality.

    On the other hand,
    \[
     | \psi(\vx^\prime) - \psi(\vx) | & = \mleft| \sum_{\ind = 1}^d \log \frac{\vx^\prime[\ind]}{\vx[\ind]}\mright| \\ 
     & \leq \sum_{\ind = 1}^d \mleft| \log \frac{\vx^\prime[\ind]}{\vx[\ind]}\mright| \\
     & \leq \sum_{\ind = 1}^d  \log \mleft( \frac{\max\{\vx^\prime[\ind], \vx[\ind]\}}{\min\{\vx^\prime[\ind], \vx[\ind]\}} \mright) \\
     & \leq \sum_{\ind = 1}^d  3 \mleft( 1 - \frac{\min\{\vx^\prime[\ind], \vx[\ind]\}}{\max\{\vx^\prime[\ind], \vx[\ind]\}}\mright) \numberthis{eq:log_intrinc_3}
    \]
    where last line follows by knowing that for all $9 \geq \omega = \frac{\max\{\vx^\prime[\ind], \vx[\ind]\}}{\min\{\vx^\prime[\ind], \vx[\ind]\}} \geq 1$ (which we know by \eqref{eq;mult_stab_x_log}), we infer that,
    \[
    \log (\omega) \leq 3 (1 - \frac{1}{\omega}).
    \]
    In the end, \eqref{eq:log_intrinc_3} combined with \eqref{eq:log_intrinc_2} concludes the overall picture,
    \[
    | \phi(\vx') - \phi(\vx) |^2 \leq 18 d D_\phi(\vx' \; \| \; \vx).
    \]
    We recognize that the dynamics of \ours, when $\psi$ is chosen as the log regularizer, align with the principles of lifted log-regularized \OFTRL (\LRLOFTRL)~\citep{farina2022near} and lead to the same regret guarantees. Despite these similarities, our proofs are substantially different from those of~\citet{farina2022near}, where they exploit elementwise multiplicative stability of actions due to the high curvature of the log regularizers and the structure of the induced intrinsic norms, which is a classical approach in the online learning and optimization literature.

    \item \textbf{Squared $\ell_p$ Norm} ($\psi(\vx) = \frac{1}{2} \| \vx \|_p^2$): with $p \in (1, 2]$ \hfill $\gamma = \frac{2}{p - 1}, \mu= (p - 1) d^{2/p - 2}$.

    To prove strong convexity and intrinsic Lipschitz, we recall the following lemma from \citet{shalev2007online}.
    
    \begin{lemma}[Lemma 17 of \citet{shalev2007online}] \label{lemma:convex_lp}
        $\psi(x) = \frac{1}{2} \|x\|_p^2$ is $(p-1)$-strongly convex with respect to $\|\cdot\|_p$, for any choice of $1 \leq p \leq 2$.
    \end{lemma}

    Now,
    \[
    |\psi(\vx') - \psi(\vx) |^2 & = \frac{1}{4} \mleft( \| \vx' \|_p^2 - \|\vx \|_p^2 \mright)^2 \\
    & = \frac{1}{4} \mleft( \| \vx' \|_p - \|\vx \|_p \mright)^2 \mleft( \| \vx' \|_p + \|\vx \|_p\mright)^2 \numberthis{eq:lp_simplex1}\\
    & \leq \mleft( \| \vx' \|_p - \|\vx \|_p \mright)^2 \\
    & \leq \| \vx' - \vx \|_p^2 \\
    & \leq \frac{2}{p - 1} D_\psi (\vx' \; \| \; \vx),
    \]
    where \eqref{eq:lp_simplex1} follows due to the fact that $\|.\|_p \leq 1$ on the simplex and the last line follows by the strong convexity indicated in \Cref{lemma:convex_lp}.

    In turn, to prove show the strong convexity of the regularizer $\psi$ w.r.t. $\ell_1$ by the following norm equivalence lemma.

    \begin{lemma}[Norm Equivalences]\label{lemma:norm_conversion}
        For any choice of $p \geq 1$,
        \[
        \| x \|_p \leq \| x \|_1 \leq d^{1 - 1/p} \| x \|_p.
        \]
    \end{lemma}

    Thus, \Cref{lemma:convex_lp} in connection with \Cref{lemma:norm_conversion} implies the convexity parameter $\mu$.

    \item \textbf{$1/2$-Tsallis Entropy} ($\psi(\vx) = 2 \mleft( 1 - \sum_{\ind = 1}^d \sqrt{\vx[\ind]} \mright)$) : \hfill $\gamma =  4 \sqrt{d}, \mu=\frac{1}{2}$.

    The gradient and Hessian of $\psi$ are,
    \[
    \nabla \psi(\vx)[\ind] = - 1 [\vx[1]^{-1/2}, \dots, \vx[d]^{-1/2}] \\
    \nabla^2 \psi(\vx) = \frac{1}{2} \textup{diag}([\frac{1}{\vx[1]^{3/2}}, \dots, \frac{1}{\vx[d]^{3/2}}]),
    \]
    From the Hessian, it can be immediately inferred that $\mu = 1/2$ since for any vector $\vec v \in \mathbb{R}^d$,
    \[
    \vec v^\top \nabla^2 \psi(\vx) \vec v & = \frac{1}{2}  \vec v^\top \textup{diag}([\frac{1}{\vx[1]^{3/2}}, \dots, \frac{1}{\vx[d]^{3/2}}]) \vec v \\
    & = \frac{1}{2} \sum_{\ind = 1}^d \frac{\vec v [\ind]^2}{\vx[\ind]^{3/2}} \\
    & \geq \frac{1}{2} \frac{1}{\sum_{\ind = 1}^d \vx[k]^{3/2}} \mleft( \sum_{\ind = 1}^d | \vec v[\ind] | \mright)^2 \numberthis{eq:half_tsallis1_convex}\\
    & \geq \frac{1}{2} \| \vec v \|_1^2,
    \]
    where in \eqref{eq:half_tsallis1_convex}, we used Cauchy Schwarz Inequality and in last line, we used the fact that on the simplex $\vx \in \Delta^d$,
    \[
    \sum_{\ind = 1}^d \vx[k]^{3/2} \leq \sum_{\ind = 1}^d \vx[k] = 1.
    \]
    Hence, strong convexity parameter is derived. 
    One the other hand,
    \[
    D_\psi(\vx' \;\|\; \vx) & = \psi(\vx') - \psi(\vx) - \nabla \psi(\vx)^\top (\vx' - \vx) \\
    & = 2 \mleft[\mleft( 1 - \sum_{\ind = 1}^d \sqrt{\vx'[\ind]} \mright) - \mleft(1 - \sum_{\ind = 1}^d \sqrt{\vx[\ind]} \mright) + \sum_{\ind = 1}^d \frac{1}{2 \sqrt{\vx[\ind]}} (\vx'[\ind] - \vx[\ind])\mright] \\
    & = 2 \sum_{\ind = 1}^d \frac{1}{\sqrt{\vx[\ind]}} \mleft[ - \sqrt{\vx[\ind]'\vx[\ind]} + \frac{1}{2} \vx[\ind] + \frac{1}{2} \vx'[\ind] \mright] \\
    & = \sum_{\ind = 1}^d \frac{1}{\sqrt{\vx[\ind]}} \mleft( \sqrt{\vx'[\ind]} - \sqrt{\vx[\ind]} \mright)^2 \\
    & \geq \frac{\mleft( \sum_{\ind  = 1}^d \sqrt{\vx'[\ind]} - \sqrt{\vx[\ind]} \mright)^2}{\sum_{\ind=1}^d \sqrt{\vx[\ind]}} \numberthis{eq:half_tsallis_1}\\
    & \geq \frac{1}{\sqrt{d}}  \mleft(  \sum_{\ind  = 1}^d \sqrt{\vx'[\ind]} - \sqrt{\vx[\ind]} \mright)^2 \numberthis{eq:half_tsallis_2}\\
    & = \frac{1}{4 \sqrt{d}} \mleft[\psi(\vx') - \psi(\vx) \mright]^2,
    \]
    where \eqref{eq:half_tsallis_1} is due to Cauchy–Schwarz inequality and \eqref{eq:half_tsallis_2} follows since on the simplex $\Delta^d$,
    \[
    \mleft(\sum_{\ind = 1}^{d} \sqrt{\vx[\ind]} \mright)^2 \leq d \sum_{\ind = 1}^d \vx[\ind] = d. 
    \]

    \item \textbf{$q$-Tsallis Entropy} ($\psi(\vx) = \frac{1}{1 - q} \mleft( 1 - \sum_{\ind = 1}^d \vx[\ind]^q \mright)$) : with $q \in (0, 1)$ \hfill $\gamma =  \frac{4 d^{1-q}}{(1 - q)^2} , \mu=q$.

    The gradient and Hessian of $\psi$ are,
    \[
    \nabla \psi(\vx)[\ind] = - \frac{q}{1 - q} [\vx[1]^{q - 1}, \dots, \vx[d]^{q - 1}] \\
    \nabla^2 \psi(\vx) = q.\textup{diag}([\frac{1}{\vx[1]^{2 - q}}, \dots, \frac{1}{\vx[d]^{2 - q}}]),
    \]
    From the Hessian, it can be immediately inferred that $\mu = 1/2$ since for any vector $\vec v \in \mathbb{R}^d$,
    \[
    \vec v^\top \nabla^2 \psi(\vx) \vec v & = q . \vec v^\top \textup{diag}([\frac{1}{\vx[1]^{2 - q}}, \dots, \frac{1}{\vx[d]^{2 - q}}]) \vec v \\
    & = q \sum_{\ind = 1}^d \frac{\vec v [\ind]^2}{\vx[\ind]^{2 - q}} \\
    & \geq q \frac{1}{\sum_{\ind = 1}^d \vx[k]^{2 - q}} \mleft( \sum_{\ind = 1}^d | \vec v[\ind] | \mright)^2 \numberthis{eq:q_tsallis1_convex}\\
    & \geq q \| \vec v \|_1^2,
    \]
    where in \eqref{eq:q_tsallis1_convex}, we used Cauchy Schwarz Inequality and in last line, we used the fact that on the simplex $\vx \in \Delta^d$,
    \[
    \sum_{\ind = 1}^d \vx[k]^{2 - q} \leq \sum_{\ind = 1}^d \vx[k] = 1.
    \]
    Hence, strong convexity parameter is inferred. 

    However, compared to the $1/2$-Tsallis entropy, determining the intrinsic Lipschitzness parameter for the general $q$-Tsallis entropy is a bit more involved. To establish this, we first state and prove the following lemma.
    \begin{lemma}\label{lemma:aux_q_tsallis}
        For any positive number $s \in \mathbb{R}_+$ and choice of $q \in (0, 1)$, we know,
        \[
        q (s - 1) - (s^q - 1) \geq (1 - q) (1 - s^{\frac{q}{2}})^2. \numberthis{eq:weighted_AM_GM_1}
        \]
    \end{lemma}
    \begin{proof}
        Substitute $r \defeq s^{q/2}$, 
        This is because the LHS of \eqref{eq:weighted_AM_GM_1} is,
        \[
        q (r^{\frac{2}{q}} - 1) - (r^2 - 1) = q r^{\frac{2}{q}} - r^2 - (q - 1).
        \]
        And for the RHS of \eqref{eq:weighted_AM_GM_1}, we get,
        \[
        (1 - q) (1 - r)^2 = (1 - q)r^2 - 2 (1 - q)r + (1 - q).
        \]
        
        Hence, the Inequality \ref{eq:weighted_AM_GM_1} becomes,
        \[
        q r^{\frac{2}{q}} + 2 (1 - q) r \geq (2 - q) r^2. 
        \]
        Now, in the weighted AM-GM inequality, 
        \[
        \beta_1 t_1 + \beta_2 t_2 \geq t_1^{\beta_1} t_2^{\beta_2},
        \]
        set $\beta_1 \defeq \frac{q}{2 - q}, \beta_2 \defeq \frac{2 (1 - q)}{2 - q}$ and $t_1 \defeq \frac{q r^{2/q}}{\beta_1}, t_2 \defeq \frac{2 (1-q) r}{\beta_2}$. It is clear that $\beta_1 + \beta_2 = 1$. We directly infer that,
        \[
        \beta_1 \mleft(\frac{q r^{2/q}}{\beta_1} \mright) + \beta_2 \mleft( \frac{2 (1-q) r}{\beta_2} \mright) & = q r^{\frac{2}{q}} + 2 (1 - q) r \\
        & \geq \mleft( \frac{q r^{2/q}}{\beta_1} \mright)^{\frac{q}{2-q}} \mleft( \frac{2 (1-q) r}{\beta_2} \mright)^{\frac{2 (1 - q)}{2 - q}} \\
        & = \mleft( \mleft( 2 - q\mright)^{\frac{q}{2 - q}}. \mleft(2 - q\mright)^{\frac{2 (1 - q)}{2 - q}} \mright) \mleft( r^{\frac{2}{2 - q}}. r^{\frac{2 ( 1 - q)}{2 - q}}\mright)  \\
        & = (2 - q)r^2,
        \]
        and the proof is concluded.
    \end{proof}

    In sequel, we expand the Bregman divergence.
    \[
    D_\psi(\vx' \;\|\; \vx) & = \psi(\vx') - \psi(\vx) - \nabla \psi(\vx)^\top (\vx' - \vx) \\
    & = \frac{1}{1 - q} \mleft[\mleft( 1 - \sum_{\ind = 1}^d \vx'[\ind]^q \mright) - \mleft(1 - \sum_{\ind = 1}^d \vx[\ind]^q \mright) + \sum_{\ind = 1}^d q \vx[\ind]^{q - 1} (\vx'[\ind] - \vx[\ind])\mright] \\
    & = \frac{1}{1 - q} \sum_{\ind = 1}^d  \mleft[ \vx[\ind]^q - \vx'[\ind]^q + q \vx[\ind]^{q - 1} (\vx'[\ind] - \vx[\ind]) \mright] \\
    & = \frac{1}{1 - q} \sum_{\ind = 1}^d  \vx[\ind]^q  \mleft[1 -   \mleft(\frac{\vx'[\ind]}{\vx[\ind]} \mright)^q + q (\frac{\vx'[\ind]}{\vx[\ind]} - 1) \mright].
    \]
    Next, replace $s \defeq \frac{\vx'[\ind]}{\vx[\ind]}$ in \Cref{lemma:aux_q_tsallis}, we get,
    \[
    D_\psi(\vx' \;\|\; \vx) & \geq \frac{1}{1 - q} \sum_{\ind = 1}^d  \vx[\ind]^q  \mleft[(1 - q) \mleft(1 - \mleft(\frac{\vx'[\ind]}{\vx[\ind]}\mright)^{\frac{q}{2}} \mright)^2 \mright] \\
    & =  \sum_{\ind = 1}^d  \mleft( \vx'[\ind]^{\frac{q}{2}} - \vx[\ind]^{\frac{q}{2} }\mright)^2. \numberthis{eq:q_tsallis_gamma1}
    \]
    In turn, we inspect the term $\mleft[\psi(\vx') - \psi(\vx) \mright]^2$,
    \[
    \mleft[\psi(\vx') - \psi(\vx) \mright]^2 & = \frac{1}{(1-q)^2}\mleft( \sum_{\ind = 1}^d \vx'[k]^q - \vx[k]^q \mright)^2 \\
    & = \frac{1}{(1-q)^2}\mleft( \sum_{\ind = 1}^d \mleft(\vx'[\ind]^\frac{q}{2} - \vx[\ind]^\frac{q}{2} \mright) \mleft(\vx'[\ind]^\frac{q}{2} + \vx[\ind]^\frac{q}{2}\mright) \mright)^2 \\
    & \leq \frac{1}{(1-q)^2} \mleft( \sum_{\ind = 1}^d \mleft(\vx'[\ind]^\frac{q}{2} - \vx[\ind]^\frac{q}{2} \mright)^2 \mright) \mleft( \sum_{\ind = 1}^d \mleft(\vx'[\ind]^\frac{q}{2} + \vx[\ind]^\frac{q}{2}\mright)^2 \mright) \numberthis{eq:q_tsallis_gamma2} \\
    & \leq \frac{2}{(1-q)^2} \mleft( \sum_{\ind = 1}^d \mleft(\vx'[\ind]^\frac{q}{2} - \vx[\ind]^\frac{q}{2} \mright)^2 \mright) \mleft( \sum_{\ind = 1}^d \vx'[\ind]^q + \sum_{\ind = 1}^d \vx[\ind]^q \mright) \numberthis{eq:q_tsallis_gamma3} \\
    & \leq \frac{2}{(1-q)^2} \mleft( \sum_{\ind = 1}^d \mleft(\vx'[\ind]^\frac{q}{2} - \vx[\ind]^\frac{q}{2} \mright)^2 \mright) 2 d^{1 - q} \numberthis{eq:q_tsallis_gamma4} \\
    & = \frac{4 d^{1 - q}}{(1-q)^2}  \sum_{\ind = 1}^d \mleft(\vx'[\ind]^\frac{q}{2} - \vx[\ind]^\frac{q}{2} \mright)^2, \numberthis{eq:q_tsallis_gamma5}
    \]
    where \eqref{eq:q_tsallis_gamma2} follows by Cauchy Schwarz inequality, \eqref{eq:q_tsallis_gamma3} follows by Young's inequality and \eqref{eq:q_tsallis_gamma4} follows by the fact that on the simplex $\vx \in \Delta^d$, the maximum of $\sum_{\ind = 1}^d \vx'[\ind]^q$ happens when for all $\ind \in [d]$, $\vx'[\ind] = \frac{1}{d}$. Thus,
    \[
    \sum_{\ind = 1}^d \vx'[\ind]^q \leq d \times d^{-q} = d^{1 - q}.
    \]
    \eqref{eq:q_tsallis_gamma5}, in combination with \eqref{eq:q_tsallis_gamma1}, completes the proof of intrinsic Lipschitzness for the $q$-Tsallis entropy.

\end{enumerate}

\section{Proofs for \Cref{sec:analysis}}

\subsection{Proofs for \Cref{sec:stepI}}

\begin{restatelemma}{lemma:breg_identity}
For arbitrary values of $\vy, \vy^\prime \in (0, 1]\mathcal{X}$, let $\lambda \defeq \vec 1^\top \vy, \lambda^\prime \defeq \vec 1^\top \vy^\prime$. Then,
\[
  D_\phi (\vy^\prime \; \| \; \vy) = \alpha D_{- \log} (\lambda^\prime \; \| \; \lambda) + \frac{\lambda^\prime}{\lambda} D_\psi (\vx^\prime, \vx) + \mleft(1 - \frac{\lambda^\prime}{\lambda} \mright) [\psi(\vx^\prime) - \psi(\vx)].
\]
\end{restatelemma}
\begin{proof}
    Denote as before $\vx = \frac{\vy}{{\vec 1}^\top \vy}$ and $\vx' = \frac{\vy'}{{\vec 1}^\top \vy'}$. Observe that,
    \[
    \frac{\partial}{\partial \vy_i} \phi(\vy) & = - \frac{\alpha}{\vec 1^\top \vy} - \frac{1}{\vec 1^\top \vy} (\nabla \psi(\vx)^\top \vx) + \frac{1}{\vec 1^\top \vy}. \frac{\partial}{\partial \vx_i} \psi(\vx) \\
    & = \frac{1}{\lambda} \mleft[ -\alpha - \nabla\psi(\vx)^\top \vx + \frac{\partial}{\partial \vx_i} \psi(\vx) \mright].
    \]
    Hence, expanding the definition of Bregman divergence,
    \[
    D_\phi(\vy' \; \| \; \vy) & = - \alpha \log \lambda' + \alpha \log \lambda + \psi(\vx') - \psi(\vx) + \alpha \frac{\lambda' - \lambda}{\lambda} + \frac{\nabla \psi(\vx)^\top \vx}{\lambda} (\lambda' - \lambda) - \nabla \psi(\vx)^\top \mleft( \vx' \frac{\lambda'}{\lambda} - \vx \mright) \\
    & = \alpha \mleft[ - \log \lambda' + \log \lambda + \frac{\lambda' - \lambda}{\lambda} \mright] + \psi(\vx') - \psi(\vx) + (\nabla \psi(\vx)^\top \vx) \frac{\lambda'}{\lambda} - \nabla \psi(\vx)^\top \vx' \frac{\lambda'}{\lambda} \\
    & = \alpha \mleft[ - \log \lambda' + \log \lambda + \frac{\lambda' - \lambda}{\lambda} \mright] + \frac{\lambda'}{\lambda} \mleft[ \psi(\vx') - \psi(\vx) - \nabla \psi(\vx)^\top (\vx' - \vx) \mright]  + \mleft(1 - \frac{\lambda'}{\lambda} \mright) [\psi(\vx') - \psi(\vx)] \\
    & = \alpha D_{- \log} (\lambda' \; \| \; \lambda) + \frac{\lambda'}{\lambda} D_\psi(\vx' \; \| \; \vx) + \mleft( 1 - \frac{\lambda'}{\lambda}\mright) [\psi(\vx') - \psi(\vx)],
    \]
    as claimed.
\end{proof}

\subsection{Proofs for \Cref{sec:stepII}} \label{app:proof_stepII}

\begin{restateproposition}{prop:convex_phi}
    For arbitrary choices of $\vy',\vy \in (0, 1]\mathcal{X}$, let $\lambda = \vec 1^\top \vy, \lambda' = \vec 1\^t \vy'$. Then, we infer that,
    \[
    D_\phi(\vy' \; \| \; \vy) + D_\phi(\vy \; \| \; \vy') \geq (\alpha - \gamma) \mleft[ \frac{\lambda'}{\lambda} + \frac{\lambda}{\lambda'} - 2 \mright],
    \]
    where $\psi$ is $\gamma$-intrinsically Lipschitz.
\end{restateproposition} 

\begin{proof}
    Expanding the Bregman divergences $D_\phi(\vy' \; \| \; \vy)$ and $D_\phi(\vy \; \| \; \vy')$ according to \Cref{lemma:breg_identity},
    \[
    D_\phi(\vy' \; \| \; \vy) + D_\phi(\vy \; \| \; \vy') & = \alpha [D_{-\log} (\lambda' \; \| \; \lambda) + D_{-\log} (\lambda \; \| \; \lambda') ] + \frac{\lambda'}{\lambda} D_\psi(\vx' \; \| \; \vx) + \frac{\lambda}{\lambda'} D_\psi(\vx \; \| \; \vx') \\
    & \qquad - \mleft( \frac{\lambda'}{\lambda} - \frac{\lambda}{\lambda'}\mright) [\psi(\vx') - \psi(\vx)].
    \]
    Let $\omega \defeq \frac{\lambda'}{\lambda}$ and $B \defeq [\psi(\vx') - \psi(\vx)]$, we can write,
    \[
    & D_\phi(\vy' \; \| \; \vy) + D_\phi(\vy \; \| \; \vy')  = \alpha [D_{-\log} (\lambda' \; \| \; \lambda) + D_{-\log} (\lambda \; \| \; \lambda') ] + \omega D_\psi(\vx' \; \| \; \vx) + \frac{1}{\omega} D_\psi(\vx \; \| \; \vx') - (\omega - \frac{1}{\omega}) B \\
    & \quad = (\alpha - \gamma) [D_{-\log} (\lambda' \; \| \; \lambda) + D_{-\log} (\lambda \; \| \; \lambda')] + \gamma (\omega + \frac{1}{\omega} - 2) + \omega D_\psi(\vx' \; \| \; \vx) + \frac{1}{\omega} D_\psi(\vx \; \| \; \vx') - (\omega - \frac{1}{\omega}) B \numberthis{eq:stepII_1}\\
    & \quad \geq (\alpha - \gamma) [D_{-\log} (\lambda' \; \| \; \lambda) + D_{-\log} (\lambda \; \| \; \lambda')] + \gamma (\omega + \frac{1}{\omega} - 2) +  (\omega + \frac{1}{\omega}) \frac{B^2}{\gamma} - (\omega - \frac{1}{\omega}) B, \numberthis{eq:last_three_terms}
    \]
    where \eqref{eq:stepII_1} follows by the observation that 
    \[
    D_{-\log} (\lambda' \; \| \; \lambda) + D_{-\log} (\lambda \; \| \; \lambda') = (- \log \omega + \omega - 1 ) + ( - \log \frac{1}{\omega} + \frac{1}{\omega} - 1 ) = \omega + \frac{1}{\omega} - 2, \numberthis{eq:obs_log}
    \]
    and the last line follows by $\gamma$-intrinsically Lipschitzness assumption (\Cref{def:intrin_lips}).
    
    Focus on the last three terms of \eqref{eq:last_three_terms} and let $R \defeq B/\gamma$:
    \[
    \gamma \mleft( \omega + \frac{1}{\omega} - 2 \mright) 
    + \mleft( \omega + \frac{1}{\omega} \mright) \frac{B^2}{\gamma} 
    - \mleft( \omega - \frac{1}{\omega} \mright) B 
    & = \frac{\gamma}{\omega} \mleft[ \mleft(\omega^2 - 2\omega + 1\mright) + \mleft(\omega^2 + 1\mright)R^2 - \mleft(\omega^2 - 1\mright)R \mright] \\
    & = \frac{\gamma}{\omega} \mleft[ \mleft(1 + R^2 - R\mright)\omega^2 - 2\omega + \mleft(1 + R^2 + R\mright) \mright] \\
    & = \frac{\gamma}{\omega} \mleft[ \mleft(1 + R^2 - R\mright) \mleft( \omega - \frac{1}{1 + R^2 - R} \mright)^2 + \mleft(1 + R^2 + R\mright) - \frac{1}{1 + R^2 - R} \mright] \\
    & = \frac{\gamma}{\omega} \mleft[ \mleft(1 + R^2 - R\mright) \mleft( \omega - \frac{1}{1 + R^2 - R} \mright)^2 + \frac{\mleft(1 + R^2\mright)^2 - \mleft(R^2 + 1\mright)}{1 + R^2 - R} \mright] \\
    & = \frac{\gamma}{\omega} \mleft[ \mleft(1 + R^2 - R\mright) \mleft( \omega - \frac{1}{1 + R^2 - R} \mright)^2 + R^2 \frac{1 + R^2}{1 + R^2 - R} \mright].
    \]
    Noting that $R^2 - R + 1 \geq 0$ for all values of $R \in \mathbb{R}$, we therefore conclude that,
    \[
    D_\phi(\vy' \; \| \; \vy) + D_\phi(\vy \; \| \; \vy') & \geq (\alpha - \gamma) [D_{-\log} (\lambda' \; \| \; \lambda) + D_{-\log} (\lambda \; \| \; \lambda')] \\
    & = (\alpha - \gamma) \mleft[ \frac{\lambda'}{\lambda} + \frac{\lambda}{\lambda'} - 2 \mright],
    \]
    which in turn implies the statement.
\end{proof}

\begin{restatecorollary}{corr:convex_phi}
    The function $\phi(\vy)$ is convex in $\vy \in (0,1]\mathcal{X}$ as long as $\psi$ is $\gamma$-intrinsically Lipschitz and $\alpha \geq \gamma$.
\end{restatecorollary}

\begin{proof}
    Directly by \Cref{prop:convex_phi},
    \[
    D_\phi(\vy' \; \| \; \vy) + D_\phi(\vy \; \| \; \vy') \geq (\alpha - \gamma) \mleft[ \frac{\lambda'}{\lambda} + \frac{\lambda}{\lambda'} - 2 \mright] \geq 0,
    \]
    since $\alpha \geq \gamma$ and $\frac{\lambda'}{\lambda} + \frac{\lambda}{\lambda'} \geq 2$ for positive values of $\lambda', \lambda$. Additionally, by expanding the Bregman divergences $D_\phi(\vy' \; \| \; \vy)$ and $D_\phi(\vy \; \| \; \vy')$,
    \[
    D_\phi(\vy' \; \| \; \vy) + D_\phi(\vy \; \| \; \vy') & = [\phi(\vy') - \phi(\vy) - \nabla \phi(\vy)^\top (\vy' - \vy)] + [\phi(\vy) - \phi(\vy') - \nabla \phi(\vy')^\top (\vy - \vy')] \\
    & = \langle \nabla \phi(\vy') - \nabla \phi(\vy), \vy' - \vy \rangle.
    \]
    Hence, $\nabla \phi$ is a monotone operator and thus, $\phi(\vy)$ is a convex function of $\vy$.
\end{proof}

\begin{proposition} \label{prop:local_convex_phi}
Let $\vy',\vy \in (0, 1]\mathcal{X}$ such that,
\[
    D_{\psi} (\vy^\prime \; \| \; \vy) + D_{\psi} (\vy \; \| \; \vy^\prime) \leq \epsilon, \numberthis{eq:local_intr_locality}
\]
Let $\lambda' = {\vec 1}^\top \vy'$ and $\lambda = {\vec 1}^\top \vy$. Then,
\[
    D_\phi(\vy' \; \| \; \vy) + D_\phi(\vy \; \| \; \vy') \geq (\alpha - \gamma) \mleft[ \frac{\lambda'}{\lambda} + \frac{\lambda}{\lambda'} - 2 \mright]
\]
for $\gamma$-locally intrinsically Lipschitz regularizer $\psi$.
\end{proposition}
\begin{proof}
    The proof follows exactly the same as the proof of \Cref{prop:convex_phi} except that instead of $\gamma$-IL assumption, $\gamma$-LIL assumption is employed.
\end{proof}

\begin{lemma}\label{lemma:proximal_step}
For the proximal steps of \OFTRL in Formulation~\ref{eq:lifted_FTRL_y}, we have,
\[
D_\phi (\vy' \; \| \; \vy) + D_\phi (\vy \; \| \; \vy') \leq 2 \eta \| \mathcal{X} \|_1  \| \at' - \at \|_{\infty} = 2 \eta \| \at' - \at \|_{\infty},
\]
where $\at',\at$ are the regret vectors of proximal steps of \OFTRL,
\[
\vy  \leftarrow \argmax_{\vy \in (0, 1] \mathcal{X}} \mleft\{ \eta \langle \at, \vy \rangle - \phi(\vy) \mright\}, \quad \textrm{ and } \quad
\vy'  \leftarrow \argmax_{\vy \in (0, 1] \mathcal{X}} \mleft\{ \eta \langle \at', \vy \rangle - \phi(\vy) \mright\}.
\]
\end{lemma}
\begin{proof}
    By first order optimality of the proximal steps,
    \[
    \langle \eta \at' - \nabla \phi(\vy'), \vy - \vy' \rangle & \leq 0 \\
    \langle \eta \at - \nabla \phi(\vy), \vy' - \vy \rangle & \leq 0.
    \]
    Hence,
    \[
    D_\phi (\vy' \; \| \; \vy) + D_\phi (\vy \; \| \; \vy') & = \langle \nabla \phi(\vy') - \nabla \phi(\vy), \vy' - \vy \rangle \numberthis{eq:prox_lemma1} \\  
    & \leq \eta \langle \at - \at', \vy - \vy' \rangle \\
    & \leq \eta \| \at - \at' \|_\infty \| \vy - \vy' \|_1  \numberthis{eq:prox_lemma2} \\
    & \leq \eta \| \at - \at' \|_\infty (\| \vy \|_1 + \|\vy' \|_1 ) \\ \numberthis{eq:prox_lemma3} 
    & \leq 2 \eta \| \mathcal{X} \|_1 \| \at - \at' \|_\infty \\
    & = 2 \eta \| \at - \at' \|_\infty.
    \]
    Line \ref{eq:prox_lemma1} follows from the expansion of Bregman divergences, \eqref{eq:prox_lemma2} follows from the Cauchy–Schwarz inequality, and \eqref{eq:prox_lemma3} follows from the triangle inequality.
\end{proof}

\begin{theorem}[Multiplicative Stability] \label{theorem:mult_stable_general}
    Consider regret vectors $\at',\at$ and the corresponding proximal steps of \OFTRL,
    \[
    \vy  \leftarrow \argmax_{\vy \in (0, 1] \mathcal{X}} \mleft\{ \eta \langle \at, \vy \rangle - \phi(\vy) \mright\},
    \] 
    and 
    \[
    \vy'  \leftarrow \argmax_{\vy \in (0, 1] \mathcal{X}} \mleft\{ \eta \langle \at', \vy \rangle - \phi(\vy) \mright\},
    \]
    where $\phi$ is such that $\psi$ is $\gamma$-IL.
    Additionally, let $\lambda' = \vy'/(\vec 1^\top \vy'), \lambda = \vy/(\vec 1^\top \vy)$. Then,
    \[
        \frac{\lambda'}{\lambda} \in [\frac{2 + \Upsilon - \sqrt{(\Upsilon+2)^2 - 4}}{2}, \frac{2 + \Upsilon + \sqrt{(\Upsilon+2)^2 - 4}}{2}],
    \]
    where $\Upsilon = \frac{2 \eta}{\alpha - \gamma} \| \at' - \at \|_{\infty}$.
\end{theorem}
\begin{proof}
    By \Cref{prop:convex_phi,lemma:proximal_step},
    \[
    (\alpha - \gamma) \mleft[ \frac{\lambda'}{\lambda} + \frac{\lambda}{\lambda'} - 2 \mright] \leq D_\phi(\vy' \; \| \; \vy) + D_\phi(\vy \; \| \; \vy') \leq 2 \eta \| \at' - \at \|_{\infty}
    \]
    By choosing $\omega \defeq \lambda'/\lambda$ and $\Upsilon \defeq \frac{2 \eta}{\alpha - \gamma} \| \at' - \at \|_{\infty}$, we need to solve for $\omega$ such that,
    \[
    \omega + \frac{1}{\omega} \leq \Upsilon + 2.
    \]
    Therefore,
    \[
    \omega \in [\frac{2 + \Upsilon - \sqrt{(\Upsilon+2)^2 - 4}}{2}, \frac{2 + \Upsilon + \sqrt{(\Upsilon+2)^2 - 4}}{2}].
    \]
    
\end{proof}

\begin{restatetheorem}{theorem:mult_stable}[Stability of learning rates]
    Given the regularizer $\phi$ with $\gamma$-IL component $\psi$ and choice of $\alpha \geq 4 \gamma$. consider regret vectors $\at',\at$ such that $\| \at' - \at\|_\infty \leq 4$ the corresponding proximal steps of \OFTRL,
    \[
    \vy  \leftarrow \argmax_{\vy \in (0, 1] \mathcal{X}} \mleft\{ \eta \langle \at, \vy \rangle - \phi(\vy) \mright\},
    \] 
    and 
    \[
    \vy'  \leftarrow \argmax_{\vy \in (0, 1] \mathcal{X}} \mleft\{ \eta \langle \at', \vy \rangle - \phi(\vy) \mright\}.
    \]
    Additionally, let $\lambda' = \vy'/(\vec 1^\top \vy'), \lambda = \vy/(\vec 1^\top \vy)$. Then, for small enough learning rate $\eta \leq 3 \gamma/160$, dynamic learning rates are multiplicatively stable,
    \[
        \frac{\lambda'}{\lambda} \in \mleft[\frac{1}{2}, \frac{3}{2}\mright].
    \]
\end{restatetheorem}
\begin{proof}
    By plugging in $\eta \leq 3 \gamma/80$, $\| \at' - \at\|_\infty \leq 4$ and $\alpha \geq 3 \gamma$ in \Cref{theorem:mult_stable_general},
    \[
    \Upsilon = \frac{2 \eta}{\alpha - \gamma} \| \at' - \at \|_{\infty} \leq \frac{1}{10}.
    \]
    Consequently,
    \[
    \frac{2 + \Upsilon - \sqrt{(\Upsilon+2)^2 - 4}}{2} & \geq \frac{1}{2} \\
    \frac{2 + \Upsilon + \sqrt{(\Upsilon+2)^2 - 4}}{2} & \leq \frac{3}{2},
    \]
    and the proof is concluded.
\end{proof}

\begin{theorem}[Stability of learning rates for $\gamma$-LIL] \label{theorem:mult_stable_local_lipsc}
    Given the regularizer $\phi$ with $\gamma$-LIL component $\psi$ and choice of $\alpha \geq 4 \gamma$, consider regret vectors $\at',\at$ such that $\| \at' - \at\|_\infty \leq 4$ the corresponding proximal steps of \OFTRL,
    \[
    \vy  \leftarrow \argmax_{\vy \in (0, 1] \mathcal{X}} \mleft\{ \eta \langle \at, \vy \rangle - \phi(\vy) \mright\},
    \] 
    and 
    \[
    \vy'  \leftarrow \argmax_{\vy \in (0, 1] \mathcal{X}} \mleft\{ \eta \langle \at', \vy \rangle - \phi(\vy) \mright\}.
    \]
    Additionally, let $\lambda' = \vy'/(\vec 1^\top \vy'), \lambda = \vy/(\vec 1^\top \vy)$. Then, for small enough learning rate $\eta \leq \min\{3 \gamma(1)/80, 1/8\}$, dynamic learning rates are multiplicatively stable,
    \[
        \frac{\lambda'}{\lambda} \in \mleft[\frac{1}{2}, \frac{3}{2}\mright].
    \]
\end{theorem}
\begin{proof}
    First by \Cref{lemma:proximal_step},
    \[
    D_\phi (\vy' \; \| \; \vy) + D_\phi (\vy \; \| \; \vy') \leq 2 \eta \| \at' - \at \|_{\infty} \leq 1.
    \]
    Thus, by \Cref{def:local_intrin_lips}, the definition of local intrinsic Lipschitzness,
    \[
     \mleft| \psi(\vx^\prime) - \psi(\vx) \mright|^2 \leq \gamma(1) D_{\psi} (\vx^\prime \; \| \; \vx).
    \]
    The result of the proof follows the same way as \Cref{theorem:mult_stable_general,theorem:mult_stable}
\end{proof}

\subsection{Proofs for \Cref{sec:stepIII}} \label{app:proof_stepIII}

\begin{restateproposition}{prop:curvature}
    As long as $\alpha \geq 4 \gamma + a$ for any positive number $a \in \mathbb{R}_+$, we have that,
    \[
    D_\phi (\vy' \;\|\; \vy) \geq (\gamma + a) D_{- \log} (\lambda' \; \| \; \lambda) + \frac{1}{4} D_\phi(\vx' \; \| \; \vx)
    \]
    for all $\vy', \vy \in (0, 1]\mathcal{X}$ such that
    \[
    \frac{\lambda'}{\lambda} = \frac{\vec 1^\top \vy'}{\vec 1^\top \vy} \in \mleft[ \frac{1}{2}, \frac{3}{2} \mright],
    \]
    where we denoted $\vx' = \frac{\vy'}{\vec 1^\top \vy'}, \vx = \frac{\vy}{\vec 1^\top \vy}$.
\end{restateproposition}
\begin{proof}
    By the expression for $D_\phi(\vy' \; \| \; \vy)$ in \Cref{lemma:breg_identity},
    \[
    D_\phi(\vy' \; \| \; \vy) & = \alpha D_{- \log} (\lambda^\prime \; \| \; \lambda) + \frac{\lambda^\prime}{\lambda} D_\psi (\vx^\prime, \vx) + \mleft(1 - \frac{\lambda^\prime}{\lambda} \mright) [\psi(\vx^\prime) - \psi(\vx)].
    \]
    Using Young's inequality on the last term, we can infer
    \[
    \mleft(1 - \frac{\lambda^\prime}{\lambda} \mright) [\psi(\vx^\prime) - \psi(\vx)] & \geq - \gamma \mleft( 1 - \frac{\lambda^\prime}{\lambda} \mright)^2 - \frac{1}{4 \gamma} [\psi(\vx^\prime) - \psi(\vx)]^2 \\
    & \geq - 3 \gamma D_{- \log} (\lambda^\prime \; \| \; \lambda) - \frac{1}{4} D_\psi(\vx' \; \| \; \vx),
    \]
    where in the last line, for the first term, we used the fact that
    \[
    - \log \omega + \omega - 1 \geq \frac{1}{3} (\omega - 1)^2 \qquad \forall \; \omega = \frac{\lambda'}{\lambda} \in \mleft[\frac{1}{2}, \frac{3}{2} \mright],
    \]
    and for the second term, we used \Cref{def:intrin_lips}. Setting $\alpha \geq 3 \gamma + a$ yields the expression.
\end{proof}

\begin{theorem}[Curvature in subsequent iterates of \OFTRL]\label{theorem:cruv}
    Given the regularizer $\phi$ with $\gamma$-IL component (or with $\gamma$-LIL component) and choice of $\alpha \geq 4 \gamma + a$ for any positive number $a \in \mathbb{R}_+$. consider regret vectors $\at',\at$ such that $\| \at' - \at\|_\infty \leq 4$ the corresponding proximal steps of \OFTRL,
    \[
    \vy  \leftarrow \argmax_{\vy \in (0, 1] \mathcal{X}} \mleft\{ \eta \langle \at, \vy \rangle - \phi(\vy) \mright\},
    \] 
    and 
    \[
    \vy'  \leftarrow \argmax_{\vy \in (0, 1] \mathcal{X}} \mleft\{ \eta \langle \at', \vy \rangle - \phi(\vy) \mright\}.
    \]
    Additionally, let $\lambda' = \vy'/(\vec 1^\top \vy'), \lambda = \vy/(\vec 1^\top \vy)$. Then, for small enough learning rate $\eta \leq 3 \gamma/80$ (or $\eta \leq \min\{3 \gamma(1)/80, 1/8\}$ when $\psi$ is $\gamma$-LIL), we have
    \[
    D_\phi (\vy' \;\|\; \vy) \geq \gamma D_{- \log} (\lambda' \; \| \; \lambda) + \frac{1}{4} D_\phi(\vx' \; \| \; \vx).
    \]
\end{theorem}
\begin{proof}
    The proof follows directly by combination of \Cref{theorem:mult_stable,theorem:mult_stable_local_lipsc} with \Cref{prop:curvature}.
\end{proof}

\section{Detailed analysis of regret (\Cref{sec:regret})} \label{app:detailed}

To analyze the regret of \ours, we utilize its equivalent \OFTRL formulation \ref{eq:lifted_FTRL} and demonstrate that its regret, referred to as nonnegative regret, satisfies the RVU property~\citep{Syrgkanis15}. Consequently, leveraging the nonnegativity of the regret for the formulation in \ref{eq:lifted_FTRL}, we establish a nonnegative RVU bound for \ours. The approach of examining regret through nonnegative RVU bounds is inspired by the work of \citet{farina2022near}. This idea is formalized in the following section.

\subsection{Nonnegative Regret} \label{app:nonneg}

In terms, we analyze the nonnegative regret, defined as
\[
\tildereg\^T \defeq \max_{\vy^* \in [0,1]\mathcal{X} } \sum_{t = 1}^{T} \langle  {\ut}\^{t}, \vy^* - \vy\^{t} \rangle.
\]

The following proposition, characterizes that not only $\tildereg\^T$ is positive but also it gives us an upperbound on $\reg$ of \ours. This proposition is pivotal as it establishes a direct translation of RVU bounds for $\tildereg\^T$ into nonnegative RVU bounds for $\reg^T$.

\begin{proposition}[Nonnegative Regret]\label{prop:reg+}
    For any time horizon $T \in \mathbb{N}$, it holds that $\tildereg\^T = \max\{0, \reg\^T\}$. Consequently, $\tildereg\^T \geq 0$ and $\tildereg\^T \geq \reg\^T$.
\end{proposition}

\begin{proof}\allowdisplaybreaks
    To establish the result, consider the definition of the reward signal $ {\ut}\^t = \nut\^t -\langle \nut\^t, \vx\^t\rangle \vec1_d$ and the induced action $\vx\^t = \frac{\vy\^t}{\langle \vy\^t, {\vec 1}_d \rangle}$. Based on these, the regret can be analyzed as follows:
    \[
        \tildereg\^T & = \max_{\vy^* \in [0,1]\Delta^d } \sum_{t = 1}^{T} \langle  {\vec\ut}\^{t}, \vy^* - \vy\^{t} \rangle                              \\
        & = \max_{\vy^* \in [0,1]\mathcal{X} } \sum_{t = 1}^{T} \langle \nut\^t -\langle \nut\^t, \frac{\vy\^t}{\langle \vy\^t, {\vec 1}_d \rangle} \rangle \vec1_d, \vy^* - \vy\^{t} \rangle                                                   \\
        & =  \max_{\vy^* \in [0,1]\mathcal{X} } \sum_{t = 1}^{T} \langle \nut\^t, \vy^* \rangle - \mleft\langle \langle \nut\^t, \frac{\vy\^t}{\langle \vy\^t, {\vec 1}_d \rangle} \rangle \vec1_d, \vy^* \mright\rangle \numberthis{eq:reg+_1} \\
        & \geq \max_{\vy^* \in \mathcal{X} } \sum_{t = 1}^{T} \langle \nut\^t, \vy^* \rangle - \langle \nut\^t, \vx\^{t} \rangle . \langle \vec1_d, \vy^* \rangle                                                                           \\
        & \geq \max_{\vy^* \in \mathcal{X} } \sum_{t = 1}^{T} \langle \nut\^t, \vy^* \rangle - \langle \nut\^t, \vx\^{t} \rangle                                                                                                            \\
        & = \reg\^T,
    \]
    where \eqref{eq:reg+_1} follows from the orthogonality condition $\vec \ut\^{t} \perp \vy\^{t}$. Furthermore, it is straightforward to see that $\tildereg\^T \geq 0$ by selecting $0$ as the comparator. This completes the proof.
\end{proof}

\subsection{Analysis of \OFTRL in Formulation \ref{eq:lifted_FTRL}}

To proceed, we define $ {\vy}\^{t}$ as the outputs produced by the \OFTRL algorithm in Formulation \ref{eq:lifted_FTRL}:
\begin{align*}
    {\vy}\^{t} = \argmax_{  {\vy} \in \Omega} - F_t(  {\vy}) = \argmin_{  {\vy} \in \Omega} F_t(  {\vy}), \textup{ where } F_t(  {\vy}) \defeq - \mleft\langle  {\Ut}\^t +  {\vec\ut}\^{t-1},  {\vy} \mright\rangle + \frac{1}{\eta} \phi(  {\vy}). \numberthis{eq:Fdef}
\end{align*}
Additionally, we introduce the auxiliary sequence $ {\vz}\^{t}$, which corresponds to the solutions obtained from the standard \FTRL algorithm at each time step $t$:
\begin{align*}
    {\vz}\^{t} = \argmax_{  {\vz} \in \Omega} - G_t(  {\vz}) = \argmin_{  {\vz} \in \Omega} G_t(  {\vz}), \textup{ where } G_t(  {\vz}) \defeq - \mleft\langle  {\Ut}\^t,  {\vz} \mright\rangle + \frac{1}{\eta} \phi(  {\vz}). \numberthis{eq:Gdef}
\end{align*}
Both $G_t$ and $F_t$ are convex, as established in \Cref{corr:convex_phi}. Let $\lambda\^t = \vec 1^\top \vy\^t$ be the dynamic learning rate and $\vx\^t = \vy\^t / \lambda\^t$ the corresponding as noted in Formulation \ref{eq:OFTRL}. Similarly for \FTRL step \eqref{eq:Gdef}, let the dynamic learning rate $\xi\^t = \vec 1^\top \vz\^t$ and the corresponding actions $\vtheta\^t = \vz\^t / \xi\^t$.

With these definitions in place, we proceed to the next step of the analysis using the following standard lemma from the \OFTRL framework~\citep{rakhlin2013online}.

\begin{lemma}\label{lemma:basi_OFTRL}
    For any $ {\vy} \in \Omega$, let the sequences $\{ {\vy}\^{t}\}_{t = 1}^{T}$ and $\{ {\vz}\^{t}\}_{t = 1}^{T}$ be generated by the \FTRL update rules specified in \eqref{eq:Fdef} and \eqref{eq:Gdef}, respectively. Then, the following inequality holds:
    \[
        \sum_{t = 1}^{T} \mleft \langle  {\vy} -  {\vy}\^{t},  {\ut}\^{t} \mright \rangle & \leq \frac{1}{\eta} \phi( {\vy}) - \frac{1}{\eta} \phi( {\vy}\^{1}) + \sum_{t = 1}^{T} \mleft \langle  {\vz}\^{t+1} -  {\vy}\^{t},  {\ut}\^{t} -   {\ut}\^{t - 1}\mright \rangle \\
        & \qquad - \frac{1}{\eta} \sum_{t = 1}^{T} \mleft( D_{\phi}( {\vy}\^{t} \,\big\|\,  {\vz}\^{t}) + D_{\phi}( {\vz}\^{t+1} \,\big\|\,  {\vy}\^{t}) \mright).
    \]
\end{lemma}
\begin{proof}
     Using \Cref{lemma:opt_bregman} and the optimality condition for $ {\vz}\^{t}$, we have
     \[
     G_t( {\vz}\^{t}) & \leq G_t( {\vy}\^{t}) - \frac{1}{\eta} D_{\phi} ( {\vy}\^{t} \,\big\|\, {\vz}\^{t}) \\ & \leq F_t( {\vy}\^{t}) + \langle {\vy}\^{t} , {\ut}\^{t - 1} \rangle - \frac{1}{\eta} D_{\phi} ( {\vy}\^{t} \,\big\|\, {\vz}\^{t}).
     \]
     Similarly, from the optimality of $ {\vy}\^{t}$, it follows that
     \[
     F_t( {\vy}\^{t}) & \leq F_t( {\vz}\^{t+1}) - \frac{1}{\eta} D_{\phi} ( {\vz}\^{t+1} \,\big\|\, {\vy}\^{t}) \\ & \leq G_{t+1}( {\vz}\^{t+1}) + \langle {\vz}\^{t+1}, {\ut}\^{t} - {\ut}\^{t-1} \rangle - \frac{1}{\eta} D_{\phi} ( {\vz}\^{t+1} \,\big\|\, {\vy}\^{t}).
     \]
     By combining these inequalities and summing over all $t$, we obtain
     \[
     G_1( {\vz}\^{1}) & \leq G_{T+1}( {\vz}\^{T + 1}) + \sum_{t = 1}^{T} (\langle {\vy}\^{t}, {\ut}\^{t} \rangle + \langle {\vz}\^{t+1} - {\vy}\^{t}, {\ut}\^{t} - {\ut}\^{t-1}\rangle) \\ & \qquad \quad - \frac{1}{\eta} \sum_{t=1}^{T} ( D_{\phi} ( {\vy}\^{t} \,\big\|\, {\vz}\^{t}) + D_{\phi} ( {\vz}\^{t+1} \,\big\|\, {\vy}\^{t})).
     \] 
     Substituting $G_{T+1}( {\vz}\^{T + 1}) \leq - \langle {\vy}, {\Ut}\^{T+1} \rangle + \frac{1}{\eta} \phi(\vy)$ and $G_{1}( {\vz}\^{1}) = \frac{1}{\eta} \phi( {\vy}\^{1})$, we complete the proof as follows,
     \[
     \sum_{t = 1}^{T} \mleft \langle {\vy} - {\vy}\^{t}, {\ut}\^{t} \mright \rangle & \leq \frac{1}{\eta} \phi( {\vy}) - \frac{1}{\eta} \phi( {\vy}\^{1}) + \sum_{t = 1}^{T} \mleft \langle {\vz}\^{t+1} - {\vy}\^{t}, {\ut}\^{t} - {\ut}\^{t - 1}\mright \rangle \\ & \qquad - \frac{1}{\eta} \sum_{t = 1}^{T} \mleft( D_{\phi}( {\vy}\^{t} \,\big\|\, {\vz}\^{t}) + D_{\phi}( {\vz}\^{t+1} \,\big\|\, {\vy}\^{t}) \mright). 
     \]
\end{proof}

\begin{lemma} \label{lemma:opt_bregman}
    Let $F: \Omega \rightarrow \bbR$ be a convex function defined on the compact set $\Omega$. The minimizer $ {\vz}^* = \argmin_{ {\vz} \in \Omega} F( {\vz})$ satisfies:
    \[
        F( {\vz}^*) \leq F( {\vz}) - D_F( {\vz} \|  {\vz}^*) \qquad \forall \vz \in \Omega,
    \]
    where $D_F$ is the Bregman divergence associated with the function $F$.
\end{lemma}

\begin{proof}
    From the definition of the Bregman divergence, we know,
    \[
        F( {\vz}^*) = F( {\vz}) - \langle \nabla F( {\vz}^*),  {\vz} -  {\vz}^* \rangle - D_F( {\vz} \|  {\vz}^*).
    \]
    Applying the first-order optimality condition for $ {\vz}^*$, we conclude
    \[
        F( {\vz}^*) \leq F( {\vz}) - D_F( {\vz} \|  {\vz}^*).
    \]
    This completes the proof.
\end{proof}

\begin{lemma}\label{lemma:u_correction}
    Suppose that $\| \nut\^{t} \|_\infty \leq 1$ holds for all $t \in [T]$ according to \Cref{assumption:bounded}. Then, the following inequality is satisfied:
    \[
        \| \ut\^{t} - \ut\^{t - 1} \|_{\infty}^2 \leq 6 \|\nut\^t - \nut\^{t - 1}\|_{\infty}^2 + 4 \| \vx\^t - \vx\^{t-1} \|_1^2.
    \]
\end{lemma}

\begin{proof}
    From the definition, we have:
    \[
    \hspace{-0.5cm}
        \| {\ut}\^{t} -   {\ut}\^{t - 1} \|_{\infty}^2 & = \| ({\nut}\^{t} - \langle \nut\^t, \vx\^t \rangle \vec{1} ) -   ({\nut}\^{t-1} - \langle \nut\^{t-1}, \vx\^{t-1} \rangle \vec{1} ) \|_{\infty}^2 \\
        & \leq \mleft(\|\nut\^t - \nut\^{t - 1}\|_{\infty} + \| \langle \nut\^t, \vx\^t \rangle \vec{1} - \langle \nut\^{t-1}, \vx\^{t-1} \rangle \vec{1} \|_{\infty} \mright)^2 \numberthis{eq:u_correction1} \\
        & =  \mleft(\|\nut\^t - \nut\^{t - 1}\|_{\infty} + \big| \langle \nut\^t, \vx\^t \rangle  - \langle \nut\^{t-1}, \vx\^{t-1} \rangle \big|\mright)^2 \\
        & \leq 2 \|\nut\^t - \nut\^{t - 1}\|_{\infty}^2 + 2 \big| \langle \nut\^t, \vx\^t \rangle  - \langle \nut\^{t-1}, \vx\^{t-1} \rangle \big|^2 \numberthis{eq:u_correction2} \\
        & \leq  2 \|\nut\^t - \nut\^{t - 1}\|_{\infty}^2 + 2 \big| \big(\langle \nut\^t, \vx\^t \rangle  - \langle \nut\^t, \vx\^{t-1} \rangle\big) + \big(\langle \nut\^t, \vx\^{t-1} \rangle - \langle \nut\^{t-1}, \vx\^{t-1} \rangle\big) \big|^2 \\
        & \leq  2 \|\nut\^t - \nut\^{t - 1}\|_{\infty}^2 + 4 \big|\langle \nut\^t, \vx\^t - \vx\^{t-1}\rangle\big|^2 + 4 \big| \langle \nut\^t - \nut\^{t-1}, \vx\^{t-1} \rangle \big|^2 \numberthis{eq:u_correction3} \\
        & \leq  2 \|\nut\^t - \nut\^{t - 1}\|_{\infty}^2 + 4 \| \vx\^t - \vx\^{t-1} \|_1^2 + 4 \| \nut\^t - \nut\^{t-1} \|_{\infty}^2 \numberthis{eq:u_correction4} \\
        & = 6 \|\nut\^t - \nut\^{t - 1}\|_{\infty}^2 + 4 \| \vx\^t - \vx\^{t-1} \|_1^2,
    \]
    where \eqref{eq:u_correction1} follows from the triangle inequality, \eqref{eq:u_correction2} and \eqref{eq:u_correction3} make use of Young’s inequality, and \eqref{eq:u_correction4} applies Hölder’s inequality.
\end{proof}

\begin{lemma} \label{lemma:correct_y}
    Let $\vz, \vy$ be two arbitrary vectors in $\Omega = (0, 1] \mathcal{X}$, then,
    \[
    \| \vz - \vy \|_1^2 \leq 2 \mleft| \vec 1^\top \vz - \vec 1 ^\top \vy \mright|^2 + 2 \| \frac{\vz}{1^\top \vz} - \frac{\vy}{1^\top \vy} \|_1^2.
    \]
\end{lemma}
\begin{proof}
    \[
    \| \vz - \vy \|_1 & = \sum_{\ind}^d \mleft| \vz[\ind] - \vy[\ind] \mright| \\
    & = \sum_{\ind}^d \mleft|  \vec 1^\top \vz \frac{\vz[\ind]}{ \vec 1^\top \vz} - \vec 1 ^\top \vy \frac{\vy[\ind]}{\vec 1 ^\top \vy}\mright|  \\
    & = \sum_{\ind}^d \mleft|  \vec 1^\top \vz \frac{\vz[\ind]}{ \vec 1^\top \vz} -  \vec 1^\top \vy \frac{\vz[\ind]}{ \vec 1^\top \vz} + \vec 1^\top \vy \frac{\vz[\ind]}{ \vec 1^\top \vz} - \vec 1 ^\top \vy \frac{\vy[\ind]}{\vec 1 ^\top \vy}\mright| \\
    & \leq \sum_{\ind}^d \mleft( \frac{\vz[\ind]}{ \vec 1^\top \vz} \mleft| \vec 1^\top \vz - \vec 1^\top \vy \mright| + \vec 1 ^\top \vy \mleft|  \frac{\vz[\ind]}{ \vec 1^\top \vz} - \frac{\vy[\ind]}{\vec 1 ^\top \vy}\mright| \mright) \numberthis{eq:beta_terms_with_lambda_trick1} \\
    & \leq \mleft| \vec 1^\top \vz - \vec 1^\top \vy \mright| + \| \frac{\vz}{1^\top \vz} - \frac{\vy}{1^\top \vy} \|_1,
    \]
    where \eqref{eq:beta_terms_with_lambda_trick1} follows by triangle inequality and the last line follows since $ \vy \in (0, 1]\mathcal{X}$. Next, by Young's inequality,
    \[
    \| \vz - \vy \|_1^2 \leq 2 \mleft| \vec 1^\top \vz - \vec 1 ^\top \vy \mright|^2 + 2 \| \frac{\vz}{1^\top \vz} - \frac{\vy}{1^\top \vy} \|_1^2.
    \]
    
\end{proof}

\subsection{RVU Bounds and Regret Analysis} \label{app:rvu}

\begin{restatetheorem}{theorem:nonnegative_rvu}
[Nonnegative RVU bound of \ours]
    Under \Cref{assumption:bounded}, by choosing a sufficiently small learning rate $\eta \leq \min\{ 3\gamma/80, \mu/(32\sqrt{2}) \}$\footnote{If $\psi$ is $\gamma$-LIL, it is sufficient to select $\eta \leq \min \{ 3\gamma(1)/80, 1/8, \mu/(32\sqrt{2})\}$, after which guarantees naturally follow.} and $\alpha \geq 4 \gamma + \mu$, the cumulative regret ($\reg\^T$) incurred by \ours up to time $T$ is bounded as follows:
    \[
        \max\{\reg\^T, 0\} = \tildereg\^T \leq 3 + \frac{1}{\eta} \mleft( \alpha \log T + \mathcal{R} \mright) + \eta \frac{48}{\mu} \sum_{t = 1}^{T-1} \|\nut\^{t+1} - \nut\^{t}\|_{\infty}^2  - \frac{1}{\eta} \frac{\mu}{64} \sum_{t = 1}^{T-1} \| {\vx}\^{t+1} -  {\vx}\^{t} \|_1^2.
    \]
\end{restatetheorem}
\begin{proof}
    For any choice of comparator $ {\vy} \in \Omega = (0,1]\mathcal{X}$, let $ {\vy}^\prime = \frac{T - 1}{T} {\vy} + \frac{1}{T} {\vy}\^{1} \in \Omega$. Define $ \lambda' = \vec 1^\top \vy$ and $\vx' = \vy / \vec 1^\top \vy $ such that ${\vy}^\prime = \lambdap \vx^\prime$. By direct calculations,
    \[
        \sum_{t = 1}^T  \mleft\langle {\vy} -  {\vy}\^{t},  {\ut}\^{t} \mright\rangle  & = \sum_{t = 1}^T \langle  {\vy} -  {\vy}^\prime,  {\ut}\^{t} \rangle + \sum_{t = 1}^T \langle  {\vy}^\prime -  {\vy}\^{t},  {\ut}\^{t} \rangle           \\
        & = \frac{1}{T} \sum_{t = 1}^T \langle  {\vy} -  {\vy}\^{1},  {\ut}\^{t} \rangle + \sum_{t = 1}^T \langle  {\vy}^\prime -  {\vy}\^{t},  {\ut}\^{t} \rangle \\
        & \leq 2 + \sum_{t = 1}^T \langle  {\vy}^\prime -  {\vy}\^{t},  {\ut}\^{t} \rangle,
    \]
    where the last line follows because of Hölder's inequality and $\|  {\ut}\^{t} \|_\infty \leq 2 \| \nut\^t \|_\infty \leq 2$, under \Cref{assumption:bounded}. 

    In turn, we need to upperbound the $\sum_{t = 1}^T \langle  {\vy}^\prime -  {\vy}\^{t},  {\ut}\^{t} \rangle$ term. By \Cref{lemma:basi_OFTRL},
    \[
    \sum_{t = 1}^{T} \mleft \langle  {\vy}' -  {\vy}\^{t},  {\ut}\^{t} \mright \rangle & \leq \underbrace{\frac{1}{\eta} \phi( {\vy'}) - \frac{1}{\eta} \phi( {\vy}\^{1})}_{(\textup{I})} + \underbrace{\sum_{t = 1}^{T} \mleft \langle  {\vz}\^{t+1} -  {\vy}\^{t},  {\ut}\^{t} -   {\ut}\^{t - 1}\mright \rangle}_{(\textup{II})} \\
    & \qquad \underbrace{- \frac{1}{\eta} \sum_{t = 1}^{T} \mleft( D_{\phi}( {\vy}\^{t} \,\big\|\,  {\vz}\^{t}) + D_{\phi}( {\vz}\^{t+1} \,\big\|\,  {\vy}\^{t}) \mright)}_{(\textup{III})}.
    \]
    For the term (I), following some calculations,
    \[
    (\textup{I}) & = \frac{1}{\eta} \mleft(\phi( {\vy'}) - \phi( {\vy}\^{1}) \mright) \\
    & = \frac{1}{\eta} \mleft( - \alpha \log (\vec{1}^\top \vy') + \psi (\frac{\vy'}{{\vec 1}^\top \vy'}) +  \alpha \log (\vec{1}^\top \vy\^1) - \psi (\frac{\vy\^1}{{\vec 1}^\top \vy\^1}) \mright) \\
    & \leq \frac{1}{\eta} \mleft( \alpha \log \mleft( \frac{{\vec 1}^\top \vy\^1}{{\vec{1}}^\top \vy'} \mright) + \psi(\vx') \mright) \\
    & \leq \frac{1}{\eta} \mleft(\alpha \log T + \mathcal{R} \mright), 
    \]
    where $\mathcal{R} = \argmax_{\vx \in \mathcal{X}} \psi(\vx)$.

    In sequel, using Hölder's and Young's inequalities, we conclude that term (II) is upper bounded by,
    \[
    (\textup{II}) & \leq \sum_{t = 1}^{T} \mleft \langle  {\vz}\^{t+1} -  {\vy}\^{t},  {\ut}\^{t} -   {\ut}\^{t - 1}\mright \rangle \\
        & \leq \sum_{t = 1}^{T} \| {\vz}\^{t+1} -  {\vy}\^{t}\|_{1} \cdot \| {\ut}\^{t} -   {\ut}\^{t - 1} \|_{\infty}  \\
        & \leq \sum_{t = 1}^{T} \mleft( \frac{\mu}{32 \eta} \| {\vz}\^{t+1} -  {\vy}\^{t}\|_{1}^2 + \frac{8}{\mu} \eta \| {\ut}\^{t} -   {\ut}\^{t - 1} \|_{\infty}^2 \mright) \\
        & \leq \sum_{t = 1}^{T} \mleft(\frac{\mu}{32 \eta} \| {\vz}\^{t+1} -  {\vy}\^{t}\|_{1}^2 + \frac{48}{\mu} \eta \|\nut\^t - \nut\^{t - 1}\|_{\infty}^2 + \frac{32}{\mu} \eta \| \vx\^t - \vx\^{t-1} \|_1^2 \mright) \numberthis{eq:pluging_u}, \\
        & \leq \sum_{t = 1}^{T} \mleft(\frac{\mu}{16 \eta} \mleft| \xi\^{t+1} - \lambda\^{t} \mright|^2 + \frac{\mu}{16 \eta} \| {\vtheta}\^{t+1} -  {\vx}\^{t}\|_{1}^2 + \frac{48}{\mu} \eta \|\nut\^t - \nut\^{t - 1}\|_{\infty}^2 + \frac{32}{\mu} \eta \| \vx\^t - \vx\^{t-1} \|_1^2 \mright)
    \]
    where we used \Cref{lemma:u_correction} in \eqref{eq:pluging_u} and \Cref{lemma:correct_y} in the last line.

    Subsequently for the term (III), by setting $a = \mu$ in \Cref{theorem:cruv},
    \[
    \hspace{-0.5cm}(\textup{III}) & = - \frac{1}{\eta} \sum_{t = 1}^{T} \mleft( D_{\phi}( {\vy}\^{t} \,\big\|\,  {\vz}\^{t}) + D_{\phi}( {\vz}\^{t+1} \,\big\|\,  {\vy}\^{t}) \mright) \\
    & \leq - \frac{1}{\eta} \sum_{t = 1}^{T} \mleft( (\gamma + \mu) D_{-\log} (\lambda\^t \; \| \; \xi\^t) + \frac{1}{4} D_{\psi}( {\vx}\^{t} \,\big\|\,  {\vtheta}\^{t}) + (\gamma + \mu) D_{-\log} (\xi\^{t+1} \; \| \; \lambda\^t)+ \frac{1}{4} D_{\psi}( {\vtheta}\^{t+1} \,\big\|\,  {\vx}\^{t}) \mright) \\
    & = - \frac{1}{4 \eta} \sum_{t = 1}^{T} \mleft( D_{\psi}( {\vx}\^{t} \,\big\|\,  {\vtheta}\^{t}) + D_{\psi}( {\vtheta}\^{t+1} \,\big\|\,  {\vx}\^{t})\mright) - \frac{\gamma + \mu}{\eta} \sum_{t = 1}^{T} \mleft( - \log(\frac{\xi\^{t + 1}}{\lambda\^t}) + \frac{\xi\^{t + 1}}{\lambda\^t} - 1 \mright) \\
    & \leq - \frac{1}{4 \eta} \sum_{t = 1}^{T} \mleft( D_{\psi}( {\vx}\^{t} \,\big\|\,  {\vtheta}\^{t}) + D_{\psi}( {\vtheta}\^{t+1} \,\big\|\,  {\vx}\^{t})\mright) - \frac{\gamma + \mu}{3 \eta} \sum_{t = 1}^{T} \mleft| \frac{\xi\^{t + 1}}{ \lambda\^t} - 1 \mright|^2 \numberthis{eq;log_breg_to_quadratic1} \\
    & \leq - \frac{1}{4 \eta} \sum_{t = 1}^{T} \mleft( D_{\psi}( {\vx}\^{t} \,\big\|\,  {\vtheta}\^{t}) + D_{\psi}( {\vtheta}\^{t+1} \,\big\|\,  {\vx}\^{t})\mright) - \frac{\gamma + \mu}{3 \eta} \sum_{t = 1}^{T} \mleft| \xi\^{t + 1}- \lambda\^t \mright|^2 \numberthis{eq;log_breg_to_quadratic2} \\
    & \leq - \frac{1}{8 \eta} \sum_{t = 1}^{T}  D_{\psi}( {\vtheta}\^{t+1} \,\big\|\,  {\vx}\^{t}) - \frac{\mu}{16 \eta} \sum_{t = 1}^{T} \mleft( \| {\vx}\^{t} -  {\vtheta}\^{t} \|_1^2 + \| {\vtheta}\^{t+1} -  {\vx}\^{t} \|_1^2 \mright) - \frac{\gamma + \mu}{3 \eta} \sum_{t = 1}^{T} \mleft| \xi\^{t + 1} - \lambda\^t \mright|^2 \numberthis{eq;breg_to_quadratic3} \\
    & \leq - \frac{1}{8 \eta} \sum_{t = 1}^{T}  D_{\psi}( {\vtheta}\^{t+1} \,\big\|\,  {\vx}\^{t}) - \frac{\mu}{16 \eta} \sum_{t = 1}^{T - 1} \mleft( \| {\vx}\^{t+1} -  {\vtheta}\^{t+1} \|_1^2 + \| {\vtheta}\^{t+1} -  {\vx}\^{t} \|_1^2 \mright) - \frac{\gamma + \mu}{3 \eta} \sum_{t = 1}^{T} \mleft| \xi\^{t + 1} - \lambda\^t \mright|^2 \\
    & \leq - \frac{\mu}{16 \eta} \sum_{t = 1}^{T}  \| {\vtheta}\^{t+1} -  {\vx}\^{t} \|_1^2 - \frac{\mu}{32 \eta} \sum_{t = 1}^{T-1}  \| {\vx}\^{t+1} -  {\vx}\^{t} \|_1^2 - \frac{\gamma + \mu}{3 \eta} \sum_{t = 1}^{T} \mleft| \xi\^{t + 1} - \lambda\^t \mright|^2, \numberthis{eq;breg_to_quadratic4}
    \]
    where \eqref{eq;log_breg_to_quadratic1} follows by mutiplicative stability of learning rates for proximal steps established in \Cref{theorem:mult_stable,theorem:mult_stable_local_lipsc} and the fact that
    \[
    - \log \omega + \omega - 1 \geq \frac{1}{3} (\omega - 1)^2 \qquad \forall \; \omega = \frac{\xi\^{t + 1}}{\lambda\^t} \in \mleft[\frac{1}{2}, \frac{3}{2} \mright]. 
    \]
    \eqref{eq;log_breg_to_quadratic2} follows simply by remembering that $\lambda\^t \in (0, 1]$, \eqref{eq;breg_to_quadratic3} follows by $\mu$-strong convexity of $\psi$ w.r.t. $\ell_1$ norm. And finally, \eqref{eq;breg_to_quadratic4} follows by Young's inequality.

    In turn, we assemble the complete picture by summing (II) and (III) terms,
    \[
    (\textup{II}) + (\textup{III}) & \leq \eta \frac{48}{\mu} \sum_{t = 1}^T  \|\nut\^t - \nut\^{t - 1}\|_{\infty}^2 - \frac{1}{\eta} \mleft( \frac{\mu}{32} - \frac{32}{\mu} \eta^2 \mright) \sum_{t = 1}^{T-1} \| {\vx}\^{t+1} -  {\vx}\^{t} \|_1^2 \\
    & \leq 1 + \eta \frac{48}{\mu} \sum_{t = 1}^{T-1} \|\nut\^{t+1} - \nut\^{t}\|_{\infty}^2  - \frac{1}{\eta} \mleft( \frac{\mu}{32} - \frac{32}{\mu} \eta^2 \mright) \sum_{t = 1}^{T-1} \| {\vx}\^{t+1} -  {\vx}\^{t} \|_1^2 \\
    & \leq 1 + \eta \frac{48}{\mu} \sum_{t = 1}^{T-1} \|\nut\^{t+1} - \nut\^{t}\|_{\infty}^2  - \frac{1}{\eta} \frac{\mu}{64} \sum_{t = 1}^{T-1} \| {\vx}\^{t+1} -  {\vx}\^{t} \|_1^2,
    \]
    where in the last line we used $\eta \leq \frac{\mu}{32 \sqrt{2}}$.
    Summing all terms,
    \[
    \sum_{t = 1}^T  \mleft\langle {\vy} -  {\vy}\^{t},  {\ut}\^{t} \mright\rangle & \leq 2 + (\textup{I}) + (\textup{II}) + (\textup{III}) \\
    & \leq 3 + \frac{1}{\eta} \mleft( \alpha \log T + \mathcal{R} \mright) + \eta \frac{48}{\mu} \sum_{t = 1}^{T-1} \|\nut\^{t+1} - \nut\^{t}\|_{\infty}^2  - \frac{1}{\eta} \frac{\mu}{64} \sum_{t = 1}^{T-1} \| {\vx}\^{t+1} -  {\vx}\^{t} \|_1^2.
    \]
\end{proof}

\begin{restateproposition}{prop:path_length}
    Under \Cref{assumption:bounded}, if all players adhere to the \ours algorithm with  a learning rate $\eta \leq \min\{ 3\gamma/80, \mu/(32\sqrt{2}), \mu/ (L n 32 \sqrt{6}) \}$\footnote{If $\psi$ is $\gamma$-LIL, it is sufficient to select $\eta \leq \min\{ 3\gamma(1)/80, 1/8, \mu/(32\sqrt{2}), \mu/ (L n 32 \sqrt{6})\}$ to ensure the same guarantees.}, the total path length is bounded as follows,
    \[
    \sum_{i = 1}^{n} \sum_{t = 1}^{T - 1} \|\vx\^{t+1}_i - \vx\^{t}_i \|_1^2 \leq \frac{128 \eta}{\mu} \mleft( 3 n + n \frac{\alpha \log T + \mathcal{R}}{\eta} \mright).
    \]
\end{restateproposition}
\begin{proof}
    \Cref{assumption:bounded} implies that,
    \[
    \| \nut\^{t + 1}_i - \nut\^{t}_i \|_{\infty}^2 & \leq L^2 \mleft(\sum_{i = 1}^{n} \| \vx\^{t+1}_i - \vx\^{t}_i \|_1 \mright)^2 \\
    & \leq L^2 n \sum_{i = 1}^{n} \| \vx\^{t+1}_i - \vx\^{t}_i \|_1^2,
    \]
    where the final line follows from Jensen's inequality. Subsequently, we combine this result with the nonnegative RVU bound on $\tildereg^T$ for the $i$th player, as stated in \Cref{theorem:nonnegative_rvu},
    \[
    \tildereg_i\^T & \leq 3 + \frac{1}{\eta} \mleft( \alpha \log T + \mathcal{R} \mright) + \eta \frac{48}{\mu} \sum_{t = 1}^{T-1} \|\nut\^{t+1} - \nut\^{t}\|_{\infty}^2  - \frac{1}{\eta} \frac{\mu}{64} \sum_{t = 1}^{T-1} \| {\vx}\^{t+1} -  {\vx}\^{t} \|_1^2 \\
    & \leq 3 + \frac{1}{\eta} \mleft( \alpha \log T + \mathcal{R} \mright) + \eta \frac{48}{\mu} L^2 n \sum_{j = 1}^{n} \sum_{t = 1}^{T-1} \| \vx\^{t+1} - \vx\^{t}\|_{\infty}^2  - \frac{1}{\eta} \frac{\mu}{64} \sum_{t = 1}^{T-1} \| {\vx}\^{t+1} -  {\vx}\^{t} \|_1^2.
    \]
    Summing the nonnegative regret over all the players $i \in [n]$,
    \[
    \sum_{i = 1}^{n}  \tildereg\^T_i & \leq 3 n +  \frac{n}{\eta} \mleft( \alpha \log T + \mathcal{R} \mright) + \mleft( \eta \frac{48 L^2 n^2}{\mu} - \frac{1}{\eta} \frac{\mu}{64} \mright) \sum_{j = 1}^{n} \sum_{t = 1}^{T-1} \| \vx\^{t+1} - \vx\^{t}\|_{\infty}^2 \\
    & \leq 3 n +  \frac{n}{\eta} \mleft( \alpha \log T + \mathcal{R} \mright) - \frac{1}{\eta} \frac{\mu}{128} \sum_{j = 1}^{n} \sum_{t = 1}^{T-1} \| \vx\^{t+1} - \vx\^{t}\|_{\infty}^2,
    \]
    where we used the choice that $\eta \leq \frac{\mu}{L n 32 \sqrt{6}}$.

    In turn, using the nonnegativity of regret, $0 \leq \tildereg^T_i$ for all $i \in [n]$, we deduce that,
    \[
    0 \leq 3 n +  \frac{n}{\eta} \mleft( \alpha \log T + \mathcal{R} \mright) - \frac{1}{\eta} \frac{\mu}{128} \sum_{j = 1}^{n} \sum_{t = 1}^{T-1} \| \vx\^{t+1} - \vx\^{t}\|_{\infty}^2,
    \]
    and hence,
    \[
    \sum_{j = 1}^{n} \sum_{t = 1}^{T-1} \| \vx\^{t+1} - \vx\^{t}\|_{\infty}^2 \leq \frac{128 \eta}{\mu} \mleft( 3 n + n \frac{\alpha \log T + \mathcal{R}}{\eta} \mright).
    \]
\end{proof}

\begin{restatetheorem}{theorem:regret_final_bound}[Main Theorem]
Under \Cref{assumption:bounded}, if all players $i \in [n]$ follow \ours with a $\gamma$-IL and $\mu$-strongly convex regularizer $\psi$, and choosing a small enough learning rate $\eta \leq \min\{ 3\gamma/80, \mu/(32\sqrt{2}), \mu/ (L n 32 \sqrt{6}) \}$\footnote{To maintain the same type of guarantees whenever $\psi$ is $\gamma$-LIL, it is sufficient to set $\eta \leq \min\{3\gamma/80, 1/8, \mu/(32\sqrt{2}), \mu/(32\sqrt{6}Ln)\}$, we formalize this result in \Cref{thm:main_ILI}.}, then the regret for each player $i \in [n]$ is bounded as follows:
\[
\reg_i\^T \leq 6 + 2 \frac{\alpha \log T + \mathcal{R}}{\mu} \max \mleft\{ \frac{80 \mu}{3\gamma}, 32 \sqrt{2}, L n 32 \sqrt{6} \mright\} = O ( L n \Gamma_\psi(d) \log T).
\]
where $\Gamma_\psi(d) = \gamma/\mu$ and the algorithm for each player $i \in [n]$ is adaptive to adversarial utilities, i.e., the regret that each player incurs is $\reg_i\^T = O(\sqrt{T \log d})$. 
\end{restatetheorem}
\begin{proof}
    Similar to the proof of \Cref{prop:path_length},
    \[
    \| \nut\^{t + 1}_i - \nut\^{t}_i \|_{\infty}^2 & \leq L^2 \mleft(\sum_{i = 1}^{n} \| \vx\^{t+1}_i - \vx\^{t}_i \|_1 \mright)^2 \\
    & \leq L^2 n \sum_{i = 1}^{n} \| \vx\^{t+1}_i - \vx\^{t}_i \|_1^2.
    \]
    Summing over $t$ from $1$ to $T-1$ yields in combination with \Cref{prop:path_length},
    \[
    \sum_{t = 1}^{T-1} \| \nut\^{t + 1}_i - \nut\^{t}_i \|_{\infty}^2 & \leq  L^2 n \sum_{t = 1}^{T-1} \sum_{i = 1}^{n} \| \vx\^{t+1}_i - \vx\^{t}_i \|_1^2 \\
    & \leq L^2 n^2 \frac{128 \eta}{\mu} \mleft( 3 +  \frac{\alpha \log T + \mathcal{R}}{\eta} \mright) \numberthis{eq:bound_on_beta_terms}.
    \]
    In sequel, by \Cref{prop:reg+,theorem:nonnegative_rvu} we get that,
    \[
    \reg\^T & \leq \tildereg\^T \\
    & \leq 3 + \frac{1}{\eta} \mleft( \alpha \log T + \mathcal{R} \mright) + \eta \frac{48}{\mu} \sum_{t = 1}^{T-1} \|\nut\^{t+1} - \nut\^{t}\|_{\infty}^2  - \frac{1}{\eta} \frac{\mu}{64} \sum_{t = 1}^{T-1} \| {\vx}\^{t+1} -  {\vx}\^{t} \|_1^2 \\
    & \leq  3 + \frac{1}{\eta} \mleft( \alpha \log T + \mathcal{R} \mright) + \eta \frac{48}{\mu} \sum_{t = 1}^{T-1} \|\nut\^{t+1} - \nut\^{t}\|_{\infty}^2 \\
    & \leq 3 + \frac{1}{\eta} \mleft( \alpha \log T + \mathcal{R} \mright) + \eta \frac{48}{\mu} \mleft(  L^2 n^2 \frac{128 \eta}{\mu} \mleft( 3 +  \frac{\alpha \log T + \mathcal{R}}{\eta} \mright) \mright) \numberthis{eq:plugin_bounded_beta} \\
    & \leq 3 + \frac{\mleft( \alpha \log T + \mathcal{R} \mright)}{\eta} + \eta^2 \frac{48 \times 128 L^2 n^2}{\mu^2} \mleft( 3 + \frac{\alpha \log T + \mathcal{R}}{\eta} \mright) \\
    & \leq 6 + 2 \frac{\mleft( \alpha \log T + \mathcal{R} \mright)}{\eta} \numberthis{eq:need_later_for_local},
    \]
    where \eqref{eq:plugin_bounded_beta} follows by \eqref{eq:bound_on_beta_terms} and the last line follows by our choice that
    \[
    \eta^2 \leq \frac{\mu^2}{L^2 n^2 48 \times 128}.
    \]
    We conclude the proof by recalling that $\eta \leq \min\{ 3\gamma/80, \mu/(32\sqrt{2}), \mu/ (L n 32 \sqrt{6})\}$ and plug-in this consideration into \eqref{eq:need_later_for_local}.
    \[
    \reg\^T & \leq 6 + 2 \frac{\mleft( \alpha \log T + \mathcal{R} \mright)}{\eta}  \\
    & \leq 6 + 2 (\alpha \log T + \mathcal{R}) \max \mleft\{ \frac{80}{3\gamma}, \frac{32 \sqrt{2}}{\mu}, \frac{L n 32 \sqrt{6}}{\mu}\mright\} \\
    & \leq 6 + 2 \frac{\alpha \log T + \mathcal{R}}{\mu} \max \mleft\{ \frac{80 \mu}{3\gamma}, 32 \sqrt{2}, L n 32 \sqrt{6} \mright\} \\
    & = O\mleft( (L n + 1) \frac{\alpha \log T + \mathcal{R}}{\mu} + \frac{\alpha \log T + \mathcal{R}}{\gamma} \mright), \\
    & = O ( L n \frac{\gamma \log T + \mathcal{R}}{\mu}), \\
    & = O( Ln \Gamma_\psi(d) \log T) & \textrm{ with } \Gamma_\psi(d) = \frac{\gamma}{\mu},
    \]
    where in the last line, we assumed $\gamma \geq \mu$\footnote{Otherwise, the regularizer can be scaled appropriately to satisfy this property. Notably, the leading term in the regret depends on $\Gamma_\psi = \gamma/\mu$, which remains invariant under such scalings.} and the term dependent on $T$ is the leading term.

    To establish the adversarial bound for each player $i \in [n]$, player $i$ simply checks if there exists a time $t \in [T]$ such that,
    \[
    \sum_{t = 1}^{T-1} \| \nut\^{t + 1}_i - \nut\^{t}_i \|_{\infty}^2 > L^2 n^2 \frac{128 \eta}{\mu} \mleft( 3 +  \frac{\alpha \log T + \mathcal{R}}{\eta} \mright).
    \]
    And, upon detecting such a condition, switches to any no-regret learning algorithm, such as Multiplicative Weights Update (\Hedge) \citep{cesa2006prediction,orabona2019modern}, ensuring a regret bound of $O(\sqrt{T \log d})$. This argument relies on the fact that if all players follow the \ours dynamics, then \eqref{eq:bound_on_beta_terms} must hold.
\end{proof}

\begin{theorem} [Main Theorem for Locally Intrinsically Lipschitz Regularizer] \label{thm:main_ILI}
    Under \Cref{assumption:bounded}, if all players $i \in [n]$ follow \ours with a $\gamma$-LIL and $\mu$-strongly convex regularizer $\psi$, and choosing a small enough learning rate $\eta \leq \min\{3\gamma(1)/80, 1/8, \mu/(32\sqrt{2}), \mu/(32\sqrt{6}Ln)\}$, then the regret for each player $i \in [n]$ is bounded as follows:
    \[
    \reg_i\^T \leq 6 + 2 \frac{\alpha \log T + \mathcal{R}}{\mu} \max \mleft\{ \frac{80 \mu}{3\gamma(1)}, 8\mu, 32 \sqrt{2}, L n 32 \sqrt{6} \mright\} = O ( L n \gamma(1) \log T).
    \]
    Moreover, the algorithm for each player $i \in [n]$ is adaptive to adversarial utilities, i.e., the regret that each player incurs is $\reg_i\^T = O(\sqrt{T \log d})$. 
\end{theorem}

\begin{proof}
    The proof for $\gamma$-LIL $\psi$ follows similar to the proof of \Cref{theorem:regret_final_bound} except that we consider the choice of $ \eta \leq \min\{ 3\gamma(1)/80, 1/8, \mu/(32\sqrt{2}), \mu/ (L n 32 \sqrt{6})\}$ for \eqref{eq:need_later_for_local} and infer that,
    \[
    \reg\^T =  O ( L n \gamma(1) \log T),
    \]
    assuming that $\mu \geq 1$, otherwise scale the regularizer $\psi$ accordingly.
\end{proof}

\section{Analysis of Social Regret}
\label{app:social_regret}

In this section, we show that the social regret $\sum_{i = 1}^n \reg_i^T$ for \ours converges at a fast rate that is constant in $T$. In the spirit of \citet{Syrgkanis15}, for this result it suffices to show that players running \ours have regrets that enjoy the RVU property. We state and briefly prove this result below. In comparison to the nonnegative RVU bound in \Cref{theorem:nonnegative_rvu}, the following has two differences. 

First, it upper-bounds the regret instead of the nonnegative regret; hence, the left-hand side of \eqref{eq:rvu_social} can be negative. Second, the right-hand side of \eqref{eq:rvu_social} does not depend on the time horrizon $T$. The key in the analysis is that, in contrast to \Cref{theorem:nonnegative_rvu}, since we do not need nonnegativity, we can restrict the comparator to the simplex $\Delta^d$ instead of $(0, 1]\Delta^d$.

\begin{restatetheorem} {theorem:rvu_simple}
(RVU bound of \ours)
    Under \Cref{assumption:bounded}, by choosing a sufficiently small learning rate $\eta \leq \min\{ 3\gamma/80, \mu/(32\sqrt{2}) \}$\footnote{If $\psi$ is $\gamma$-LIL, it is sufficient to select $\eta \leq \min \{ 3\gamma(1)/80, 1/8, \mu/(32\sqrt{2})\}$, after which guarantees naturally follow.} and $\alpha \geq 4 \gamma + \mu$, the cumulative regret ($\reg\^T$) incurred by \ours up to time $T$ is bounded as follows:
    \[
    \reg\^T \leq \frac{1}{\eta} \mathcal{R} + \eta \frac{48}{\mu} \sum_{t = 1}^{T-1} \|\nut\^{t+1} - \nut\^{t}\|_{\infty}^2  - \frac{1}{\eta} \frac{\mu}{64} \sum_{t = 1}^{T-1} \| {\vx}\^{t+1} -  {\vx}\^{t} \|_1^2, \numberthis{eq:rvu_social}
    \]
    where $\mathcal{R} = \argmax_{\vx \in \mathcal{X}} \psi(\vx)$.
\end{restatetheorem}
\begin{proof}
    For the regret we know that,
    \[
         \reg\^T & =  \max_{\vx^* \in \Delta^d } \sum_{t = 1}^{T} \langle \nut\^t, \vx^* \rangle - \langle \nut\^t, \vx\^{t} \rangle \\
         & = \max_{\vx^* \in \Delta^d } \sum_{t = 1}^{T} \langle \nut\^t, \vx^* \rangle - \langle \nut\^t, \vx\^{t} \rangle . \langle \vec 1_d, \vx^* \rangle. \\
         & = \max_{\vx^* \in \Delta^d } \sum_{t = 1}^{T} \langle \nut\^t, \vx^* \rangle - \mleft\langle \langle \nut\^t, \frac{\vy\^t}{\langle \vy\^t, {\vec 1}_d \rangle} \rangle \vec1_d, \vx^* \mright\rangle \\
         & = \max_{\vx^* \in \Delta^d } \sum_{t = 1}^{T} \langle \nut\^t - \langle \nut\^t, \frac{\vy\^t}{\langle \vy\^t, {\vec 1}_d \rangle} \rangle \vec1_d, \vx^* \rangle \\
         & = \max_{\vx^* \in \Delta^d } \sum_{t = 1}^{T} \langle \nut\^t - \langle \nut\^t, \frac{\vy\^t}{\langle \vy\^t, {\vec 1}_d \rangle} \rangle \vec1_d, \vx^* - \vy\^t \rangle \numberthis{eq:temp_rvu_regret_conversion}\\
         & = \max_{\vx^* \in \Delta^d } \sum_{t = 1}^{T} \langle \ut\^t, \vx^* - \vy\^t \rangle, \numberthis{eq:temp_rvu_regret_result}
    \]
    where in \eqref{eq:temp_rvu_regret_conversion}, similar to the proof of \Cref{prop:reg+}, we exploited the orthogonality condition $\vec \ut\^{t} \perp \vy\^{t}$ that was introduced by construction. Thus, according to \eqref{eq:temp_rvu_regret_result}, in the analysis of the regret it suffices to consider only comparators restricted to the simplex $\Delta^d$, while the actions $\vy\^t$ may belong to the lifted space $\Omega = (0, 1]\Delta^d$.

    For any choice of comparator $ {\vx} \in \Delta^d$, by \Cref{lemma:basi_OFTRL},
    \[
    \sum_{t = 1}^{T} \mleft \langle  {\vx} -  {\vy}\^{t},  {\ut}\^{t} \mright \rangle & \leq \underbrace{\frac{1}{\eta} \phi( {\vx}) - \frac{1}{\eta} \phi( {\vy}\^{1})}_{(\textup{I})} + \underbrace{\sum_{t = 1}^{T} \mleft \langle  {\vz}\^{t+1} -  {\vy}\^{t},  {\ut}\^{t} -   {\ut}\^{t - 1}\mright \rangle}_{(\textup{II})} \\
    & \qquad \underbrace{- \frac{1}{\eta} \sum_{t = 1}^{T} \mleft( D_{\phi}( {\vy}\^{t} \,\big\|\,  {\vz}\^{t}) + D_{\phi}( {\vz}\^{t+1} \,\big\|\,  {\vy}\^{t}) \mright)}_{(\textup{III})}.
    \]
    The proof for the terms (I) and (II) are exactly the same as the proof of \Cref{theorem:nonnegative_rvu}. For the term (I), following direct calculations,
    \[
    (\textup{I}) & = \frac{1}{\eta} \mleft(\phi( {\vx}) - \phi( {\vy}\^{1}) \mright) \leq \frac{1}{\eta} \phi( {\vx}) 
     \leq \frac{1}{\eta} \mleft(\alpha \log 1 + \psi(\vx) \mright) \leq \frac{1}{\eta} \mathcal{R},
    \]
    where $\mathcal{R} = \argmax_{\vx \in \mathcal{X}} \psi(\vx)$. In the calculations, we used the fact that comparator $\vx$ belongs to the simplex $\Delta^d$. This concludes the proof.
\end{proof}

We recall the following Theorem from \citet{Syrgkanis15}.

\begin{theorem}[Theorem 4 of \citet{Syrgkanis15}] \label{theorem:social_rvu}
    Suppose that players of the game follow an algorithm that satisfies RVU property with parameter $a, b$ and $c$ such that $b \leq c / (n - 1)^2$, and $\|.\| = \|.\|_1$. Then, $\sum_{i = 1}^n \reg_i^T \leq a n$.
\end{theorem}

Having \Cref{theorem:rvu_simple,theorem:social_rvu} at hand, we are ready to state and prove constant rate of the social regret of \ours.

\begin{restatetheorem}{theorem:constant_social}
    Under the bounded utilities assumption~\ref{assumption:bounded}, if all players $i \in [n]$ follow \ours with a $\gamma$-IL and $\mu$-strongly convex regularizer $\psi$, and choosing a small enough learning rate $\eta \leq \min\{ 3\gamma/80, \mu/(32\sqrt{2}), \mu/ (L n 32 \sqrt{6}) \}$, then the social regret is constant in time horizon $T$, i.e.,
    \[
    \sum_{i = 1}^n \reg_i^T \leq O\mleft( L \frac{n^2 \mathcal{R}}{\mu} + \frac{n \mathcal{R}}{\gamma} \mright),
    \]
    where $\mathcal{R} = \argmax_{\vx \in \mathcal{X}} \psi(\vx)$.
\end{restatetheorem}

\begin{proof}
    Since $\eta^2 \leq \frac{\mu^2}{48 \times 64 (n - 1)^2}$, by \Cref{theorem:rvu_simple}, we infer that \ours satisfies RVU property as required in \Cref{theorem:social_rvu}. Therefore,
    \[
    \sum_{i = 1}^n \reg_i^T & \leq \frac{n}{\eta} \mathcal{R} \\
    & \leq \max \mleft\{ \frac{80}{3\gamma}, \frac{32 \sqrt{2}}{\mu}, \frac{L n 32 \sqrt{6}}{\mu}\mright\} n \mathcal{R} \\
    & \leq \max \mleft\{ \frac{80 \mu}{3\gamma}, 32 \sqrt{2}, L n 32 \sqrt{6} \mright\} \frac{n}{\mu} \mathcal{R} \\
    & = O\mleft( (L n + 1) \frac{n \mathcal{R}}{\mu} + \frac{n \mathcal{R}}{\gamma} \mright).
    \]
\end{proof}

\section{Proof of Strong Concavity of the Learning Rate Control Problem} \label{app:proof_convexity}

For notational convenience, we use $\vx_\lambda$ to denote the iterate produced by \ours when the learning rate is set to $\lambda$, that is, in Formulation \ref{eq:ftrl_step_fixed_lambda},
\[
\vx_\lambda \defeq \argmax_{\vx \in \mathcal{X}} \mleft\{  \langle \lambda \at, \vx \rangle - \psi(\vx) \mright\}.
\]
Recall that
\[
\psi^*_{\mathcal{X}} (\lambda \at) = \max_{\vx \in \mathcal{X}} \mleft\{  \langle \lambda \at, \vx \rangle - \psi(\vx) \mright\}.
\]

The learning rate control objective, which is to be \emph{maximized}, can be rewritten using the new definition of $\vx_\lambda$ as
\[
f(\lambda) & \defeq \alpha \log \lambda + \psi_\mathcal{X}^*( \lambda \at) \\
& = \alpha \log \lambda + \langle \lambda \at, \vx_\lambda \rangle - \psi(\vx_\lambda).
\]
Using $\phi$ as defined in \eqref{eq:def_phi},
\[
\phi(\lambda \vx) = - \alpha \log(\lambda) + \psi(\vx), 
\]
we immediately infer that
\[
\phi(\lambda \vx_\lambda) = - f(\lambda) + \langle \lambda \at, \vx_\lambda \rangle. \numberthis{eq:identitiy_f}
\]

\begin{lemma} \label{lemma:convex_phi_lambda}
    The function $\phi(\vy)$ is a strongly convex function, and, for all $\vy_1,\vy_2 \in (0, 1]\Delta^d$,
    \begin{enumerate}
        \item $\phi(\vy_2)  \geq \phi(\vy_1) + \nabla \phi(\vy_1)^\top (\vy_2 - \vy_1) +  \frac{1}{2} \frac{(\alpha - \gamma)}{\lambda_1} (\lambda_2 - \lambda_1)^2 $
        \item $\phi(t \vy_1 + (1 - t) \vy_2) \leq t \phi(\vy_1) + (1 - t) \phi(\vy_2) - \frac{(\alpha - \gamma)}{2} t (t - 1) (\lambda_2 - \lambda_1)^2$,
    \end{enumerate}
    where $\lambda_1 = \vec 1^\top \vy_1, \lambda_2 = \vec 1^\top \vy_2$. 
\end{lemma}
\begin{proof}
    Consider $\vy', \vy \in (0, 1]\mathcal{X}$. Let $\lambda = \vec 1^\top \vy$ and $\lambda' = \vec 1^\top \vy'$. By expansion of Bregman divergences and
    \Cref{prop:convex_phi}, we infer,
    \[
    \langle \nabla \phi(\vy') - \nabla \phi(\vy), \vy' - \vy \rangle & = D_\phi(\vy' \; \| \; \vy) + D_\phi(\vy \; \| \; \vy') \\
    & \geq (\alpha - \gamma) \mleft[ \frac{\lambda'}{\lambda} + \frac{\lambda}{\lambda'} - 2 \mright] \\
    & = \frac{(\alpha - \gamma)}{\lambda' \lambda} (\lambda' - \lambda)^2 \\
    & \geq \frac{(\alpha - \gamma)}{\lambda} (\lambda' - \lambda)^2, \numberthis{eq:monotone_temp2}
    \]
    for all $\vy', \vy \in (0, 1]\mathcal{X}$.

    Now, we start by fundamental theorem of calculus along the line segment from $\vy_1$ to $\vy_2$,
    \[
    \phi(\vy_2) & = \phi(\vy_1) + \int_{0}^{1} \nabla \phi(\vy_1 + t (\vy_2 - \vy_1))^\top (\vy_2 - \vy_1) dt \\
    & = \phi(\vy_1) + \nabla \phi(\vy_1)^\top (\vy_2 - \vy_1) +  \int_{0}^{1} \frac{1}{t} \mleft(\nabla \phi(\vy_1 + t (\vy_2 - \vy_1)) - \nabla \phi(\vy_1) \mright)^\top \mleft (t (\vy_2 -  \vy_1) \mright) dt \\
    & \geq \phi(\vy_1) + \nabla \phi(\vy_1)^\top (\vy_2 - \vy_1) +  \int_{0}^{1} \frac{(\alpha - \gamma)}{\lambda_1} t (\lambda_2 - \lambda_1)^2  dt \\
    & = \phi(\vy_1) + \nabla \phi(\vy_1)^\top (\vy_2 - \vy_1) +  \frac{1}{2} \frac{(\alpha - \gamma)}{\lambda_1} (\lambda_2 - \lambda_1)^2. \\
    & \geq \phi(\vy_1) + \nabla \phi(\vy_1)^\top (\vy_2 - \vy_1) +  \frac{(\alpha - \gamma)}{2} (\lambda_2 - \lambda_1)^2, \numberthis{eq:tighten_strong_convex1}
    \]
    where we used \eqref{eq:monotone_temp2}. Thus, part $(1)$ is proved.

    For part $(2)$, set $\vz = t \vy_1 + (1 - t) \vy_2$. By applying \eqref{eq:tighten_strong_convex1} twice,
    \[
    t \phi(\vy_1) \geq t \phi(\vz) + t (1 - t) \nabla \phi(\vz)^\top (\vy_1 - \vy_2) + t (1 - t)^2 \frac{(\alpha - \gamma)}{2} (\lambda_2 - \lambda_1)^2
    \]
    and,
    \[
    (1 - t) \phi(\vy_2) \geq (1 - t) \phi(\vz) + t (1 - t) \nabla \phi(\vz)^\top (\vy_2 - \vy_2) + (1 - t) t^2 \frac{(\alpha - \gamma)}{2} (\lambda_2 - \lambda_1)^2
    \]
    Summing these equations conclude the proof.
\end{proof}

Given the development in \Cref{lemma:convex_phi_lambda}, we are now ready to state and prove the strong convexity of the learning-rate control problem. The proof proceeds by showing that the curvature of the $\phi$-regularizer along the $\lambda$-ray translates into strong concavity of the learning-rate control problem.

\begin{theorem}\label{thm:strong_convexity_simple}
    For a $\gamma$-intrinsic Lipschitz regularizer $\psi$, the dynamic learning-rate control optimization problem \eqref{eq:dynamic_learning_rate},  
    \[
    f(\lambda) \defeq \alpha \log \lambda + \psi_\mathcal{X}^*(\lambda \at),
    \]
    is $(\alpha - \gamma)$-strongly concave.    
\end{theorem}

\begin{proof}
We note that since the domain $\mathcal{X}$ is closed and compact, it suffices to check only for \emph{midpoint} strong convexity, i.e., convex combinations with a coefficient of $1/2$. 

The key idea is that, by \Cref{lemma:convex_phi_lambda}, we already know that the strong convexity of $\psi(\vy)$ is characterized by a quadratic term in the learning rates, $\vec 1^\top \vy$, and we aim to transfer this curvature to establish the strong concavity of $f(\lambda)$ in $\lambda$ directly.

Define
\[
\vy_\lambda & \defeq \lambda \vx_\lambda \\
\vy_{\lambda'} & \defeq {\lambda'} \vx_{\lambda'}.
\]
By the mid-point version of \Cref{lemma:convex_phi_lambda}, we know that,
\[
\phi\mleft(\frac{1}{2} \vy_\lambda + \frac{1}{2} \vy_{\lambda'}\mright) \leq \frac{1}{2} \phi(\vy_\lambda) + \frac{1}{2} \phi(\vy_{\lambda'}) - \frac{\alpha - \gamma}{8} (\lambda' - \lambda)^2.
\]
Using the Identity \ref{eq:identitiy_f}, we have,
\[
\phi(\vy_\lambda) = - f(\lambda) + \langle \lambda \at, \vx_\lambda \rangle \qquad \phi(\vy_{\lambda'}) = - f(\lambda') + \langle \lambda' \at, \vx_{\lambda'} \rangle
\]
and so,
\[
\phi \mleft(\frac{1}{2} \vy_\lambda + \frac{1}{2} \vy_{\lambda'} \mright) \leq - \frac{1}{2} f(\lambda)  - \frac{1}{2} f(\lambda') + 
\mleft \langle \at, \frac{1}{2} \vy_{\lambda} + \frac{1}{2} \vy_{\lambda'} \mright \rangle -\frac{\alpha - \gamma}{8} (\lambda' - \lambda)^2. \numberthis{eq:concave_lambda_temp}
\]
Rearranging,
\[
\frac{1}{2} f(\lambda)  + \frac{1}{2} f(\lambda') \leq - \phi\mleft(\frac{1}{2} \vy_\lambda + \frac{1}{2} \vy_{\lambda'}\mright) + \mleft \langle \at, \frac{1}{2} \vy_{\lambda} + \frac{1}{2} \vy_{\lambda'} 
\mright \rangle -\frac{\alpha - \gamma}{8} (\lambda' - \lambda)^2.
\]

In sequel, let
\[
\overline{\lambda} \defeq \frac{1}{2} \lambda + \frac{1}{2} \lambda', \qquad \overline{\vx} \defeq \frac{\lambda}{\lambda + \lambda'} \vx_{\lambda} + \frac{\lambda'}{\lambda + \lambda'} \vx_{\lambda'} \in \Delta^d.
\]
Note that,
\[
\frac{1}{2} \vy_{\lambda} + \frac{1}{2} \vy_{\lambda'} = \overline{\lambda} \overline{\vx}.
\]
Thus,
\[
- \phi\mleft(\frac{1}{2} \vy_{\lambda} + \frac{1}{2} \vy_{\lambda'} \mright) + \mleft\langle \at, \frac{1}{2} \vy_{\lambda} + \frac{1}{2} \vy_{\lambda'} \mright \rangle & = \alpha \log \overline{\lambda} - \psi(\overline{\vx}) + \langle \overline{\lambda} \at, \overline{\vx} \rangle \\
& = \mleft[ \alpha \log \overline{\lambda} - \psi(\vx_{\overline{\lambda}}) + \langle \overline{\lambda} \at, \vx_{\overline{\lambda}} \rangle \mright] + \mleft(  \langle \overline{\lambda} \at, \overline{\vx} \rangle - \psi(\overline{\vx}) \mright) - \mleft(  \langle \overline{\lambda} \at, \vx_{\overline{\lambda}} \rangle - \psi(\vx_{\overline{\lambda}}) \mright) \\
& = f(\overline{\lambda}) + \mleft(  \langle \overline{\lambda} \at, \overline{\vx} \rangle - \psi(\overline{\vx}) \mright) - \mleft(  \langle \overline{\lambda} \at, \vx_{\overline{\lambda}} \rangle - \psi(\vx_{\overline{\lambda}}) \mright).
\]
Now, we remember the crucial observation that,
\[
\langle \overline{\lambda} \at, \overline{\vx} \rangle - \psi(\overline{\vx}) \leq \langle \overline{\lambda} \at, \vx_{\overline{\lambda}} \rangle - \psi(\vx_{\overline{\lambda}}), 
\]
since by definition $\vx_{\overline{\lambda}}$ maximizes the objective $\langle \overline{\lambda} \at, \vx \rangle - \psi(\vx)$. Hence,
\[
- \phi\mleft(\frac{1}{2} \vy_{\lambda} + \frac{1}{2} \vy_{\lambda'} \mright) + \mleft\langle \at, \frac{1}{2} \vy_{\lambda} + \frac{1}{2} \vy_{\lambda'} \mright \rangle \leq f(\overline{\lambda}).
\]

At this point, substituting into \eqref{eq:concave_lambda_temp},
\[
\frac{1}{2} f(\lambda) + \frac{1}{2} f(\lambda') \leq f(\overline{\lambda}) - \frac{\alpha - \gamma}{8} (\lambda' - \lambda)^2,
\]
which is the definition of strong concavity of $f$ as a function of $\lambda$.
\end{proof}

When the intrinsic Lipschitz regularizer $\psi$ is additionally Legendre, we can establish a stronger form of strong concavity, showing that the learning-rate control problem exhibits high curvature, up to multiplicative factors as large as $\log \lambda$. We formalize this result in the following, starting with useful lemmas.

For convenience, we use $\overline{f}$ to denote the negative of the learning rate control objective $-f$, where
\[
\overline{f}(\lambda) \defeq - f(\lambda) = - \alpha \log \lambda - \psi_{\mathcal{X}}^* (\lambda \at).
\]

And similar to the proof of \Cref{thm:strong_convex_learning_rate_control}, define
\[
\vy_\lambda & \defeq \lambda \vx_\lambda \\
\vy_{\lambda'} & \defeq {\lambda'} \vx_{\lambda'}.
\]

The following lemma shows that, when the regularizer $\psi$ is Legendre, the monotonicity of the derivative of the learning-rate control problem interestingly coincides with the curvature of the composite regularizer $\phi$.

\begin{lemma}\label{lemma:identity_learning_rate-control}
    If the regularizer $\psi$ is Legendre, then the objective of the learning-rate control problem satisfies the following identity,
    \[
    \mleft( \overline{f}'(\lambda') - \overline{f}'(\lambda) \mright) \mleft(\lambda' - \lambda\mright) = D_{\phi} (\vy_{\lambda'} \; \big|\big| \; \vy_{\lambda}) + D_{\phi} (\vy_{\lambda} \; \big|\big| \; \vy_{\lambda'}). 
    \]
\end{lemma}
\begin{proof}
    By the KKT conditions for the objective of the optimization problem of the convex conjugate,  
    \[
    \psi^*_{\mathcal{X}} (\lambda \at) = \max_{\vx \in \mathcal{X}} \left\{ \langle \lambda \at, \vx \rangle - \psi(\vx) \right\},
    \]
    we have
    \[
    \nabla \psi(\vx_\lambda) = \lambda \at + \eta_\lambda \vec{1} \numberthis{eq:high_curve_temp0},
    \]
    for some scalar $\eta_\lambda \in \mathbb{R}$. We do not need to take into account the constraints associated with the positivity of $\vx_\lambda$, since $\psi$ is Legendre and the optimum $\vx_\lambda$ lies in the interior of the simplex.
    Thus,
    \[
    \langle \nabla \psi(\vx_\lambda), \vx_\lambda \rangle & = \lambda \langle \at, \vx_\lambda \rangle + \eta_\lambda \langle \vec 1, \vx_\lambda \rangle  \\
    & = \lambda \langle \at, \vx_\lambda \rangle + \eta_\lambda, \numberthis{eq:high_curve_temp1}
    \]
    since $\vx_\lambda \in \Delta^d$ and $\vec 1^\top \vx_\lambda = 1$. Also, element-wise from \eqref{eq:high_curve_temp0}, we get
    \[
    \frac{\partial \psi(\vx)}{\partial \vx[i]} = \lambda \at[i] + \eta_\lambda. \numberthis{eq:high_curve_temp2}
    \]
    
    Subsequently, recall the definition of the composite regularizer $\phi$,
    \[
    \phi(\vy) = -\log \lambda + \psi\mleft(\frac{\vy}{\lambda}\mright) 
    \]
    where $\lambda = \vec 1^\top \vy$ and $\vx = \vy / \lambda$.
    Taking derivatives,
    \[
    \frac{\partial \lambda}{\partial \vy[i]} = 1 \quad \textrm{ and } \quad  \frac{\partial \vx[i]}{\partial \vy[j]} = \frac{\delta_{i,j} \lambda - \vy [j]}{\lambda^2} = \frac{1}{\lambda} \mleft(  \delta_{i,j} - \vx[j] \mright),
    \]
    where $\delta_{i,j}$ is the delta dirac function, 
    \[
    \delta_{i,j} = \begin{cases}
        1 & \textrm{if }i = j \\
        0 & \textrm{otherwise}
    \end{cases}.
    \]
    Taking derivatives from $\phi$, we get
    \[
    \frac{\partial \phi (\vy) }{\partial \vy[i]} & = - \frac{\alpha}{\lambda} + \sum_{j = 1}^d \frac{\partial \psi(\vx)}{\partial \vx[j]} \frac{1}{\lambda} \mleft( \delta_{i, j} - \vx_j \mright) \\
    & = \frac{1}{\lambda} \mleft(- \alpha +  \frac{\partial \psi(\vx)}{\partial \vx[i]} - \langle \nabla \psi(\vx), \vx \rangle  \mright).
    \]
    Instantiating with $\vy_\lambda$ and incorporating \eqref{eq:high_curve_temp1} and \eqref{eq:high_curve_temp2}, 
    \[
    \frac{\partial \phi (\vy_\lambda) }{\partial \vy_\lambda [i]} & = \frac{1}{\lambda} \mleft( -\alpha + \lambda \at[i] + \eta_\lambda - \lambda \langle \at, \vx_\lambda \rangle - \eta_\lambda \mright) \\
    & = \at[i] - \mleft(\frac{\alpha}{\lambda} + \langle \at, \vx_\lambda \rangle \mright).
    \]
    Therefore, in vector form, we have
    \[
    \nabla \phi(\vy_\lambda) = \at - \mleft( \frac{\alpha}{\lambda} + \langle \at, \vx_\lambda \rangle \mright) \vec 1. \numberthis{eq:high_curve_temp3}
    \]

    In turn, by Danskin's theorem for the learning-rate control problem, we have  
    \[
    f'(\lambda) = \frac{\alpha}{\lambda} + \langle \at, \vx_\lambda \rangle.
    \]
    Thus, back into \eqref{eq:high_curve_temp3},
    we get,
    \[
    \nabla \phi(\vy_\lambda) & = \at - f'(\lambda) \vec 1 \\
    & = \at + \overline{f}'(\lambda) \vec 1. \numberthis{eq:high_curve_temp4}
    \]

    Now, let us consider the symmetric Bregman summation formula and substitute \eqref{eq:high_curve_temp4} into its formula to obtain
    \[
    D_{\phi} (\vy_{\lambda'} \; \big|\big| \; \vy_\lambda ) + D_{\phi} (\vy_{\lambda} \; \big|\big| \; \vy_\lambda' ) & = \langle \nabla \phi(\vy_{\lambda'}) - \nabla \phi(\vy_\lambda) , \vy_{\lambda'} - \vy_{\lambda} \rangle \\
    & = \langle \mleft( \at + \overline{f}'(\lambda') \vec 1 \mright) - \mleft( \at + \overline{f}'(\lambda) \vec 1 \mright), \vy_{\lambda'} - \vy_{\lambda} \rangle \\
    & = \mleft( \overline{f}'(\lambda')  - \overline{f}'(\lambda)  \mright) \langle \vec 1, \vy_{\lambda'} - \vy_{\lambda} \rangle \\
    & = \mleft( \overline{f}'(\lambda') - \overline{f}'(\lambda) \mright) \mleft(\lambda' - \lambda\mright),
    \]
    where in the last line, we used the fact that
    \[
    \sum_{\ind = 1}^d \vy_{\lambda'} = \lambda' \quad \textrm{ and } \quad \sum_{\ind = 1}^d \vy_{\lambda} = \lambda,
    \]
    and this completes the proof.
\end{proof}

We are now ready to state and prove the high-curvature property of the dynamic learning-rate control problem when the regularizer $\psi$ is additionally Legendre. Although the proof is technically very different, it follows the same high-level idea as the proof of \Cref{thm:strong_convexity_simple}: the curvature of the $\phi$-regularizer along the $\lambda$-ray translates into high curvature of the learning-rate control problem.

\begin{theorem} \label{thm:strong_convex_learning_rate_control_legendre}
Whenever the $\gamma$-(locally) intrinsic Lipschitz regularizer $\psi$ is Legendre, the dynamic learning-rate control objective \eqref{eq:dynamic_learning_rate},  
\[
f(\lambda) \defeq \alpha \log \lambda + \psi_\mathcal{X}^*(\lambda \at),
\]
is strongly concave and exhibits high curvature, as  
\[
f''(\lambda) \leq -\frac{\alpha - \gamma}{\lambda^2}.
\]
\end{theorem}
\begin{proof}
    Define
    \[
    \omega \defeq \frac{\lambda'}{\lambda},
    \]
    and 
    \[
    S(\lambda, \omega) & \defeq \mleft( \overline{f}'(\lambda') - \overline{f}'(\lambda) \mright) \mleft(\lambda' - \lambda\mright) \\
    & = D_{\phi} (\vy_{\lambda'} \; \big|\big| \; \vy_{\lambda}) + D_{\phi} (\vy_{\lambda} \; \big|\big| \; \vy_{\lambda'}).
    \]

    For a fixed $\lambda$, define
    \[
    h(\omega) \defeq S(\lambda, \omega) - (\alpha - \gamma) \mleft(\omega + \frac{1}{\omega} - 2 \mright).
    \]
    It is easy to observe that,
    \[
    h(1) & = S(\lambda, 1) - (\alpha - \gamma) \mleft(1 + \frac{1}{1} - 2 \mright) \\
    & = \mleft( \overline{f}'(\lambda ) - \overline{f}'(\lambda) \mright) \mleft(\lambda  - \lambda\mright) - 0 \\
    & = 0.
    \]
    By \Cref{prop:convex_phi,prop:local_convex_phi}, we know that when the regularizer $\psi$ is $\gamma$-IL, $\omega = 1$ is the global minimizer of $h$, and when $\psi$ is $\gamma$-LIL, $\omega = 1$ is a local minimizer of $h$, where the local neighborhood is characterized by \eqref{eq:local_intr_locality} in \Cref{prop:local_convex_phi}. Moreover, the Legendre property of the regularizer $\psi$ implies that $h$ is twice differentiable. Thus, we have  
    \[
    \frac{d h}{d \omega}(1) = 0 \quad \text{and} \quad \frac{d^2 h}{d \omega^2}(1) \geq 0. \numberthis{eq:goal_h}
    \]
    To avoid confusion with derivatives with respect to $\lambda$, define the auxiliary functions
    \[
    A(\omega) &\defeq \overline{f}'(\lambda \omega) - \overline{f}'(\lambda), \\
    B(\omega) &\defeq \lambda (\omega - 1), \\
    C(\omega) &\defeq S(\lambda, \omega), \\
    R(\omega) &\defeq \omega + \frac{1}{\omega} - 2.
    \]
    By direct calculations, it is easy to observe that
    \[
    A(1) = 0, B(1) = 0, \quad \textrm{ and } \quad \frac{d A(\omega)}{d \omega} = \lambda \overline{f}''(\lambda \omega) \quad, \frac{d B(\omega)}{d \omega} = \lambda. \numberthis{eq:deriva_eval}
    \]
    It follows further that
    \[
    C(\omega) & = A(\omega) B(\omega) \\
    \frac{d C(1)}{d \omega} & = \frac{d A(1)}{d \omega} B(1) + A(1) \frac{d B(1)}{d \omega} = 0 \\
    \frac{d^2 C(1)}{d \omega^2} & = \frac{d^2 A(1)}{d \omega^2} B(1) + 2  \frac{d A(1)}{d \omega} \frac{d B(1)}{d \omega} + A(1) \frac{d^2 B(1)}{d \omega^2} = 2 \lambda^2 \overline{f}''(\lambda),
    \]
    and
    \[
    \frac{d R(\omega)}{d \omega} & = 1 - \frac{1}{\omega^2} \Rightarrow \frac{d R(1)}{d \omega} = 0 \\
    \frac{d^2 R(\omega)}{d \omega^2} & = \frac{2}{\omega^3} \Rightarrow \frac{d^2 R(1)}{d \omega^2} = 2.
    \]
    In this new notation,
    \[
    h(\omega) & = C(\omega) - (\alpha - \gamma) R(\omega) \\
    \frac{d h(\omega)}{d \omega} & =  \frac{d C(\omega)}{d \omega} - (\alpha - \gamma) \frac{d R(\omega)}{d \omega} \\
    \frac{d^2 h(\omega)}{d \omega^2} & =  \frac{d^2 C(\omega)}{d \omega^2} - (\alpha - \gamma) \frac{d^2 R(\omega)}{d \omega^2}.
    \]
    Thus,
    \[
    \frac{d h(1)}{d \omega} = \frac{d C(1)}{d \omega} - (\alpha - \gamma) \frac{d R(1)}{d \omega},
    \]
    as expected in \eqref{eq:goal_h} and moreover again by \eqref{eq:goal_h},
    \[
    0 \leq  \frac{d^2 h(1)}{d^2 \omega} & = \frac{d^2 C(1)}{d^2 \omega} - (\alpha - \gamma) \frac{d^2 R(1)}{d^2 \omega} \\
    & = 2 \lambda^2 \overline{f}''(\lambda) - 2 (\alpha - \gamma).
    \]
    Hence,
    \[
    \overline{f}''(\lambda) = - f''(\lambda) \geq \frac{\alpha - \gamma}{\lambda^2}, 
    \]
    and this establishes the result.
\end{proof}

To conclude, we present the following theorem summarizing the results of this section.

\begin{restatetheorem}{thm:strong_convex_learning_rate_control}
The dynamic learning-rate control optimization problem \eqref{eq:dynamic_learning_rate},  
\[
f(\lambda) \defeq \alpha \log \lambda + \psi_\mathcal{X}^*(\lambda \at),
\]
\begin{itemize}
    \item is $(\alpha - \gamma)$-strongly concave, differentiable and admits high curvature locally,  
    \[
    f(\lambda') \leq f(\lambda) + f'(\lambda)(\lambda' - \lambda) 
    - \frac{\alpha - \gamma}{2 \max \{\lambda',\lambda\} ^2} (\lambda' - \lambda)^2,
    \]  
    whenever the regularizer $\psi$ is $\gamma$-intrinsic Lipschitz.

    \item is $(\alpha - \gamma)$-strongly concave, twice differentiable, and admits high curvature globally,  
    \[
    - \frac{\alpha}{\lambda^2} \leq f''(\lambda) \leq -\frac{\alpha - \gamma}{\lambda^2},
    \]
    whenever the regularizer $\psi$ is $\gamma$-(locally) intrinsic Lipschitz and additionally Legendre.
\end{itemize}
\end{restatetheorem}
\begin{proof}

    For the first part, by invoking \Cref{thm:strong_convexity_simple}, consider the function  
    \[
    f_1(\lambda) \defeq \gamma \log \lambda + \psi_\mathcal{X}^*(\lambda \at). \numberthis{eq:temp_f1_def}
    \]
    It follows that $f_1(\lambda)$ is concave (since $\gamma - \gamma = 0$), and thus  
    \[
    f_1(\lambda') \leq f_1(\lambda) + \vec{s}_1(\lambda)(\lambda' - \lambda), \numberthis{eq:concav_f_1}
    \]
    for any subgradient $\vec{s}_1(\lambda) \in \partial f_1$, where we know that  
    \[
    \vec{s}_1(\lambda) = \frac{\gamma}{\lambda} + \vec{g}(\lambda), \numberthis{eq:s_grad_definition}
    \]
    and the subgradient $\vec{g}(\lambda) \in \partial \psi_{\mathcal{X}}^*$, by Danskin's theorem~\citep{danskin2012theory,bertsekas1971control}, arises from the set
    \[
    \partial \psi_{\mathcal{X}}^* & = \operatorname{conv}\left\{ \langle \at, \vx_\lambda \rangle \;\middle|\; 
    \vx_\lambda \in \argmax_{\vx \in \mathcal{X}} \left\{ \lambda \langle \at, \vx \rangle - \psi(\vx) \right\} \right\} \\
    & = \mleft\{ \langle \at, \vx_\lambda \rangle \mright\},
    \]
    since the solution $\vx_\lambda$ to the optimization problem  
    \[
    \vx_\lambda \in \argmax_{\vx \in \mathcal{X}} \left\{ \lambda \langle \at, \vx \rangle - \psi(\vx) \right\}
    \]
    is unique, because $\psi$ is strongly convex. Thus, $\vec{g}(\lambda) = \langle \at, \vx_\lambda \rangle$, and $f(\lambda)$ is differentiable with  
    \[
    f'(\lambda) = \frac{\gamma}{\lambda} +  \langle \at, \vx_\lambda \rangle.
    \]

    Observe that
    \[
    f(\lambda') & = (\alpha - \gamma) \log \lambda' + f_1(\lambda') \\
    & \leq (\alpha - \gamma) \log \lambda' + f_1(\lambda) + f'(\lambda)(\lambda' - \lambda) \\
    & \leq (\alpha - \gamma) \mleft( \log \lambda + \frac{1}{\lambda} (\lambda' - \lambda) - \frac{1}{2 \xi^2} (\lambda' - \lambda)^2 \mright) + f_1(\lambda) + f'(\lambda)(\lambda' - \lambda) ,
    \]
    for some $\xi$ between $\lambda$ and $\lambda'$, by Taylor’s remainder theorem. Incorporating \eqref{eq:s_grad_definition} and \eqref{eq:temp_f1_def}, we get
    \[
    f(\lambda') & \leq (\alpha - \gamma) \mleft( \log \lambda + \frac{1}{\lambda} (\lambda' - \lambda) - \frac{1}{ 2 \max \{\lambda',\lambda\}^2} (\lambda' - \lambda)^2 \mright) + f_1(\lambda) + \mleft( \frac{\gamma}{\lambda} + \vec g(\lambda) \mright) (\lambda' - \lambda) \\
    & = (\alpha - \gamma) \mleft( \log \lambda + \frac{1}{\lambda} (\lambda' - \lambda) - \frac{1}{ 2 \max \{\lambda',\lambda\}^2} (\lambda' - \lambda)^2 \mright) + \gamma \log \lambda + \psi_\mathcal{X}^*(\lambda \at) + \mleft( \frac{\gamma}{\lambda} + \vec g(\lambda) \mright) (\lambda' - \lambda) \\
    & = \mleft( \alpha \log \lambda + \psi_\mathcal{X}^*(\lambda \at) \mright) + \mleft(\frac{\alpha}{\gamma} + \vec g(\lambda) \mright) (\lambda' - \lambda) - \frac{1}{ 2 \max \{\lambda',\lambda\}^2} (\lambda' - \lambda)^2 \\
    & = f(\lambda) + \vec f'(\lambda) (\lambda' - \lambda) - \frac{1}{ 2 \max \{\lambda',\lambda\}^2} (\lambda' - \lambda)^2,
    \]
    which finishes the proof for the first part.

    The proof of the upper bound for the second part is given in \Cref{thm:strong_convex_learning_rate_control_legendre}. For the lower bound, by recalling the convexity of the conjugate $\psi_{\mathcal{X}}^*$,  
    \[
    f''(\lambda) = - \frac{\alpha}{\lambda^2} + \frac{d^2 \psi_{\mathcal{X}}^* (\lambda \at)}{d^2 \lambda} \geq - \frac{\alpha}{\lambda^2},
    \]
    we conclude the proof.

\end{proof}

\section{Proof of Self-Concordance of the Learning Rate Control Problem for Legendre Regularizers} \label{proof_selfconcordance}

Throughout this section, we discuss properties of the Learning Rate Control Problem when the regularizer $\psi : \Delta^d_\circ \rightarrow \mathbb{R}$ is Legandre and third order differentiable. 

Consider the simplex $\Delta$, with interior $\Delta^\circ$. We define its tangent space $\mathcal{T}$ as,

\[
\mathcal{T} \defeq \mleft\{ h \in \mathbb{R}^d \; \big| \; \vec1^\top h = 0 \mright\},
\]

and the orthogonal projection matrix into the tangent space $\mathcal{T}$ as,

\[
\vP \defeq \vec I - \frac{1}{d} \vec1 {\vec1}^\top.
\]

Similar to \Cref{app:proof_convexity}, for a fixed regret vector $\at$,  we use $\vx_\lambda$ to denote the iterate produced by \ours when the learning rate is set to $\lambda$, that is, 
\[
\vx_\lambda \defeq \argmax_{\vx \in \mathcal{X}} \mleft\{ \langle \lambda \at, \vx \rangle - \psi(\vx) \mright\}.
\]
For ease of notation, define the function $g$ as a function of the learning rate $\lambda$,
\[
g(\lambda) \defeq \psi^*_{\mathcal{X}} (\lambda \at) = \max_{\vx \in \mathcal{X}}  \mleft\{ \langle \lambda \at, \vx \rangle - \psi(\vx) \mright\} \numberthis{eq:definition_g}
\]

Because the regularizer $\psi$ is Legendre, the maximizer $\vx_\lambda \in \Delta^d_\circ $ is unique and belongs to the interior of the simplex and it is (3rd order) smooth in $\lambda$. 

Define the projected into the simplex version of Hessian and Third-order tensor (third derivative) of the regularizer $\psi$ as,
\[
\vH(\vx) \defeq \nabla^2 \psi(\vx) \big|_{\mathcal{T}^2} = \vP \nabla^2 \psi(\vx) \vP,
\]
and 
\[
\vT(\vx) \defeq \nabla^3 \psi(\vx) \big|_{\mathcal{T}^3} = \nabla^3 \psi(\vx) [\vP \cdot, \vP \cdot, \vP \cdot].
\]

Finally, for ease of notation, throughout this section we denote derivatives of variables with respect to $\lambda$ by placing dots over the variables.
For example, we use
\[
\dot{\vx}_{\lambda} = \frac{d \vx_{\lambda}}{d \lambda} = \mleft[\frac{d \vx_{\lambda}[1]}{d \lambda}, \frac{d \vx_{\lambda}[2]}{d \lambda}, \ldots,  \frac{d \vx_{\lambda}[d]}{d \lambda}\mright]^\top.
\]

\begin{lemma} \label{lemma:g_derivatives}
    Given a strongly convex Legendre and third order differentiable regularizer $\psi$, for the convex conjugate function $g(\lambda) \defeq \psi_{\mathcal{X}}^* (\lambda \vx)$ as a function of  $\lambda$, we have,
    \begin{itemize}[itemsep=0.4em]
        \item $g''(\lambda) = \dot{\vx}_{\lambda}^\top \vH(\vx_\lambda) \dot{\vx}_{\lambda} = \dfrac{1}{\lambda^2} \vvv(\vx_{\lambda})^\top \vH(\vx_\lambda) \vvv(\vx_{\lambda}) = \frac{1}{\lambda^2} \vvv(\vx_{\lambda})^\top \nabla^2 \psi(\vx_\lambda) \vvv(\vx_{\lambda})$
        \item $g'''(\lambda) = - \vT(\vx_\lambda) [\dot{\vx}_{\lambda}, \dot{\vx}_{\lambda}, \dot{\vx}_{\lambda}] = -\dfrac{1}{\lambda^3} \vT(\vx_\lambda) [\vvv(\vx_{\lambda}), \vvv(\vx_{\lambda}), \vvv(\vx_{\lambda})] = -\dfrac{1}{\lambda^3} \nabla^3 \psi(\vx_\lambda) [\vvv(\vx_{\lambda}), \vvv(\vx_{\lambda}), \vvv(\vx_{\lambda})]$, 
    \end{itemize}
    where $\dot{\vx}_{\lambda}$ follows the projected Newton step at point $\vx_\lambda$ (with $1/\lambda$ as learning rate),
    \[
        \dot{\vx}_{\lambda} = \frac{1}{\lambda} \vvv(\vx_{\lambda}) \in \mathcal{T},
    \]
    with $\vvv(\vx_{\lambda}) \defeq \vH(\vx_\lambda)^{-1} \vP \nabla \psi(\vx_\lambda) $.
\end{lemma}
\begin{proof}
    Consider the optimization problem of the conjugate,
    \[
    \vx_\lambda \defeq \argmax_{\vx \in \mathcal{X}} \mleft\{ \langle \lambda \at, \vx \rangle - \psi(\vx) \mright\}.
    \]
    We form the Lagrangian with parameter $\eta$ (The unique optimum lies in the interior of the simplex, so there is no need to account for the positivity constraints.),
    \[
    \mathcal{L}(\vx, \eta; \lambda) = \lambda \langle \at, \vx \rangle - \psi (\vx) + \eta ( 1 - \vec 1^\top \vx), 
    \]
    with the constraint $1^\top \vx_\lambda = 1$.

    Denote the optimal Lagrange variable $\eta$ associated with fixed $\lambda$ by $\eta_\lambda$. Then, by first order optimality,
    \[
    \nabla \psi(\vx_\lambda) = \lambda \at + \eta_\lambda \vec 1. \numberthis{eq:self_temp_first_order}
    \]
    Projecting \eqref{eq:self_temp_first_order} into the simplex with projection matrix $\vP$ yields
    \[
    \vP \nabla \psi(\vx_\lambda) = \lambda \vP \at, \numberthis{eq:self_temp_first_order_P}
    \]
    since $\vP \vec 1 = 0$.

    Taking derivatives with respect to $\lambda$, once and twice, from the equality constraint of the simplex $1^\top \vx_\lambda = 1$, we infer that
    \[
    \vec 1^\top \dot{\vx}_{\lambda} & = 0 \Rightarrow  \dot{\vx}_{\lambda} \in \mathcal{T} \numberthis{eq:velocity_in_the_tangent} \\
    \vec 1^\top \ddot{\vx}_{\lambda} & = 0 \Rightarrow  \ddot{\vx}_{\lambda} \in \mathcal{T}. \numberthis{eq:acc_in_the_tangent}
    \]
    Taking another derivatives from \eqref{eq:self_temp_first_order}, we obtain that
    \[
    \nabla^2 \psi(\vx_\lambda) \dot{\vx}_\lambda = \at + \dot{\eta}_\lambda \vec 1. \numberthis{eq:self_temp_second_order}
    \]
    Again by projecting \eqref{eq:self_temp_second_order} back to the simplex, 
    \[
    \vP \nabla^2 \psi(\vx_\lambda)  \dot{\vx}_\lambda = \vP \at. \numberthis{eq:self_temp_second_order_P}
    \]
    Recalling that $\dot{\vx}_\lambda \in \mathcal{T}$, we know that $\dot{\vx} = \vP \dot{\vx}$. Hence, we can rewrite \eqref{eq:self_temp_second_order_P} as,
    \[
    \vH(\vx_\lambda) \dot{\vx}_\lambda & = \mleft(\vP \nabla^2 \psi(\vx) \vP \mright) \dot{\vx}_\lambda \\
    & = \vP \nabla^2 \psi(\vx_\lambda) \dot{\vx}_\lambda \\
    & = \vP \at. \numberthis{eq:x_dot_formula}
    \]
    Since $\vH(\vx_\lambda)$ is an invertible matrix on the simplex (the regularizer $\psi$ is convex), we can reformulate this equation in terms of $\dot{\vx}_\lambda$ as
    \[
    \dot{\vx}_\lambda & = \vH(\vx_\lambda)^{-1} \vP \at \\
    & = \frac{1}{\lambda} \vH(\vx_\lambda)^{-1} \vP \nabla \psi(\vx_\lambda) \\
    & = \frac{1}{\lambda} \vvv(\vx) \in \mathcal{T}. \numberthis{eq:x_dot_formula_inverse}
    \]
    In sequel, we calculate derivatives of the function $g(\lambda)$. By Danskin's theorem,
    \[
    g'(\lambda) = \langle \at, \vx_\lambda \rangle. 
    \]
    Taking another derivative w.r.t. $\lambda$,
    \[
    g''(\lambda) & = \langle \at, \dot{\vx}_\lambda \rangle \\
    & = \langle \at, \vP \dot{\vx}_\lambda \rangle \numberthis{eq:g_2_temp1}\\
    & = \langle \vP^\top \at, \dot{\vx}_\lambda \rangle  \\
    & = \langle \vP \at, \dot{\vx}_\lambda \rangle \numberthis{eq:g_2_temp3} \\
    & = \langle \vH(\vx_\lambda) \dot{\vx}_\lambda, \dot{\vx}_\lambda \rangle \numberthis{eq:g_2_temp4} \\
    & = \dot{\vx}_\lambda^\top \vH(\vx_\lambda) \dot{\vx}_\lambda, \numberthis{eq:g_2_final}
    \]
    where \eqref{eq:g_2_temp1} follows since $\dot{\vx}_\lambda \in \mathcal{T}$ and hence $\vP \dot{\vx}_\lambda = \dot{\vx}_\lambda$, \eqref{eq:g_2_temp3} follows since $\vP^\top = \vP$, and finally \eqref{eq:g_2_temp4} follows by the formula in \eqref{eq:x_dot_formula_inverse}.

    Taking another derivative from \eqref{eq:self_temp_second_order} w.r.t. $\lambda$, by chain rule we get,
    \[
    \nabla^3 \psi(\vx_\lambda) [\dot{\vx}_\lambda, \dot{\vx}_\lambda, \cdot] + \nabla^2 \psi(\vx_\lambda) \ddot{\vx}_\lambda = \ddot{\eta}_\lambda \vec 1.
    \]
    Again, by projecting back to the simplex with the projection matrix $\vP$,
    \[
    \vP \nabla^3 \psi(\vx_\lambda) [\dot{\vx}_\lambda, \dot{\vx}_\lambda, \cdot] + \vP \nabla^2 \psi(\vx_\lambda) \ddot{\vx}_\lambda = 0, \numberthis{eq:third_order_temp1}
    \]
    as $\vP \vec 1 = 0$. 

    Now, recalling that $\ddot{\vx}_\lambda \in \mathcal{T}$ and substituting $\vP = \ddot{\vx}_\lambda$ into \eqref{eq:third_order_temp1} implies that
    \[
    \vP \nabla^3 \psi(\vx_\lambda) [\dot{\vx}_\lambda, \dot{\vx}_\lambda, \cdot] + \vP \nabla^2 \psi(\vx_\lambda) \vP \ddot{\vx}_\lambda = 0,
    \]
    and thus,
    \[
    \vP \nabla^3 \psi(\vx_\lambda) [\dot{\vx}_\lambda, \dot{\vx}_\lambda, \cdot] + \vH(\vx_\lambda) \ddot{\vx}_\lambda = 0. \numberthis{eq:third_order_final_temp}
    \]

    Now, we take another derivative w.r.t. $\lambda$ from the two sides of $g''(\lambda) = \langle \at, \dot{\vx}_\lambda \rangle$.
    \[
    g'''(\lambda) & = \langle \at, \ddot{\vx}_\lambda \rangle \\
    & = \langle \nabla^2 \psi(\vx_\lambda) \dot{\vx}_\lambda - \dot{\eta}_\lambda \vec 1. , \ddot{\vx}_\lambda \rangle \numberthis{eq:third_order_g_temp1}\\
    & = \langle \nabla^2 \psi(\vx_\lambda) \dot{\vx}_\lambda, \ddot{\vx}_\lambda \rangle \numberthis{eq:third_order_g_temp2} \\
    & = \langle \nabla^2 \psi(\vx_\lambda) \vP \dot{\vx}_\lambda, \vP \ddot{\vx}_\lambda \rangle \numberthis{eq:third_order_g_temp3}\\
    & = \langle \vP^\top \nabla^2 \psi(\vx_\lambda) \vP \dot{\vx}_\lambda,  \ddot{\vx}_\lambda \rangle \\
    & = \langle \vP \nabla^2 \psi(\vx_\lambda) \vP \dot{\vx}_\lambda,  \ddot{\vx}_\lambda \rangle \numberthis{eq:third_order_g_temp4} \\
    & = \dot{\vx}_\lambda^\top \vH(\vx_\lambda) \ddot{\vx}_\lambda, \numberthis{eq:third_order_g_temp5}
    \]
    where \eqref{eq:third_order_g_temp1} is implied by \eqref{eq:self_temp_second_order}, \eqref{eq:third_order_g_temp2} a consequence of the fact that since $\ddot{\vx}_\lambda \in \mathcal{T}$, $\langle \ddot{\vx}_\lambda, \vec 1 \rangle = 0 $,  \eqref{eq:third_order_g_temp3} follows by $\vP \dot{\vx}_\lambda = \dot{\vx}_\lambda$ and finally \eqref{eq:third_order_g_temp4} results from $\vP^\top = \vP$.

    Next, incorporating \eqref{eq:third_order_final_temp} into \eqref{eq:third_order_g_temp5}, we get that
    \[
    g'''(\lambda) & = -  \dot{\vx}_\lambda^\top \vP \nabla^3 \psi(\vx_\lambda) [\dot{\vx}_\lambda, \dot{\vx}_\lambda, \cdot] \\
    & = - \nabla^3 \psi(\vx_\lambda) [\vP \dot{\vx}_\lambda, \vP \dot{\vx}_\lambda, \vP \dot{\vx}_\lambda] \numberthis{eq:third_order_g_temp6} \\
    & = - \vT(\vx_\lambda) [\dot{\vx}_\lambda, \dot{\vx}_\lambda, \dot{\vx}_\lambda] \numberthis{eq:g_3_final}, 
    \]
    where \eqref{eq:third_order_g_temp6} follows by $\vP \vx_\lambda = \vx$ since $\vx_\lambda \in \mathcal{T}$.

    In the end, placing $ \dot{\vx}_\lambda = \frac{1}{\lambda} \vvv(\vx_{\lambda})$ from \eqref{eq:x_dot_formula_inverse} into \eqref{eq:g_2_final} and \eqref{eq:g_3_final} concludes the proof.
\end{proof}

Now, we derive simplified formulation for the velocity $\vvv(\vx_{\lambda}) = \vH(\vx_\lambda)^{-1} \vP \nabla \psi(\vx_\lambda) $ defined in \Cref{lemma:g_derivatives} in the following lemma.

\begin{lemma} \label{lemma:v_formula}
    Given a strongly convex Legendre and third order differentiable regularizer $\psi$, we have
    \begin{enumerate}
        \item[a)] The term $\vvv(\vx_\lambda)$ attains the following formulation as a function of the regularizer $\psi$,
        \[
        \vvv(\vx_\lambda) = \mleft( \mleft[\nabla^2 \psi(\vx_\lambda) \mright]^{-1} - \dfrac{\mleft[\nabla^2 \psi(\vx_\lambda) \mright]^{-1} \vec 1 \vec 1^\top \nabla^2 \psi(\vx_\lambda)}{\vec 1^\top \mleft[\nabla^2 \psi(\vx_\lambda) \mright]^{-1} \vec 1} \mright) \nabla \psi (\vx_\lambda),
        \]
        \item[b)] and when the regularizer $\psi$ is additionally separable\footnote{With a slight abuse of notation, we use the same symbol $\psi$ to denote each element of the separable regularizer.},
        \[
        \psi(\vx) = \sum_{\ind = 1}^{d} \psi(\vx[\ind]),
        \]
        for each component $\ind \in [d]$,
        \[
        \vvv(\vx_\lambda) [\ind] = \frac{\psi'(\vx_\lambda[\ind]) - m_{\psi} (\vx_\lambda)}{\psi''(\vx_\lambda[\ind])},
        \]
        where $ m_{\psi} (\vx_\lambda)$ is an average of $\psi'$, weighted by $1/{\psi''}$, i.e.,
        \[
        m_{\psi} (\vx_\lambda) \defeq \dfrac{\displaystyle\sum_{\ind = 1}^d \dfrac{\psi'(\vx_\lambda[\ind])}{\psi''(\vx_\lambda[\ind])}}{\displaystyle\sum_{\ind = 1}^d \dfrac{1}{\psi''(\vx_\lambda[\ind])}}.
        \]
    \end{enumerate}
\end{lemma}
\begin{proof}
    By definition in the \Cref{lemma:g_derivatives}, $\vvv(\vx_{\lambda}) \in \mathcal{T}$ is the unique solution to the linear system $ \vP \nabla^2 \psi(\vx_\lambda) \vvv(\vx_{\lambda}) = \vP \nabla \psi(\vx_\lambda) $ and since $\vP \vec 1 = 0$, this is equivalent to 
    \[
    \nabla^2 \psi(\vx_\lambda) \vvv(\vx_{\lambda}) = \nabla \psi(\vx_\lambda) + \eta \vec 1, \numberthis{eq:elementwise_temp1}
    \]
    for some scalar number $\eta \in \mathbb{R}$.

    Writing $\vvv(\vx_{\lambda}) = \mleft[ \nabla^2 \psi(\vx_\lambda) \mright]^{-1} \mleft( \nabla \psi(\vx_\lambda) + \eta \vec 1 \mright)$ and enforcing $\vvv(\vx_{\lambda}) \in \mathcal{T}$ gives,
    \[
    0 = \vec 1^\top \vvv(\vx_\lambda) = \vec 1^\top \mleft[ \nabla^2 \psi(\vx_\lambda) \mright]^{-1} \nabla \psi(\vx_\lambda) + \eta \vec 1^\top  \mleft[ \nabla^2 \psi(\vx_\lambda) \mright]^{-1}  \vec 1.
    \]
    Thus, by solving for $\eta$, we get
    \[
    \eta = \frac{1^\top \mleft[ \nabla^2 \psi(\vx_\lambda) \mright]^{-1} \nabla \psi(\vx_\lambda) }{\vec 1^\top  \mleft[ \nabla^2 \psi(\vx_\lambda) \mright]^{-1}  \vec 1},
    \]
    and the proof for part a) is concluded.

    For part (b), starting from \eqref{eq:elementwise_temp1} componentwise, we know that for each $\ind \in [d]$, we have
    \[   \psi''(\vx[\ind]).\vvv(\vx_\lambda) [\ind] = \psi'(\vx_\lambda[\ind]) + \eta,
    \]
    Thus,
    \[
    \vvv(\vx_\lambda) [\ind] = \frac{\psi'(\vx[\ind]) + \eta}{\psi''(\vx_\lambda[\ind])}. \numberthis{eq:elementwise_temp2}
    \]
    We still must enforce $\vvv(\vx_\lambda) \in \mathcal{T}$, i.e., $\sum_{\ind = 1}^d \vvv(\vx_\lambda)[\ind] = 0$. Summing over elements of $\vvv(\vx_\lambda)$,
    \[
    0 = \sum_{\ind = 1}^d \vvv(\vx_\lambda)[\ind] = \sum_{\ind = 1}^d \frac{\psi'(\vx_\lambda[\ind])}{\psi''(\vx_\lambda[\ind])} + \eta \sum_{\ind = 1}^d \frac{1}{\psi''(\vx_\lambda[\ind])}.
    \]
    Thus,
    \[
    \eta = - m_{\psi} (\vx_\lambda) = -\dfrac{\displaystyle\sum_{\ind = 1}^d \dfrac{\psi'(\vx_\lambda[\ind])}{\psi''(\vx_\lambda[\ind])}}{\displaystyle\sum_{\ind = 1}^d \dfrac{1}{\psi''(\vx_\lambda[\ind])}}.
    \]
    Substituting back into \eqref{eq:elementwise_temp2}, we get
    \[
    \vvv(\vx_\lambda) [\ind] = \frac{\psi'(\vx_\lambda[\ind]) - m_{\psi} (\vx_\lambda)}{\psi''(\vx_\lambda[\ind])},
    \]
    and the proof is concluded.
\end{proof}

Now, with the preliminary derivations in place in \Cref{lemma:g_derivatives,lemma:v_formula}, we are ready to characterize the regularizers $\psi$ that give rise to a self-concordant learning-rate control problem~\eqref{eq:dynamic_learning_rate}. We refer to these Legendre regularizers as \emph{log-dominated}. We use this term since, as shown in \Cref{thm:log_dom_self_con}, the learning-rate control problem induced by these regularizers is self-concordant, and its second- and third-order derivatives are governed by the $\alpha \log \lambda$ term (See \Cref{thm:strong_convex_learning_rate_control,lemma:g_derivatives}).

\begin{definition}\label{def:log_dominated}
A strongly convex, third-order differentiable Legendre regularizer $\psi$ is said to be $\kappa$-log-dominated if there exists a universal constant $\kappa$ (possibly depending on $d$) such that  
\[
\mleft| \nabla^3 \psi(\vx)\big[\vvv(\vx), \vvv(\vx), \vvv(\vx)\big] \mright| \leq \kappa,
\]
where $\vvv(\vx)$ denotes the projected Newton step, whose closed form is given in \Cref{lemma:v_formula}.
\end{definition}

\begin{restatetheorem}{thm:log_dom_self_con}
    When the $\gamma$-(L)IL regularizer $\psi$ is $\kappa$-log-dominated, then the dynamic learning rate control objective (\ref{eq:dynamic_learning_rate}),
    \[
    f(\lambda) = \alpha \log \lambda + \psi_\mathcal{X}^*( \lambda \at).
    \]
    is $M$-generalized self-concordant, i.e.,
    \[
    \mleft| f'''(\lambda) \mright|^2 \leq M \mleft| f''(\lambda) \mright|^3,
    \]
    where 
    \[
    M = \frac{\mleft(2 \alpha + \kappa\mright)^2}{\mleft( \alpha - \gamma \mright)^3}.
    \]
\end{restatetheorem}

\begin{proof}
    Taking derivatives of $f(\lambda)$ with respect to $\lambda$, we obtain  
    \[
    f'(\lambda) = \frac{\alpha}{\lambda} + \frac{d}{d \lambda}\,\psi_\mathcal{X}^*(\lambda \at) 
    = \frac{\alpha}{\lambda} + g'(\lambda),
    \]
    in the terminology of \eqref{eq:definition_g}. Further, differentiating more gives  
    \[
    f''(\lambda) = -\frac{\alpha}{\lambda^2} + g''(\lambda) \\
    f'''(\lambda) = \frac{2 \alpha}{\lambda^3} + g'''(\lambda).
    \]
    Since $\psi$ is Legendre, by application of \Cref{thm:strong_convex_learning_rate_control_legendre},
    \[
      \mleft| f''(\lambda) \mright| \geq \frac{\alpha - \gamma}{\lambda^2}, \numberthis{eq:self_con_temp1}
    \]
    and by \Cref{lemma:g_derivatives,def:log_dominated} and triangle inequality,
    \[
    \mleft| f'''(\lambda) \mright| & = \mleft| \frac{2 \alpha}{\lambda^3} + g'''(\lambda) \mright| \\
    & = \mleft| \frac{2 \alpha}{\lambda^3} - \dfrac{1}{\lambda^3} \nabla^3 \psi(\vx_\lambda) [\vvv(\vx_{\lambda}), \vvv(\vx_{\lambda}), \vvv(\vx_{\lambda})] \mright| \\
    & \leq \frac{2 \alpha}{\lambda^3} + \frac{1}{\lambda^3} \mleft| \nabla^3 \psi(\vx_\lambda) [\vvv(\vx_{\lambda}), \vvv(\vx_{\lambda}), \vvv(\vx_{\lambda})] \mright|  \\
    & \leq \frac{2 \alpha + \kappa}{\lambda^3}. \numberthis{eq:self_con_temp2}
    \]
    Putting \eqref{eq:self_con_temp1} and \eqref{eq:self_con_temp2} together,
    \[
    \mleft| f'''(\lambda) \mright|^2 & \leq \mleft( \frac{2 \alpha + \kappa}{\lambda^3} \mright)^2 \\
    & = M \mleft( \frac{\alpha - \gamma}{\lambda^2} \mright)^3 \\
    & \leq M | f''(\lambda) |^3,
    \]
    where
    \[
    M = \frac{\mleft(2 \alpha + \kappa\mright)^2}{\mleft( \alpha - \gamma \mright)^3},
    \]
    and proof is concluded.
\end{proof}

\subsection{Characterization of Log-dominance and Self-concordance Parameters of Different Legendre Regularizers} \label{app:find_kappa}

In this section, we state and prove log-dominance for different (locally) intrinsic Lipschitz and Legendre regularizers, noted as in \Cref{table:regularizers_self_concordant}. Since these regularizers are separable, we used the simplified formula of $\vvv(\vx)$ written in \Cref{lemma:v_formula} for calculation of $\nabla^3 \psi(\vx) [\vvv(\vx), \vvv(\vx), \vvv(\vx)] $. In terms, 
\[
\nabla^3 \psi(\vx) [\vvv(\vx), \vvv(\vx), \vvv(\vx)] & = \sum_{\ind = 1}^d \psi'''(\vx[\ind]) \vvv(\vx)[\ind]^3 \\
& = \sum_{\ind = 1}^d \psi'''(\vx) \mleft( \frac{\psi'(\vx[\ind]) - m_{\psi} (\vx)}{\psi''(\vx[\ind])}\mright)^3 \\
& = \sum_{\ind = 1}^d  \beta (\vx[\ind]) \mleft( \psi'(\vx[\ind]) - m_{\psi} (\vx) \mright)^3,
\]
where
\[
\beta(t) \defeq \frac{\psi'''(t)}{\psi''(t)^3} \quad \textrm{and} \quad m_{\psi} (\vx) \defeq \dfrac{\displaystyle\sum_{\ind = 1}^d \dfrac{\psi'(\vx[\ind])}{\psi''(\vx[\ind])}}{\displaystyle\sum_{\ind = 1}^d \dfrac{1}{\psi''(\vx[\ind])}}.
\]

\begin{enumerate}
    \setlength{\itemsep}{10pt}  %
    \setlength{\parskip}{0pt}
    \item \textbf{Negative Entropy}  ($\psi(\vx) = \sum_{\ind = 1}^d \vx[\ind] \log \vx[\ind]$): \hfill $\kappa = \log^3 d$.

   Let us start with the element-wise regularizer,
    \[
    \psi(t) = t \log t,
    \]
    for which we have,
    \[
    \psi'(t) = 1 + \log t, \quad \psi''(t) = \frac{1}{t}, \quad  \psi'''(t) = - \frac{1}{t^2}, \quad \beta(t) = \dfrac{- \frac{1}{t^2}}{(\frac{1}{t})^3} = -t,
    \]
    and,
    \[
    m_{\psi} (\vx) & = \dfrac{\displaystyle\sum_{\ind = 1}^d \dfrac{\psi'(\vx[\ind])}{\psi''(\vx[\ind])}}{\displaystyle\sum_{\ind = 1}^d \dfrac{1}{\psi''(\vx[\ind])}} \\
    & = \dfrac{\displaystyle\sum_{\ind = 1}^d \dfrac{1 + \log \vx[\ind]}{\frac{1}{\vx[\ind]}}}{\displaystyle\sum_{\ind = 1}^d \dfrac{1}{\frac{1}{\vx[\ind]}}} \\
    & =  - \dfrac{\sum_{\ind = 1}^d \vx[\ind] \mleft( 1 + \log \vx[\ind]\mright)}{\sum_{\ind = 1}^d \vx[\ind]} \\
    & = 1 + \sum_{\ind = 1}^d \vx[\ind] \log \vx[\ind],
    \]
    where in the last line, we used $\vx \in \Delta^d$ and hence $\sum_{\ind = 1}^d \vx[\ind] = 1$.

    Thus,
    \[
    \nabla^3 \psi(\vx) [\vvv(\vx), \vvv(\vx), \vvv(\vx)] & = \sum_{\ind = 1}^d  \beta (\vx[\ind]) \mleft( \psi'(\vx[\ind]) - m_{\psi} (\vx) \mright)^3 \\
    & = \sum_{\ind = 1}^d -\vx[\ind] \mleft(1 + \log \vx[\ind] - \mleft(1 + \sum_{j = 1}^d \vx[j] \log \vx[j] \mright) \mright)^3 \\
    & = - \sum_{\ind = 1}^d \vx[\ind] \mleft(\log \vx[\ind] - \sum_{j = 1}^d \vx[j] \log \vx[j] \mright)^3 \\
    & = - \mathbb{E}[\mleft(\log a - \mathbb{E} [\log a] \mright)^3],
    \]
    where random variable $a$ is drawn from discrete distribution with law $\vx \in \Delta^d$.

    Thus,
    \[
    \mleft| \nabla^3 \psi(\vx) [\vvv(\vx), \vvv(\vx), \vvv(\vx)] \mright| & = \mleft| \mathbb{E}[\mleft(\log a - \mathbb{E} [\log a] \mright)^3] \mright| \\
    & \leq \mleft| \mathbb{E}[\log^3 a] \mright| \\
    & \leq \log^3 d,
    \]
    where the maximum is attained when $\vx$ is the uniform distribution.

    \item \textbf{Log} ($\psi(\vx) = - \sum_{\ind = 1}^d \log \vx[\ind]$): \hfill $\kappa = 2(d-1)$.

    Consider the element-wise regularizer,
    \[
    \psi(t) = \log t,
    \]
    for which we have,
    \[
    \psi'(t) = - \frac{1}{t}, \quad \psi''(t) = \frac{1}{t^2}, \quad  \psi'''(t) = - \frac{2}{t^3}, \quad \beta(t) = \dfrac{- \frac{2}{t^3}}{(\frac{1}{t^2})^3} = -2 t^3,
    \]
    and,
    \[
    m_{\psi} (\vx) & = \dfrac{\displaystyle\sum_{\ind = 1}^d \dfrac{\psi'(\vx[\ind])}{\psi''(\vx[\ind])}}{\displaystyle\sum_{\ind = 1}^d \dfrac{1}{\psi''(\vx[\ind])}} \\
    & = \dfrac{\displaystyle\sum_{\ind = 1}^d \dfrac{- \frac{1}{\vx[\ind]}}{\frac{1}{\vx[\ind]^2}}}{\displaystyle\sum_{\ind = 1}^d \dfrac{1}{\frac{1}{\vx[\ind]^2}}} \\
    & =  - \dfrac{\sum_{\ind = 1}^d \vx[\ind]}{\sum_{\ind = 1}^d \vx[\ind]^2} \\
    & = - \dfrac{1}{\sum_{\ind = 1}^d \vx[\ind]^2},
    \]
    where in the last line, we used $\vx \in \Delta^d$.

    Thus,
    \[
    \nabla^3 \psi(\vx) [\vvv(\vx), \vvv(\vx), \vvv(\vx)] & = \sum_{\ind = 1}^d  \beta (\vx[\ind]) \mleft( \psi'(\vx[\ind]) - m_{\psi} (\vx) \mright)^3 \\
    & = \sum_{\ind = 1}^d - 2 \vx[\ind]^3 \mleft(- \frac{1}{\vx[\ind]} + \dfrac{1}{\sum_{j = 1}^d \vx[j]^2} \mright)^3. \numberthis{eq:log_expansion_self_con}
    \]
    Expand $(a - b)^3 = a^3 - 3 a b^2 + 3a^2 b - b^3$ for $a = \dfrac{1}{\sum_{j = 1}^d \vx[j]^2}, b = \frac{1}{\vx[\ind]}$ and note,
    \[
    \mleft(\dfrac{1}{\sum_{j = 1}^d \vx[j]^2} - \frac{1}{\vx[\ind]}  \mright)^3 & = \frac{1}{\mleft( \sum_{j = 1}^d \vx[j]^2 \mright)^3} - \frac{3}{\mleft( \sum_{j = 1}^d \vx[j]^2 \mright)^2} \frac{1}{\vx[k]} + \frac{3}{\sum_{j = 1}^d \vx[j]^2} \frac{1}{\vx[k]^2} - \frac{1}{\vx[k]^3}.
    \]
    Putting all the pieces into \eqref{eq:log_expansion_self_con},
    \[
    \nabla^3 \psi(\vx) [\vvv(\vx), \vvv(\vx), \vvv(\vx)] & =  \mleft( - 2 \frac{\sum_{\ind = 1}^d \vx[k]^3}{\mleft( \sum_{j = 1}^d \vx[j]^2 \mright)^3} \mright) + \mleft( 6 \frac{\sum_{\ind = 1}^d \vx[k]^2}{\mleft( \sum_{j = 1}^d \vx[j]^2 \mright)^2} \mright) - \mleft( 6 \frac{\sum_{\ind = 1}^d \vx[k]}{\sum_{j = 1}^d \vx[j]^2} \mright) + 2 \sum_{\ind = 1}^d 1 \\
    & = - 2 \frac{\sum_{\ind = 1}^d \vx[k]^3}{\mleft( \sum_{j = 1}^d \vx[j]^2 \mright)^3} + 6 \frac{\sum_{\ind = 1}^d \vx[k]}{\sum_{j = 1}^d \vx[j]^2} - 6 \frac{\sum_{\ind = 1}^d \vx[k]}{\sum_{j = 1}^d \vx[j]^2} + 2 d \\
    & = - 2 \frac{\sum_{\ind = 1}^d \vx[k]^3}{\mleft( \sum_{j = 1}^d \vx[j]^2 \mright)^3} + 2 d
    \]
    where we used the fact that $\vx \in \Delta^d$ and $\sum_{\ind = 1}^d \vx[\ind] = 1$. 
    Thus, 
    \[
    \mleft| \nabla^3 \psi(\vx) [\vvv(\vx), \vvv(\vx), \vvv(\vx)] \mright| \leq 2 \max \mleft\{  \frac{\sum_{\ind = 1}^d \vx[k]^3}{\mleft( \sum_{j = 1}^d \vx[j]^2 \mright)^3} - d, d - \frac{\sum_{\ind = 1}^d \vx[k]^3}{\mleft( \sum_{j = 1}^d \vx[j]^2 \mright)^3}\mright\},
    \]
    where we used the fact that,
    \[
    1 \leq \frac{\sum_{\ind = 1}^d \vx[k]^3}{\mleft( \sum_{j = 1}^d \vx[j]^2 \mright)^3} \leq d,
    \]
    its minimum is attained at a vertex and its maximum is attained at the uniform distribution. Hence, we deduce that $\kappa = 2(d-1)$.

    \item \textbf{$q$-Tsallis Entropy} ($\psi(\vx) = \frac{1}{1 - q} \mleft( 1 - \sum_{\ind = 1}^d \vx[\ind]^q \mright)$): \hfill $\kappa = \frac{q (2 - q)}{(1 - q)^3} \mleft( d^{\frac{(3 - q)(1 - q)}{2 - q}} - 1 \mright)$.

    Omitting the constant term for simplicity, the element-wise regularizer for $q$-Tsallis Entropy is,
    \[
    \psi(t) = \frac{t^q}{q - 1}, 
    \]
    for which we have,
    \[
    \psi'(t) = \frac{q}{q - 1} t^{q - 1}, \quad \psi''(t) = q t^{q - 2}, \quad  \psi'''(t) = q (q - 2) t^{q - 3}, \quad \beta(t) = \frac{q - 2}{q^2} t^{3 - 2q},
    \]
    and,
    \[
    m_{\psi} (\vx) & = \dfrac{\displaystyle\sum_{\ind = 1}^d \dfrac{\psi'(\vx[\ind])}{\psi''(\vx[\ind])}}{\displaystyle\sum_{\ind = 1}^d \dfrac{1}{\psi''(\vx[\ind])}} \\
    & = \dfrac{\displaystyle\sum_{\ind = 1}^d \dfrac{\frac{q}{q - 1} \vx[\ind]^{q - 1}}{q \vx[\ind]^{q - 2}}}{\displaystyle\sum_{\ind = 1}^d \dfrac{1}{q \vx[\ind]^{q - 2}}} \\
    & = \frac{q}{q - 1} \dfrac{\sum_{\ind = 1}^d \vx[\ind]}{\sum_{\ind = 1}^d \vx[\ind]^{2 - q}} \\
    & = \frac{q}{q - 1} \dfrac{1}{\sum_{\ind = 1}^d \vx[\ind]^{2 - q}},
    \]
    where in the last line, we used $\vx \in \Delta^d$.

    Thus,
    \[
    \nabla^3 \psi(\vx) [\vvv(\vx), \vvv(\vx), \vvv(\vx)] & = \sum_{\ind = 1}^d  \beta (\vx[\ind]) \mleft( \psi'(\vx[\ind]) - m_{\psi} (\vx) \mright)^3 \\
    & = \sum_{\ind = 1}^d \frac{q - 2}{q^2} \vx[\ind]^{3 - 2q} \mleft( \frac{q}{q - 1} \vx[\ind]^{q - 1} - \frac{q}{q - 1} \dfrac{1}{\sum_{j = 1}^d \vx[j]^{2 - q}} \mright)^3 \\
    & = \frac{q (q - 2)}{(q - 1)^3} \sum_{\ind = 1}^d \vx[\ind]^{q} \mleft( 1 - \frac{\vx[\ind]^{1 - q}}{\sum_{j = 1}^d \vx[j]^{2 - q}} \mright)^3. \numberthis{eq:tsallis_expansion_self_con}
    \]
    Expand $(1 - w)^3 = 1 - 3w + 3w^2 - w^3$ for $w = \frac{\vx[\ind]^{1 - q}}{\sum_{j = 1}^d \vx[j]^{2 - q}}$ and note,
    \[
    \sum_{\ind = 1}^d \vx[\ind]^{q} w & = \sum_{\ind = 1}^d \vx[\ind]^{q} \mleft( \frac{\vx[\ind]^{1 - q}}{\sum_{j = 1}^d \vx[j]^{2 - q}} \mright) = \dfrac{\sum_{\ind = 1}^d \vx[\ind] }{\sum_{j = 1}^d \vx[j]^{2 - q}} = \dfrac{1}{\sum_{j = 1}^d \vx[j]^{2 - q}} \\
    \sum_{\ind = 1}^d \vx[\ind]^{q} w^2 & = \sum_{\ind = 1}^d \vx[\ind]^{q} \mleft( \frac{\vx[\ind]^{1 - q}}{\sum_{j = 1}^d \vx[j]^{2 - q}} \mright)^2 = \dfrac{\sum_{\ind = 1}^d \vx[\ind]^{2 - q}}{\mleft( \sum_{j = 1}^d \vx[j]^{2 - q} \mright)^2} =  \dfrac{1}{\sum_{j = 1}^d \vx[j]^{2 - q}} \\
    \sum_{\ind = 1}^d \vx[\ind]^{q} w^3 & = \sum_{\ind = 1}^d \vx[\ind]^{q} \mleft( \frac{\vx[\ind]^{1 - q}}{\sum_{j = 1}^d \vx[j]^{2 - q}} \mright)^3 = \dfrac{\sum_{\ind = 1}^d \vx[\ind]^{3 - 2q}}{\mleft( \sum_{j = 1}^d \vx[j]^{2 - q} \mright)^3}.
    \]
    Putting all the pieces into \eqref{eq:tsallis_expansion_self_con},
    \[
    \nabla^3 \psi(\vx) [\vvv(\vx), \vvv(\vx), \vvv(\vx)] & = \frac{q (q - 2)}{(q - 1)^3} \mleft( \sum_{\ind = 1}^d \vx[\ind]^k - 3 \dfrac{1}{\sum_{j = 1}^d \vx[j]^{2 - q}} + 3 \dfrac{1}{\sum_{j = 1}^d \vx[j]^{2 - q}} - \dfrac{\sum_{\ind = 1}^d \vx[\ind]^{3 - 2q}}{\mleft( \sum_{j = 1}^d \vx[j]^{2 - q} \mright)^3}\mright) \\
    & = \frac{q (q - 2)}{(q - 1)^3} \mleft( \sum_{\ind = 1}^d \vx[\ind]^k - \dfrac{\sum_{\ind = 1}^d \vx[\ind]^{3 - 2q}}{\mleft( \sum_{j = 1}^d \vx[j]^{2 - q} \mright)^3}\mright).
    \]
    Thus,
    \[
    \mleft| \nabla^3 \psi(\vx) [\vvv(\vx), \vvv(\vx), \vvv(\vx)] \mright| \leq \frac{q (q - 2)}{(q - 1)^3} 
    \mleft|  A(\vx) - B(\vx) \mright|,
    \]
    where,
    \[
    A(\vx) \defeq \sum_{\ind = 1}^d \vx[\ind]^q \quad \textrm{and} \quad B(\vx) \defeq \dfrac{\sum_{\ind = 1}^d \vx[\ind]^{3 - 2q}}{\mleft( \sum_{j = 1}^d \vx[j]^{2 - q} \mright)^3}.
    \]
    For the term $A(\vx)$, we know that the minimum is attained at a vertex and the maximum is attained at the uniform distribution. Hence,  
    \[
    1 \leq A(\vx) \leq \sum_{\ind = 1}^d \left( \frac{1}{d} \right)^q = d^{1-q}.
    \]
    For the other term $B(\vx)$, using Cauchy–Schwarz,
    \[
    \mleft(  \sum_{j = 1}^d \vx[j]^{2 - q} \mright)^2 & \leq \mleft( \sum_{j = 1}^d \vx[j] \mright) \mleft( \sum_{j = 1}^d \vx[j]^{3 - 2q} \mright) \\
    & = \mleft( \sum_{j = 1}^d \vx[j]^{3 - 2q} \mright),
    \]
    since $\vx \in \Delta^d$ and $\sum_{j = 1}^d \vx[j] = 1$.
    Thus,
    \[
    B(\vx) & = \dfrac{\sum_{\ind = 1}^d \vx[\ind]^{3 - 2q}}{\mleft( \sum_{j = 1}^d \vx[j]^{2 - q} \mright)^3} \\
    & \geq \frac{1}{\sum_{j = 1}^d \vx[j]^{3 - 2q}} \\
    & \geq 1,
    \]
    as the maximum of $\sum_{j = 1}^d \vx[j]^{2 - q}$ is attained at a vertex of $\Delta^d$. And on the other hand, by finite dimensional $\ell_p$ norms equivalence,
    \[
    \| \vx \|_{3 - 2q} \leq \| \vx \|_{2 - q},
    \]
    Thus,
    \[
    \mleft( \sum_{\ind = 1}^d \vx[\ind]^{3 - 2q} \mright)^{\frac{1}{3 - 2q}} \leq \mleft( \sum_{\ind = 1}^d \vx[\ind]^{2 - q} \mright)^{\frac{1}{2 - q}},
    \]
    and hence,
    \[
    \mleft( \sum_{\ind = 1}^d \vx[\ind]^{3 - 2q} \mright) \leq \mleft( \sum_{\ind = 1}^d \vx[\ind]^{2 - q} \mright)^{\frac{3 - 2q}{2 - q}}.
    \]

    As a result for $B(\vx)$, we infer 
    \[
    \dfrac{\sum_{\ind = 1}^d \vx[\ind]^{3 - 2q}}{\mleft( \sum_{j = 1}^d \vx[j]^{2 - q} \mright)^3} & \leq \mleft( \sum_{\ind = 1}^d \vx[\ind]^{2 - q} \mright)^{- \frac{3 - q}{2 - q}} \\
    & \leq \mleft( \sum_{\ind = 1}^d \mleft(\frac{1}{d}\mright)^{2 - q} \mright)^{- \frac{3 - q}{2 - q}} \\
    & = \mleft( d^{q-1} \mright)^{- \frac{3 - q}{2 - q}} \\
    & = d^{\frac{(3 - q)(1 - q)}{2 - q}},
    \]
    since the maximum of $\sum_{\ind = 1}^d \vx[\ind]^{2 - q}$ is attained at uniform distribution.

    Assembling the results obtained so far,
    \[
    \mleft| \nabla^3 \psi(\vx) [\vvv(\vx), \vvv(\vx), \vvv(\vx)] \mright| & \leq \frac{q (q - 2)}{(q - 1)^3} 
    \max \mleft\{ d^{1 - q} - 1, d^{\frac{(3 - q)(1 - q)}{2 - q}} - 1\mright\} \\
    & = \frac{q (q - 2)}{(q - 1)^3} \mleft( d^{\frac{(3 - q)(1 - q)}{2 - q}} - 1 \mright).
    \]

    Thus,
    \[
    \kappa = \frac{q (2 - q)}{(1 - q)^3} \mleft( d^{\frac{(3 - q)(1 - q)}{2 - q}} - 1 \mright).
    \]
\end{enumerate}

\section{Proofs for Iteration Complexity of \ours} \label{app:iteration_comp}

In this section, we study the iteration complexity of \ours for different choices of the regularizer $\psi$. The computational overhead of \ours compared to \OFTRL arises from solving the dynamic learning-rate control optimization problem. Similar to \Cref{app:proof_convexity}, for ease of notation, we use $\overline{f}$ to denote the negative of the learning rate control objective $-f$, where
\[
\overline{f}(\lambda) \defeq - f(\lambda) = - \alpha \log \lambda - \psi_{\mathcal{X}}^* (\lambda \at).
\]
We want to find an approximator $\widehat{\lambda}$ that is multiplicatively $\epsilon$-close to the unique minimizer $\lambda^*$ of the 1-dimensional optimization problem, 
\[
\lambda^* = \argmin_{\lambda \in (0, \eta]} \overline{f}(\lambda).
\]
In other words, we require the approximation $\widehat{\lambda}$ to satisfy the multiplicative tolerance $\epsilon$,  
\[
\left| \frac{\widehat{\lambda}}{\lambda^*} - 1 \right| \leq \epsilon.
\]

In this section, we show that iteration complexity of additional overhead of \ours is negligible are designing efficient algorithms. We summarize the results  of this section in \Cref{table:iteration_complexity}.

It is clear that  
\[
\overline{f}(0) = \infty, \quad \text{and} \quad \overline{f}(\eta) \text{ is finite,}
\]
and by \Cref{thm:strong_convex_learning_rate_control}, the function $\overline{f}$ is strongly convex on $(0, \eta]$. Hence, it is unimodal and admits a unique minimizer $\lambda^*$.

\begin{algorithm2e}[th]
    \SetNoFillComment
    \caption{Geometric Bisection Search for Learning Rate Control Problem
    }\label{algo:geometric_bisection}
    \DontPrintSemicolon
    \KwData{$l_0 \in (0, \eta)$ with $\overline{f}'(l_0) < 0$, multiplicative tolerance $\epsilon$}
    \KwResult{Solution $\widehat{\lambda}$} \vspace{2mm}

    \uIf{$\overline{f}'(\eta) \leq 0$}{
    Set  $\displaystyle \widehat{\lambda} \gets \eta$ \;
    \Return $\displaystyle \widehat{\lambda}$\;
    } 
    Set  $\displaystyle l \gets l_0$ \;
    Set  $\displaystyle u \gets \eta$ \; \medskip

    \While{$\dfrac{u}{l} > \mleft( 1 + \epsilon \mright)^2$}{
    \tcc{\color{commentcolor}\texttt{Geometric Mean of Interval $(l, u)$}}
    
    Set  $\displaystyle m \gets \sqrt{l \cdot u}$ \; \medskip

    \tcc{\color{commentcolor}\texttt{Geometric Bisection on $(l, u)$}}
    \uIf{$\overline{f}'(m) < 0$}{
    Set  $\displaystyle l \gets m$ \;
    }
    \Else{
    Set  $\displaystyle u \gets m$ \;
    }

    } \medskip
    Set  $\displaystyle \widehat{\lambda} \gets \sqrt{l \cdot u}$ \;
    \Return $\displaystyle \widehat{\lambda}$\;

\end{algorithm2e}

For a general intrinsic Lipschitz regularizer $\psi$, we show that one can employ geometric variants of the golden-section (\Cref{algo:geometric_golden}) and bisection (\Cref{algo:geometric_bisection}) search algorithms, which rely only on zeroth-order and first-order oracles respectively, to find an approximate dynamic learning rate $\widehat{\lambda}$ with multiplicative error $\epsilon$ in $O(\log (1/\epsilon))$ time. We formalize these results in the followings (\Cref{prop:bisection,prop:golden}).

\begin{proposition}\label{prop:bisection}
    The geometric bisection search algorithm (\Cref{algo:geometric_bisection}) outputs an approximation $\widehat{\lambda}$ that is multiplicatively $\epsilon$-close to the optimal solution $\lambda^*$, using only $O(\log (1/\epsilon))$ evaluations of the $\overline{f}'$, derivative of the learning-rate control problem.
\end{proposition}
\begin{proof}
    Because of strong convexity, the derivative $\overline{f}'$ is strictly increasing on $(0, \eta]$. If $\overline{f}'(1^-) \leq 0$, then $\overline{f}'(\lambda) \leq 0$ for all $\lambda \in (0, \eta]$. Since $\overline{f}'$ is increasing, the function $\overline{f}$ is strictly decreasing, and the unique minimizer is $\widehat{\lambda} = \lambda^* = \eta$.

    Otherwise, since $\overline{f}'(1^-) > 0$, then by monotonicity and first-order optimality, the unique solution $\lambda^*$ is the root of the equation $\overline{f}'(\lambda^*) = 0$. 

    We use the subscript $k$ to denote the bracket $[l_k, u_k]$ at iteration $k$. We maintain the bracket $[l_k, u_k] \subseteq (0, \eta]$ with  
    \[
    \overline{f}'(l_k) \leq 0 \leq \overline{f}'(u_k),
    \]
    so that $\lambda^* \in [l_k, u_k]$. At each iteration, we shrink the interval by the ratio  
    \[
    r_k = \frac{u_k}{l_k}.
    \]
    
    Consider the geometric mean $m_k = \sqrt{l_k u_k}$.  
    Since $\overline{f}'$ is monotone and satisfies $\overline{f}'(\lambda^*) = 0$, if $\overline{f}'(m_k) < 0$, then $m_k < \lambda^* \geq u_k$; otherwise, $l_k \leq \lambda^* < m_k$. Thus, always $\lambda^* \in [l_{k+1}, u_{k+1}]$.

    In both cases,  
    \[
    r_{k+1} = \frac{u_k}{m_k} = \frac{m_k}{l_k} = \sqrt{\frac{u_k}{l_k}} = \sqrt{r_k}.
    \]
    Hence, by simple induction,  
    \[
    r_k = (r_0)^{1/2^k},
    \]
    and thus $\log(u_k/l_k)$ halves at each iteration.

    For any bracket $\lambda^* \in (l, u)$ any candidate $\widehat{\lambda} = \sqrt{l \cdot u}$ has worst-case multiplicative factor,
    \[
    \sup_{\lambda^* \in (l, u)} \max \mleft\{ \frac{\widehat{\lambda}}{\lambda^*}, \frac{\lambda^*}{\widehat{\lambda}}\mright\} = \max \mleft\{ \frac{u}{\sqrt{l \cdot u}}, \frac{\sqrt{l \cdot u}}{l}\mright\} = \sqrt{\frac{u}{l}}
    \]
    
    Our stopping criterion is  
    \[
    \frac{u_k}{l_k} \leq (1 + \epsilon)^2,
    \]
    which implies $\sqrt{\frac{u_k}{l_k}} \leq 1 + \epsilon$ and consequently,  
    \[
    \left| \frac{\widehat{\lambda}}{\lambda^*} - 1 \right| \leq \epsilon.
    \]
    Therefore, the algorithm outputs a good approximation $\widehat{\lambda}$ in $k$ iterations, since  
    \[
    r_k = (r_0)^{1/2^k} \leq (1 + \epsilon)^2.
    \]
    
    As a result, it suffices that $k$ satisfies  
    \[ 
    r_k = (r_0)^{1/2^k} \leq (1 + \epsilon)^2, \] which as a result, we infer that $k$ iterations suffices, 
    \[ 
    k & = \log_2 \mleft(\frac{\log r_0}{2 \log ( 1 + \epsilon) } \mright) \\ 
    & = \log_2 \log r_0 - \log_2 \mleft( 2 \log (1 + \epsilon) \mright) \\
    & \leq \log_2 \log r_0 - \log_2 \epsilon \numberthis{eq:bisection_temp}\\ 
    & = O\mleft(\log \frac{1}{\epsilon}\mright),
    \]
    where \eqref{eq:bisection_temp} follows from $\frac{\epsilon}{2} \leq \log(1 + \epsilon)$ for all $\epsilon \in (0,1)$.

\end{proof}

\begin{algorithm2e}[th]
    \SetNoFillComment
    \caption{Geometric Golden-section Search for Learning Rate Control Problem
    }\label{algo:geometric_golden}
    \DontPrintSemicolon
    \KwData{$l_0 \in (0, \eta)$ with $l_0 < \lambda^*$, multiplicative tolerance $\epsilon$}
    \KwResult{Solution $\widehat{\lambda}$} \vspace{2mm}

    \tcc{\color{commentcolor}\texttt{Golden Ratio}}
    Set $\tau = \dfrac{\sqrt{5} - 1}{2}  \approx 0.618$ \;
    \medskip
    
    Set  $\displaystyle a \gets \log l_0$ \;
    Set  $\displaystyle b \gets \log \eta$ \; 
    \medskip

    \tcc{\color{commentcolor}\texttt{Split Interval in Golden Ratio}}
    Set  $\displaystyle c \gets b - \tau(b - a)$ \;
    Set  $\displaystyle d \gets a + \tau(b - a)$ \;
    \medskip
    
    \While{$b - a > 2 \log \epsilon$}{
    \tcc{\color{commentcolor}\texttt{Geometric Split on $(a, b)$}}
    \uIf{$f(c) < f(d)$}{
    \tcc{\color{commentcolor}\texttt{The Minimum Lies on $(a, d)$}}
    Set  $\displaystyle b \gets d$ \;
    Set  $\displaystyle d \gets c$ \;
    Set  $\displaystyle c \gets b - \tau(b - a)$ \;
    }
    \Else{
    \tcc{\color{commentcolor}\texttt{The Minimum Lies on $(c, b)$}}
    Set  $\displaystyle a \gets c$ \;
    Set  $\displaystyle c \gets d$ \;
    Set  $\displaystyle d \gets a + \tau(b - a)$ \;
    }

    } \medskip

    \tcc{\color{commentcolor}\texttt{Geometric Mean}}
    Set  $\displaystyle \widehat{\lambda} \gets \exp\mleft\{ \frac{a + b}{2}\mright\}$ \;
    \Return $\displaystyle \widehat{\lambda}$\;

\end{algorithm2e}

\begin{proposition}\label{prop:golden}
    The geometric golden-section search algorithm (\Cref{algo:geometric_golden}) outputs an approximation $\widehat{\lambda}$ that is multiplicatively $\epsilon$-close to the optimal solution $\lambda^*$, using only $O(\log (1/\epsilon))$ evaluations of $\overline{f}$, i.e., zeroth-order evaluations of the learning-rate control problem.
\end{proposition}
\begin{proof}
    Define the function  
    \[
    \overline{g}(t) \defeq \overline{f}(e^t),
    \]
    on the interval $(-\infty, \log \eta]$. Since the function $\overline{f}$ is unimodal and the exponential map is monotone, the function $\overline{g}$ is also unimodal. Its minimum is attained at the unique point  
    \[
    t^* = \log \lambda^* \in (\log l_0, \log \eta).
    \]

    At its core, \Cref{algo:geometric_golden} runs the classical golden-section search~\citep{kiefer1953sequential} on the function $\overline{g}$, starting from the interval $(a_0, b_0]$, where $a_0 = \log l_0$ and $b_0 = \log \eta$. By the standard analysis of the golden-section search, each iteration maintains a bracket $[a_k, b_k]$ containing $t^*$ and shrinks its length by a factor of $\tau \approx 0.618$,  
    \[
    b_k - a_k \leq \tau^k \,(b_0 - a_0).
    \]

    The stopping criterion is $b_k - a_k \leq 2 \log (1 + \epsilon)$, and we output  
    \[
    \widehat{\lambda} = \exp\!\left(\frac{a_k + b_k}{2}\right).
    \]
    Since $a_k = \log l_k$, $b_k = \log u_k$, and $\lambda^* \in [l_k, u_k]$, we have  
    \[
    \frac{u_k}{l_k} = \exp\!\{b_k - a_k\} \leq \exp\!\{2 \log (1 + \epsilon)\} = (1 + \epsilon)^2.
    \]
    Thus, by the same argument as in the proof of \Cref{prop:bisection},  
    \[
    \left| \frac{\widehat{\lambda}}{\lambda^*} - 1 \right| \leq \epsilon.
    \]

    Again, similar to the proof of \Cref{prop:bisection}, for the stopping condition we require $b_k - a_k \leq 2 \log (1 + \epsilon)$, which amounts to  
    \[
    k \geq \log_{1/\tau} \!\left( \frac{b_0 - a_0}{2 \log (1 + \epsilon)} \right) 
    = O\!\left(\log \frac{1}{\epsilon}\right),
    \]

\end{proof}

If the regularizer $\psi$ is additionally Legendre and log-dominated (see \Cref{def:log_dominated}), then the learning-rate control problem is self-concordant by \Cref{thm:log_dom_self_con}. In that case, we can apply Newton's method to solve the learning-rate control problem and obtain an approximation $\widehat{\lambda}$ that is $\epsilon$-close to the true solution $\lambda^*$ in the local norm induced by the second order derivative of $\overline{f}$, in $O(\log \log (1/\epsilon))$ time. Moreover, by \Cref{thm:strong_convex_learning_rate_control_legendre}, $\widehat{\lambda}$ is also multiplicatively $\epsilon$-close to $\lambda^*$. The next proposition formalizes this result.

\begin{proposition}\label{prop:newton}
    When the regularizer $\psi$ is additionally Legendre and log-dominated and hence the learning-rate control problem is self-concordant, starting from a close-enough initialization $\lambda_0$, Newton's method outputs an approximation $\widehat{\lambda}$ that is multiplicatively $\epsilon$-close to the optimal solution $\lambda^*$, using only $O(\log \log (1/\epsilon))$ evaluations of $\overline{f}'$ and $\overline{f}''$, i.e., the first- and second-order derivatives of the learning-rate control problem.
\end{proposition}

\begin{proof}
    By the standard analysis of Newton's method for self-concordant functions~\citep{nesterov1994interior}, we know that after $O(\log \log (1/\epsilon))$ iterations,  
    \[
    \|\widehat{\lambda} - \lambda^*\|_{\overline{f}''(\lambda^*)} \leq \epsilon,
    \]
    where $\|\cdot\|_{\overline{f}''}$ denotes the local norm induced by the Hessian $\overline{f}''$ of the function $\overline{f}$. Thus, by \Cref{thm:strong_convex_learning_rate_control_legendre},  
    \[
         \epsilon^2 & \geq \|\widehat{\lambda} - \lambda^*\|_{\overline{f}''(\lambda^*)} \\
         & = \mleft(\widehat{\lambda} - \lambda^* \mright)^2 {\overline{f}''(\lambda^*)} \\
         & \geq (\alpha - \gamma) \mleft( \dfrac{\widehat{\lambda} - \lambda^*}{\lambda^*}\mright)^2 \\
         & \geq \mleft( \frac{\widehat{\lambda}}{\lambda^*} - 1 \mright)^2.
    \]
    Therefore, $\widehat{\lambda}$ is multiplicatively $\epsilon$-close to $\lambda^*$.
\end{proof}

\begin{restatetheorem}{thm:iter_comp}
    The iteration complexity of the additional overhead of \ours compared to \OFTRL, i.e., solving the learning-rate control problem, is:  
    \begin{itemize}
        \item $O(\log T)$ with access to \textbf{either a zeroth- or a first-order oracle}, whenever the regularizer is intrinsic Lipschitz,
        \item $O(\log \log T)$ with access to \textbf{first- and second-order oracles}, whenever the regularizer is (locally) intrinsic Lipschitz, Legendre, and log-dominated.  
    \end{itemize}
\end{restatetheorem}
\begin{proof}
    As shown in \Cref{app:approx}, at each iteration $t \in [T]$, it suffices to use approximate dynamic learning rates $\widehat{\lambda}\^t$ that are multiplicatively $\epsilon$-close to the exact dynamic learning rate $\lambda\^t$, with $\epsilon = 1/T$.

    For the first part, we can employ \Cref{algo:geometric_golden,algo:geometric_bisection}, respectively, and make use of \Cref{prop:golden,prop:bisection} with $\epsilon = 1/T$. However, both of these algorithms require an initial lower bound $l_0 \in (0, \eta)$ as input such that  
    \[
    \overline{f}'(l_0) \leq 0 \quad \text{or, equivalently,} \quad l_0 \leq \lambda\^t.
    \]
    It is clear that such $l_0$ exists, since $f(\eta)$ is finite, and $f$ cannot be non-decreasing near zero. Thus, it remains to efficiently find it.
    
    For this purpose, we use a warm start with the approximate learning rate from the previous iteration. We know that from the previous iteration, $\widehat{\lambda}\^{t - 1}$ is multiplicatively $\epsilon$-close to the exact value $\lambda\^{t - 1}$,  
    \[
     \frac{\widehat{\lambda}\^{t - 1}}{\lambda\^{t - 1}} \in [1 - \epsilon, 1 + \epsilon]. \numberthis{eq:warm_start_next1}
    \]
    Moreover, by the stability of dynamic learning rates in consecutive iterations (\Cref{theorem:mult_stable}),  
    \[
    \dfrac{\lambda\^{t-1}}{\lambda\^{t}} \in \mleft[ \dfrac{1}{2}, \dfrac{3}{2} \mright]. \numberthis{eq:warm_start_next2}
    \]
    Thus, multiplying \eqref{eq:warm_start_next1} and \eqref{eq:warm_start_next2}, we get  
    \[
    \frac{\widehat{\lambda}\^{t - 1}}{\lambda\^{t}} \in \mleft[ \dfrac{1}{2} (1 - \epsilon), \dfrac{3}{2} (1 + \epsilon) \mright], \numberthis{eq:warm_start_mult}
    \]
    and as a result,  
    \[
    \lambda\^t \geq \frac{2}{3 (1 + \epsilon)} \widehat{\lambda}\^{t-1}.
    \]
    Accordingly, at each iteration $t \in [T]$, it is enough to initialize \Cref{algo:geometric_bisection,algo:geometric_golden} with  
    \[
    l_0 = \frac{2}{3 (1 + \epsilon)} \widehat{\lambda}\^{t-1},
    \]
    and the proof is complete by \Cref{prop:golden,prop:bisection}.

    The second part follows directly from \Cref{prop:newton}, except that we need to find a good initialization $\lambda_0$ that lies in Newton's quadratic convergence regime. To this end, by \Cref{lemma:quadratic_convergence}, it is enough to show that at each iteration $t$,  
    \[
    \| \lambda_0 - \lambda\^t \|_{\overline{f}''(\lambda\^t)} \leq \frac{1}{M + 4},
    \]
    for an $M$-generalized self-concordant function $\overline{f} = -f$. We show that at each iteration $t$, the warm-start initialization $\lambda_0 = \widehat{\lambda}\^{t-1}$ satisfies this property. In terms, by \Cref{thm:strong_convex_learning_rate_control},
    \[
    \| \widehat{\lambda}\^{t- 1} - \lambda\^t \|_{\overline{f}''(\lambda\^t)} & = \mleft(\widehat{\lambda}\^{t-1} - \lambda\^t \mright)^2 \overline{f}''(\lambda\^t) \\
    & \leq \alpha \mleft( \dfrac{\widehat{\lambda}\^{t-1} - \lambda\^t}{\lambda\^t} \mright)^2 \\
    & = \alpha \mleft( \dfrac{\widehat{\lambda}\^{t-1}}{\lambda\^t} - 1 \mright)^2 \\
    & \leq \alpha \max \mleft\{ \mleft| -\frac{1}{2} (1 + \epsilon) \mright|, \mleft|  \frac{1}{2} (1 + 3 \epsilon) \mright| \mright\}^2, \numberthis{eq:warm_start_mult2} \\
    & = \alpha \frac{(1 + 3 \epsilon)^2}{4}
    \]
    where in \eqref{eq:warm_start_mult2}, we use \eqref{eq:warm_start_mult}. Choosing a small enough tolerance $\epsilon$ establishes the claim.

    To sum up, for both parts, we can compute $\widehat{\lambda}\^t$ efficiently by using warm-start ideas from $\widehat{\lambda}\^{t-1}$. For $t = 0$, we directly observe that since $\at\^0 = 0$  
    \[
    \lambda\^0 & = \argmax_{\lambda \in (0, \eta]} \mleft\{ \alpha \log \lambda + \psi_{\mathcal{X}}^* (0)\mright\}, \\
    & = \argmax_{\lambda \in (0, \eta]} \mleft\{ \alpha \log \lambda\mright\}, \\
    & = \eta.
    \]
    Hence, we can simply output $\widehat{\lambda}\^0 = \lambda\^0 = \eta$ for the starting iteration.
\end{proof}

For our proof in \Cref{thm:iter_comp}, when the learning-rate control problem is self-concordant, we use the following lemma from the theory of self-concordant functions~\citep{nesterov1994interior,boyd2004convex}, for which we provide a proof for completeness.

\begin{lemma}\label{lemma:quadratic_convergence}
 Consider the convex and $M$-generalized self-concordant function $h$,  
\[
\mleft | h'''(x) \mright| \leq M \mleft[ h''(x) \mright]^{3/2},
\]
with global minimizer $x^*$. If  
\[
\| x - x^* \|_{h''(x^*)} \leq \frac{1}{4 + M}, 
\]
then $x$ is in the quadratic convergence regime of Newton's method.
\end{lemma}
\begin{proof}
    Define,
    \[
    r^* \defeq \| x - x^* \|_{h''(x^*)}.
    \]
    By Dikin's Ellipsoid relations~\citep{boyd2004convex}, we know that if $r^* \leq \frac{2}{M}$, then
    \[
    \| x - x^* \|_{h''(x)} \leq \frac{\| x - x^* \|_{h''(x^*)}}{1 - \frac{M}{2} \| x - x^* \|_{h''(x^*)}} = \frac{r^*}{1 - \frac{M}{2} r^*}. \numberthis{eq:dikins}
    \]
    On the other hand, let $w = x - x^*$ and parameterize the segment $x_t = x^* + t w$. Using the self-concordant Hessian comparison bounds along this segment,  
    \[
    h''(x_t) \leq \frac{1}{\mleft( 1 - \frac{M}{2} \| x - x_t \|_{h''(x)} \mright)^2} h''(x),
    \]
    and integrating $h'(x) - h'(x^*) = \int_0^1 h''(x_t)\, dt$, we obtain the following standard bound for the Newton decrement:  
    \[
    \| h'(x) \|_{h''(x)}^* \leq \frac{\| w \|_{h''(x)}}{1 - \frac{M}{2} \| w \|_{h''(x)}}.
    \]
    Replacing \eqref{eq:dikins}, we get  
    \[
    \| h'(x) \|_{h''(x)}^* \leq \frac{r^*}{1 - M r^*}.
    \]
    Imposing the quadratic convergence regime $\| h'(x) \|_{h''(x)}^* \leq \frac{1}{4}$~\citep{boyd2004convex} in the last inequality,  
    \[
    \frac{r^*}{1 - M r^*} \leq \frac{1}{4},
    \]
    we obtain  
    \[
    r^* \leq \frac{1}{4 + M},
    \]
    which completes the proof.

\end{proof}
\section{Proofs for Extension of the Analysis under Approximate Iterates} \label{app:approx}

In this section, we discuss the regret analysis and properties of \ours, where at each iteration $t$, instead of the $\lambda\^t$, \emph{exact} solution of the learning rate control problem
\[
\lambda\^t = \argmax_{\lambda \in (0, \eta]} \mleft\{ \alpha \log \lambda + \psi_{\mathcal{X}}^* (\lambda \at\^t)\mright\},
\]
we compute a \emph{multiplicative $\epsilon$-approximation} $\widehat{\lambda}\^t$ with 
\[
\widehat{\lambda}\^t \in \mleft[ (1 - \epsilon) \lambda\^t, (1 + \epsilon) \lambda\^t \mright], \qquad \textrm{with } \epsilon \in (0, \frac{1}{2}],
\]
with algorithms discussed in \Cref{table:iteration_complexity,app:iteration_comp}.

Moreover, instead of the canonical (\emph{exact}) lifted action $\vy\^t = \lambda\^t \vx\^t$ with  
\[
\vx\^t = \argmax_{\vx \in \mathcal{X}} \mleft\{ \langle \lambda\^t \at\^t, \vx \rangle - \psi(\vx) \mright\},
\]
we consider the approximate lifted iteration  
\[
\widehat{\vy}\^{t} \defeq \widehat{\lambda}\^t \widehat{\vx}\^t,
\]
where  
\[
\widehat{\vx}\^t = \argmax_{\vx \in \mathcal{X}} \mleft\{ \langle \widehat{\lambda}\^t \at\^t, \vx \rangle - \psi (\vx) \mright\}.
\]
And, we have the reward signal $\ut\^t = \nut\^t - \langle \nut\^t, \widehat{\vx}\^t \rangle \vec1_d$, and the regret vector $\at\^t = \sum_{\tau=1}^{t-1} \ut\^t + \ut\^{t-1}$, where the utility vector $\nut\^t$ is received by playing the approximate action $\widehat{\vx}\^t$.

In this section, we will show that the regret gurantees of \ours continues to hold (up to constants) for the sequence $\mleft\{ \widehat{\vx}\^t \mright\}_{t = 1}^T$ if we take $\epsilon = 1/T$. To this end, define the per-round suboptimality in the lifted \OFTRL objective,
\[
\delta_t \defeq F_t (\widehat{\vy}\^t) - \min_{\vy \in \Omega} F_t(\vy) = F_t(\widehat{\vy}\^t) - F_t(\vy\^t) \geq 0.
\]
Because $F_t(\vy) = \frac{1}{\eta} \phi(\vy) - \langle \at\^t, \vy \rangle$ and $\vy = \lambda \vx$, the minimization in $\vy$ is equivalent to the maximization of the dynamic learning rate control problem over $\lambda$ (with the inner maximization in $\vx$ captured by $\psi^*_{\mathcal{X}}$). In particular,  
\[
\delta_t = \frac{1}{\eta} \mleft( f_t(\lambda\^t) - f_t(\widehat{\lambda}\^t) \mright),
\]
with
\[
f_t(\lambda) \defeq \alpha \log \lambda + \psi_\mathcal{X}^*(\lambda \at\^t).
\]
In the following lemma, we show that the suboptimality terms $\delta_t$ can be upper bounded by the Bregman divergence of $-\log$ between the \emph{approximate} and \emph{exact} dynamic learning rates.

\begin{lemma}\label{lemma:subopt_basic_ineq}
    For every time $t \in [T]$, and for every exact dynamic learning rate $\lambda\^t \in (0, \eta)$ and approximate dynamic learning rate $\widehat{\lambda} > 0$,  
    \[
    \delta_t(\widehat{\lambda}) = \frac{1}{\eta} \mleft( f_t(\lambda\^t) - f_t(\widehat{\lambda}) \mright) \leq \frac{\alpha}{\eta} D_{- \log} (\widehat{\lambda} \; \| \; \lambda\^t ) = \frac{\alpha}{\eta} \mleft( - \log \frac{\widehat{\lambda}}{\lambda\^t} + \frac{\widehat{\lambda}}{\lambda\^t} - 1\mright).
    \]
\end{lemma}
\begin{proof}
    Let 
    \[
    \widehat{\vx} = \argmax_{\vx \in \mathcal{X}} \mleft\{ \langle \widehat{\lambda} \at\^t, \vx \rangle - \psi(\vx) \mright\}.
    \]
    By definition, 
    \[
    \psi^*_{\mathcal{X}}(\widehat{\lambda} \at\^t) \geq \langle \widehat{\lambda} \at\^t, \vx\^t \rangle - \psi(\vx\^t).
    \]
    Consequently,
    \[
    f_t (\lambda\^t) - f_t(\widehat{\lambda}) & = \mleft( \alpha \log \lambda\^t + \langle \lambda\^t \at\^t, \vx\^t \rangle - \psi(\vx\^t) \mright) - \mleft( \alpha \log \widehat{\lambda} + \psi^*_{\mathcal{X}}(\widehat{\lambda} \at\^t) \mright) \\
    & \leq \mleft( \alpha \log \lambda\^t + \langle \lambda\^t \at\^t, \vx\^t \rangle - \psi(\vx\^t) \mright) - \mleft( \alpha \log \widehat{\lambda} + \langle \widehat{\lambda} \at\^t, \vx\^t \rangle - \psi(\vx\^t) \mright) \\
    & = \alpha \log \frac{\lambda\^t}{\widehat{\lambda}} + (\lambda\^t - \widehat{\lambda}) \langle \at\^t, \vx\^t \rangle. \numberthis{eq:subopt_basic_ineq1}
    \]
    
    Since $\lambda\^t \in (0, \eta)$ lies in the interior of the set, by first-order optimality and Danskin's theorem,  
    \[
    f_t'(\lambda\^t) = \frac{\alpha}{\lambda\^t} +  \langle \at\^t, \vx\^t \rangle  = 0.
    \]

    Thus, by substituting $\langle \at\^t, \vx\^t \rangle = - \frac{\alpha}{\lambda\^t} $ back into \eqref{eq:subopt_basic_ineq1},
    \[
    f_t (\lambda\^t) - f_t(\widehat{\lambda}) & \leq \alpha \mleft( \log \frac{\lambda\^t}{\widehat{\lambda}} - \frac{\lambda\^t - \widehat{\lambda}}{\lambda\^t} \mright) \\
    & = \alpha \mleft(- \log \frac{\widehat{\lambda}}{\lambda\^t} + \frac{ \widehat{\lambda}}{\lambda\^t} - 1 \mright) \\
    & = \alpha D_{- \log} (\widehat{\lambda} \; \| \; \lambda\^t ), 
    \]
    which finishes the proof.
\end{proof}

In turn, we can characterize how the suboptimality gap $\delta_t$ governs the approximation error in the lifted space $ D_\phi( \widehat{\vy}\^t \; \| \; \vy\^t)$. Moreover, by \Cref{prop:curvature} and the fact that $\widehat{\lambda}\^t$ and $\lambda\^t$ are multiplicatively close, we show that the suboptimality gap $\delta_t$ also governs $\| \widehat{\vx}\^t - \vx\^t \|_1$ on the simplex.

\begin{lemma}\label{lemma:subopt_approx_error_bound}
    For every time $t \in [T]$,  
    \[
    \delta_t \geq \frac{1}{\eta} D_\phi( \widehat{\vy}\^t \; \| \; \vy\^t),
    \]
    and whenever the approximate dynamic learning rate $\widehat{\lambda}\^t > 0$ is multiplicatively $\epsilon$-close ($\epsilon \in (0, 1/2]$) to the exact dynamic learning rate $\lambda\^t \in (0, \eta]$, i.e.,  
    \[
    \omega_t \defeq \frac{\widehat{\lambda}\^t}{\lambda\^t} \in [1 - \epsilon, 1 + \epsilon],
    \]
    the suboptimality gap $\delta_t$ provides an upper bound on the approximation error of the actions:  
    \[
    \delta_t \geq \frac{\gamma}{\eta} D_{- \log} (\widehat{\lambda}\^t \; \| \; \lambda\^t) + \frac{1}{4 \eta} D_\psi (\widehat{\vx}\^t \; \| \; \vx\^t) \geq  \frac{\gamma}{3\eta} 
    \mleft( \frac{\widehat{\lambda}\^t}{\lambda\^t} - 1\mright)^2 + \frac{\mu}{8 \eta} \| \widehat{\vx}\^t - \vx\^t \|_1^2, 
    \]
    as long as we choose the hyperparameter $\alpha \geq 4 \gamma$.
\end{lemma}
\begin{proof}
    By definition and \Cref{lemma:opt_bregman},  
    \[
    \delta_t = F_t(\widehat{\vy}\^t) - F_t(\vy\^t) \geq D_{F_t} (\widehat{\vy}\^t \;\|\; \vy\^t) = \frac{1}{\eta} D_\phi (\widehat{\vy}\^t \;\|\; \vy\^t).
    \]
    On the other hand, by the multiplicative closeness of $\widehat{\lambda}\^t$ and $\lambda\^t$, and since $\epsilon \leq \frac{1}{2}$,  
    \[
    \frac{\widehat{\lambda}\^t}{\lambda\^t} = \frac{\vec 1^\top \widehat{\vy}\^t}{\vec 1^\top \vy\^t} \in \mleft[ \frac{1}{2}, \frac{3}{2} \mright];
    \]
    hence, by \Cref{prop:curvature},  
    \[
    \delta_t & \geq \frac{1}{\eta} D_\phi (\widehat{\vy}\^t \;\|\; \vy\^t) \\
    & \geq \frac{1}{\eta} \mleft( \gamma D_{- \log} (\widehat{\lambda}\^t \; \| \; \lambda\^t) + \frac{1}{4} D_{\psi}(\widehat{\vx}\^t \; \| \; \vx\^t) \mright) \\
    & \geq \frac{\gamma}{\eta} \mleft( - \log \frac{\widehat{\lambda}\^t}{\lambda\^t} + \frac{\widehat{\lambda}\^t}{\lambda\^t} - 1\mright) + \frac{\mu}{8 \eta} \| \widehat{\vx}\^t - \vx\^t \|_1^2 \numberthis{eq:subopt_curvature_temp1}\\
    & \geq \frac{\gamma}{3 \eta} 
   \mleft( \frac{\widehat{\lambda}\^t}{\lambda\^t} - 1\mright)^2 + \frac{\mu}{8 \eta} \| \widehat{\vx}\^t - \vx\^t \|_1^2 \numberthis{eq:subopt_curvature_temp2}
    \]
    In \eqref{eq:subopt_curvature_temp1}, we used the $\mu$-strong convexity of the regularizer with respect to the $\ell_1$ norm and in \eqref{eq:subopt_curvature_temp2} we leveraged the the fact that
    \[
    - \log \omega + \omega - 1 \geq \frac{1}{3} (\omega - 1)^2 \qquad \textrm{for any} \omega = \frac{\widehat{\lambda}\^t}{\lambda\^t} = \frac{\vec 1^\top \widehat{\vy}\^t}{\vec 1^\top \vy\^t} \in \mleft[ \frac{1}{2}, \frac{3}{2} \mright]. 
    \]
\end{proof}

Next, we characterize how the suboptimality gap $\delta_t$ is controlled by the multiplicative tolerance $\epsilon$ in approximate iterations.

\begin{restatelemma}{lemma:subopt_value}
    For every time $t \in [T]$, and for every approximate dynamic learning rate $\widehat{\lambda}\^t > 0$ that is multiplicatively $\epsilon$-close ($\epsilon \in (0, 1/2]$) to the exact dynamic learning rate $\lambda\^t \in (0, \eta]$, i.e.,  
    \[
    \omega_t \defeq \frac{\widehat{\lambda}\^t}{\lambda\^t} \in [1 - \epsilon, 1 + \epsilon],
    \]
    we have that per-round suboptimality is bounded as
    \[
    \delta_t \leq \frac{\alpha}{\eta} \epsilon^2.
    \]

\end{restatelemma}
\begin{proof}
    Without loss of generality, we can assume that  
    \[
    \lambda\^t \in (0, \eta),
    \]
    since otherwise, by a simple test on $f'(\eta)$ (see, e.g., \Cref{algo:geometric_bisection}), the algorithm infers $\widehat{\lambda}\^t = \eta$ and $\delta_t = 0$.  Subsequently, by \Cref{lemma:subopt_basic_ineq},
    \[
    \delta_t \leq \frac{\alpha}{\eta} D_{- \log} (\widehat{\lambda}\^t \; \| \; \lambda\^t ) 
    \]
    with $\omega_t = \widehat{\lambda}\^t/\lambda\^t \in [1 - \epsilon, 1 + \epsilon] \subset [1/2, 3/2]$.
    In turn,
    \[
    D_{- \log} (\widehat{\lambda}\^t \; \| \; \lambda\^t ) & = - \log w_t + w_t - 1 \\
    & \leq (w_t - 1)^2 \\
    & \leq \epsilon^2,
    \]
    where we used the fact that for all $x \in [1/2, 3/2]$,
    \[
    - \log x + x - 1 \leq (x - 1)^2.
    \]
\end{proof}

With \Cref{lemma:subopt_basic_ineq,lemma:subopt_approx_error_bound,lemma:subopt_value} in place, we are ready to study how the per-round suboptimality terms $\delta_t$ arise in the nonnegative RVU bounds and hence in the regret analysis of \ours. We show an extension of \Cref{theorem:nonnegative_rvu} under approximate iterations of the dynamic learning rate. We omit the rederivation of all details, especially new constants in the interest of space, and instead focus on the broad idea.

\begin{theorem}[Extension of \Cref{theorem:nonnegative_rvu} under approximate iterations  ]\label{theorem:nonnegative_rvu_approx} 
    Under the bounded utilities assumption~\ref{assumption:bounded}, by choosing a sufficiently small learning rate $\eta \leq \min\{ 3\gamma/80, \mu/(32\sqrt{2}) \}$ and $\alpha \geq 4 \gamma + \mu$, the cumulative regret $\reg\^T$ incurred by \ours with approximate iterates $(\widehat{\lambda}\^t, \widehat{\vx}\^t)$ up to time $T$ is bounded as follows:  
    \[
        \max\{\reg\^T, 0\} = \tildereg\^T & \leq 3 + \frac{1}{\eta} \mleft( \alpha \log T + \mathcal{R} \mright) + \eta \frac{48}{\mu} \sum_{t = 1}^{T-1} \|\nut\^{t+1} - \nut\^{t}\|_{\infty}^2  \\
        & \qquad - \frac{1}{\eta} \frac{\mu}{192} \sum_{t = 1}^{T-1} \| \widehat{\vx}\^{t+1} -  \widehat{\vx}\^{t} \|_1^2  + \frac{1}{\eta} \frac{\mu}{32} \sum_{t = 1}^T \delta_t \\
        & \qquad + \sqrt{\min\mleft\{ \frac{81 \gamma^2}{64000}, 2048 \sqrt{2} \mright\} } \sum_{t = 1}^T  \sqrt{\delta_t}.
    \numberthis{eq:positive_rvu_rhs_approx},
    \]
    where $\mathcal{R} \defeq \max_{\vx \in \mathcal{X}} \psi(\vx)$.
\end{theorem}
\begin{proof}
    Starting from the regret under approximate iterations, for any comparator $\vy \in \Omega$, we have
    \[
    \sum_{t = 1}^T  \mleft\langle {\vy} -  \widehat{\vy}\^{t},  \ut\^{t} \mright\rangle & = \sum_{t = 1}^T  \mleft\langle {\vy} -  \vy\^{t},  \ut\^t \mright\rangle + \sum_{t = 1}^T  \mleft\langle \vy\^{t} - \widehat{\vy}\^{t} ,  \ut\^t \mright\rangle \\
    & \leq \sum_{t = 1}^T  \mleft\langle {\vy} -  \vy\^{t},  \ut\^t \mright\rangle + \sum_{t = 1}^T  \| \vy\^{t} - \widehat{\vy}\^{t} \|_1 \cdot  \| \ut\^t \|_\infty  \numberthis{eq:approx_non_rvu_temp1} \\
    & \leq \sum_{t = 1}^T  \mleft\langle {\vy} -  \vy\^{t},  \ut\^t \mright\rangle + 2 \sum_{t = 1}^T  \| \lambda\^t \vx\^{t} - \widehat{\lambda}\^t \widehat{\vx}\^{t} \|_1 \numberthis{eq:approx_non_rvu_temp2} \\
    & \leq \sum_{t = 1}^T  \mleft\langle {\vy} -  \vy\^{t},  \ut\^t \mright\rangle + 2 \sum_{t = 1}^T  \sqrt{ 2 \mleft|\lambda\^t - \widehat{\lambda}\^t \mright|^2 + 2 \| \vx\^{t} - \widehat{\vx}\^{t} \|_1^2 } \numberthis{eq:approx_non_rvu_temp3} \\
    & \leq \sum_{t = 1}^T  \mleft\langle {\vy} -  \vy\^{t},  \ut\^t \mright\rangle + 2 \sum_{t = 1}^T  \sqrt{ \frac{2}{c_0} \delta_t } \numberthis{eq:approx_non_rvu_temp4},
    \]
    with universal constant $c_0 = \min\mleft\{ \frac{512000}{81 \gamma^2}, \frac{1}{256 \sqrt{2}} \mright\}$. In our analysis, \eqref{eq:approx_non_rvu_temp1} is derived by Hölder's inequality, \eqref{eq:approx_non_rvu_temp2} is derived by $\| \ut\^t \|_\infty \leq 2$, \eqref{eq:approx_non_rvu_temp3} follows by \Cref{lemma:correct_y} and \eqref{eq:approx_non_rvu_temp4} follows since we know that by \Cref{lemma:subopt_approx_error_bound},
    \[
    \delta_t & \geq \frac{\gamma}{3\eta} 
    \mleft( \frac{\widehat{\lambda}\^t}{\lambda\^t} - 1\mright)^2 + \frac{\mu}{8 \eta} \| \widehat{\vx}\^t - \vx\^t \|_1^2 \\
    & = \frac{\gamma}{3\eta {\lambda\^t}^2} 
    \mleft( \widehat{\lambda}\^t - {\lambda\^t} \mright)^2 + \frac{\mu}{8 \eta} \| \widehat{\vx}\^t - \vx\^t \|_1^2 \\
    & \geq \frac{\gamma}{3 \eta^3} 
    \mleft( \widehat{\lambda}\^t - {\lambda\^t} \mright)^2 + \frac{\mu}{8 \eta} \| \widehat{\vx}\^t - \vx\^t \|_1^2 \numberthis{eq:approx_non_rvu_temp5} \\
    & \geq \frac{\gamma}{3 \mleft(3 \gamma/80 \mright)^3} 
    \mleft( \widehat{\lambda}\^t - {\lambda\^t} \mright)^2 + \frac{\mu}{8 \mleft(\mu/ 32 \sqrt{2}\mright) } \| \widehat{\vx}\^t - \vx\^t \|_1^2 \numberthis{eq:approx_non_rvu_temp6} \\
    & = \frac{512000}{81 \gamma^2} \mleft( \widehat{\lambda}\^t - {\lambda\^t} \mright)^2 + \frac{1}{256 \sqrt{2}} \| \widehat{\vx}\^t - \vx\^t \|_1^2 \\
    & \geq c_0 \mleft( \mleft( \widehat{\lambda}\^t - {\lambda\^t} \mright)^2 + \| \widehat{\vx}\^t - \vx\^t \|_1^2 \mright),
    \]
    with universal constant $c_0 = \min\mleft\{ \frac{512000}{81 \gamma^2}, \frac{1}{256 \sqrt{2}} \mright\}$. In our derivations, \eqref{eq:approx_non_rvu_temp5} follows by $\lambda\^t \leq \eta$ and \eqref{eq:approx_non_rvu_temp6} follows by $\eta \leq \min\{ 3\gamma/80, \mu/(32\sqrt{2}) \}$.

    For the term $\sum_{t = 1}^T  \mleft\langle {\vy} -  \vy\^{t},  \ut\^t \mright\rangle$ on the right-hand side of \eqref{eq:approx_non_rvu_temp4}, by the nonnegative RVU bounds for exact iterates of \ours in \Cref{theorem:nonnegative_rvu},
    \[
    \sum_{t = 1}^T  \mleft\langle {\vy} -  \vy\^{t},  \ut\^t \mright\rangle \leq 3 + \frac{1}{\eta} \mleft( \alpha \log T + \mathcal{R} \mright) + \eta \frac{48}{\mu} \sum_{t = 1}^{T-1} \|\nut\^{t+1} - \nut\^{t}\|_{\infty}^2  - \frac{1}{\eta} \frac{\mu}{64} \sum_{t = 1}^{T-1} \| {\vx}\^{t+1} -  {\vx}\^{t} \|_1^2. \numberthis{eq:temp_rvu1}
    \]
    On the other hand, 
    \[
    \sum_{t = 1}^{T-1} \| {\vx}\^{t+1} -  {\vx}\^{t} \|_1^2 & = \sum_{t = 1}^{T-1} \mleft\| \mleft( {\vx}\^{t+1} - \widehat{\vx}\^{t + 1}\mright) + \mleft(\widehat{\vx}\^{t + 1} - \widehat{\vx}\^{t}\mright) + \mleft(  \widehat{\vx}\^{t} -  {\vx}\^{t} \mright) \mright\|_1^2 \\
    & \geq \sum_{t = 1}^{T-1} \mleft( \mleft\| \widehat{\vx}\^{t + 1} - \widehat{\vx}\^{t} \mright\|_1 - \mleft\|  {\vx}\^{t+1} - \widehat{\vx}\^{t + 1} \mright\|_1 - \mleft\|  {\vx}\^{t} - \widehat{\vx}\^{t} \mright\|_1 \mright)^2 \numberthis{eq:temp_rvu2} \\
    & \geq \sum_{t = 1}^{T-1}  \mleft( \frac{1}{3} \mleft\| \widehat{\vx}\^{t + 1} - \widehat{\vx}\^{t} \mright\|_1^2 - \mleft\|  {\vx}\^{t+1} - \widehat{\vx}\^{t + 1} \mright\|_1^2 - \mleft\|  {\vx}\^{t} - \widehat{\vx}\^{t} \mright\|_1^2  \mright) \numberthis{eq:temp_rvu3} \\
    & \geq \sum_{t = 1}^{T-1}  \mleft( \frac{1}{3} \mleft\| \widehat{\vx}\^{t + 1} - \widehat{\vx}\^{t} \mright\|_1^2 - \delta_{t+1} - \delta_t  \mright) \numberthis{eq:temp_rvu4} \\
    & \geq \frac{1}{3} \sum_{t = 1}^{T-1} \mleft\| \widehat{\vx}\^{t + 1} - \widehat{\vx}\^{t} \mright\|_1^2 - 2 \sum_{t = 1}^{T} \delta_t, \numberthis{eq:path_length_delta}
    \]
    where \eqref{eq:temp_rvu2} follows from the triangle inequality, \eqref{eq:temp_rvu3} is derived using the following inequality for all positive numbers $a, b, c \in \mathbb{R}_+$,
    \[
    (a - b - c)^2 \geq \tfrac{1}{3} a^2 - b^2 - c^2,
    \]
    and \eqref{eq:temp_rvu4} follows by the definition of $\delta_t$.

    Combining \eqref{eq:approx_non_rvu_temp4}, \eqref{eq:temp_rvu1} and \eqref{eq:path_length_delta}, we complete the proof:
    \[
    \sum_{t = 1}^T  \mleft\langle {\vy} -  \widehat{\vy}\^{t},  \ut\^{t} \mright\rangle & \leq \sum_{t = 1}^T  \mleft\langle {\vy} -  \vy\^{t},  \ut\^t \mright\rangle + 2 \sqrt{\frac{2}{c_0}} \sum_{t = 1}^T  \sqrt{\delta_t} \\
    & \leq 3 + \frac{1}{\eta} \mleft( \alpha \log T + \mathcal{R} \mright) + \eta \frac{48}{\mu} \sum_{t = 1}^{T-1} \|\nut\^{t+1} - \nut\^{t}\|_{\infty}^2  \\
    & \qquad - \frac{1}{\eta} \frac{\mu}{64} \sum_{t = 1}^{T-1} \| {\vx}\^{t+1} -  {\vx}\^{t} \|_1^2 + 2 \sqrt{\frac{2}{c_0}} \sum_{t = 1}^T  \sqrt{\delta_t} \\
    & \leq 3 + \frac{1}{\eta} \mleft( \alpha \log T + \mathcal{R} \mright) + \eta \frac{48}{\mu} \sum_{t = 1}^{T-1} \|\nut\^{t+1} - \nut\^{t}\|_{\infty}^2  \\
    & \qquad - \frac{1}{\eta} \frac{\mu}{192} \sum_{t = 1}^{T-1} \| \widehat{\vx}\^{t+1} -  \widehat{\vx}\^{t} \|_1^2  + \frac{1}{\eta} \frac{\mu}{32} \sum_{t = 1}^T \delta_t \\
    & \qquad + 2 \sqrt{\frac{2}{c_0}} \sum_{t = 1}^T  \sqrt{\delta_t}.
    \]
\end{proof}

We are now ready to state and prove the main result of this section: \ours with a multiplicatively $\epsilon$-close approximate dynamic learning rate, where $\epsilon = \tfrac{1}{T}$, enjoys the same guarantees as \ours with exact iterates. More concretely, \ours with a multiplicatively $\frac{1}{T}$-approximate dynamic learning rate achieves individual regret of order $O(n \Gamma(d) \log T)$. As a direct byproduct, we can infer the iteration complexity of \ours, as elaborated extensively in \Cref{thm:iter_comp,table:iteration_complexity,app:iteration_comp}.

\begin{restatetheorem}{thm:main_approx}[Extension of \Cref{theorem:regret_final_bound} under approximate dynamic learning rates]
    If all players $i \in [n]$ follow \ours with a multiplicatively $\epsilon$-approximated dynamic learning rate $\lambda\^t$ at each iteration $t$,  
    \[
    \widehat{\lambda}\^t \in \mleft[ (1 - \epsilon) \lambda\^t, (1 + \epsilon) \lambda\^t \mright], \qquad \text{with } \epsilon = \frac{1}{T},
    \]
    and a small enough learning-rate cap $\eta = O_{\scaleto{T}{4pt}}(1)$, then the regret for each player $i \in [n]$ is bounded by  
    \[
    \reg_i\^T = O \big( n \Gamma_\psi(d) \log T \big), 
    \]
    where $\Gamma_\psi(d) = \gamma/\mu$. Moreover, the algorithm for each player $i \in [n]$ is adaptive to adversarial utilities; that is, the regret that each player incurs satisfies $\reg_i\^T = O(\sqrt{T \log d})$.
\end{restatetheorem}
\begin{proof}
    Given the nonnegative RVU bounds under approximate iterates in \Cref{theorem:nonnegative_rvu_approx}, we observe that there are only two additional terms stemmed from the approximation of the dynamic learning rate compared to \Cref{theorem:nonnegative_rvu}. These two terms, up to constant factors, are: 
    \begin{enumerate}
        \item the sum of suboptimalities, $\sum_{t = 1}^T \delta_t$, and 
        \item the sum of the square roots of suboptimalities, $\sum_{t = 1}^T \sqrt{\delta_t}$.
    \end{enumerate}

    In sequel, by \Cref{lemma:subopt_value}, we can connect the suboptimality terms to the approximation factor $\epsilon$, and hence by choosing $\epsilon = \tfrac{1}{T}$, we can bound these two terms by constants:  
    \begin{enumerate}
        \item For the first term, we have  
        \[
        \sum_{t = 1}^T \delta_t & \leq \frac{\alpha}{\eta} \epsilon^2 
         = \frac{\alpha}{\eta} \sum_{t = 1}^T \frac{1}{T^2} 
         = \frac{\alpha}{\eta} \frac{1}{T} 
         = o_{\scaleto{T}{4pt}}(1),
        \]
        \item For the second term,  
        \[
        \sum_{t = 1}^T \sqrt{\delta_t} & \leq \sqrt{\frac{\alpha}{\eta} \epsilon^2} 
         = \sqrt{\frac{\alpha}{\eta}} \sum_{t = 1}^T \epsilon 
         = \sqrt{\frac{\alpha}{\eta}} \sum_{t = 1}^T \frac{1}{T}  
         = O_{\scaleto{T}{4pt}}(1),
        \]
    \end{enumerate}
    Thus, the nonnegative bounds also hold for approximate iterations of \ours. The rest of the argument follows the same path as \Cref{theorem:regret_final_bound}. Similar to the proof of \Cref{prop:path_length}, we can bound the path length and, consequently, the individual regret as in \Cref{theorem:regret_final_bound}. To avoid repetition, we omit the details.
\end{proof}

\section{Details for Cautious Optimistic Multiplicative Weights Update (\COMWU)} \label{app:comwu}

In this section, we examine the special case of Cautious Optimistic Multiplicative Weights Update (\COMWU) in greater detail. \COMWU is a key instantiation of \ours, where negative entropy is chosen as the regularizer,  
\[
\psi (\vx) = \sum_{\ind = 1}^d \vx[\ind] \log \vx[\ind].
\]  

\COMWU updates the next distribution of play, similar to \Opthedge, proportionally to the exponential of the optimistic regret accumulated for each action, following the formula,
\[
    \vx\^{t}[k] \defeq\frac{\exp\{\lambda\^t \at\^t[k]\}}{\sum_{k' \in \mathcal{A}} \exp\{\lambda\^t \at\^t[k']\}} \qquad \forall k \in \mathcal{A},
    \numberthis{eq:softmax}
\]
where the learning rate $\lambda\^t$ is dynamically adjusted according to the following optimization problem
\[
    \lambda\^t \defeq \argmax_{\lambda \in (0, \eta]} \left\{f(\lambda; \at\^t) \defeq \alpha \log \lambda + \log \left( \sum_{\ind=1}^d e^{\lambda \at\^t[\ind]} \right) \right\}, \numberthis{eq:opt_problem_lambda}
\]
where the hyperparameter $\alpha$ is chosen as $\alpha = \Theta(\log^2 d)$

We prove in the following proposition that this learning rate control problem is induced by the negative entropy instantiation of \ours.

\begin{restateproposition}{proposition:closed_form_comwu}
    \ours with negative entropy admits the following dynamic learning rate control problem,
    \[
    \lambda\^t \defeq \argmax_{\lambda \in (0, \eta]} \left\{f(\lambda; \at\^t) \defeq \alpha \log \lambda + \log \left( \sum_{\ind=1}^d e^{\lambda \at\^t[\ind]} \right) \right\}. 
    \]
\end{restateproposition}

\begin{proof}
    In this case the mixed dynamics of \COMWU evolves according to,
    \[
        \vstack{\lambda\^t}{\vx\^t} \gets \argmax_{\lambda \in (0, \eta], \vx \in \Delta^d}\mleft\{ {\lambda} \mleft\langle  \at\^t, \vx \mright\rangle + \alpha \log \lambda - \sum_{\ind=1}^d \vx[\ind] \log \vx[\ind]\mright\}.
        \numberthis{eq:drlc_ftrl_special}
    \]
    By KKT condition\footnote{There is no need to do this step, as equivalently one can remember the of underlying \OFTRL, which is \Opthedge; however, we do this for completeness.}, 
    \[
        \lambda\^t \at\^t - \big[\log x\^t[1], \log x\^t[2], \ldots, \log x\^t[d]\big]^\top = \mu \in \bbR.
    \]
    Therefore, $\vx\^t[r] \propto \exp\{\at\^t[r]\}$. Now, after renormalization since $\vx \in \Delta^d$,
    \[
        \vx\^t[r] = \frac{\exp\{\lambda\^t \at\^t[r]\}}{\sum_{\ind = 1}^{d} \exp\{\lambda\^t \at\^t[\ind]\}}.
    \]
    In sequel, plugging this update formula into \eqref{eq:drlc_ftrl_special} entails,
    \[
        & \argmax_{\lambda \in [0, \eta]} \left\{ \lambda \sum_{r=1}^d \at\^t[r] \frac{e^{\lambda \at\^t[r]}}{\sum_{\ind=1}^d e^{\lambda \at\^t[\ind]}} + \alpha \log \lambda - \sum_{r=1}^d \frac{e^{\lambda \at\^t[r]}}{\sum_{r=1}^d e^{\lambda \at\^t[\ind]}} \log \Big( \frac{e^{\lambda \at\^t[r]}}{\sum_{\ind=1}^d e^{\lambda \at\^t[\ind]}} \Big) \right\}                                                  \\
        &  = \argmax_{\lambda \in [0, \eta]}  \\ 
        & \qquad \left\{ \lambda \sum_{r=1}^d \at\^t[r] \frac{e^{\lambda \at\^t[r]}}{\sum_{\ind=1}^d e^{\lambda \at\^t[\ind]}} + \alpha \log \lambda - \lambda \sum_{r=1}^d \at\^t[r] \frac{e^{\lambda \at\^t[r]}}{\sum_{\ind=1}^d e^{\lambda \at\^t[\ind]}} + \frac{ \sum_{r=1}^d  e^{\lambda \at\^t[r]}}{\sum_{\ind=1}^d e^{\lambda \at\^t[\ind]}} \log \Big( {\sum_{\ind=1}^d e^{\lambda \at\^t[\ind]}} \Big) \right\} \\
        & \;\; = \argmax_{\lambda \in [0, \eta]} \left\{ \log \Big( {\sum_{\ind=1}^d e^{\lambda \at\^t[\ind]}} \Big) + \alpha \log \lambda \right\}.          
    \]
    \end{proof}

    Now that the learning rate control problem is determined, we present the full algorithm for \COMWU in \Cref{algo:comwu}. When the maximum regret $\max_k \{\at\^t[k]\}$ accumulated on the actions is not too negative, i.e., $\max_k \{\at\^t[k]\} \geq - \alpha + \log d$ it is easy to observe that optimal solution is $\lambda\^t = \eta$. We formalize this observation in the following lemma.

    \begin{restatelemma}{lemma:lambda_is_one}
        Given an arbitrary regret vector $\at\^t \in \bbR^d$, and $\lambda\^t$ as the solution to
        \begin{align}
            \lambda\^t \defeq \argmax_{\lambda \in (0, \eta]} \left\{f(\lambda; \at\^t) \defeq \alpha \log \lambda + \log \Big( {\sum_{\ind=1}^d e^{\lambda \at\^t[\ind]}} \Big)  \right\},
        \end{align}
        then as long as $\max_{r \in [d]} \{ \at\^t[r]\} \geq - \alpha + \log d$, we have $\lambda\^t = \eta$.
    \end{restatelemma}
    \begin{proof}
    By KKT conditions, strong concavity and uniqueness of the solution, it is obvious that whenever $f^\prime (\eta; \at) \geq 0$, then $\widehat{\lambda} = \eta$. 
    Taking derivatives from $f(\lambda; \at)$ w.r.t. $\lambda$,
    \[
        \frac{d f(\lambda; \at)}{d \lambda}     & = \frac{\sum_{\ind = 1}^d \at[\ind] e^{\lambda \at[\ind]}}{\sum_{\ind = 1}^d e^{\lambda \at[\ind]}} + \frac{\alpha}{\lambda}.     
    \]
    Substituting $\at[r] = \frac{\log \vx[r] + \log \Gamma}{\lambda}$ for every coordinate $r \in [d]$, where $\Gamma = \sum_{\ind = 1}^{d} \exp\{\lambda \at[\ind]\}$ by the update step \eqref{eq:softmax},
    \[
    \frac{d f(\lambda; \at)}{d \lambda} & = \frac{\alpha}{\lambda} - \frac{1}{\lambda} \sum_{\ind = 1}^{d} \vx[\ind] \log \vx[\ind] + \frac{1}{\eta} \log \mleft( \sum_{\ind = 1}^d e^{\lambda \at[\ind]}\mright)
    \]
    Thus, 
    \[
        f^\prime(\eta; \at) 
        & = \frac{\alpha}{\eta} + \frac{1}{\eta} \sum_{\ind = 1}^{d} \vx[\ind] \log \vx[\ind] + \frac{1}{\eta} \log \mleft( \sum_{\ind = 1}^{d} e^{\lambda \at[\ind]}\mright) \\
        & \geq \frac{\alpha}{\eta} - \frac{\log d}{\eta} + \frac{1}{\eta} \max_{r \in [d]} \{ \at[r]\}                                     \\
        & \geq \frac{1}{\eta} \mleft( \alpha - \log d + \max_{r \in [d]} \{ \at[r]\} \mright).
    \]
    Thus, whenever $\max_{r \in [d]} \{ \at[r]\} \geq - \alpha + \log d$, then $\hat{\lambda} = \eta$.
\end{proof}    

Additionally, for completeness, we provide the equivalent Formulation~\ref{eq:lifted_FTRL} of \COMWU, incorporating the usual invertible change of variables $\vy = \lambda \vx$:
\[
 \vy\^t \gets \argmax_{\vy \in (0,1]\Delta^d } \mleft\{ {\eta} \mleft\langle  \at\^t, \vy \mright\rangle + (\alpha + 1) \log \left(\sum_{\ind=1}^d \vy[\ind]\right) - \frac{1}{\sum_{\ind=1}^d \vy[\ind]} \sum_{\ind=1}^d \vy[\ind] \log \vy[\ind] \mright\}.
\]

We conclude this section by demonstrating that the learning-rate control problem for \COMWU is strongly concave and self-concordant. Consequently, using Newton-type algorithms, it achieves an iteration complexity of $O(d \log \log T)$. While strong concavity and self-concordance was already established in \Cref{app:proof_convexity,proof_selfconcordance} for broad Legendre and log-dominated regularizers, including negative entropy, we provide the following derivations for illustrative purposes, given the closed form of the learning-rate control problem~\eqref{eq:opt_problem_lambda}.

\begin{proposition}[Properties of learning rate control step of \COMWU] \label{lemma:self_concordance}
    For any $\at \in \bbR^d$, the rate control objective $f(\lambda; \at)$, defined in
    \[
        \hat{\lambda} \defeq \argmax_{\lambda \in (0, \eta]} \left\{f(\lambda; \at) \defeq \alpha \log \lambda + \log \Big( {\sum_{\ind=1}^d e^{\lambda \at[\ind]}} \Big)  \right\},
    \]
    satisfies the following properties:
    \begin{itemize}
        \item Strong concavity: $f''(\lambda; \at) \leq -(\alpha - \log^2 d)/\lambda^2$ for all $\lambda \in (0,\infty)$.
        \item Self-concordance: $(f'''(\lambda; \at))^2 \leq -4 f''(\lambda; \at)^{3}$,
    \end{itemize}
    where all derivatives are with respect to $\lambda$.
\end{proposition}

\begin{proof}
    Taking derivatives from $f(\lambda; \at)$ w.r.t. $\lambda$,
    \[
        \frac{d f(\lambda; \at)}{d \lambda}     & = \frac{\sum_{\ind = 1}^d \at[\ind] e^{\lambda \at[\ind]}}{\sum_{\ind = 1}^d e^{\lambda \at[\ind]}} + \frac{\alpha}{\lambda}.                                                                                                                                                        \\
        \frac{d^2 f(\lambda; \at)}{d \lambda^2} & = \frac{ \Big( \sum_{\ind = 1}^d \at[\ind]^2 e^{\lambda \at\^t[\ind]} \Big) \Big(\sum_{\ind = 1}^d e^{\lambda \at\^t[\ind]} \Big) - \Big(\sum_{\ind = 1}^d \at[\ind] e^{\lambda \at[\ind]} \Big)^2}{\Big(\sum_{\ind = 1}^d e^{\lambda \at[\ind]} \Big)^2} - \frac{\alpha}{\lambda^2} \\
        & = \frac{\sum_{\ind = 1}^d \at[\ind]^2 e^{\lambda \at[\ind]} }{\sum_{\ind = 1}^d e^{\lambda \at[\ind]}} - \mleft( \frac{\sum_{\ind = 1}^d \at[\ind] e^{\lambda \at[\ind]}}{\sum_{\ind = 1}^d e^{\lambda \at[\ind]}}\mright)^2 - \frac{\alpha}{\lambda^2}.
    \]
    Substituting $\at[r] = \frac{\log \vx[r] + \log \Gamma}{\lambda}$ for every coordinate $r \in [d]$, where $\Gamma = \sum_{\ind = 1}^{d} \exp\{\lambda \at[\ind]\}$ and $\vx \in \Delta^d$, we get
    \[
        \frac{d^2 f(\lambda; \at)}{d \lambda^2} & = \frac{1}{\lambda^2} \sum_{\ind = 1}^{d } \big( \vx[\ind] (\log \vx[\ind] + \log \Gamma)^2 \big) - \frac{1}{\lambda^2} \Big ( \sum_{\ind =1}^{k}  \vx[\ind] (\log \vx[\ind] + \log \Gamma) \Big)^2 -\frac{\alpha}{\lambda^2} \\
        & = \frac{1}{\lambda^2} \sum_{\ind = 1}^{d } \big( \vx[\ind] (\log \vx[\ind] + \log \Gamma)^2 \big) - \frac{1}{\lambda^2} \Big ( \sum_{\ind =1}^{k}  \vx[\ind] \log \vx[\ind]  + \log \Gamma \Big)^2 -\frac{\alpha}{\lambda^2}  \\
        & = \frac{1}{\lambda^2} \mleft( \sum_{\ind =1}^{d} \vx[\ind] \log^2 \vx[\ind] - \Big( \sum_{\ind =1}^{k} \vx[\ind] \log \vx[\ind] \Big)^2 - \alpha \mright)                                                                   \\
        & \leq \frac{1}{\lambda^2} \mleft( \log^2 d - \alpha \mright).
    \]
    Therefore, the function $f(\lambda; \at)$ is $\frac{\alpha - \log^2 d}{\lambda^2}$-strongly concave.

    Taking derivatives,
    \[
        \frac{d^3 f(\lambda; \at)}{d \lambda^3} & = \frac{\sum_{\ind = 1}^d \at[\ind]^3 e^{\lambda \at[\ind]} }{\sum_{\ind = 1}^d e^{\lambda \at[\ind]}} - \frac{\Big( \sum_{\ind = 1}^d \at[\ind]^2 e^{\lambda \at[\ind]} \Big) \Big(  \sum_{\ind = 1}^d \at[\ind] e^{\lambda \at[\ind]}\Big)}{\Big(\sum_{\ind = 1}^d e^{\lambda \at[\ind]}\Big)^2}                                                                                                      \\
        & \quad - 2 \mleft( \frac{\sum_{\ind = 1}^d \at[\ind] e^{\lambda \at[\ind]}}{\sum_{\ind = 1}^d e^{\lambda \at[\ind]}} \mright) \mleft( \frac{\sum_{\ind = 1}^d \at[\ind]^2 e^{\lambda \at[\ind]} }{\sum_{\ind = 1}^d e^{\lambda \at[\ind]}} - \mleft( \frac{\sum_{\ind = 1}^d \at[\ind] e^{\lambda \at[\ind]}}{\sum_{\ind = 1}^d e^{\lambda \at[\ind]}}\mright)^2 \mright) + \frac{2 \alpha}{\lambda^3}.
    \]
    Substituting $\at[r] = \frac{\log \vx[r] + \log \Gamma}{\lambda}$ for every coordinate $r \in [d]$, where $\Gamma = \sum_{\ind = 1}^{d} \exp\{\lambda \at[\ind]\}$ and $\vx \in \Delta^d$, we get
    \[
        & \frac{d^3 f(\lambda; \at)}{d \lambda^3}                                                                                                                                                                                                                                     \\
        & \quad  = \frac{1}{\lambda^3} \sum_{\ind = 1}^{d } \vx[\ind] (\log \vx[\ind] + \log \Gamma)^3 - 3 \frac{1}{\lambda^3} \Big( \sum_{\ind = 1}^{d } \vx[\ind] (\log \vx[\ind] + \log \Gamma)^2 \Big)  \Big( \sum_{\ind = 1}^{d } \vx[\ind] (\log \vx[\ind] + \log \Gamma) \Big)  \\
        & \quad \qquad + 2 \frac{1}{\lambda^3} \Big( \sum_{\ind = 1}^{d } \vx[\ind] (\log \vx[\ind] + \log \Gamma) \Big)^3 + \frac{2 \alpha}{\lambda^3}                                                                                                                          \\
        & \quad = \frac{1}{\lambda^3} \mleft( \sum_{\ind = 1}^{d} \vx[k]\log^3 \vx[k] - 3 \big( \sum_{\ind = 1}^{d} \vx[k] \log^2 \vx[k] \big) \big( \sum_{\ind = 1}^{d} \vx[k] \log \vx[k] \big) + 2 \big ( \sum_{\ind = 1}^{d} \vx[k] \log \vx[k] \big)^3  + 2 (\alpha - 1) \mright) \\
        & \quad \leq \frac{1}{\lambda^3} \mleft( 3 \log^3 d +  2 \alpha \mright)
    \]
    Thus,
    \[
        \mleft(\frac{d^3 f(\lambda; \at)}{d \lambda^3}\mright)^2 \leq \mleft(\frac{3 \log^3 d +  2 \alpha}{\lambda^3}  \mright)^2 \leq 4 \mleft( \frac{\alpha - \log^2 d}{\lambda^2} \mright)^3 \leq
        - 4 \mleft( \frac{d^2 f(\lambda; \at)}{d \lambda^3} \mright)^3,
    \]
    for all $d \geq 1$,\footnote{the equality happens at $d = 1$.} where $\alpha = \beta \log^2 d + 2 \log d + 1$ and $\beta \geq 2$, hence the proof is concluded.
\end{proof}

We note that the strong concavity established in \Cref{lemma:self_concordance} is slightly stronger (constant factors) than the general result in \Cref{thm:strong_convex_learning_rate_control}. We are now ready to prove the iteration complexity of \COMWU, in the following corollary.

\begin{restatecorollary}{coro:newton_comwu}
    Given any $\at \in \bbR^d$ and a desired multiplicative accuracy $\epsilon > 0$, $O(\log \log (1/\epsilon))$ iterations of Newton's method, starting from the initialization point $\lambda_0 = \alpha / \big(- \max_{r \in [d]} \{\at[r]\}\big)$, are sufficient to compute a point $\widehat{\lambda}$ that approximates $\lambda^* \defeq \argmax_{\lambda \in (0, \eta]} f(\lambda; \at)$ with relative error at most $\epsilon$, i.e., $(1 - \epsilon)\lambda^* < \widehat{\lambda} < (1 + \epsilon)\lambda^*$.
\end{restatecorollary}

\begin{proof}
    By the self-concordance result in \Cref{lemma:self_concordance}, it is enough to show that $\lambda_0$ is a good initialization point for the Newton algorithm. Specifically, by demonstrating that the intrinsic norm of the second-order ascent direction of $f$ at $\lambda_0$ is small, we conclude that it is a valid initialization.

    Recall that from proof of \Cref{lemma:self_concordance},
    \[
        \frac{d f(\lambda; \at)}{d \lambda}      & = \frac{\sum_{\ind = 1}^d \at[\ind] e^{\lambda \at[\ind]}}{\sum_{\ind = 1}^d e^{\lambda \at[\ind]}} + \frac{\alpha}{\lambda}.                                                                                                                            \\
        \frac{d^2 f(\lambda; \at)}{d \lambda^2} & = \frac{\sum_{\ind = 1}^d \at[\ind]^2 e^{\lambda \at[\ind]} }{\sum_{\ind = 1}^d e^{\lambda \at[\ind]}} - \mleft( \frac{\sum_{\ind = 1}^d \at[\ind] e^{\lambda \at[\ind]}}{\sum_{\ind = 1}^d e^{\lambda \at[\ind]}}\mright)^2 - \frac{\alpha}{\lambda^2}.
    \]
    Let us choose the primary guess $\lambda_0 = \dfrac{\alpha}{\min_{r \in [d]}{(- \at[r]})}$. By change of variables, we know that
    \[
        \frac{d f}{d \lambda}(\lambda_0; \at) = \frac{\alpha - 1}{\lambda_0} + \frac{1}{\lambda_0} \ent{\vx} + \frac{1}{\lambda_0} \log \mleft( \sum_{\ind = 1}^{d} e^{\lambda_0 \at[\ind]}\mright),
    \]
    where $\ent{\vx}$ is the negative entropy of the vector $\vx [r] = \dfrac{\exp\{\lambda_0 \at[r]\}}{\sum_{\ind = 1}^{d} \exp\{\lambda_0 \at[\ind]\}}$. Therefore, $0 \geq H(\vx) \geq - \log d$.
    On the other hand, by \Cref{lemma:softmax_bound},
    \[
        \max_{r \in [d]} \{\at[r]\} \leq \frac{1}{\lambda_0} \log \mleft( \sum_{\ind = 1}^{d} e^{\lambda_0 \at[\ind]}\mright) \leq  \max_{r \in [d]} \{\at[r]\} + \frac{\log d}{\lambda_0}.
    \]
    Hence,
    \[
        - \frac{\log d}{\lambda_0} = \min_{r \in [d]} \{-\at[r]\} +  \max_{r \in [d]} \{\at[r]\}- \frac{\log d}{\lambda_0} \leq f^\prime(\lambda_0, \at) \leq  \min_{r \in [d]} \{-\at[r]\} +  \max_{r \in [d]} \{\at[r]\} + \frac{\log d}{\lambda_0} = \frac{\log d}{\lambda_0}.
    \]
    On the other hand, from \Cref{lemma:self_concordance} we know that
    \[
        | f^{\dprime}(\lambda_0; \at) | \geq \frac{1}{\lambda_0^2} \mleft( \alpha - \log^2 d\mright).
    \]
    Therefore, the local norm of Newton step $n(\lambda_0)$ is
    \[
        \| n(\lambda_0) \|_{f^{\dprime}(\lambda_0; \at)}^2 & = \frac{f^\prime(\lambda_0, \at)^2}{|f^{\dprime}(\lambda_0, \at) |} \\
        & \leq \frac{\log^2 d}{(\beta - 1) \log^2 d + 2 \log d + 1}           \\
        & \leq \frac{1}{\beta - 1},
    \]
    controlled by the hyperparameter $\beta \geq 40$, where $\alpha = \beta \log^2 d + 2 \log d + 1$. Thus, the Newton decrement $\| n(\lambda_0) \|_{f^{\dprime}(\lambda_0; \at)}^2$ is bounded by $\frac{1}{4}$; hence $\lambda_0$ is in the quadratic convergence regime of Newton’s method~\citep{boyd2004convex,nesterov1994interior}. This completes the proof.
\end{proof}

For our proof of \Cref{coro:newton_comwu}, we use the following lemma connecting the log-sum-exp function to the maximum function.

\begin{lemma} \label{lemma:softmax_bound}
    For any set of numbers $s_1, s_2, \dots, s_d$, the log-sum-exp function $g(t) \defeq h(t s_1, t s_2, \dots, t s_d) = \log \mleft( \sum_{\ind = 1}^{d} \exp\{t s_\ind\}\mright) $ satisfies,
    \[
        \max\{ s_1, s_2, \dots, s_d \} \leq \frac{g(t)}{t} \leq \max\{ s_1, s_2, \dots, s_d \} + \frac{\log d}{t}.
    \]
\end{lemma}

\section{Details for Kernelized Cautious Optimistic Multiplicative Weights Update (\kours)} \label{app:kernelized_comwu}

In this section, we discuss the omitted proofs for the kernelized version of \COMWU, which we refer to as \kours. A key observation for our proofs is that, by \eqref{eq:softmax_vertices}, $\vchi$ can be interpreted as the distribution of a random variable from which we draw $\vertex$. From this perspective, \Cref{prop:kernel_paper} concerns the efficient computation of the first-order moments of $\vchi$, while \Cref{prop:kernel_order_two} addresses the efficient computation of its second-order moments. We provide detailed proofs for these two statements below, beginning with \Cref{prop:kernel_paper}.

\begin{restateproposition}{prop:kernel_paper}[Theorems 4.1 and 4.2 of \citep{farina2022kernelized}]
For all time steps $t \in T$, 
let $\mut\^t \defeq \nut\^t + \sum_{\tau=1}^t  \nut\^\tau$ be the optimistic sum of the utility vectors $\nut\^t$, and define the embedding vector $\vb\^t \in \mathbb{R}^d$ as
\[
    \vb\^t[k] \defeq \exp\{\lambda\^t \mut\^t[k]\} \quad \forall \, k \in [d]. 
\]
Then, the distributions $\vchi\^t$ are proportional to $\phi_\Lambda(\vb\^t)$,
\[
    \vchi\^t = \frac{\phi_\Lambda(\vb\^t)}{K_\Lambda(\vb\^t, \vec{1}_d)},
\]
and the iterates $\vx\^t$ produced by \COMWU are computed as
\[
    \vx\^t = \sum_{\vertex \in \mathcal{V}_{\Lambda}} \vchi\^t [\vertex] \vertex = \mleft[1 - \frac{K_\Lambda(\vb\^t, \vebar_1)}{K_\Lambda(\vb\^t, \vec{1}_d)}, 1 - \frac{K_\Lambda(\vb\^t, \vebar_2)}{K_\Lambda(\vb\^t, \vec{1}_d)}, \ldots, 1 - \frac{K_\Lambda(\vb\^t, \vebar_d)}{K_\Lambda(\vb\^t, \vec{1}_d)} \mright].
\]
\end{restateproposition}

\begin{proof}
For all $\vertex \in \mathcal{V}_\Lambda$,
    \[
    \At\^t[\vertex] & = (\Nut\^t[\vertex] - \langle \Nut\^t, \vchi\^t\rangle) + \sum_{\tau=1}^t \left[\Nut\^\tau[\nu] - \langle \Nut\^\tau, \vchi\^\tau\rangle\right] \\
    & = (\langle \nut\^t, \vertex \rangle - \sum_{\vertex^\prime \in \mathcal{V}_\Lambda} \Nut\^t[\vertex^\prime] \vchi\^t[\vertex^\prime])  + \sum_{\tau=1}^t \left[ \langle \nut\^\tau, \vertex \rangle - \sum_{\vertex^\prime \in \mathcal{V}_\Lambda} \Nut\^\tau[\vertex^\prime] \vchi\^\tau[\vertex^\prime] \right] \\
    & = (\langle \nut\^t, \vertex \rangle - \sum_{\vertex^\prime \in \mathcal{V}_\Lambda} \langle \nut\^t, \vertex^\prime \rangle \vchi\^t[\vertex^\prime])  + \sum_{\tau=1}^t \left[ \langle \nut\^\tau, \vertex \rangle - \sum_{\vertex^\prime \in \mathcal{V}_\Lambda} \langle \nut\^\tau, \vertex^\prime \rangle \vchi\^\tau[\vertex^\prime] \right] \\
    & = (\langle \nut\^t, \vertex \rangle -  \langle \nut\^t, \sum_{\vertex^\prime \in \mathcal{V}_\Lambda} \vertex^\prime \vchi\^t[\vertex^\prime] \rangle)  + \sum_{\tau=1}^t \left[ \langle \nut\^\tau, \vertex \rangle -  \langle \nut\^\tau, \sum_{\vertex^\prime \in \mathcal{V}_\Lambda} \vertex^\prime \vchi\^\tau[\vertex^\prime] \rangle \right]  \\
    & = (\langle \nut\^t, \vertex \rangle -  \langle \nut\^t, \vx\^t \rangle)  + \sum_{\tau=1}^t \left[ \langle \nut\^\tau, \vertex \rangle -  \langle \nut\^\tau, \vx\^\tau \rangle \right] \\
    & = \langle \nut\^t, \vertex \rangle + \sum_{\tau=1}^t  \langle \nut\^\tau, \vertex \rangle  - (\langle \nut\^t, \vx\^t \rangle + \sum_{\tau=1}^t \langle \nut\^\tau, \vx\^\tau \rangle) \\
    & = \langle \mut\^t, \vertex \rangle + \sigma\^t,
    \]
    where we defined $\sigma\^t \defeq - (\langle \nut\^t, \vx\^t \rangle + \sum_{\tau=1}^t \langle \nut\^\tau, \vx\^\tau \rangle)$.
    By definition \eqref{eq:softmax_vertices}, for all $\vertex \in \mathcal{V}_\Lambda$,
    \[
    \vchi\^t [\vertex] & \propto \exp\{\lambda\^t \At\^t[\vertex]\} \\
    & = \exp\{\lambda\^t (\langle \mut\^t, \vertex \rangle + \sigma\^t) \} \\ %
    & = \exp\{\lambda\^t \sigma\^t\} \exp\{ \lambda\^t \sum_{\ind: \vertex[\ind] = 1} \mut\^t[k] \} \\
    & = \exp\{\lambda\^t \sigma\^t\} \prod_{\ind: \vertex[\ind] = 1} \exp\{ \lambda\^t \mut\^t[k] \} \\
    & = \exp\{\lambda\^t \sigma\^t\} \phi_{\Lambda} (\vb\^t)[\vertex].
    \] 
    On the other hand, 
    \[
    \sum_{\vertex \in \mathcal{V}_{\Lambda}}  (\exp\{\lambda\^t \sigma\^t\} \phi_{\Lambda} (\vb\^t) [\vertex]) & = \exp\{\lambda\^t \sigma\^t\} \left(\sum_{\vertex \in \mathcal{V}_{\Lambda}} \phi_{\Lambda} (\vb\^t) [\vertex]\right) \\
    & = \exp\{\lambda\^t \sigma\^t\} \left(\sum_{\vertex \in \mathcal{V}_{\Lambda}} \phi_{\Lambda} (\vb\^t) [\vertex]. \phi_{\Lambda}(\vec 1_d)[\vertex] \right) \\
    & = \exp\{\lambda\^t \sigma\^t\} K_{\Lambda} (\vb\^t, \vec 1_d) 
    \]
    Thus,
    \[
    \vchi\^t [\vertex] = \frac{\exp\{\lambda\^t \At\^t[\vertex]\}}{ \sum_{\vertex^\prime \in \mathcal{V}_{\Lambda}} \exp\{\lambda\^t \At\^t[\vertex^\prime]\} } =  \frac{\phi_\Lambda(\vb\^t)[\vertex]}{K_\Lambda(\vb\^t, \vec{1}_d)}.
    \]
    The proof of the second part follows similarly to the proof of Theorem 4.2 in \citep{farina2022kernelized}.
\end{proof}

In what follows, we state and prove \Cref{prop:kernel_order_two}.

\begin{restateproposition}{prop:kernel_order_two}
    Define the vector $\vb \in \mathbb{R}^d$ as
    \[
        \vb[k] \defeq \exp\{\lambda \mut[k]\} \quad \forall \, k \in [d]. 
    \]
    Then,
    \[
    \expect\mleft[\vertex \mright] = \mleft[1 - \frac{K_\Lambda(\vb, \vebar_1)}{K_\Lambda(\vb, \vec{1}_d)}, 1 - \frac{K_\Lambda(\vb, \vebar_2)}{K_\Lambda(\vb, \vec{1}_d)}, \ldots, 1 - \frac{K_\Lambda(\vb, \vebar_d)}{K_\Lambda(\vb, \vec{1}_d)} \mright], 
    \]
    And,
    \[
    \expect\mleft[\vertex \vertex^\top \mright]_{ij} = 1 + \frac{K_\Lambda(\vb, \vebar_i)}{K_\Lambda(\vb, \vec{1}_d)} + \frac{K_\Lambda(\vb, \vebar_j)}{K_\Lambda(\vb, \vec{1}_d)} - \frac{K_\Lambda(\vb, \vebar_{ij})}{K_\Lambda(\vb, \vec{1}_d)},
    \]
    for all $i, j \in [d]$, where $i \neq j$, and
    \[
        \expect\mleft[\vertex \vertex^\top \mright]_{ii} = 1 + \frac{K_\Lambda(\vb, \vebar_i)}{K_\Lambda(\vb, \vec{1}_d)},
    \]
    for all $i \in [d]$.
\end{restateproposition}
\begin{proof}
    We start by $\expect\mleft[\vertex \mright]$. 
    \[
    \expect\mleft[\vertex \mright] = \sum_{\vertex \in \mathcal{V}_{\Lambda}} \vchi [\vertex] \vertex = \vx,
    \]
    the rest follows similar to \Cref{prop:kernel_paper}. Next, we analyze $\expect\mleft[\vertex \vertex^\top \mright]$. First we show that, for every $i. j \in [d]$,
    \[
    \phi_\Lambda (\vebar_i) [\vertex] = \prod_{k: \vertex[k] = 1} \vebar_i[k] = \prod_{k: \vertex[k] = 1} \mathbb{I}_{k \neq i} =  \mathbb{I}_{k \notin \vertex}
    \]
    and by the fact that $\phi_\Lambda (\vec 1_d) = \vec 1 $,
    \[
    \phi_\Lambda (\vec 1_d)[\vertex] - \phi_\Lambda (\vebar_i) [\vertex] =  \mathbb{I}_{i \in \vertex}.
    \]
    Similarly,
    \[
    \phi_\Lambda (\vebar_{ij}) [\vertex] = \prod_{k: \vertex[k] = 1} \vebar_{ij}[k] = \prod_{k: \vertex[k] = 1} \mathbb{I}_{k \neq i \text{ and } k \neq j} =  \mathbb{I}_{i, j \notin \vertex}
    \]
    For every $i. j \in [d]$,
    \[
    \expect\mleft[\vertex_i \vertex_j \mright] & = \sum_{\vertex \in \mathcal{V}_{\Lambda}} \vchi [\vertex] \vertex_i \vertex_j \\
    & = \sum_{\vertex \in \mathcal{V}_{\Lambda}} \vchi [\vertex].\mathbb{I}_{i, j \in \vertex} \\
    & = \sum_{\vertex \in \mathcal{V}_{\Lambda}} \vchi [\vertex] . ( 1 - \mathbb{I}_{i \notin \vertex} - \mathbb{I}_{j \notin \vertex} + \mathbb{I}_{i, j \notin \vertex}) \numberthis{eq:inc_exc_princ} \\
    & = 1 - \sum_{\vertex \in \mathcal{V}_{\Lambda}} \vchi [\vertex] . ( - \phi_\Lambda (\vebar_i) [\vertex] - \phi_\Lambda (\vebar_j) [\vertex] + \phi_\Lambda (\vebar_{ij}) [\vertex] ) \\
    & = 1 + \frac{K_\Lambda(\vb, \vebar_i)}{K_\Lambda(\vb, \vec{1}_d)} + \frac{K_\Lambda(\vb, \vebar_j)}{K_\Lambda(\vb, \vec{1}_d)} - \frac{K_\Lambda(\vb, \vebar_{ij})}{K_\Lambda(\vb, \vec{1}_d)}
    \] 
    where line (\ref{eq:inc_exc_princ}) follows from the inclusion–exclusion principle, and in the last line, we applied \Cref{prop:kernel_paper} multiple times. The case with $i = j$ is quite similar.
\end{proof}

\section{Proofs for Cautious Optimism for Convex Games} \label{app:convex}

In this section, we analyze how \ours naturally applies to the broad class of \emph{convex games}. We begin by establishing the nonnegativity of the regret in the lifted space for \Cref{algo:cftrl_convex}.

\begin{restateproposition}{prop:reg+_convex}
    For any time horizon $T \in \mathbb{N}$, it holds that $\tildereg\^T = \max\{0, \reg\^T\}$. Consequently, $\tildereg\^T \geq 0$ and $\tildereg\^T \geq \reg\^T$.
\end{restateproposition}
\begin{proof}
    To establish the result, recall the definition of the reward signal  
    \[
    \ut\^t \defeq \begin{pmatrix} -\langle \nut\^t, \vx\^t \rangle \\ \nut\^t \end{pmatrix}
    \]
    and the induced action  
    \[
    \begin{pmatrix} 1 \\ \vx\^t \end{pmatrix} = \frac{\vy\^t}{\langle \vy\^t, {\vec 1}_{d+1} \rangle}.
    \]
    It is clear that $\vec \ut\^{t} \perp \vy\^{t}$.

 Based on these definitions, the regret can be analyzed as follows:
    \[
        \tildereg\^T & = \max_{\vy^* \in [0,1] \widetilde{\mathcal{X}} } \sum_{t = 1}^{T} \langle  {\vec\ut}\^{t}, \vy^* - \vy\^{t} \rangle                              \\
        & = \max_{\vy^* \in [0,1]\widetilde{\mathcal{X}} } \sum_{t = 1}^{T} \langle \begin{pmatrix}
        -\langle \nut\^t, \vx\^t \rangle \\
        \nut\^t
        \end{pmatrix}, \vy^* - \vy\^{t} \rangle                                                   \\
        & =  \max_{\vy^* \in [0,1]\widetilde{\mathcal{X}} } \sum_{t = 1}^{T} \langle \begin{pmatrix}
        -\langle \nut\^t, \vx\^t \rangle \\
        \nut\^t
        \end{pmatrix}, \vy^* \rangle\numberthis{eq:reg+_1_convex} \\
        & \geq \max_{\vy^* \in \widetilde{\mathcal{X}} } \sum_{t = 1}^{T} \langle \begin{pmatrix}
        -\langle \nut\^t, \vx\^t \rangle \\
        \nut\^t
        \end{pmatrix}, \vy^* \rangle       \\
        & = \max_{\begin{pmatrix}
        1 \\
        \vx^*
        \end{pmatrix} \in \widetilde{\mathcal{X}} } \sum_{t = 1}^{T} \langle \begin{pmatrix}
        -\langle \nut\^t, \vx\^t \rangle \\
        \nut\^t
        \end{pmatrix}, \begin{pmatrix}
        1 \\
        \vx^*
        \end{pmatrix} \rangle   
        \\
        & \geq \max_{\vx^* \in \mathcal{X} } \sum_{t = 1}^{T} \langle \nut\^t, \vx^* \rangle - \langle \nut\^t, \vx\^{t} \rangle                            \\
        & = \reg\^T,
    \]
    where \eqref{eq:reg+_1_convex} follows from the orthogonality condition $\vec \ut\^{t} \perp \vy\^{t}$. Furthermore, it is straightforward to verify that $\tildereg\^T \geq 0$ by choosing $0$ as the comparator. This completes the proof.
\end{proof}

\begin{lemma}\label{lemma:u_correction_convex}
    Suppose that $\| \nut\^{t} \|_\infty \leq 1$ holds for all $t \in [T]$ according to \Cref{assumption:convex}. Then, the following inequality holds:
    \[
        \| \ut\^{t} - \ut\^{t - 1} \|_{\infty}^2 \leq 3 \|\nut\^t - \nut\^{t - 1}\|_{\infty}^2 + 2 \| \vx\^t - \vx\^{t-1} \|_1^2.
    \]
\end{lemma}
\begin{proof}
    From the definition, we have:
    \[
    \hspace{-0.5cm}
        \| {\ut}\^{t} -   {\ut}\^{t - 1} \|_{\infty}^2 & = \| \begin{pmatrix} -\langle \nut\^t, \vx\^t \rangle \\ \nut\^t \end{pmatrix} -   \begin{pmatrix} -\langle \nut\^{t-1}, \vx\^{t-1} \rangle \\ \nut\^{t-1} \end{pmatrix} \|_{\infty}^2 \\
        & \leq \max\mleft\{ \| \nut\^t - \nut\^{t-1} \|_\infty^2,  \big| \langle \nut\^t, \vx\^t \rangle  - \langle \nut\^{t-1}, \vx\^{t-1} \rangle \big|^2  \mright\} \\
        & \leq \|\nut\^t - \nut\^{t - 1}\|_{\infty}^2 + \big| \langle \nut\^t, \vx\^t \rangle  - \langle \nut\^{t-1}, \vx\^{t-1} \rangle \big|^2 \numberthis{eq:u_correction_convex2} \\
        & \leq   \|\nut\^t - \nut\^{t - 1}\|_{\infty}^2 + \big| \big(\langle \nut\^t, \vx\^t \rangle  - \langle \nut\^t, \vx\^{t-1} \rangle\big) + \big(\langle \nut\^t, \vx\^{t-1} \rangle - \langle \nut\^{t-1}, \vx\^{t-1} \rangle\big) \big|^2 \\
        & \leq  \|\nut\^t - \nut\^{t - 1}\|_{\infty}^2 + 2 \big|\langle \nut\^t, \vx\^t - \vx\^{t-1}\rangle\big|^2 + 2 \big| \langle \nut\^t - \nut\^{t-1}, \vx\^{t-1} \rangle \big|^2 \numberthis{eq:u_correction_convex3} \\
        & \leq   \|\nut\^t - \nut\^{t - 1}\|_{\infty}^2 + 2 \| \vx\^t - \vx\^{t-1} \|_1^2 + 2\| \nut\^t - \nut\^{t-1} \|_{\infty}^2 \numberthis{eq:u_correction_convex4} \\
        & = 3 \|\nut\^t - \nut\^{t - 1}\|_{\infty}^2 + 2 \| \vx\^t - \vx\^{t-1} \|_1^2,
    \]
    where \eqref{eq:u_correction_convex2} and \eqref{eq:u_correction_convex3} make use of Young’s inequality, and \eqref{eq:u_correction_convex4} applies Hölder’s inequality and the fact that $\| \vx \|_1 \leq 1 $ for all $\vx \in \mathcal{X}$.
\end{proof}

In turn, we show how the nonnegative RVU bounds hold for \ours in convex games.

\begin{theorem}[Nonnegative RVU bound of \ours for convex games]\label{theorem:nonnegative_rvu_convex}
    The cumulative regret $\reg\^T$ up to time $T$ for the players in a convex game satisfying \Cref{assumption:convex}, who follow \ours with a sufficiently small learning rate $\eta \leq \min\{3\gamma/80,\, \mu/(32\sqrt{2})\}$,\footnote{If $\psi$ is $\gamma$-LIL, it suffices to choose $\eta \leq \min\{3\gamma(1)/80,\, 1/8,\, \mu/(32\sqrt{2})\}$, after which the guarantees follow naturally.} and $\alpha \geq 4\gamma + \mu$, is bounded as
    \[
        \max\{\reg\^T, 0\} = \tildereg\^T \leq 3 + \frac{1}{\eta} \mleft( \alpha \log T + \mathcal{R} \mright) + \eta \frac{48}{\mu} \sum_{t = 1}^{T-1} \|\nut\^{t+1} - \nut\^{t}\|_{\infty}^2  - \frac{1}{\eta} \frac{\mu}{64} \sum_{t = 1}^{T-1} \| {\vx}\^{t+1} -  {\vx}\^{t} \|_1^2.
    \]
\end{theorem}
\begin{proof}
    The proof is similar to the proof of \Cref{theorem:nonnegative_rvu} for finite games and we avoid unncessary details in the interest of space. 
    
    We begin with \Cref{lemma:basi_OFTRL} to decompose the nonnegative regret into terms (I), (II), and (III). The analysis for term (I) follows verbatim. For term (II), instead of \Cref{lemma:u_correction}, we employ \Cref{lemma:u_correction_convex}. We then repeat the steps of the regret analysis in \Cref{sec:stepI,sec:stepII,sec:stepIII}, recalling that under \Cref{assumption:convex}, we have $\|\vx\|_1 \leq 1$ for all $\vx \in \mathcal{X}$. Thus, term (III) can be bounded in a similar manner.
\end{proof}

Finally, we are ready to state and prove the regret guarantees of \ours for convex games. The proofs are similar to the analysis for finite games discussed in \Cref{app:detailed,app:social_regret}, and we omit the repetitive details.

\begin{restatetheorem}{theorem:regret_final_bound_convex}[Regret bounds of \ours for convex games]
    If all players $i \in [n]$ in a convex game satisfying \Cref{assumption:convex} follow \ours with a $\gamma$-(L)IL and a $\mu$-strongly convex\footnote{With respect to the $\ell_1$ norm.} regularizer $\psi$ on the set $\mathcal{X}_i$, and use a sufficiently small learning-rate cap $\eta = O_{\scaleto{T}{4pt}}(1)$, then the following holds for both the individual and social regrets:
    \[
    \reg_i\^T = O\big(n \Gamma_\psi(d) \log T\big), \quad \text{and} \quad \sum_{i \in [n]} \reg_i\^T = O_{\scaleto{T}{4pt}}(1),
    \]
    where $\Gamma_\psi(d) = \gamma / \mu$. Moreover, the algorithm for each player $i \in [n]$ is adaptive to adversarial utilities; that is, each player’s regret satisfies $\reg_i\^T = O(\sqrt{T \log d})$.
\end{restatetheorem}

\begin{proof}
        We observe that, by \Cref{assumption:convex},
    \[
    \| \nut\^{t + 1}_i - \nut\^{t}_i \|_{\infty}^2 & = \| \nabla_{\vx} \nu_i(\vx_1\^{t+1}, \ldots, \vx_n\^{t+1}) - \nabla_{\vx} \nu_i(\vx_1\^t, \ldots, \vx_n\^t) \|_{\infty}^2 \\
    & \leq L^2 \mleft(\sum_{i = 1}^{n} \| \vx\^{t+1}_i - \vx\^{t}_i \|_1 \mright)^2 \\
    & \leq L^2 n \sum_{i = 1}^{n} \| \vx\^{t+1}_i - \vx\^{t}_i \|_1^2,
    \]
    where the final inequality follows from Jensen’s inequality. Combining this observation with the proofs \Cref{theorem:nonnegative_rvu_convex} yields a bound on the total path length analogous to \Cref{prop:path_length}. This bound on the path length can then be used to bound the individual regret, following the proof of \Cref{theorem:regret_final_bound}. The proof for the social regret proceeds similarly, and we omit the details in the interest of space.
\end{proof}

\end{document}